\documentclass[nohyperref]{article}

\usepackage{microtype}
\usepackage{graphicx}
\usepackage{subfig}
\usepackage{booktabs} % for professional tables
\usepackage{tcolorbox}
\usepackage{wrapfig}
\usepackage{hyperref}

\usepackage[accepted]{icml2023}
\usepackage{amsmath}
\usepackage{amssymb}
\usepackage{mathtools}
\usepackage{amsthm}
\usepackage{rotating}

\usepackage[capitalize,noabbrev]{cleveref}

\theoremstyle{plain}
\newtheorem{theorem}{Theorem}[section]

\newtheorem{lemma}[theorem]{Lemma}
\newtheorem{corollary}[theorem]{Corollary}
\theoremstyle{definition}
\newtheorem{definition}[theorem]{Definition}
\newtheorem{assumption}[theorem]{Assumption}
\theoremstyle{remark}
\newtheorem{remark}[theorem]{Remark}

\usepackage[textsize=tiny]{todonotes}

\icmltitlerunning{An SDE for Modeling SAM: Theory and Insights}

\newtheorem{observation}[theorem]{Observation}

\newcommand{\E}{{\mathbb E}}

\newcommand{\R}{{\mathbb{R}}}
\newtcolorbox{mybox}[2][]
{
  colframe = #2!25,
  colback  = #2!15!white!15,
  #1
}

\begin{document}

\twocolumn[
\icmltitle{An SDE for Modeling SAM: Theory and Insights}



\icmlsetsymbol{equal}{*}

\begin{icmlauthorlist}
\icmlauthor{Enea Monzio Compagnoni}{BS}
\icmlauthor{Luca Biggio}{ETHZ}
\icmlauthor{Antonio Orvieto}{ETHZ}
\icmlauthor{Frank Norbert Proske}{NOR}
\icmlauthor{Hans Kersting}{IN}
\icmlauthor{Aurelien Lucchi}{BS}
%\icmlauthor{Firstname7 Lastname7}{comp}
%\icmlauthor{Firstname8 Lastname8}{sch}
%\icmlauthor{Firstname8 Lastname8}{yyy,comp}
\end{icmlauthorlist}

\icmlaffiliation{BS}{Department of Mathematics \& Computer Science, University of Basel, Basel, Switzerland}
\icmlaffiliation{ETHZ}{Department of Computer Science, ETH Zürich, Zürich, Switzerland}
\icmlaffiliation{NOR}{Department of Mathematics, University of Oslo, Oslo, Norway}
\icmlaffiliation{IN}{Inria, Ecole Normale Sup\'erieure PSL Research University, Paris, France}
%\icmlaffiliation{YH}{Company Name, Location, Country}

\icmlcorrespondingauthor{Enea Monzio Compagnoni}{enea.monziocompagnoni@unibas.ch}
%\icmlcorrespondingauthor{Firstname2 Lastname2}{first2.last2@www.uk}

\icmlkeywords{Machine Learning, ICML}

\vskip 0.3in
]



\printAffiliationsAndNotice{}  % leave blank if no need to mention equal contribution

\begin{abstract}
    We study the SAM (Sharpness-Aware Minimization) optimizer which has recently attracted a lot of interest due to its increased performance over more classical variants of stochastic gradient descent.
    Our main contribution is the derivation of continuous-time models (in the form of SDEs) for SAM and two of its variants, both for the full-batch and mini-batch settings.
    We demonstrate that these SDEs are rigorous approximations of the real discrete-time algorithms (in a weak sense, scaling linearly with the learning rate).
    Using these models, we then offer an explanation of why SAM prefers flat minima over sharp ones~--~by showing that it minimizes an implicitly regularized loss with a Hessian-dependent noise structure. 
    Finally, we prove that SAM is attracted to saddle points under some realistic conditions. 
    Our theoretical results are supported by detailed experiments.
\end{abstract}

%%%%%%%%%%%%%%%%%%%%%%%%%%%%%%%%%%%%%%%%%%%%%%%%%%%%%%%%%%%%%%%%%%%%%%%%%%%%
\section{Introduction} \label{sec:intro}
\vspace{-2mm}
%%%%%%%%%%%%%%%%%%%%%%%%%%%%%%%%%%%%%%%%%%%%%%%%%%%%%%%%%%%%%%%%%%%%%%%%%%%%

Optimization plays a fundamental role in the performance of machine learning models. The core problem it addresses is the minimization of the following optimization problem:
\begin{equation}
\min_{x \in \R^d} \left[ f(x) := \frac{1}{N} \sum_{i=1}^N f_i(x) \right],
\label{eq:empirical_loss}
\end{equation}
where $f,f_i:\R^d \to \R$ for $i=1,\dots,N$. In machine learning, $f$ is an empirical risk (or loss) function where $f_i$ are the contributions due to the $i$-th data point. In this notation, $x \in \R^d $ is a vector of trainable parameters and $N$ is the size of the dataset.

Solving Eq.~\eqref{eq:empirical_loss} is typically achieved via Gradient Descent (GD) methods that, starting from a given estimate $x_0$, iteratively update an estimate $x_k$ as follows,
\begin{equation}
x_{k+1} = x_k - \eta \nabla f(x_k),
\end{equation}
where $\eta > 0$ is the learning rate.
Since $\nabla f(x)$ requires computing the average of the $N$ gradients $\nabla f_i(x)$ (which is computationally expensive for large datasets where $N \gg 1$), it is common to instead replace $\nabla f(x_k)$ with a gradient estimated on a subset $\gamma_k$ of size $B \geq 1$ of the dataset which is called a \textit{mini-batch}. The resulting algorithm is known as Stochastic Gradient Descent (SGD) whose update is
\begin{equation} \label{eq:SGD_Discr_Update}
x_{k+1} = x_k - \eta \nabla f_{\gamma_k}(x_k),
\end{equation}
where $\{ \gamma_k \}$ are modelled here as i.i.d.~random variables uniformly distributed and taking value in $\{ 1, \cdots, N \}$.

Recently, \citet{foret2020sharpness} proposed a stochastic optimizer known as Sharpness-Aware Minimization (SAM), which yields significant performance gains in various fields such as computer vision and natural language processing~\citep{bahri2021sharpness, foret2020sharpness}. The general idea behind SAM is to seek parameters in low-loss regions that have a flatter curvature, which has been shown to improve the generalization of the model~\citep{hochreiter1997flat, keskar2016large, Dziugaite2017PacBayes, jiang2019fantastic}. For a radius $\rho>0$, the iteration of SAM is
\begin{equation}\label{eq:SAM_Discr_Update}
x_{k+1}=x_k-\eta \nabla f_{\gamma_k}\left(x_k + \rho \frac{\nabla f_{\gamma_k}(x_k)}{\lVert \nabla f_{\gamma_k}(x_k) \rVert}\right).
\end{equation}
%where $\rho>0$ and $\{ \gamma_k \}$ are i.i.d.~random variables uniformly distributed in $\{ 1, \cdots, N \}$.
Numerous works have studied SAM and proposed variants such as ESAM \citep{du2021efficient}, ASAM \citep{kwon2021asam}, GSAM \citep{zhuang2022surrogate}, as well as Random SAM and Variational SAM \citep{ujvary2022rethinking}. Since SAM is hard to treat theoretically, \citep{andriushchenko2022towards} introduced USAM which is more easily analyzable as it drops the gradient normalization in Eq.~\eqref{eq:SAM_Discr_Update}, thus yielding the following update:
\begin{equation} \label{eq:USAM_Discr_Update}
x_{k+1}=x_k-\eta \nabla f_{\gamma_k}\left(x_k + \rho \nabla f_{\gamma_k}(x_k) \right).
\end{equation} 
Before analyzing the full version of SAM, we first take a smaller step toward it by considering a variant with a deterministic normalization factor. We call the resulting algorithm DNSAM (Deterministic Normalization SAM), whose update step is
\begin{equation}\label{eq:DNSAM_Discr_Update}
x_{k+1}=x_k-\eta \nabla f_{\gamma_k}\left(x_k + \rho \frac{\nabla f_{\gamma_k}(x_k)}{\lVert \nabla f(x_k) \rVert}\right).
\end{equation}
We will demonstrate both theoretically and empirically that DNSAM is a better proxy of SAM than USAM. However, we do not claim that DNSAM is an algorithm to be used in practice as its update requires the calculation of the full gradient of the loss.

Following the theoretical framework of \cite{li2017stochastic}, our work provides the first \emph{formal} derivation of the SDEs of DNSAM, USAM, and SAM. Formally, such continuous-time models are weak approximations (i.e.~approximations in distribution) of their respective discrete-time models. Importantly, SDE models are not meant to be used as practical implementations since they have to be discretized, giving rise to their discrete-time counterparts. Instead, continuous-time models have typically been used in the recent literature to derive novel insights about the discrete algorithms, see e.g.~\citep{Su2014nesterov, li2017stochastic}.

We make the following contributions:
\begin{enumerate}
\setlength\itemsep{0mm}
\item \textbf{Small $\rho$ regime.} If $\rho = \mathcal{O}(\eta)$, we show that USAM, DNSAM, and SAM essentially behave like SGD.

\item \textbf{Moderate $\rho$ regime.} For $\rho = \mathcal{O}(\sqrt{\eta})$, we derive an SDE model of USAM \eqref{eq:USAM_SDE_Insights}, of DNSAM \eqref{eq:DNSAM_SDE_Insights}, and of SAM \eqref{eq:SAM_SDE_Insights}. These can be interpreted as the SDE of SGD on an implicitly regularized loss and with an additional \textit{implicit curvature-induced} noise.
Leveraging these results, we demonstrate that the additional noise is driven by the Hessian of the loss so that the noise of the processes is larger in sharp minima. This is a key factor that leads SAM and its variants to prefer flatter minima where the additional noise decreases. However, while larger values of $\rho$ increase the noise of the process, it also amplifies the implicit bias of the optimizer toward critical points independently of whether they are minima, saddles, or maxima.
\item Both in the full and mini-batch versions, USAM and SAM have very different implicit regularizations.
\item USAM might be attracted by saddles if $\rho$ is too large. Differently, for any $\rho>0$, DNSAM and SAM might be attracted by saddles but eventually, escape them after a long time. Thus, DNSAM is a more reliable model to theoretically study SAM than USAM.
\item \textbf{Empirical validation.} We empirically validate the proposed SDEs on several models and landscapes commonly studied in the optimization and machine learning communities.
\end{enumerate}

In order to gain further insights from these continuous-time models, we also study their behaviors on quadratic losses. The latter are commonly used to model the landscape in the proximity of a critical point~\citep{ge2015escaping, levy2016power, jin2017escape, poggio2017theory, mandt2017stochastic}, including several recent works that studied SAM~\citep{bartlett2022dynamics, wen2022does}. This leads us to the following important observations:

\begin{enumerate}
\setlength\itemsep{0mm}
\item \textbf{ODE - Pitfalls.} 
After noticing that the ODE of SAM and DNSAM coincide, we derive precise conditions under which SAM and USAM converge to the origin even when it is a saddle or a maximum. 
\item \textbf{SDE - Pitfalls.} We derive the stationary distribution of the USAM SDE and find that even this model is attracted by saddles under the same condition on $\rho$ as found for the ODE~\footnote{Of course, the SDE does not point-wise converge to the origin but rather oscillates around it with a certain variance.}. In contrast to USAM, we find that the dynamics of DNSAM is more complex: while a certain region centered at the origin behaves like an attractor, the origin itself repulses the dynamics away as the volatility rapidly increases to infinity.
This behavior of DNSAM is consistent with what was empirically reported in \citep{kaddour2022flat} about SAM being able to get stuck around saddles. To the best of our knowledge, this is the first time that these potential pitfalls are formally demonstrated.
\item \textbf{Empirical validation.} We empirically validate our claims for the quadratic loss as well as other models. 
\end{enumerate}

%\paragraph{Structure.} The paper is organized as follows: In Section \ref{sec:RelWorks}, we review related works and contextualize our work in the literature; in Section \ref{sec:Insights}, we present the main theoretical results, i.e. the SDE of USAM and SAM. Then, we compare these SDEs to that of SGD and provide several novel insights about the behavior of USAM and SAM; in Section \ref{sec:pitfalls}, we focus on quadratic losses and show that both the deterministic and stochastic versions of USAM and SAM can get attracted by saddles and maxima under certain conditions.
%Finally, we present our conclusions in Section \ref{sec:conclusion}. The proofs of our results and the details of the experiments are in the Appendix.

%%%%%%%%%%%%%%%%%%%%%%%%%%%%%%%%%%%%%%%%%%%%%%%%%%%%%%%%%%%%%%%%%%%%%%%%%%%%
\section{Related Work}\label{sec:RelWorks}
\vspace{-2mm}
%%%%%%%%%%%%%%%%%%%%%%%%%%%%%%%%%%%%%%%%%%%%%%%%%%%%%%%%%%%%%%%%%%%%%%%%%%%%

\paragraph{Theoretical Understanding of SAM}

The current understanding of the dynamics of SAM and USAM is still limited. Prior work includes the recent work by \cite{bartlett2022dynamics} that shows that, for convex quadratics, SAM converges to a cycle oscillating between the sides of the minimum in the direction with the largest curvature. For the non-quadratic case, they also show that the dynamics drifts towards wider minima. A concurrent work by~\cite{wen2022does} makes similar findings to~\cite{bartlett2022dynamics} as well as provides further insights regarding which notion of sharpness SAM regularizes. Interestingly, the behavior of full-batch and mini-batch SAM is intrinsically different. The former minimizes the largest eigenvalue of the hessian of the loss, while the latter tries to uniformly reduce the magnitude of the trace of the hessian of the loss. More interestingly, ~\cite{wen2022does} show how the dynamics of $1$-SAM can be divided into two phases. The first phase follows the gradient flow with respect to the loss until the dynamics approaches a manifold of minimizers. In the second phase, the dynamics is driven towards parts of the landscape with a lower trace of the hessian of the loss. \cite{rangwani2022escaping} showed that USAM could in some cases escape saddles faster than SGD. We however note that their analysis is not completely formal as it relies on prior results by~\cite{daneshmand2018escaping} which were specifically derived for SGD, not for USAM.
On our side, we here provide a more complete and rigorous description that shows that USAM can be much slower than SGD at escaping a saddle. Finally, a concurrent work \citep{kim2023stability} \textit{informally} derived an SDE for USAM around critical points, which relies on approximating the objective function by a quadratic function. We remark that the authors did not formally derive any guarantee showing the SDE closely approximates the true discrete-time algorithm. In contrast, we formally and empirically demonstrate the validity of our SDEs. In addition, our SDEs and analyses \textit{do not require} the quadratic approximation assumption made by~\citep{kim2023stability} and are instead valid for the entire trajectory of an optimizer, including, of course, around critical points.
\vspace{-2mm}
\paragraph{ODE Approximations}

Continuous-time models in the form of (stochastic) differential equations are a well-established tool to study discrete-time optimizers; see e.g.~\citet{helmke1994optimization} and \citet{kushner2003stochastic}.
In machine learning, such models have lately received increasing interest to study both deterministic and stochastic optimizers. A notable reference is the work by~\citet{Su2014nesterov}
that derives a second-order ODE to model Nesterov’s accelerated gradient. ODE models have also been used recently to study SAM. This includes the work of \citet[Section 4.2]{wen2022does} discussed above, as well as~\citet[Appendix B.1]{andriushchenko2022towards}.
Importantly we highlight two significant differences with our work. First, our analysis focuses on the stochastic setting for which we derive SDEs. Second, the ODE representations used in~\citet{wen2022does} only hold formally in the limit $\rho \rightarrow 0$, which is not the case in practical settings where $\rho > 0$. In contrast, our analysis allows for significantly larger values of $\rho$, more precisely $\rho = \mathcal{O}(\sqrt{\eta})$. Last but not least, neither of these papers empirically validates the ODEs they derived.

\paragraph{SDE Approximations of Stochastic Optimizers.}
For \emph{stochastic} optimizers, \citet{li2017stochastic,li2019stochastic} derived an SDE that provably approximates SGD (in the weak sense, i.e.~in distribution). The validity of this SDE model was experimentally tested in \cite{Li2021validity}.
Similar results are derived for ASGD by \citep{an2020stochastic}, and for RMSprop and Adam by \citet{Malladi2022AdamSDE}. 
In this paper, we derive an SDE approximation for SAM, DNSAM, and USAM.
The proof technique employed in our work (as well as in~\citet{an2020stochastic, Malladi2022AdamSDE}) relies on the theoretical framework established by \citet{li2017stochastic,li2019stochastic} (which itself requires Assumption \ref{ass:regularity_f} to hold). SDE approximations have also been derived for different types of noise. This includes heavy-tailed noise that is shown to be a good model for the noise of SGD in~\citet{simsekli2019tailindex}, although the evidence is still somewhat contested~\citep{panigrahi2019non, xie2020diffusion, Li2021validity}. \citet{zhou2020towards} also derived a Lévy-driven stochastic differential equation to model the non-gaussianity of the noise, which however does not enjoy any type of known theoretical approximation guarantee. Finally, fractional Brownian noise, a generalization of Brownian noise that allows for correlation, was used by \cite{lucchi2022fractional}.

\paragraph{Applications of SDE Approximations.}
Continuous-time models are valuable analysis tools to study and design new optimization methods. For instance, one concrete application of such models is the use of \emph{stochastic optimal control} to select the learning rate~\citep{li2017stochastic,li2019stochastic} or the batch size~\citep{zhao2022batch}.
In addition, \emph{scaling rules} to adjust the optimization hyperparameters w.r.t.~the batch size can be recovered from SDE models~\citep{Malladi2022AdamSDE}.
Apart from these algorithmic contributions, SDE approximation can be useful to better understand stochastic optimization methods.
In this regard, \cite{jastrzkebski2017three} analyzed the factors influencing the minima found by SGD, and \cite{orvieto2019continuous} derived convergence bounds for mini-batch SGD and SVRG.
\cite{smith2020sde} used an SDE model to distinguish between "noise-dominated" and "curvature-dominated" regimes of SGD. Yet another example is the study of \emph{escape times} of SGD from minima of different sharpness~\citep{xie2020diffusion}.
Moreover, \cite{Li2020reconciling} and \cite{kunin2021neural} studied the \emph{dynamics} of the SDE approximation under some symmetry assumptions. 
Finally, SDEs can be studied through the lens of various tools in the field of stochastic calculus, e.g. the Fokker–Planck equation gives access to the stationary distribution of a stochastic process. Such tools are for instance valuable in the field of Bayesian machine learning~\citep{mandt2017SGDasABI}. For additional references, see \citep{kushner2003stochastic,ljung2012stochastic,chen2015convergence,mandt2015continuous,chaudhari2018stochastic,zhu2018anisotropic,ijcai2018p307,an2020stochastic}.

%%%%%%%%%%%%%%%%%%%%%%%%%%%%%%%%%%%%%%%%%%%%%%%%%%%%%%%%%%%%%%%%%%%%%%%%%%%%
\section{Formal Statements \& Insights: The SDEs}\label{sec:Insights}
\vspace{-2mm}
%%%%%%%%%%%%%%%%%%%%%%%%%%%%%%%%%%%%%%%%%%%%%%%%%%%%%%%%%%%%%%%%%%%%%%%%%%%%

In this section, we present the general formulations of the SDEs of USAM, DNSAM, and SAM. Due to the technical nature of the analysis, we refer the reader to the Appendix for the complete formal statements and proofs. For didactic reasons, we provide simplified versions under mild additional assumptions in the main paper.

\begin{definition}[Weak Approximation]\label{def:G_Space_Weak_Solution}
 Let $G$ denote the set of continuous functions $\mathbb{R}^d \rightarrow \mathbb{R}$ of at most polynomial growth, i.e.~$g \in G$ if there exists positive integers $\kappa_1, \kappa_2>0$ such that $|g(x)| \leq \kappa_1\left(1+|x|^{2 \kappa_2}\right)$, for all $x \in \mathbb{R}^d$. Then, we say that a continuous-time stochastic process $\left\{X_t: t \in[0, T]\right\}$ is an order $\alpha$ weak approximation of a discrete stochastic process $\left\{x_k: k=0, \ldots, N\right\}$ if for every $g \in G$, there exists a positive constant $C$, independent of $\eta$, such that $ \max _{k=0, \ldots, N}\left|\mathbb{E} g\left(x_k\right)-\mathbb{E} g\left(X_{k \eta}\right)\right| \leq C \eta^\alpha.$
\end{definition}

This definition comes from the field of numerical analysis of SDEs, see \cite{mil1986weak}. Consider the case where $g(x)=\lVert x \rVert^j$, then the bound limits the difference between the $j$-th moments of the discrete and the continuous process. 

In Theorem \ref{thm:USAM_SGD} (USAM), Theorem \ref{thm:DNSAM_SGD} (DNSAM), and Theorem \ref{thm:SAM_SGD} (SAM), we prove that if $\rho = \mathcal{O}(\eta)$ (small $\rho$ regime), the SDE of SGD (Eq. \eqref{eq:SGD_Equation_Insights}) is also an order 1 weak approximation for USAM, DNSAM, and SAM. In contrast, in the more realistic moderate $\rho$ regime where $\rho = \mathcal{O}(\sqrt{\eta})$, Eq. \eqref{eq:SGD_Equation_Insights} is no longer an order 1 weak approximation for any the models we analyze. Under such a condition, we recover more insightful SDEs.

\subsection{USAM SDE}
%%%%%%%%%%%%%%%%%%%%%%%%%%%%%%%%%%%%%%%%%%%%%%%%%%%%%%%%%%%%%%%%%%%%%%%%%%%%

\begin{theorem}[USAM SDE - Informal Statement of Theorem \ref{thm:USAM_SDE}] \label{thm:USAM_SDE_Insights}
Under sufficient regularity conditions and $\rho=\mathcal{O}(\sqrt{\eta})$ the solution of the following SDE is an order 1 weak approximation of the discrete update of USAM \eqref{eq:USAM_Discr_Update}:

\begin{align}\label{eq:USAM_SDE_Insights}
dX_t & = -\nabla \Tilde{f}^{\text{USAM}}(X_t) d t + \\
& \sqrt{\eta\left( \Sigma^{\text{SGD}}(X_t)+ \rho \left( \tilde{\Sigma}(X_t) + \tilde{\Sigma}(X_t) ^{\top} \right) \right)}dW_t, \nonumber
\end{align}
%where $\Sigma^{\text{SGD}}(x)$ is given by Eq. \eqref{eq:SGD_Covariance_Insights} 
where $\tilde{\Sigma}(x)$ is defined as
\begin{align} \label{eq:USAM_sigma_star_Insights}
 & \mathbb{E} \left[ \left( \nabla f\left(x\right) - \nabla f_{\gamma}\left(x\right) \right) \cdot \right. \\
& \left. \left(\mathbb{E} \left[ \nabla^2 f_{\gamma}(x) \nabla f_{\gamma}(x)\right] - \nabla^2 f_{\gamma}(x) \nabla f_{\gamma}(x) \right)^{\top} \right] \nonumber
\end{align}
and $\Tilde{f}^{\text{USAM}}(x):= f(x) + \frac{\rho}{2} \mathbb{E} \left[ \lVert \nabla f_{\gamma}(x) \rVert^2_2\right].$
\end{theorem}

To have a more direct comparison with the SDE of SGD (Eq. \eqref{eq:SGD_Equation_Insights}), we prove Corollary \ref{thm:USAM_SDE_Simplified_Insights}, a consequence of Theorem \ref{thm:USAM_SDE_Insights} that provides a more interpretable SDE for USAM.
\vspace{2mm}

\begin{corollary}[Informal Statement of Corollary \ref{thm:USAM_SDE_Simplified}] \label{thm:USAM_SDE_Simplified_Insights}
Under the assumptions of Theorem \eqref{thm:USAM_SDE_Insights} and assuming a constant gradient noise covariance, i.e.~$\nabla f_{\gamma}(x) = \nabla f(x) + Z$ such that $Z$ is a noise vector that does not depend on $x$, the solution of the following SDE is the order 1 weak approximation of the discrete update of USAM \eqref{eq:USAM_Discr_Update}:

\begin{align}\label{eq:USAM_SDE_Simplified_Insights}
dX_t = & -\nabla \Tilde{f}^{\text{USAM}}(X_t) d t \\
+ & \left(I_d+ \rho \nabla^2 f(X_t) \right)\left(\eta \Sigma^{\text{SGD}}\left(X_t\right)\right)^{1 / 2} dW_t, \nonumber
\end{align}
%where $\Sigma^{\text{SGD}}(x)$ is given by Eq. \eqref{eq:SGD_Covariance_Insights} and
where $\Tilde{f}^{\text{USAM}}(x):= f(x) + \frac{\rho}{2} \lVert \nabla f(x) \rVert^2_2.$
\end{corollary}

Corollary~\ref{thm:USAM_SDE_Simplified_Insights} shows that the dynamics of USAM is equivalent to that of SGD on a regularized loss and with an additional noise component that depends on the curvature of the landscape (captured by the term $\nabla^2 f$).

\subsection{DNSAM: A step towards SAM}
 
\begin{theorem}[DNSAM SDE - Informal Statement of Theorem \ref{thm:DNSAM_SDE}] \label{thm:DNSAM_SDE_Insights}
Under sufficient regularity conditions, assuming a constant gradient noise covariance, i.e.~$\nabla f_{\gamma}(x) = \nabla f(x) + Z$ such that $Z$ is a noise vector that does not depend on $x$, and $\rho = \mathcal{O}(\sqrt{\eta})$ the solution of the following SDE is the order 1 weak approximation of the discrete update of DNSAM \eqref{eq:DNSAM_Discr_Update}:

\begin{align}\label{eq:DNSAM_SDE_Insights}
dX_t = & -\nabla \Tilde{f}^{\text{DNSAM}}(X_t) d t \\
+ &  \left( I_d + \rho \frac{\nabla^2 f(X_t)}{\lVert \nabla f(X_t) \rVert_2} \right) \left(\eta \Sigma^{\text{SGD}}(X_t)\right)^{\frac{1}{2}} dW_t \nonumber
\end{align}
%where $\Sigma^{\text{SGD}}(x)$ is given by Eq. \eqref{eq:SGD_Covariance_Insights} 
and $\Tilde{f}^{\text{DNSAM}}(x) = f(x) + \rho \lVert \nabla f(x) \rVert_2$.
\end{theorem}

Theorem \ref{thm:DNSAM_SDE_Insights} shows that similarly to USAM, the dynamics of DNSAM is equivalent to that of SGD on a regularized loss with an additional noise component that depends on the curvature of the landscape. However, we notice that, unlike USAM, the volatility component explodes near critical points.

\subsection{SAM SDE}
%%%%%%%%%%%%%%%%%%%%%%%%%%%%%%%%%%%%%%%%%%%%%%%%%%%%%%%%%%%%%%%%%%%%%%%%%%%%

\begin{theorem}[SAM SDE - Informal Statement of Theorem \ref{thm:SAM_SDE}] \label{thm:SAM_SDE_Insights}
Under sufficient regularity conditions and $\rho = \mathcal{O}(\sqrt{\eta})$ the solution of the following SDE is the order 1 weak approximation of the discrete update of SAM \eqref{eq:SAM_Discr_Update}:

\begin{align}\label{eq:SAM_SDE_Insights}
dX_t = & -\nabla \Tilde{f}^{\text{SAM}}(X_t) d t \\
+ & \sqrt{\eta\left( \Sigma^{\text{SGD}}(X_t)+ \rho \left( \Hat{\Sigma}(X_t) + \Hat{\Sigma}(X_t) ^{\top} \right) \right)}dW_t \nonumber
\end{align}
%where $\Sigma^{\text{SGD}}(x)$ is given by Eq. \eqref{eq:SGD_Covariance_Insights} 
where $\Hat{\Sigma}(x)$ is defined as

\begin{align} \label{eq:SAM_sigma_star_Insights}
& \mathbb{E} \left[ \left( \nabla f\left(x\right) - \nabla f_{\gamma}\left(x\right) \right) \cdot \right. \\
& \left. \left( \E \left[\frac{\nabla^2 f_{\gamma}(x)\nabla f_{\gamma}(x) }{\lVert \nabla f_{\gamma}(x) \rVert_2} \right]- \frac{\nabla^2 f_{\gamma}(x)\nabla f_{\gamma}(x) }{\lVert \nabla f_{\gamma}(x) \rVert_2} \right)^{\top} \right] \nonumber
\end{align}
and $\Tilde{f}^{\text{SAM}}(x):= f(x) + \rho \mathbb{E} \left[ \lVert \nabla f_{\gamma}(x) \rVert_2\right].$
\end{theorem}

To have a more direct comparison with the SDE of SGD (Eq. \eqref{eq:SGD_Equation_Insights}), we derive a corollary of Theorem \ref{thm:SAM_SDE_Insights} that provides a more insightful SDE for SAM.

\begin{corollary}[Informal Statement of Corollary \ref{thm:SAM_SDE_Simplified}] \label{thm:SAM_SDE_Simplified_Insights}
Under the assumptions of Theorem \ref{thm:SAM_SDE_Insights} and assuming a constant gradient noise covariance, i.e.~$\nabla f_{\gamma}(x) = \nabla f(x) + Z$ such that $Z$ is a noise vector that does not depend on $x$, the solution of the following SDE is an order 1 weak approximation of the discrete update of SAM \eqref{eq:SAM_Discr_Update}
\begin{align}\label{eq:SAM_SDE_Simplified_Insights}
dX_t & = -\nabla \Tilde{f}^{\text{SAM}}(X_t) d t + \\ & \sqrt{\eta\left( \Sigma^{\text{SGD}}(X_t)+ \rho H_t \left( \Bar{\Sigma}(X_t) + \Bar{\Sigma}(X_t) ^{\top} \right) \right)}dW_t \nonumber
\end{align}
%where $\Sigma^{\text{SGD}}(x)$ is given by Eq. \eqref{eq:SGD_Covariance_Insights} 
where $H_t := \nabla^2 f(X_t)$ and $\Bar{\Sigma}(x)$ is defined as

\begin{align} \label{eq:SAM_sigma_bar_Insights}
& \mathbb{E} \left[ \left( \nabla f\left(x\right) - \nabla f_{\gamma}\left(x\right) \right) \cdot \right. \\
& \left. \left( \E \left[\frac{\nabla f_{\gamma}(x) }{\lVert \nabla f_{\gamma}(x) \rVert_2} \right]- \frac{\nabla f_{\gamma}(x) }{\lVert \nabla f_{\gamma}(x) \rVert_2} \right)^{\top} \right], \nonumber
\end{align}
and $\Tilde{f}^{\text{SAM}}(x):= f(x) + \rho \mathbb{E} \left[ \lVert \nabla f_{\gamma}(x) \rVert_2\right].$

\end{corollary}

We note that the regularization term of SAM is the \textit{expected} norm of some gradient. While one can of course use sampling in order to simulate the SDE in Eq.~\eqref{eq:SAM_SDE_Simplified_Insights}, it results in an additional computational cost, which is worth highlighting.

\subsection{Comparison: USAM vs (DN)SAM}
\label{subsec:Compare_SDE}
%%%%%%%%%%%%%%%%%%%%%%%%%%%%%%%%%%%%%%%%%%%%%%%%%%%%%%%%%%%%%%%%%%%%%%%%%%%%
The analyses of the SDEs we derived (Eq. \eqref{eq:USAM_SDE_Simplified_Insights}, Eq. \eqref{eq:SAM_SDE_Simplified_Insights}, and Eq. \eqref{eq:DNSAM_SDE_Insights}) reveal that all three algorithms are implicitly minimizing a regularized loss with an additional injection of noise (in addition to the SGD noise). While the regularized loss is $\frac{\rho}{2} \lVert \nabla f(x) \rVert^2_2$ for USAM, it is $ \rho \lVert \nabla f(x) \rVert_2$ (not squared) for DNSAM, and $\rho \mathbb{E} \left[ \lVert \nabla f_{\gamma}(x) \rVert_2\right]$ for SAM. Therefore, when the process is closer to a stationary point, the regularization is stronger for (DN)SAM while it is the opposite when it is far away.

Regarding the additional noise, we observe that it is \textit{curvature-dependent} as the Hessian appears in the expression of all volatility terms. Note that the sharper the minimum, the larger the noise contribution from the Hessian. This phenomenon is more extreme for DNSAM where the volatility is scaled by the inverse of the norm of the gradient which drives the volatility to explode as it approaches a critical point. Based on the SAM SDE, it is clear that the diffusion term is more regular than that of DNSAM (in the sense that the denominator does not vanish). Therefore, SAM is intrinsically less volatile than DNSAM near a critical point. In contrast, we note that the SAM dynamics exhibits more randomness than USAM which in turn is more noisy than SGD. These theoretical insights are validated experimentally in Section \ref{sec:experiments}. Therefore, it is intuitive that SAM and its variants are more likely to stop or oscillate in a flat basin and more likely to escape from sharp minima than SGD.

We conclude with a discussion of the role of $\rho$. On one hand, larger values of $\rho$ increase the variance of the process. On the other hand, they also increase the marginal importance of the factor $\frac{\rho}{2} \lVert \nabla f(x) \rVert^2_2$ (USAM) and $ \rho \lVert \nabla f(x) \rVert_2$ (DNSAM), and $\rho \mathbb{E} \left[ \lVert \nabla f_{\gamma}(x) \rVert_2\right]$ (SAM), which for sufficiently large $\rho$ might overshadow the marginal relevance of minimizing $f$ and thus implicitly bias the optimizer toward any point with zero gradients, including saddles and maxima. We study this pitfall in detail for the quadratic case in the next section and verify it experimentally in Section \ref{sec:experiments} for other models as well. See Table \ref{table:SDEComparisonDrift} and Table \ref{table:SDEComparisonDiff} for a detailed summary.

%%%%%%%%%%%%%%%%%%%%%%%%%%%%%%%%%%%%%%%%%%%%%%%%%%%%%%%%%%%%%%%%%%%%%%%%%%%%
\section{Behavior Near Saddles - Theory} \label{sec:pitfalls}
%%%%%%%%%%%%%%%%%%%%%%%%%%%%%%%%%%%%%%%%%%%%%%%%%%%%%%%%%%%%%%%%%%%%%%%%%%%%

In this section, we leverage the ODEs (modeling the full-batch algorithms) and SDEs (modeling the mini-batch algorithms) to study the behavior of SAM and its variants near critical points. We especially focus on saddle points that have been a subject of significant interest in machine learning~\citep{jin2017escape, Jin2021pgd, daneshmand2018escaping}. We consider a quadratic loss (which as mentioned earlier is a common model to study saddle points) of the form $f(x) = \frac 12 x^{\top}H x$. W.l.o.g. we assume that the Hessian matrix $H$ is diagonal~\footnote{Recall that symmetric matrices can be diagonalized.} and denote the eigenvalues of $H$ by $(\lambda_1, \dots, \lambda_d)$ where $\lambda_1 \geq \lambda_1 \geq \dots \geq \lambda_d$. If there are negative eigenvalues, we denote by $\lambda_{*}$ the largest negative eigenvalue. 

\subsection{USAM ODE} \label{subsec:USAM_ODE}
%%%%%%%%%%%%%%%%%%%%%%%%%%%%%%%%%%%%%%%%%%%%%%%%%%%%%%%%%%%%%%%%%%%%%%%%%%%%

We study the deterministic dynamics of USAM on a quadratic which is defined as
\begin{equation} \label{eq:USAM_ODE_Quad_Insights}
dX_t= -H \left( I_d + \rho H \right)X_t dt \Rightarrow X^{j}_t = X^{j}_0 e^{-\lambda_j (1 + \rho \lambda_j)t}.
\end{equation}
Therefore, it is obvious (see Lemma \ref{lemma:USAM_Quad_PSD}) that, if all the eigenvalues of $H$ are positive, for all $\rho>0$, we have that $
X^{j}_t \overset{t \rightarrow \infty}{\rightarrow} 0, \quad \forall j \in \{1, \dots, d \}$. In particular, we notice that, since $e^{-\lambda_j (1 + \rho \lambda_j)t}<e^{-\lambda_j t}$, such convergence to $0$ is faster for the flow of USAM than for the gradient flow. More interestingly, if $\rho$ is \textit{too large}, the following result states that the deterministic dynamics of USAM might be attracted by a saddle or even a maximum.

\begin{lemma}[Informal Statement of Lemma \ref{lemma:USAM_Quad_Ind}]\label{lemma:USAM_Quad_Ind_Insights}
Let $H$ have at least one strictly negative eigenvalue. Then, for all $\rho > -\frac{1}{\lambda_{*} }, \quad X^{j}_t \overset{t \rightarrow \infty}{\rightarrow} 0, \quad \forall j \in \{1, \dots, d \}$.
\end{lemma}

Therefore, if $\rho$ is not chosen appropriately, USAM might converge to $0\in \mathbb{R}^d$, even if it is a saddle point or a maximum, which is very \textit{undesirable}. Of course, we observe that if $\rho <\frac{1}{\lambda_{*} }$, USAM will diverge from the saddle (or maximum), which is \textit{desirable}. Interestingly, we also notice that since $e^{-\lambda_j (1 + \rho \lambda_j)t}<e^{-\lambda_j t}$, USAM will escape the saddle but more slowly than the gradient flow. 

\subsection{USAM SDE}
%%%%%%%%%%%%%%%%%%%%%%%%%%%%%%%%%%%%%%%%%%%%%%%%%%%%%%%%%%%%%%%%%%%%%%%%%%%%

Based on Corollary \ref{thm:USAM_SDE_Simplified}, if we assume that $\Sigma^{\text{SGD}} = \varsigma^2 I_d$, the SDE of USAM on a quadratic is given by
\begin{equation} \label{eq:SDE_USAM_Quad_Insights}
 dX_t = -H\left(I_d + \rho H\right)X_t dt + \left[(I_d+\rho H)\sqrt{\eta}\varsigma \right] dW_t. 
\end{equation}

\begin{theorem}[Stationary distribution - Theorem \ref{thm:USAM_stat_distr_quad_PSD} and Theorem \ref{thm:USAM_stat_distr_quad_Indef}]\label{thm:stat_distr_quad_informal}
If all the eigenvalues of $H$ are positive, i.e.~$0$ is a minimum, we have that for any $\rho>0$, $\forall i \in \{ 1, \cdots, d \}$,  the stationary distribution of Eq. \eqref{eq:SDE_USAM_Quad_Insights} is

\begin{align} \label{eq:USAM_Stat_Distr_Insights}
P \left(x \right) = \sqrt{\frac{\lambda_i}{\pi \eta \varsigma^2 } \frac{1}{1 + \rho \lambda_i}} \exp \left[-\frac{\lambda_i}{\eta \varsigma^2 } \frac{1}{1 + \rho \lambda_i} x^2\right].
\end{align}
If there exists a negative eigenvalue, this formula does not, in general, parametrize a probability distribution. However, if $\rho > -\frac{1}{\lambda_{*}}$, Eq. \eqref{eq:USAM_Stat_Distr_Insights} is still the stationary distribution of Eq.\eqref{eq:SDE_USAM_Quad_Insights}, $\forall i \in \{ 1, \cdots, d \}$.
\end{theorem}

Theorem~\ref{thm:stat_distr_quad_informal} states that in case the origin is a saddle (or a maximum) and $\rho$ is small enough, the stationary distribution of USAM is divergent at infinity, meaning that the process will escape the bad critical point, which is desirable. In such a case, the escape from the saddle is however slower than SGD as the variance in the direction of negative eigenvalues, e.g. the escape directions, is smaller. However, if $\rho$ is too large, then the dynamics of the USAM SDE will oscillate around the origin even if this is a saddle or a maximum, which is undesirable. This is consistent with the results derived for the SDE of USAM in Section \ref{subsec:USAM_ODE}. There, we found that under the very same condition on $\rho$, the USAM ODE converges to $0$ even when it is a saddle or a maximum.

\subsection{SAM ODE}
%%%%%%%%%%%%%%%%%%%%%%%%%%%%%%%%%%%%%%%%%%%%%%%%%%%%%%%%%%%%%%%%%%%%%%%%%%%%
We now provide insights on the dynamics of the SAM ODE on a quadratic with Hessian $H$.
\begin{lemma}[Lemma \ref{lemma:SAM_ODE_Converg}] \label{lemma:SAM_ODE_Converg_Insights}
For all $\rho>0$, if $H$ is PSD (Positive Semi-Definite), the origin is (locally) asymptotically stable. Additionally, if $H$ is not PSD and $\lVert H X_t \rVert \leq-\rho \lambda_{*}$, then the origin is still (locally) asymptotically stable.
\end{lemma}
Lemma~\ref{lemma:SAM_ODE_Converg_Insights} demonstrates that USAM and SAM have completely different behaviors. For USAM, Lemma \ref{lemma:USAM_Quad_Ind_Insights} shows that selecting $\rho$ small enough would prevent the convergence towards a saddle or a maximum. In contrast, Lemma \ref{lemma:SAM_ODE_Converg_Insights} shows that for any value of $\rho$, if the dynamics of SAM is close enough to any critical point, i.e.~enters an attractor, it is attracted by it. We also observe that if $\rho \rightarrow 0$, this attractor reduces to a single point, i.e.~the critical point itself.

To the best of our knowledge, this is the first time that these phenomena are  formally demonstrated. Importantly, these theoretical insights are consistent with the experimental results of \citep{kaddour2022flat} that show how SAM might get stuck around saddles.

Finally, by comparing the ODE of USAM (Eq. \eqref{eq:USAM_ODE_Quad_Insights}) with that of SAM (Eq. \eqref{eq:SAM_ODE_Quad}), we observe that the dynamics of SAM is equivalent to that of USAM where the radius $\rho$ has been scaled by $\frac{1}{\lVert H X_t \rVert}$. In a way, while USAM has a fixed radius $\rho$, SAM has an \textit{time-dependent} radius $\frac{\rho}{\lVert H X_t \rVert}$ which is smaller than $\rho$ if the dynamics is far from the origin ($\lVert H X_t \rVert>1$) and larger when it is close to it ($\lVert H X_t \rVert<1$). Therefore, SAM converges to the origin slower than USAM when it is far from it and it becomes faster as it gets closer.

\subsection{(DN)SAM SDE}\label{subsec:SAM_SDE_Quad}
%%%%%%%%%%%%%%%%%%%%%%%%%%%%%%%%%%%%%%%%%%%%%%%%%%%%%%%%%%%%%%%%%%%%%%%%%%%%
We now provide insights on the dynamics of the DNSAM SDE on a quadratic with Hessian $H$.
\begin{observation}[Details in \ref{lemma:SAM_SDE_Converg}]\label{lemma:SAM_SDE_Converg_Insights}
We observe that for all $\rho>0$, there exists an $\epsilon>0$ such that if $\lVert H X_t \rVert \in \left( \epsilon, -\rho \lambda_{*} \right) $, the dynamics of $X_t$ is attracted towards the origin. If the eigenvalues are all positive, the condition becomes $\lVert H X_t \rVert \in \left( \epsilon, \infty \right)$. On the contrary, if $\lVert H X_t \rVert<\epsilon$, then the dynamics is pushed away from the origin. 
\end{observation}

This insight suggests that if DNSAM is initialized close enough to a quadratic saddle, it is attracted toward it, but is also repulsed by it if it gets too close. This is due to the explosion of the volatility next to the origin. We expect that this observation extends to SAM as well, and it remains to be shown theoretically in future work. In the next section, we experimentally verify that the dynamics gets cyclically pulled to $0$ and pushed away from it, not only for the quadratic saddle but also for that of other models.

\vspace{-0.3cm}
%%%%%%%%%%%%%%%%%%%%%%%%%%%%%%%%%%%%%%%%%%%%%%%%%%%%%%%%%%%%%%%%%%%%%%%%%%%%
\section{Experiments}\label{sec:experiments}
%%%%%%%%%%%%%%%%%%%%%%%%%%%%%%%%%%%%%%%%%%%%%%%%%%%%%%%%%%%%%%%%%%%%%%%%%%%%

The main goal of this experimental section is two-fold: 1)
to verify the validity of the theorems derived in Section~\ref{sec:Insights}, and 2) to validate the claims made about the behavior of SAM and its variants near saddle points. The latter requires us to use models, for which saddle points are known to be present~\cite{safran2018spurious}, including for instance linear autoencoders~\cite{kunin2019loss}.

\subsection{Empirical Validation of the SDEs}\label{subces:Exp_Val_SDE}
%%%%%%%%%%%%%%%%%%%%%%%%%%%%%%%%%%%%%%%%%%%%%%%%%%%%%%%%%%%%%%%%%%%%%%%%%%%%

We first experimentally validate the results of Corollary \ref{thm:USAM_SDE_Simplified_Insights}, Corollary \ref{thm:SAM_SDE_Simplified_Insights}, and Theorem \ref{thm:DNSAM_SDE_Insights}. To do so, we use two different test functions ($g(x)$ in Def.~\eqref{def:G_Space_Weak_Solution}), which are $g_1(x):= \lVert x \rVert + \lVert \nabla f(x) \rVert$ and $g_2(x):= f(x)$. We test on four models. The first model is a convex quadratic landscape. The second task is a classification one on the Iris Database \cite{Dua:2019} using a linear MLP with $1$ hidden layer. The third is a classification task on the Breast Cancer Database \cite{Dua:2019} using a nonlinear MLP with $1$ hidden layer. The fourth is a Teacher-Student model where the Teacher is a linear MLP with $20$ hidden layers and the Student is a nonlinear MLP with $20$ hidden layers. Figure \ref{fig:Verif_Rho} uses the first metric $g_1(x)$ to measure the maximum absolute error (across the whole trajectory) of the SDEs of SGD, USAM, DNSAM, and SAM in approximating the respective discrete algorithms. Additionally, we plot the same error if we were to use the SDE of SGD to model/approximate the discrete iterates of SAM and its variants. We observe that when $\rho = \eta$, the absolute error is small in all cases, meaning that all the discrete iterates and SDEs behave essentially in the same way. This supports our claim that if $\rho = \mathcal{O}(\eta)$, the SDE of SGD is a good model for USAM, DNSAM, and SAM (Theorem \ref{thm:USAM_SGD}, Theorem \ref{thm:DNSAM_SGD}, and Theorem \ref{thm:SAM_SGD}). When $\rho = \sqrt{\eta}$, we see that the derived SDEs correctly approximate the respective discrete algorithms, while the SDE of SGD has a significantly larger relative error, which validates the results of Corollary \ref{thm:USAM_SDE_Simplified_Insights}, Theorem \ref{thm:DNSAM_SDE_Insights}, and Corollary \ref{thm:SAM_SDE_Simplified_Insights}. Although we do not have any theoretical guarantee for larger $\rho$, we observe empirically that the modeling error is still rather low. Finally, Figure \ref{fig:Verif_Loss} shows the evolution of the metric $g_2(x):= f(x)$ for the different algorithms. We notice that all the SDEs are matching the respective algorithms. In Appendix \ref{appendix:ImportanceCovariance}, we provide evidence that failing to include the correct diffusion terms in the USAM SDE Eq. \eqref{eq:USAM_SDE_Insights} and the DNSAM SDE Eq. \eqref{eq:DNSAM_SDE_Insights} leads to less accurate models.

\begin{wrapfigure}{o}{0.5\linewidth}
\vspace{-2mm}
\centering
    \includegraphics[width=0.99\linewidth]{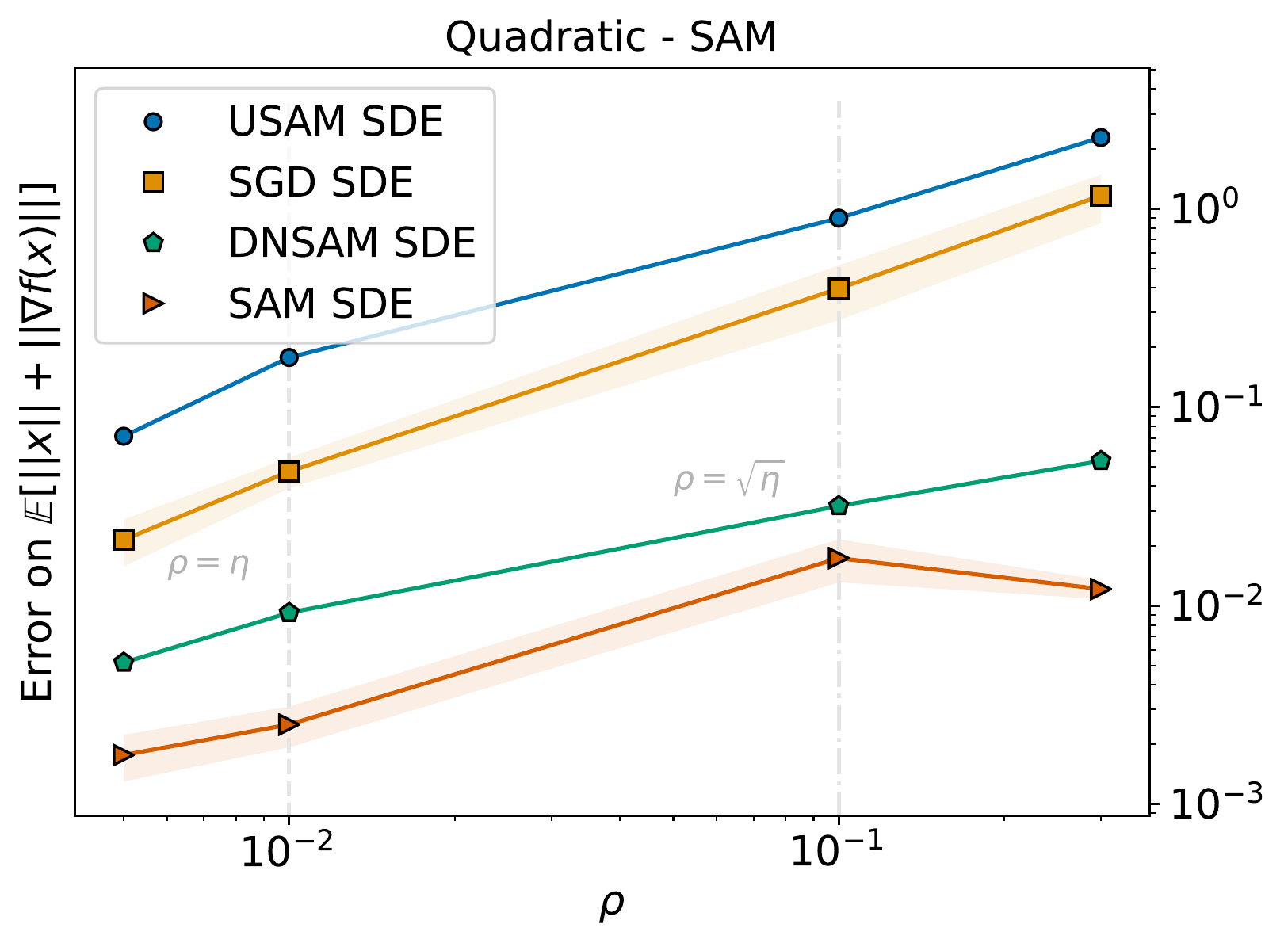}
    \vspace{-7mm}
  \caption{\small{Comparison in terms of $g_1(x)$ with respect to $\rho$.}}
  \label{fig:SAM_Best_Approx}%
  \vspace{-4mm}
\end{wrapfigure}

Finally, Figure \ref{fig:SAM_Best_Approx} shows that, in the Quadratic case, DNSAM results in a much closer approximation to SAM than other SDEs. Based on this observation and the analyses of Section \ref{sec:pitfalls}, we conclude that DNSAM is a better proxy to theoretically study SAM than USAM. It however remains not advised to employ DNSAM as a practical algorithm since its update rule requires the calculation of the full gradient, see Eq.~\eqref{eq:DNSAM_Discr_Update}.

\vspace{-3mm}
\paragraph{Interplay between noise, curvature, $\rho$, and suboptimality} Next, we check how the parameter $\rho$ and the curvature (measured by the trace operator of the Hessian) influence the noise of the stochastic process and its suboptimality. These insights substantiate the intuition that SAM and its variants are more likely to escape sharp minima faster than SGD.

First of all, we fix the value of $\rho$ as well as a diagonal Hessian $H$ with random positive eigenvalues. Then, we study the loss for SGD, USAM, DNSAM, and SAM as they optimize a convex quadratic loss of increasingly larger curvature (i.e. larger Hessian magnitude). We observe that DNSAM exhibits a loss that is orders of magnitude larger than that of SGD, with more variance, and even more so as the curvature increases. Note that SAM behaves similarly to DNSAM, but with less variance. Finally, USAM exhibits a similar pattern but less pronounced. All the observations are consistent with the insights gained from the covariance matrices in the respective SDEs. For details, we refer the reader to Figure \ref{fig:SAM_SGD_Hessian_Noise}, Figure \ref{fig:USAM_SGD_Hessian_Noise}, and Figure \ref{fig:True_SAM_SGD_Hessian_Noise} in Appendix.

In a similar experiment, we fix the Hessian as above and study the loss as we increase $\rho$. Once again, we observe that DNSAM exhibits a larger loss with more variance, and this is more and more clear as $\rho$ increases. Observations similar to the above ones can be done for SAM and USAM. For details, we refer the reader to Figure \ref{fig:SAM_SGD_Rho_Noise}, \ref{fig:True_SAM_SGD_Rho_Noise} and Figure \ref{fig:USAM_SGD_Rho_Noise} in Appendix.

Finally, we note that SAM and its variant exhibit an additional implicit curvature-induced noise compared to SGD. This leads to increased suboptimality as well as a higher likelihood to escape sharp minima. We provide an additional justification for this phenomenon in Observation \ref{observation:subopt}.

\begin{figure}%
    \centering
    \subfloat{{\includegraphics[width=0.49\linewidth]{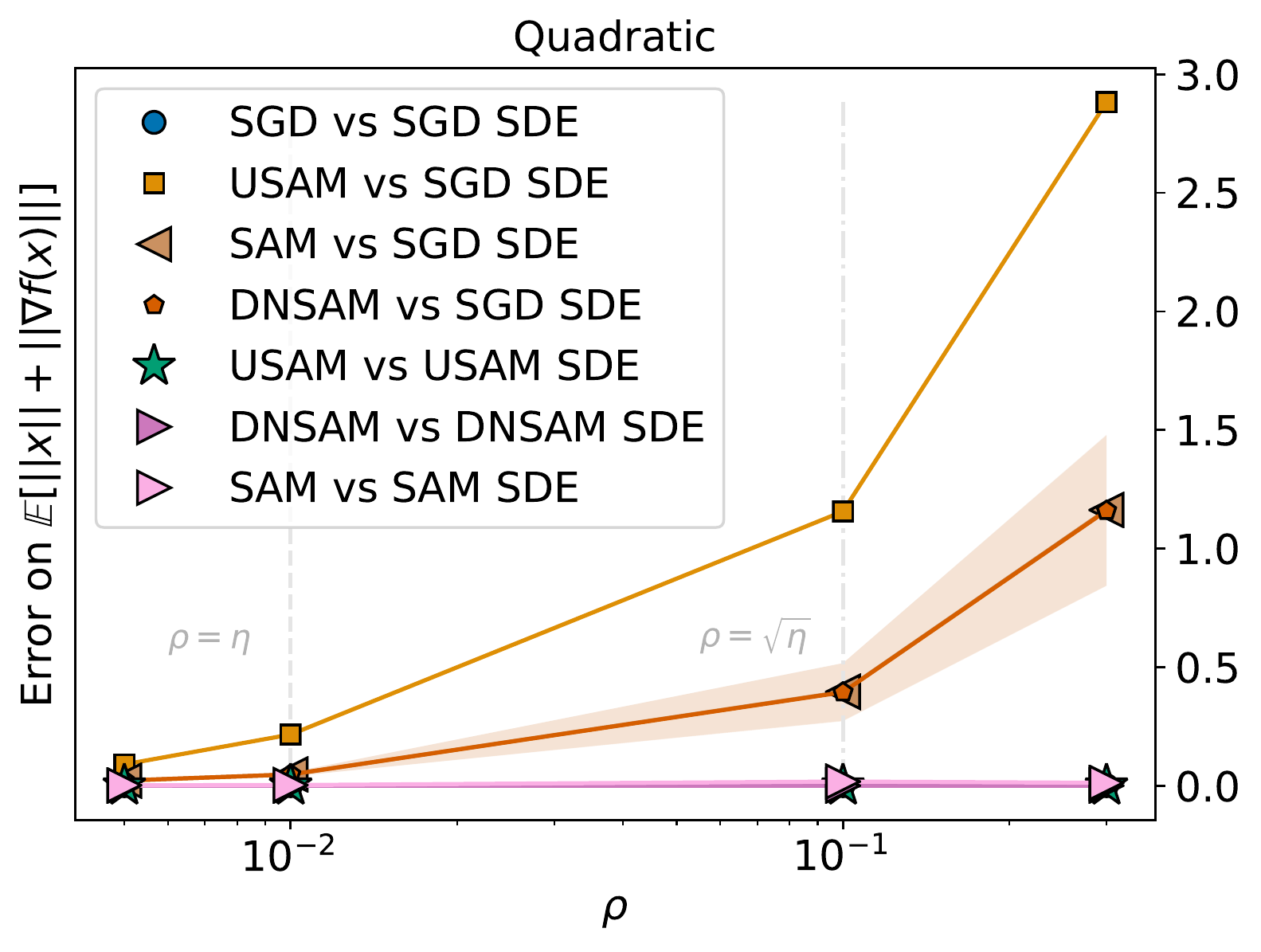} }}%
    \subfloat{{\includegraphics[width=0.49\linewidth]{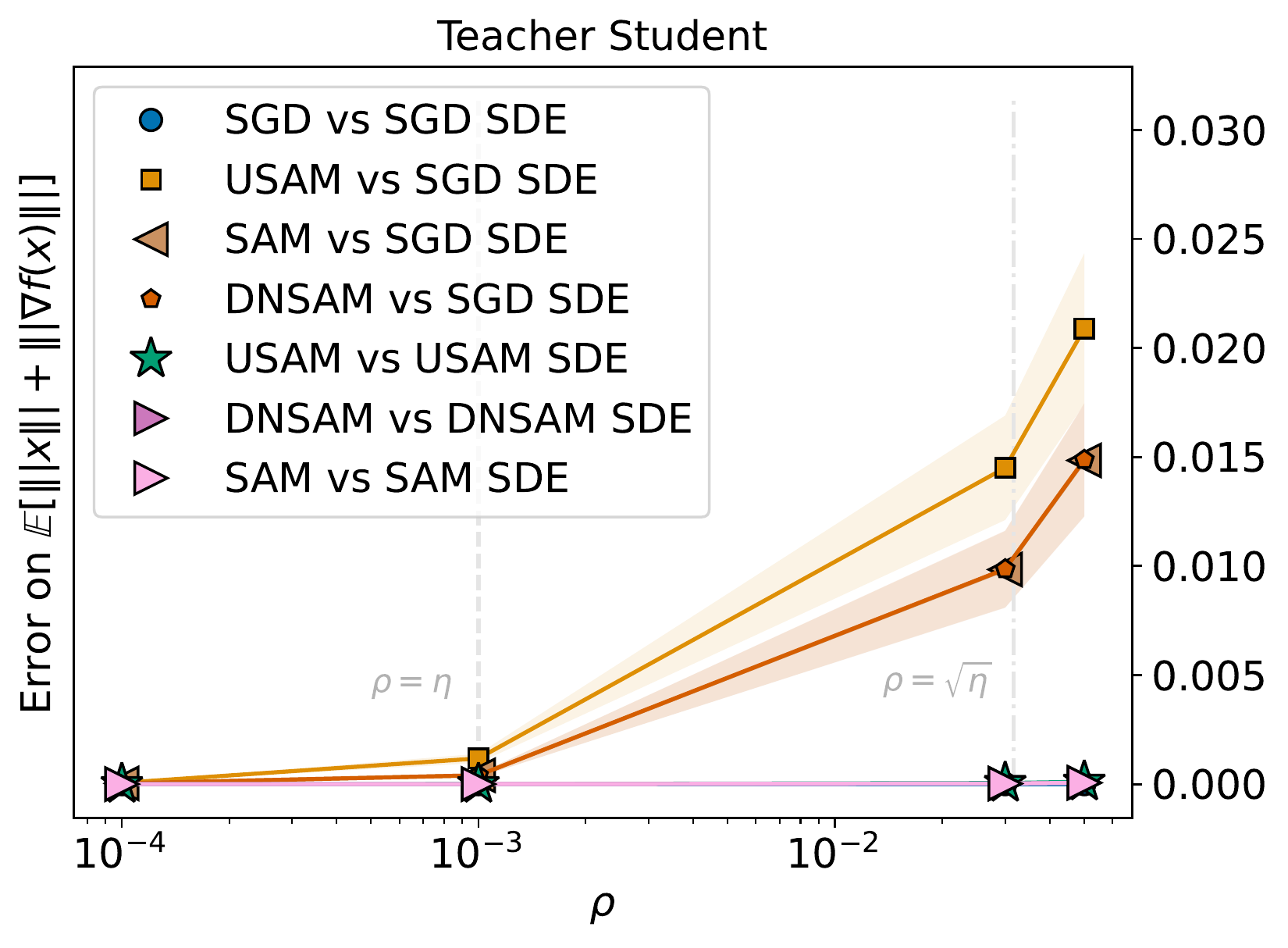} }} \\
    \subfloat{{\includegraphics[width=0.49\linewidth]{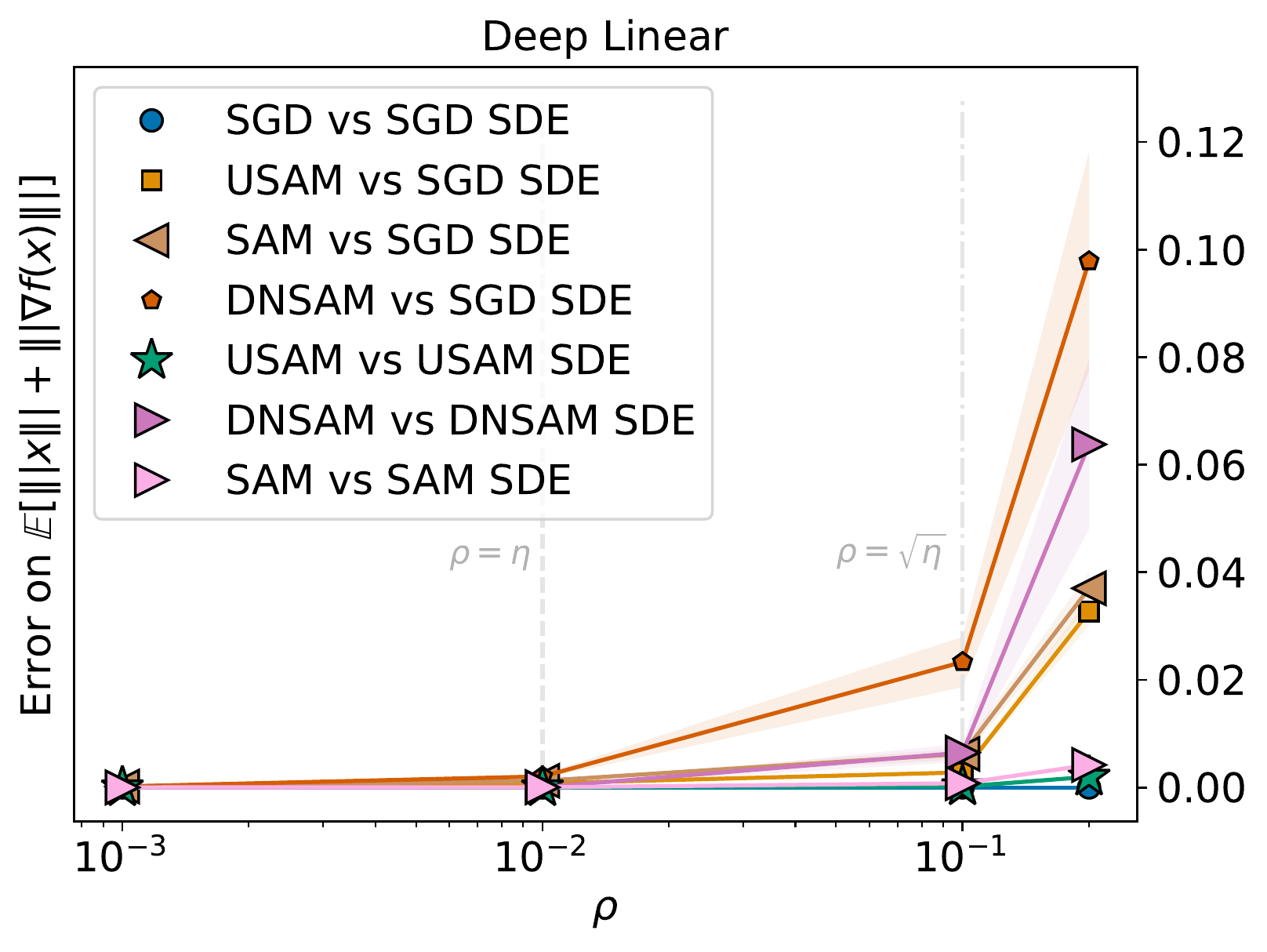} }}%
    \subfloat{{\includegraphics[width=0.49\linewidth]{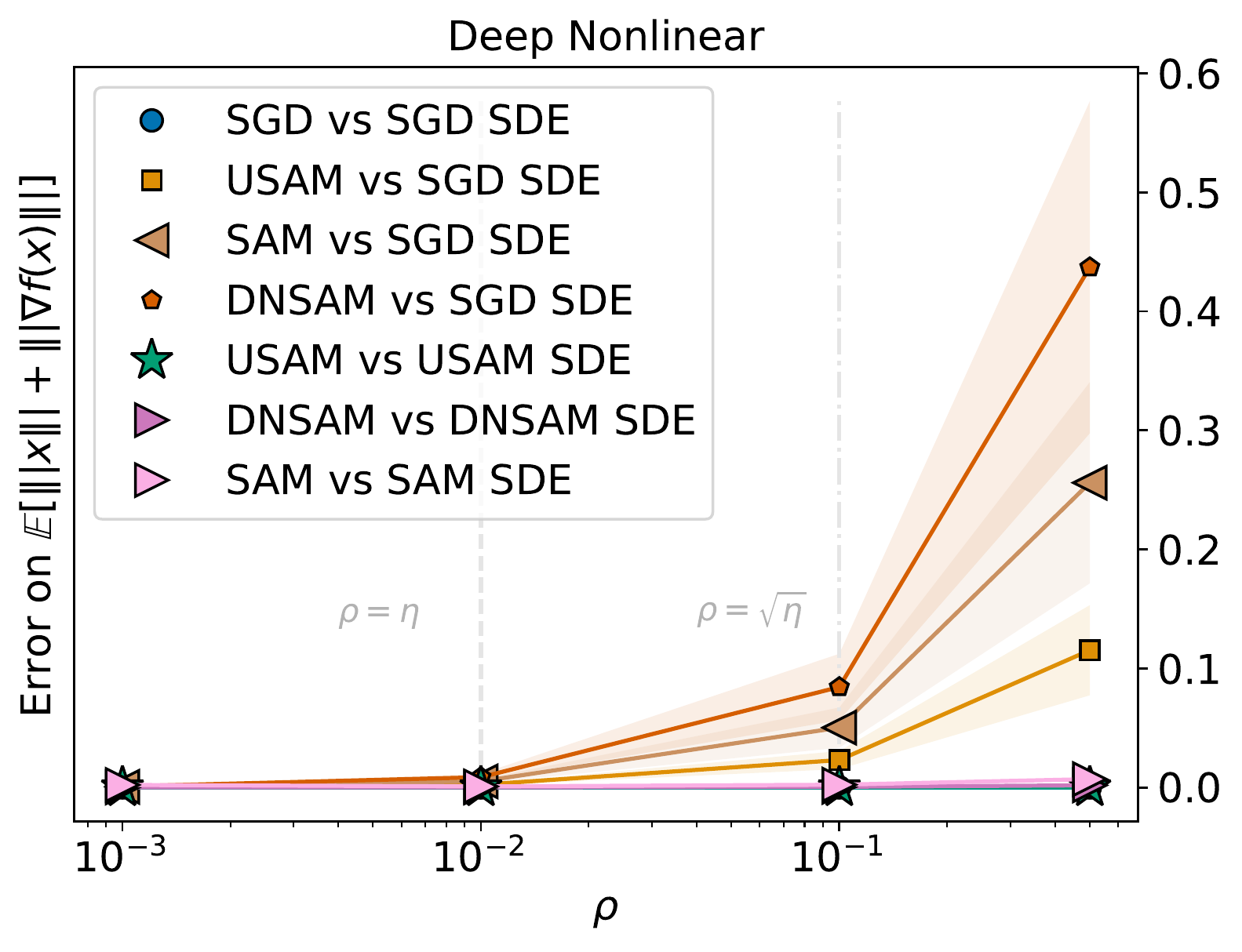} }}%
    \caption{Comparison in terms of $g_1(x)$ with respect to $\rho$ - Quadratic (top left); Teacher-Student (top right); Deep linear class (lower left); Deep Nonlinear class (lower right).}%
    \label{fig:Verif_Rho}%
\end{figure}

\begin{figure}%
    \centering
    \subfloat{{\includegraphics[width=0.49\linewidth]{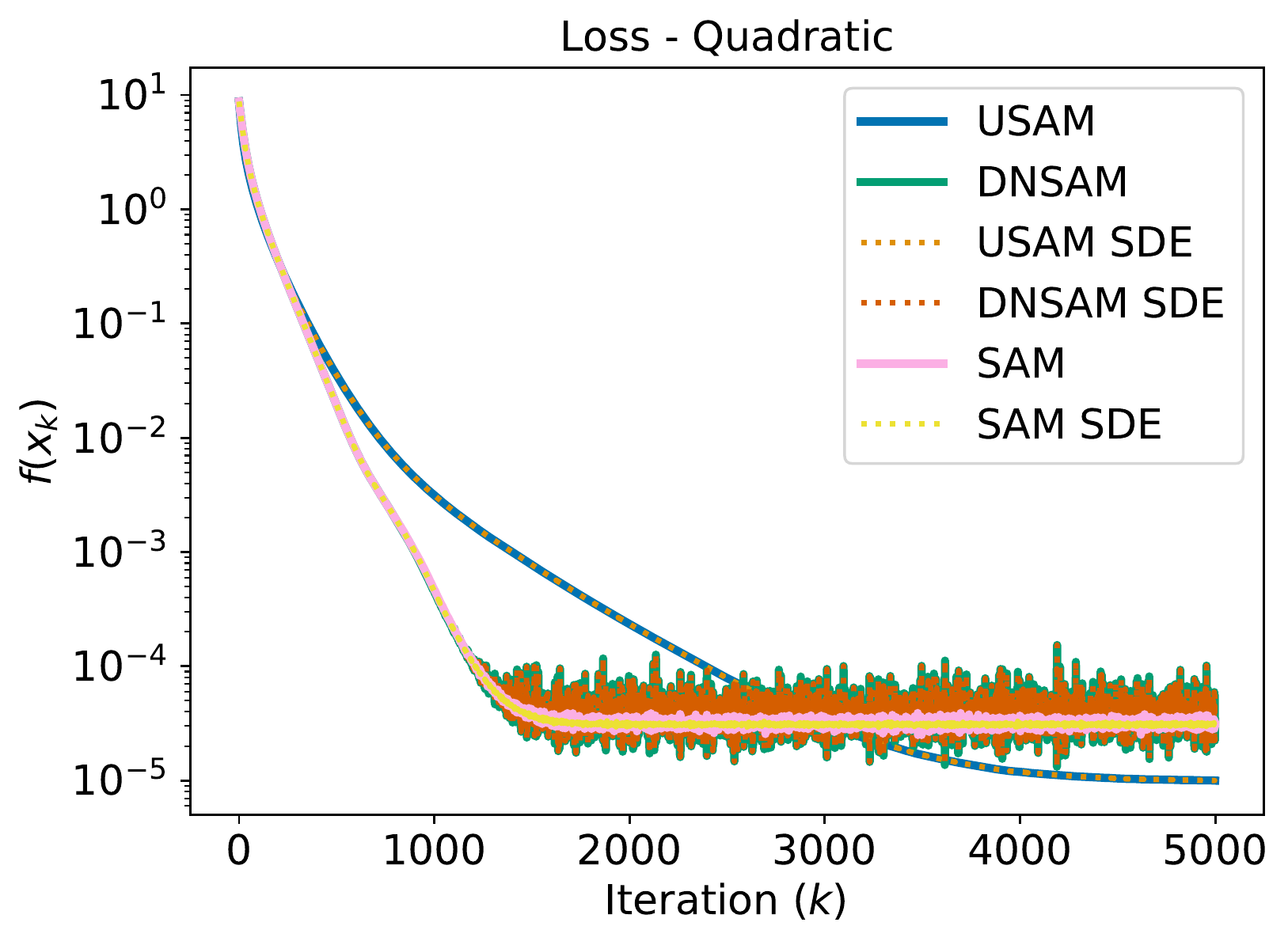} }}%
    \subfloat{{\includegraphics[width=0.49\linewidth]{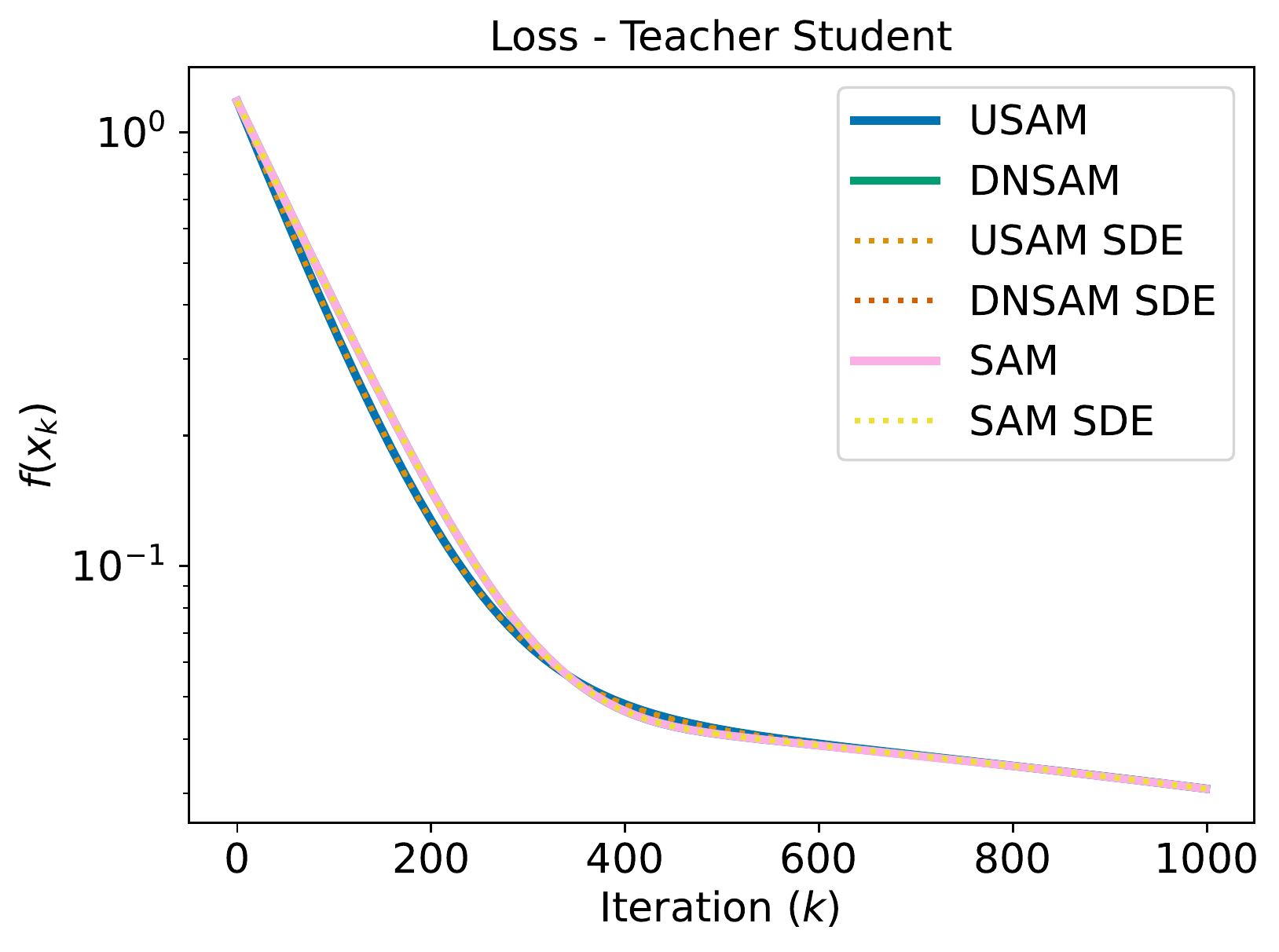} }} \\
    \subfloat{{\includegraphics[width=0.49\linewidth]{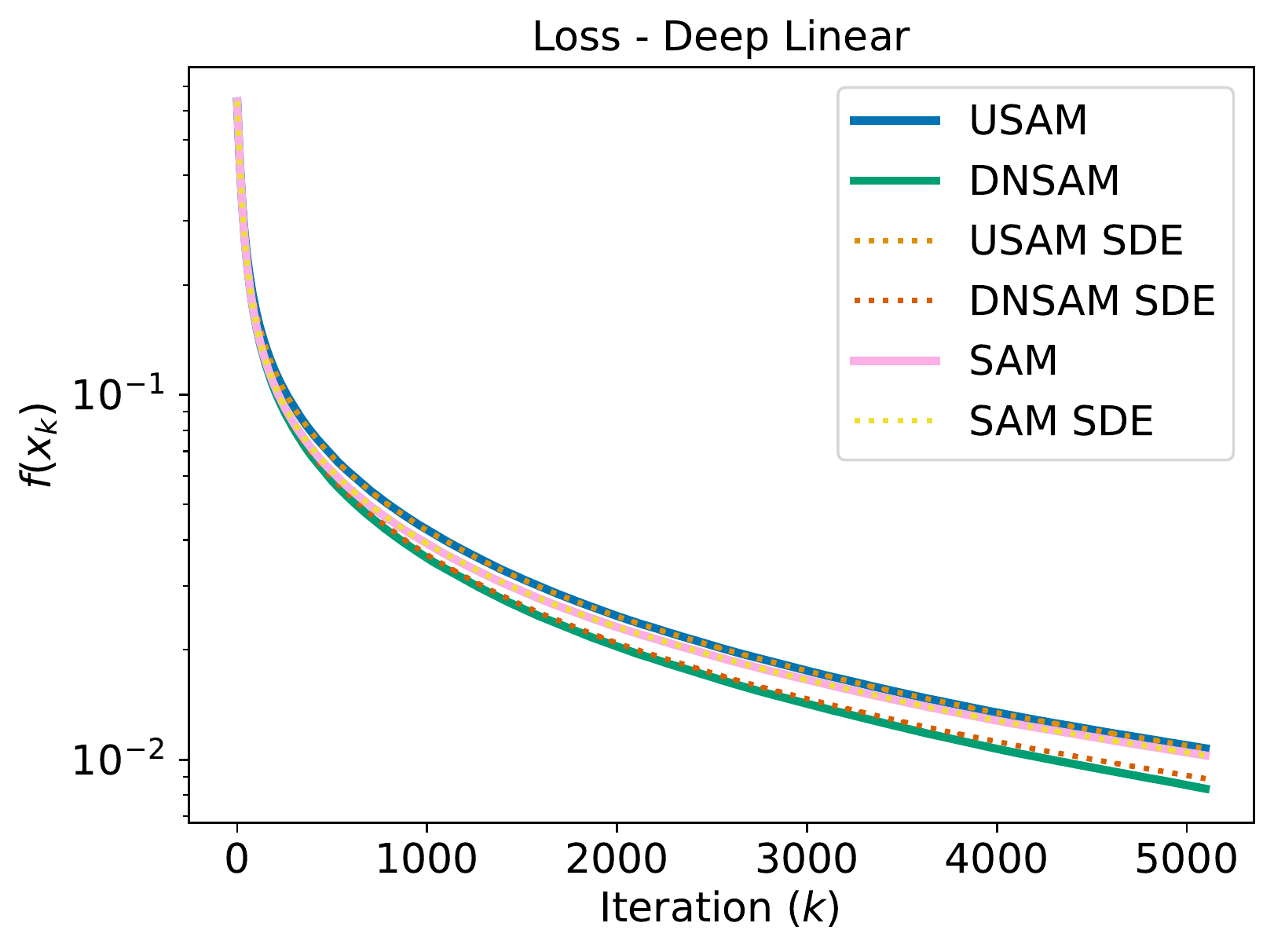} }}%
    \subfloat{{\includegraphics[width=0.49\linewidth]{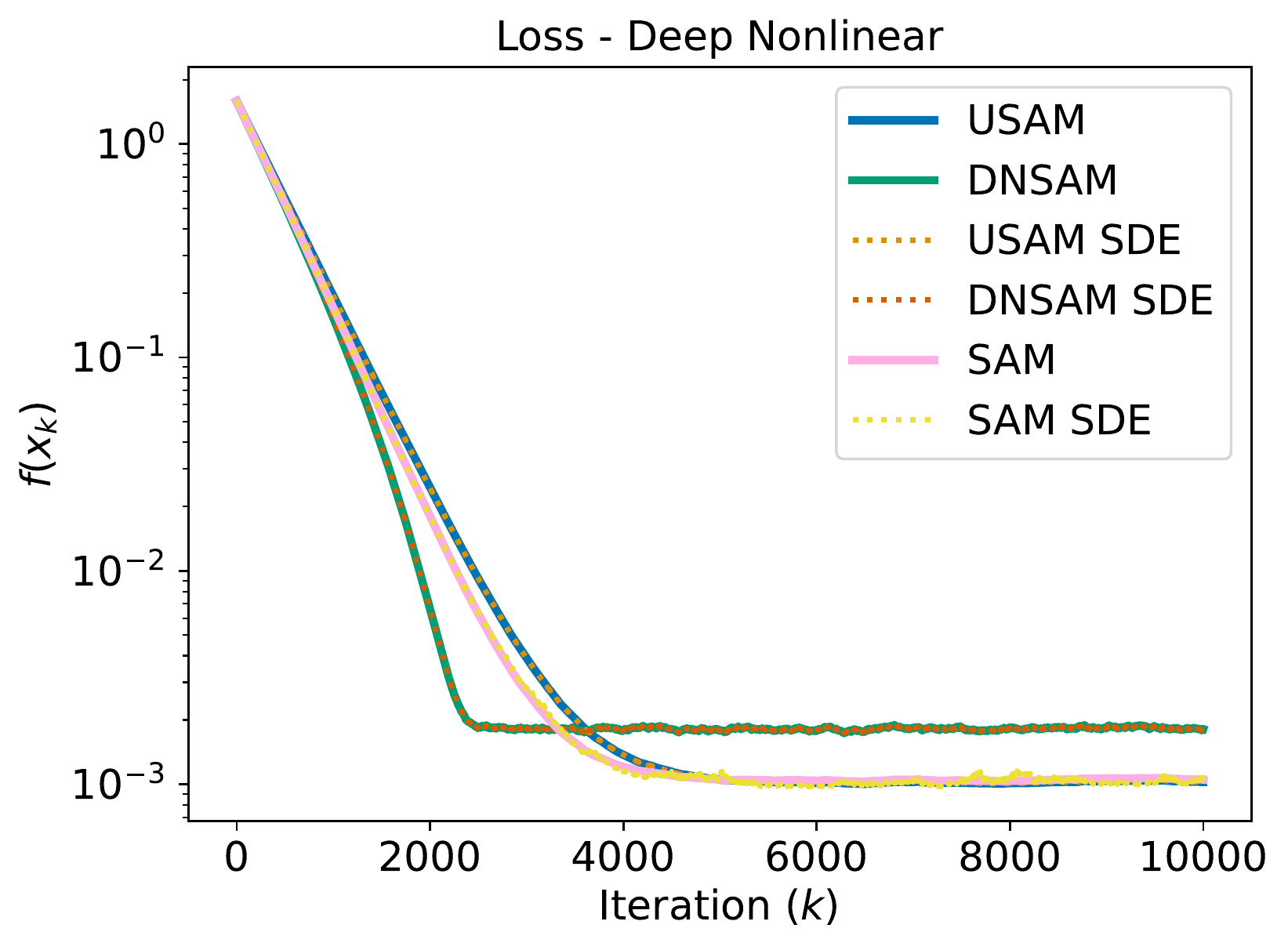} }}%
    \caption{Comparison in terms of $g_2(x)$ with respect to time - Quadratic (top left); Teacher-Student (top right); Deep linear class (lower left); Deep Nonlinear class (lower right).}%
    \label{fig:Verif_Loss}%
\end{figure}
\vspace{-0.2cm}
\subsection{Behavior Near Saddles}\label{subces:Exp_Saddle}
%%%%%%%%%%%%%%%%%%%%%%%%%%%%%%%%%%%%%%%%%%%%%%%%%%%%%%%%%%%%%%%%%%%%%%%%%%%%
In this section, we study the behavior of SAM and USAM (full batch versions), and of PSAM, DNSAM, and PUSAM (perturbed gradient versions) near saddle points. See Table \ref{table:AlgoComparison} for more details.
\vspace{-0.2cm}
\paragraph{Quadratic Landscape} We first empirically verify the insight gained in Section \ref{subsec:SAM_SDE_Quad} --- the dynamics of DNSAM is attracted to the origin, but if it gets too close, it gets repulsed away. For a quadratic saddle, in Figure \ref{fig:SAM_Saddle_StatDistr} we show the distribution of $10^{5}$ trajectories after $5 \cdot 10^{4}$ iterations. These are distributed symmetrically around the origin but the concentration is lower close to it. While this is intuitive for the convex case (see Figure \ref{fig:SAM_Convex_StatDistr} in Appendix), it is surprising for the saddle case: our insights are fully verified.
The second and third images of Figure \ref{fig:SAM_Saddle_StatDistr} show that all the trajectories are initialized outside of a certain ball around the origin and then they get pulled inside it. Then, we see that the number of points outside this ball increased again and the right-most image shows the number of points jumping in and out of it. This shows that there is a cyclical dynamics towards and away from the origin. Of course, all the points eventually escape the saddle, but much more slowly than what would happen under the dynamics of SGD where the trajectories would not even get close to the origin in the first place. In Figure \ref{fig:SAM_Saddle_Escape} in Appendix, we show the behavior of several optimizers when initialized in an escaping direction from the saddle and we observe that full-batch SAM is attracted by the saddle while the others are able to escape it. Interestingly, PSAM is anyway slower than SGD in escaping. Figure \ref{fig:SAM_Saddle_Escape} in Appendix shows that full-batch SAM and PSAM cannot escape the saddle if it is too close to it, while DNSAM can if it is close enough to enjoy a spike in volatility. More details are in the Appendix \ref{appendix:quadratic}.

\begin{figure}%
    \centering
    \subfloat{{\includegraphics[width=0.49\linewidth]{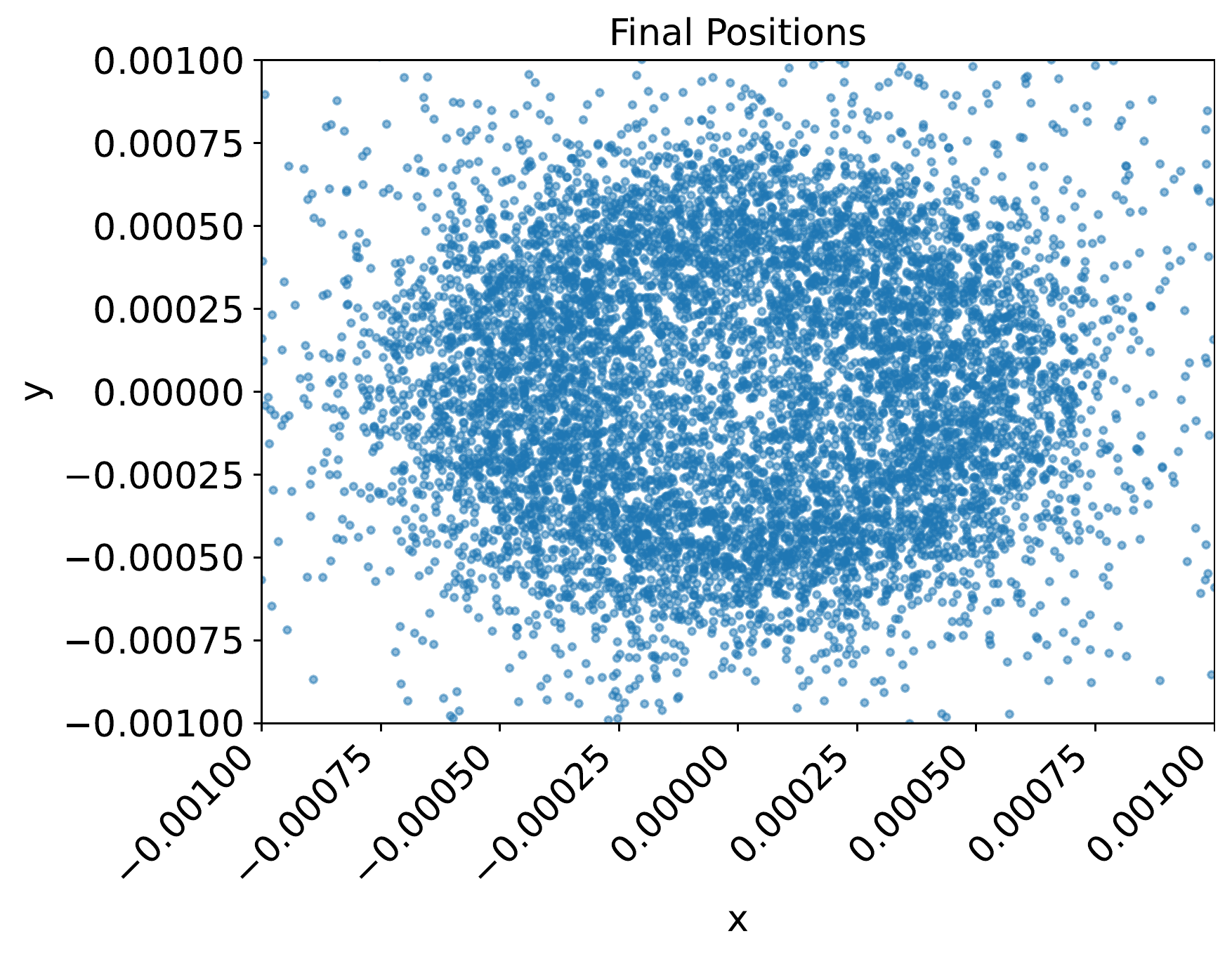} }}%
    \subfloat{{\includegraphics[width=0.49\linewidth]{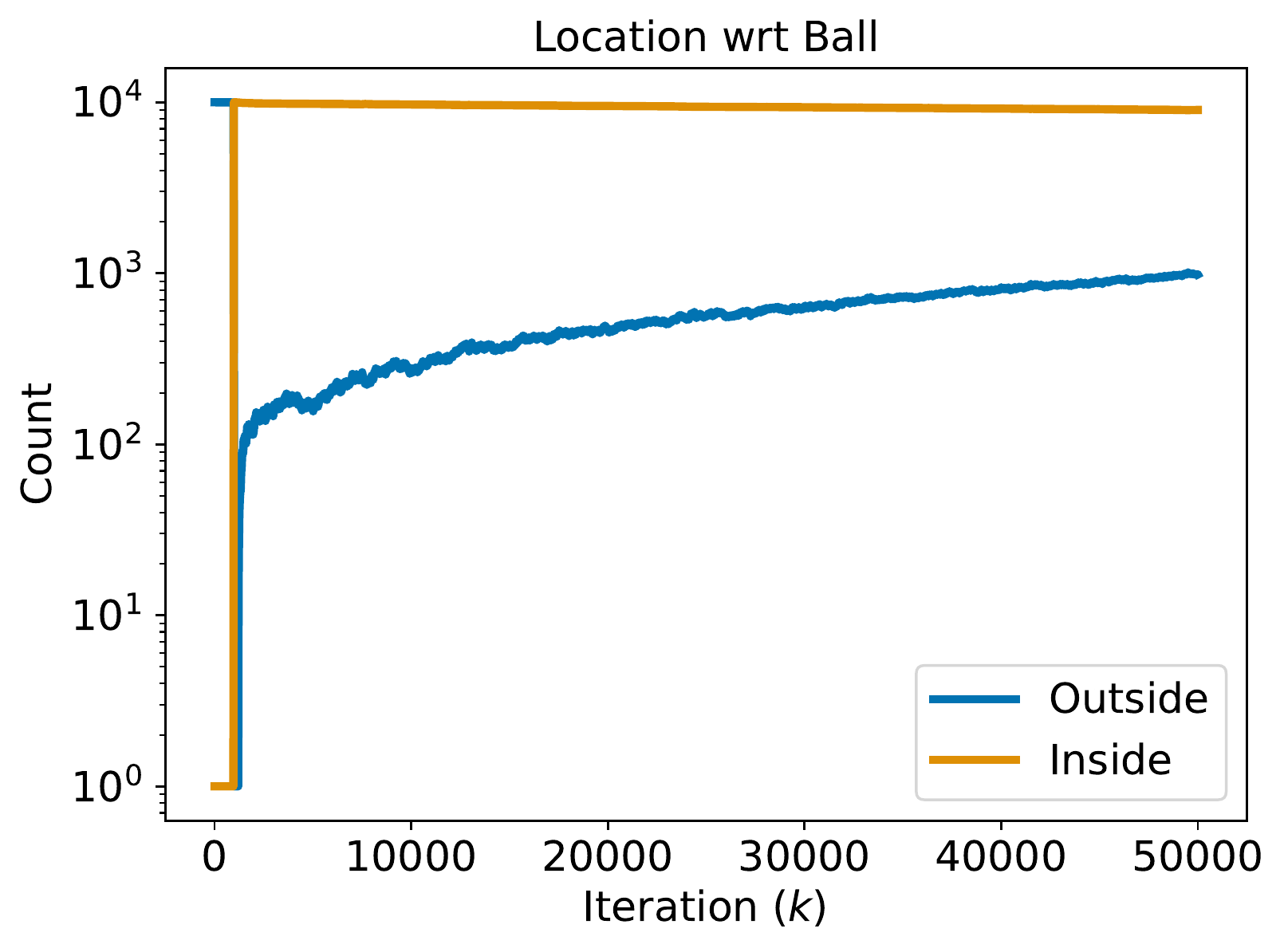} }} \\
    \subfloat{{\includegraphics[width=0.49\linewidth]{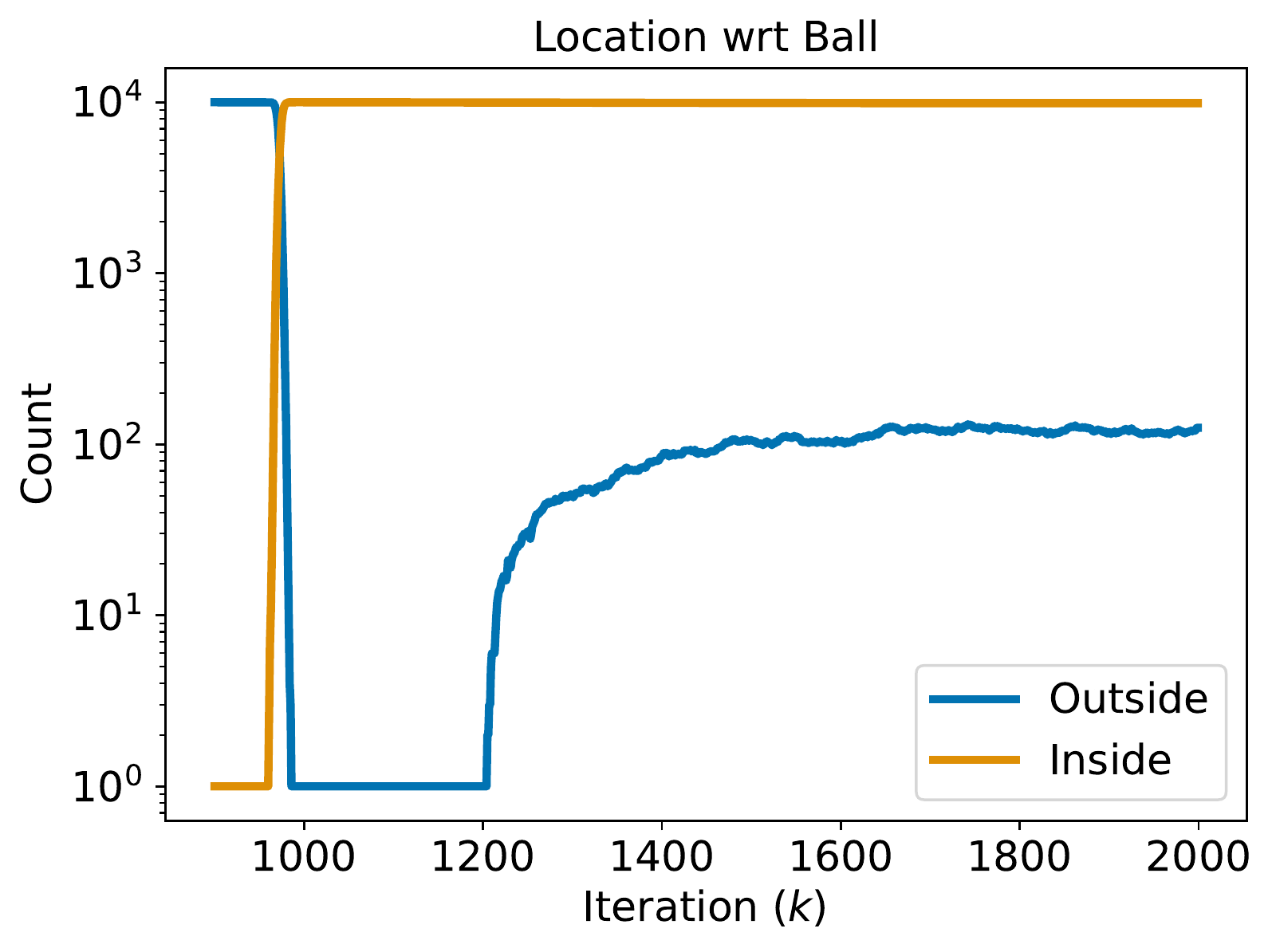} }}%
    \subfloat{{\includegraphics[width=0.49\linewidth]{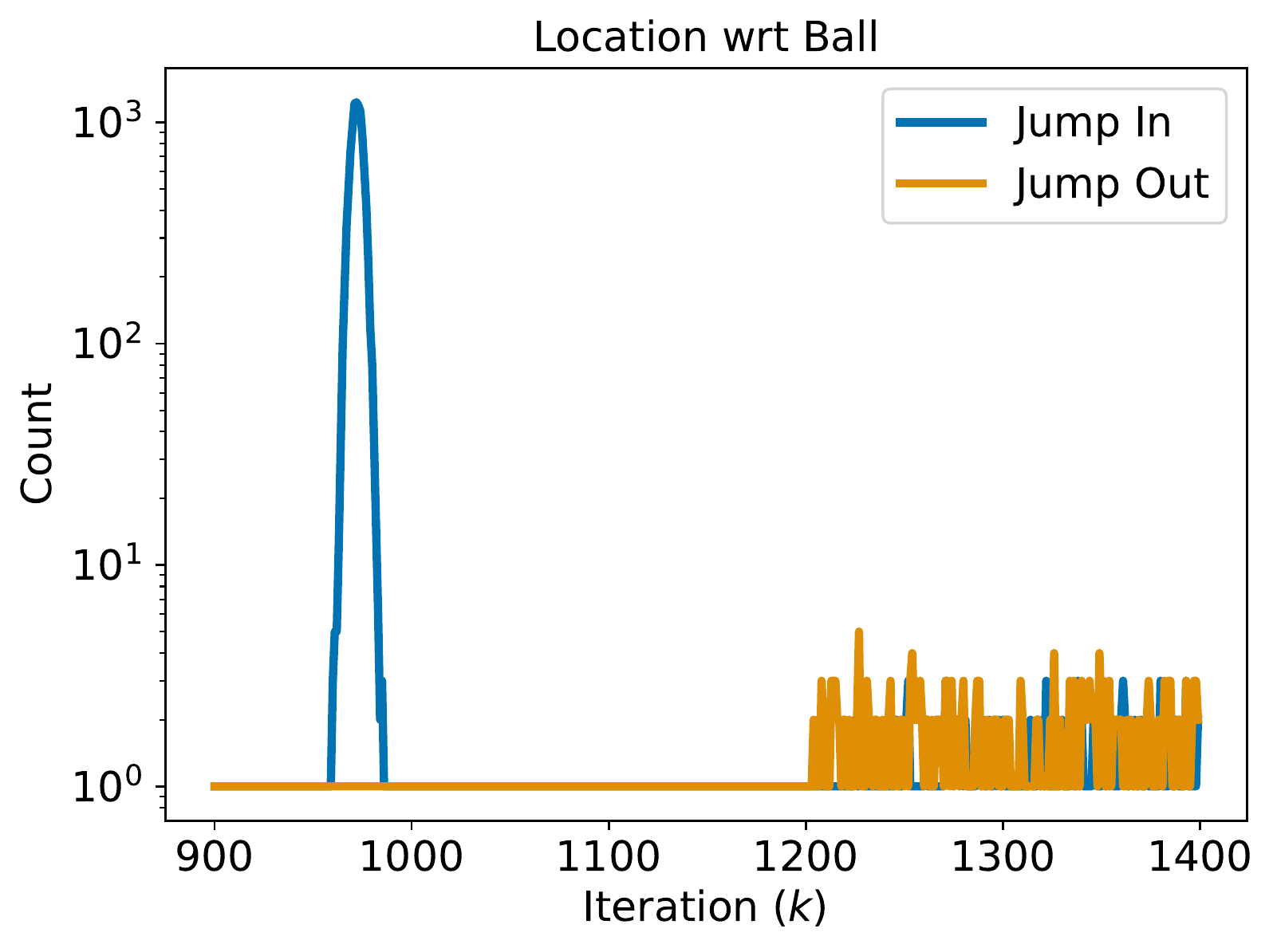} }}%
    \caption{Quadratic Saddle - Top Left: Distribution points around the origin are scarcer near to the origin; Top Right: Number of trajectories outside a small ball around the origin increases over time; Lower Left: All the trajectories eventually enter the ball and then start exiting it; Lower Right: There is a constant oscillation of points in and out of the ball.}%
    \label{fig:SAM_Saddle_StatDistr}%
\end{figure}

\vspace{-2mm}
\paragraph{Linear Autoencoder} Inspired by the insights gained so far, we study the behavior of SAM when it is initialized close to the saddle present at the origin of the linear autoencoder introduced by~\cite{kunin2019loss}. The top-left of Figure \ref{fig:SAM_Autoencoder} shows the evolution of the loss as we optimize it with SAM starting from different starting points closer and closer to the saddle in the origin. The scalar $\sigma$ parametrizes how close the initialization is to the origin. We observe that when SAM starts sufficiently far from it ($\sigma \geq 0.005$), it optimizes immediately, while the closer it is initialized to it, the more it stays around it, up to not being able to move at all ($\sigma \leq 0.001$). Regarding DNSAM, in the top-right figure, we observe the same behavior, apart from one case: if it is initialized sufficiently close to the origin, instead of getting stuck there, it jumps away following a spike in volatility. Differently, PSAM behaves more like SAM and is slower in escaping if $\sigma$ is lower. The bottom-right of Figure \ref{fig:SAM_Autoencoder} shows the comparison with other optimizers: SAM does not optimize the loss while the other optimizers do. These findings are consistent with those observed in Figure \ref{fig:SAM_Saddle_Escape} in Appendix for the quadratic landscape. In Figure \ref{fig:SAM_EmbSaddle} in Appendix, we show a similar result for a saddle landscape studied in \cite{lucchi2022fractional}. More details are in Appendix \ref{appendix:autoencoder} and Appendix \ref{appendix:embedded}, respectively. In both these experiments, we observe the suboptimality patterns forecasted by our theory.

\begin{figure}%
    \centering
    \subfloat{{\includegraphics[width=0.49\linewidth]{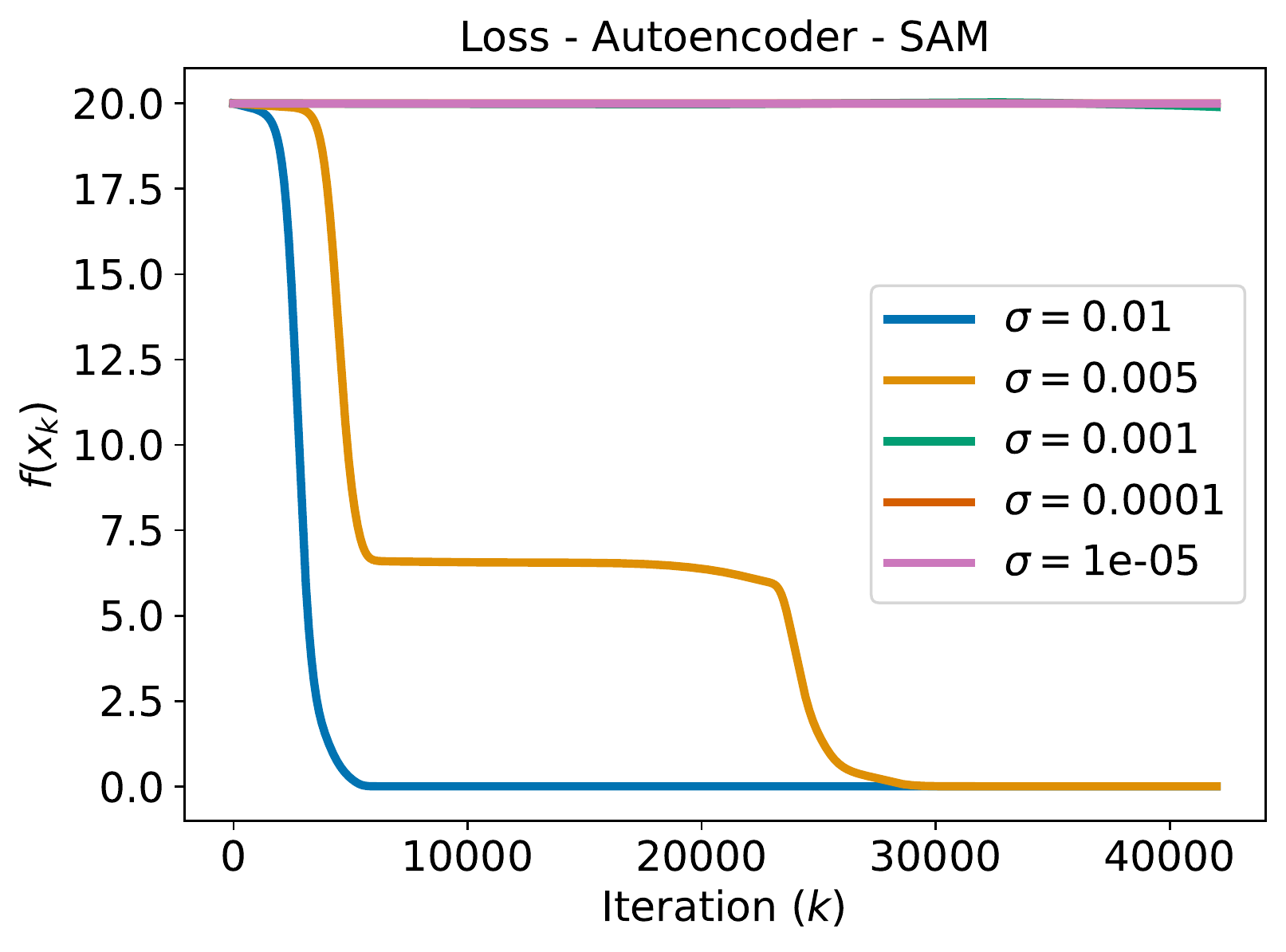} }}%
    \subfloat{{\includegraphics[width=0.49\linewidth]{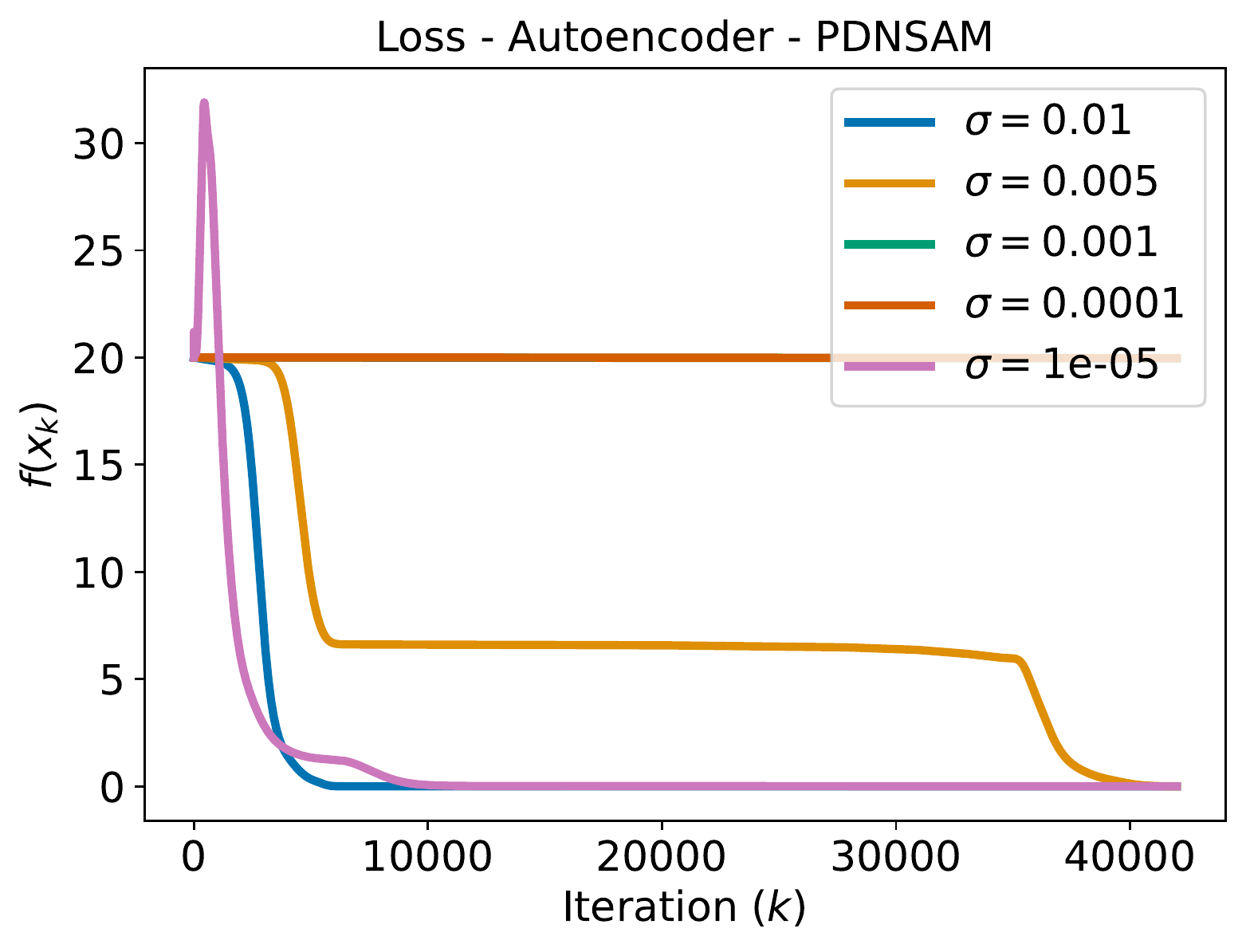} }} \\
    \subfloat{{\includegraphics[width=0.49\linewidth]{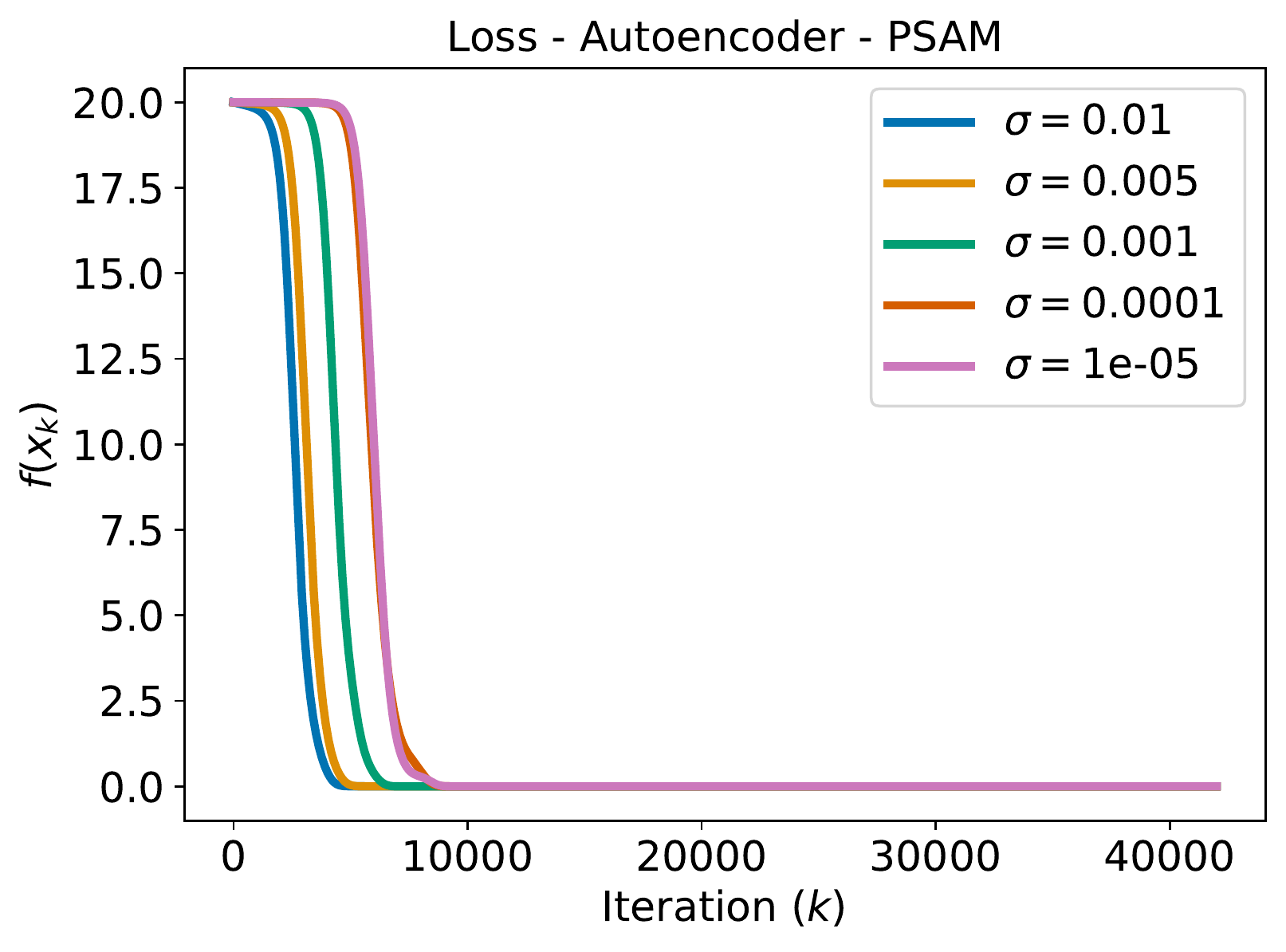} }}
    \subfloat{{\includegraphics[width=0.49\linewidth]{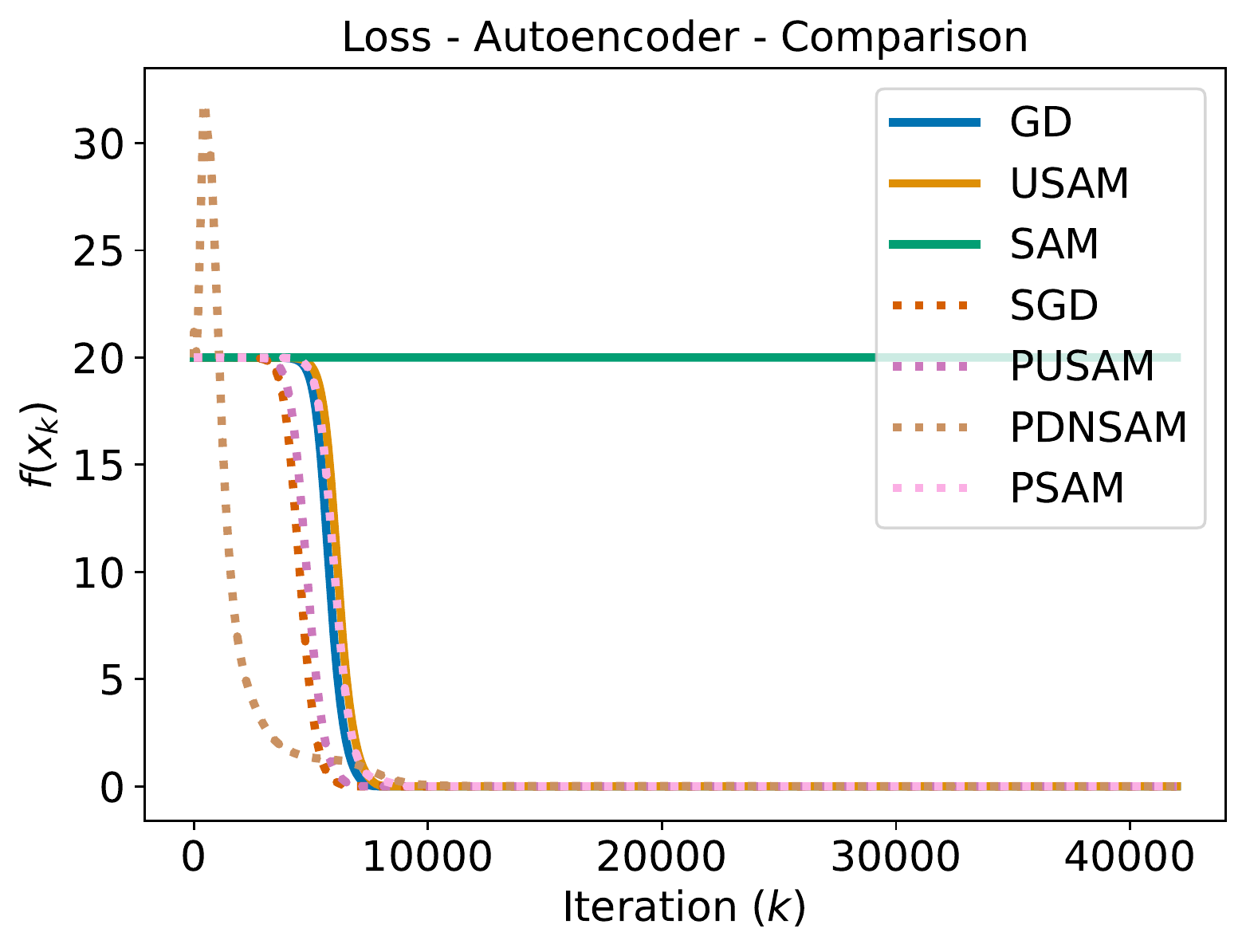} }}%
    \caption{Autoencoder - Top-Left: SAM does not escape the saddle if it is too close to it. Top-Right: DNSAM escapes if it is extremely close to the origin thanks to a volatility spike. Bottom-Left: Like SAM, PSAM does not escape if too close to the origin. Bottom-Right: DNSAM is the fastest to escape, while SAM is stuck. }%
    \label{fig:SAM_Autoencoder}%
\end{figure}

\section{Discussion}

\subsection{Future Work}
Inspired by \citep{Malladi2022AdamSDE}, it would be interesting to study possible scaling rules for SAM and its variants, thus shedding light on the interplay of the learning rate $\eta$, the batch size $B$, and $\rho$. Another direction could be to use our SDEs to study the role of $\rho$ in balancing the optimization speed and generalization properties of SAM. The insights gained could be useful in improving SAM in terms of optimization speed and generalization. Finally, we expect that the application of more analytical tools to the SDEs of SAM and DNSAM will lead to further insights into SAM. It would for instance be of particular interest to revisit claims made about other optimizers via their SDE models (see ``Applications of SDE approximations'' in Section~\ref{sec:RelWorks}). Hopefully, this will help to demystify the high performance of SAM on large-scale ML problems.

\subsection{Limitations}
We highlight that modeling discrete-time algorithms via SDEs relies on Assumption \ref{ass:regularity_f}. Furthermore, this setup cannot fully capture the regime of large learning rates. As observed in \citep{Li2021validity}, a large $\eta$ or the lack of certain conditions on $\nabla f$ and on the noise covariance matrix might lead to an approximation failure. However, the authors claim that this failure could be avoided by increasing the order of the weak approximation. Additionally, most of our discussions are focused on the case where $\rho = \mathcal{O}(\sqrt{\eta})$, which is not the only interesting setup, as some authors use $\rho < \eta$. Finally, since our work is more theoretical in nature, we did not aim at conducting SOTA experiments but rather focused on improving the understanding of the dynamics of SAM. Thus, we analyzed relatively simple models and landscapes that are relevant to the optimization and machine learning community.

\subsection{Conclusion} \label{sec:conclusion}
We proposed new continuous-time models (in the form of SDEs) for the SAM optimizer and two variants. While the USAM variant was introduced in prior work~\citep{andriushchenko2022towards}, the DNSAM variant we introduce is a step between USAM and SAM, allowing us to gain further insights into the role of the normalization. 
We formally proved (and experimentally verified) that these SDEs approximate their real discrete-time counterparts; see Theorems~\ref{thm:USAM_SDE_Insights}--\ref{thm:SAM_SDE_Simplified_Insights} for the theory and Section~\ref{subces:Exp_Val_SDE} for the experiments. An interesting side aspect of our analysis is that DNSAM appears to be a better surrogate model to describe the dynamics of SAM than USAM: SAM and DNSAM share common behaviors around saddles, they have more similar noise structures, and experiments support these claims. Of course, by no means does this paper intend to propose DNSAM as a new practical optimizer: it is instead meant to be used for theoretical analyses.

The SDEs we derived explicitly decompose the learning dynamics (in the parameter space) into a deterministic drift and a stochastic diffusion coefficient which in itself reveals some novel insights:
The drift coefficient~--~by the definition of $\tilde{f}$ in Theorems~\ref{thm:USAM_SDE_Insights}--\ref{thm:SAM_SDE_Simplified_Insights}~--~exposes how the ascent parameter $\rho$ impacts the average dynamics of SAM and its variants.
The diffusion coefficient, on the other hand, increases with the Hessian of the loss~--~thereby implying that SAM and its variants are noisier in sharp minima.
This could be interpreted as an implicit bias towards flat minima (as sharp minima will be more unstable due to the noise).

The continuous-time SDE models allow the application of tools from stochastic calculus (e.g.~integration and differentiation) to study the behavior of SAM.
As a start in this direction, we proved that the flow of USAM gets stuck around saddles if $\rho$ is too large. In contrast, SAM oscillates around saddles if initialized close to them but eventually slowly escapes them thanks to the additional noise. Importantly, our claims are substantiated by experiments on several models and invite further investigation to prevent a costly waste of computation budget near saddle points.

\paragraph{Acknowledgement} We would like to thank the reviewers for their feedback which greatly helped us improve this manuscript.
Frank Proske acknowledges the financial support of the Norwegian Research Council (project number 274410). Enea Monzio Compagnoni and Aurelien Lucchi acknowledge the financial support of the Swiss National Foundation, SNF grant No 207392.
Hans Kersting thanks the European Research Council for support through the ERC grant 724063.

\bibliography{example_paper}

\begin{thebibliography}{59}
\providecommand{\natexlab}[1]{#1}
\providecommand{\url}[1]{\texttt{#1}}
\expandafter\ifx\csname urlstyle\endcsname\relax
  \providecommand{\doi}[1]{doi: #1}\else
  \providecommand{\doi}{doi: \begingroup \urlstyle{rm}\Url}\fi

\bibitem[An et~al.(2020)An, Lu, and Ying]{an2020stochastic}
An, J., Lu, J., and Ying, L.
\newblock Stochastic modified equations for the asynchronous stochastic
  gradient descent.
\newblock \emph{Information and Inference: A Journal of the IMA}, 9\penalty0
  (4):\penalty0 851--873, 2020.

\bibitem[Andriushchenko \& Flammarion(2022)Andriushchenko and
  Flammarion]{andriushchenko2022towards}
Andriushchenko, M. and Flammarion, N.
\newblock Towards understanding sharpness-aware minimization.
\newblock In \emph{International Conference on Machine Learning}, pp.\
  639--668. PMLR, 2022.

\bibitem[Bahri et~al.(2022)Bahri, Mobahi, and Tay]{bahri2021sharpness}
Bahri, D., Mobahi, H., and Tay, Y.
\newblock Sharpness-aware minimization improves language model generalization.
\newblock \emph{ACL 2022}, 2022.

\bibitem[Bartlett et~al.(2022)Bartlett, Long, and
  Bousquet]{bartlett2022dynamics}
Bartlett, P.~L., Long, P.~M., and Bousquet, O.
\newblock The dynamics of sharpness-aware minimization: Bouncing across ravines
  and drifting towards wide minima.
\newblock \emph{arXiv preprint arXiv:2210.01513}, 2022.

\bibitem[Chaudhari \& Soatto(2018)Chaudhari and
  Soatto]{chaudhari2018stochastic}
Chaudhari, P. and Soatto, S.
\newblock Stochastic gradient descent performs variational inference, converges
  to limit cycles for deep networks.
\newblock In \emph{2018 Information Theory and Applications Workshop (ITA)},
  pp.\  1--10. IEEE, 2018.

\bibitem[Chen et~al.(2015)Chen, Ding, and Carin]{chen2015convergence}
Chen, C., Ding, N., and Carin, L.
\newblock On the convergence of stochastic gradient mcmc algorithms with
  high-order integrators.
\newblock \emph{Advances in neural information processing systems}, 28, 2015.

\bibitem[Daneshmand et~al.(2018)Daneshmand, Kohler, Lucchi, and
  Hofmann]{daneshmand2018escaping}
Daneshmand, H., Kohler, J., Lucchi, A., and Hofmann, T.
\newblock Escaping saddles with stochastic gradients.
\newblock In \emph{International Conference on Machine Learning}, pp.\
  1155--1164. PMLR, 2018.

\bibitem[Du et~al.(2022)Du, Yan, Feng, Zhou, Zhen, Goh, and
  Tan]{du2021efficient}
Du, J., Yan, H., Feng, J., Zhou, J.~T., Zhen, L., Goh, R. S.~M., and Tan, V.~Y.
\newblock Efficient sharpness-aware minimization for improved training of
  neural networks.
\newblock \emph{ICLR 2022}, 2022.

\bibitem[Dua \& Graff(2017)Dua and Graff]{Dua:2019}
Dua, D. and Graff, C.
\newblock {UCI} machine learning repository, 2017.
\newblock URL \url{http://archive.ics.uci.edu/ml}.

\bibitem[Dziugaite \& Roy(2017)Dziugaite and Roy]{Dziugaite2017PacBayes}
Dziugaite, G.~K. and Roy, D.~M.
\newblock Computing nonvacuous generalization bounds for deep (stochastic)
  neural networks with many more parameters than training data.
\newblock In \emph{Uncertainty in Artificial Intelligence}, 2017.

\bibitem[Folland(2005)]{folland2005higher}
Folland, G.~B.
\newblock Higher-order derivatives and taylor’s formula in several variables.
\newblock \emph{Preprint}, pp.\  1--4, 2005.

\bibitem[Foret et~al.(2021)Foret, Kleiner, Mobahi, and
  Neyshabur]{foret2020sharpness}
Foret, P., Kleiner, A., Mobahi, H., and Neyshabur, B.
\newblock Sharpness-aware minimization for efficiently improving
  generalization.
\newblock \emph{ICLR 2021}, 2021.

\bibitem[Gardiner et~al.(1985)]{gardiner1985handbook}
Gardiner, C.~W. et~al.
\newblock \emph{Handbook of stochastic methods}, volume~3.
\newblock springer Berlin, 1985.

\bibitem[Ge et~al.(2015)Ge, Huang, Jin, and Yuan]{ge2015escaping}
Ge, R., Huang, F., Jin, C., and Yuan, Y.
\newblock Escaping from saddle points—online stochastic gradient for tensor
  decomposition.
\newblock In \emph{Conference on Learning Theory}, pp.\  797--842, 2015.

\bibitem[Gy\"{o}ngy \& Mart\'{\i}nez(2001)Gy\"{o}ngy and Mart\'{\i}nez]{GM}
Gy\"{o}ngy, I. and Mart\'{\i}nez, T.
\newblock On stochastic differential equations with locally unbounded drift.
\newblock \emph{Czechoslovak Mathematical Journal, No. 4, p. 763--783, Vol.
  51}, 2001.

\bibitem[He et~al.(2018)He, Meng, Chen, Ma, and Liu]{ijcai2018p307}
He, L., Meng, Q., Chen, W., Ma, Z.-M., and Liu, T.-Y.
\newblock Differential equations for modeling asynchronous algorithms.
\newblock In \emph{Proceedings of the 27th International Joint Conference on
  Artificial Intelligence}, IJCAI'18, pp.\  2220–2226. AAAI Press, 2018.

\bibitem[Helmke \& Moore(1994)Helmke and Moore]{helmke1994optimization}
Helmke, U. and Moore, J.~B.
\newblock \emph{Optimization and Dynamical Systems}.
\newblock Springer London, 1st edition, 1994.

\bibitem[Hochreiter \& Schmidhuber(1997)Hochreiter and
  Schmidhuber]{hochreiter1997flat}
Hochreiter, S. and Schmidhuber, J.
\newblock Flat minima.
\newblock \emph{Neural computation}, 9\penalty0 (1):\penalty0 1--42, 1997.

\bibitem[Jastrzebski et~al.(2018)Jastrzebski, Kenton, Arpit, Ballas, Fischer,
  Bengio, and Storkey]{jastrzkebski2017three}
Jastrzebski, S., Kenton, Z., Arpit, D., Ballas, N., Fischer, A., Bengio, Y.,
  and Storkey, A.
\newblock Three factors influencing minima in sgd.
\newblock \emph{ICANN 2018}, 2018.

\bibitem[Jiang et~al.(2019)Jiang, Neyshabur, Mobahi, Krishnan, and
  Bengio]{jiang2019fantastic}
Jiang, Y., Neyshabur, B., Mobahi, H., Krishnan, D., and Bengio, S.
\newblock Fantastic generalization measures and where to find them.
\newblock In \emph{International Conference on Learning Representations}, 2019.

\bibitem[Jin et~al.(2017)Jin, Ge, Netrapalli, Kakade, and
  Jordan]{jin2017escape}
Jin, C., Ge, R., Netrapalli, P., Kakade, S.~M., and Jordan, M.~I.
\newblock How to escape saddle points efficiently.
\newblock In \emph{International Conference on Machine Learning}, pp.\
  1724--1732. PMLR, 2017.

\bibitem[Jin et~al.(2021)Jin, Netrapalli, Ge, Kakade, and Jordan]{Jin2021pgd}
Jin, C., Netrapalli, P., Ge, R., Kakade, S.~M., and Jordan, M.~I.
\newblock On nonconvex optimization for machine learning: Gradients,
  stochasticity, and saddle points.
\newblock \emph{J. ACM}, 68\penalty0 (2), 2021.

\bibitem[Kaddour et~al.(2022)Kaddour, Liu, Silva, and Kusner]{kaddour2022flat}
Kaddour, J., Liu, L., Silva, R., and Kusner, M.~J.
\newblock When do flat minima optimizers work?
\newblock \emph{Advances in Neural Information Processing Systems},
  35:\penalty0 16577--16595, 2022.

\bibitem[Keskar et~al.(2017)Keskar, Mudigere, Nocedal, Smelyanskiy, and
  Tang]{keskar2016large}
Keskar, N.~S., Mudigere, D., Nocedal, J., Smelyanskiy, M., and Tang, P. T.~P.
\newblock On large-batch training for deep learning: Generalization gap and
  sharp minima.
\newblock \emph{ICLR 2017}, 2017.

\bibitem[Kim et~al.(2023)Kim, Park, Choi, and Lee]{kim2023stability}
Kim, H., Park, J., Choi, Y., and Lee, J.
\newblock Stability analysis of sharpness-aware minimization.
\newblock \emph{arXiv preprint arXiv:2301.06308}, 2023.

\bibitem[Kunin et~al.(2019)Kunin, Bloom, Goeva, and Seed]{kunin2019loss}
Kunin, D., Bloom, J., Goeva, A., and Seed, C.
\newblock Loss landscapes of regularized linear autoencoders.
\newblock In \emph{International Conference on Machine Learning}, pp.\
  3560--3569. PMLR, 2019.

\bibitem[Kunin et~al.(2021)Kunin, Sagastuy-Brena, Ganguli, Yamins, and
  Tanaka]{kunin2021neural}
Kunin, D., Sagastuy-Brena, J., Ganguli, S., Yamins, D.~L., and Tanaka, H.
\newblock Neural mechanics: Symmetry and broken conservation laws in deep
  learning dynamics.
\newblock In \emph{International Conference on Learning Representations}, 2021.

\bibitem[Kushner \& Yin(2003)Kushner and Yin]{kushner2003stochastic}
Kushner, H. and Yin, G.~G.
\newblock \emph{Stochastic approximation and recursive algorithms and
  applications}, volume~35.
\newblock Springer Science \& Business Media, 2003.

\bibitem[Kwon et~al.(2021)Kwon, Kim, Park, and Choi]{kwon2021asam}
Kwon, J., Kim, J., Park, H., and Choi, I.~K.
\newblock Asam: Adaptive sharpness-aware minimization for scale-invariant
  learning of deep neural networks.
\newblock In \emph{International Conference on Machine Learning}, pp.\
  5905--5914. PMLR, 2021.

\bibitem[Levy(2016)]{levy2016power}
Levy, K.~Y.
\newblock The power of normalization: Faster evasion of saddle points.
\newblock \emph{arXiv preprint arXiv:1611.04831}, 2016.

\bibitem[Li et~al.(2017)Li, Tai, and Weinan]{li2017stochastic}
Li, Q., Tai, C., and Weinan, E.
\newblock Stochastic modified equations and adaptive stochastic gradient
  algorithms.
\newblock In \emph{International Conference on Machine Learning}, pp.\
  2101--2110. PMLR, 2017.

\bibitem[Li et~al.(2019)Li, Tai, and Weinan]{li2019stochastic}
Li, Q., Tai, C., and Weinan, E.
\newblock Stochastic modified equations and dynamics of stochastic gradient
  algorithms i: Mathematical foundations.
\newblock \emph{The Journal of Machine Learning Research}, 20\penalty0
  (1):\penalty0 1474--1520, 2019.

\bibitem[Li et~al.(2020)Li, Lyu, and Arora]{Li2020reconciling}
Li, Z., Lyu, K., and Arora, S.
\newblock Reconciling modern deep learning with traditional optimization
  analyses: the intrinsic learning rate.
\newblock In \emph{Advances in Neural Information Processing Systems}, 2020.

\bibitem[Li et~al.(2021)Li, Malladi, and Arora]{Li2021validity}
Li, Z., Malladi, S., and Arora, S.
\newblock On the validity of modeling {SGD} with stochastic differential
  equations ({SDE}s).
\newblock In Beygelzimer, A., Dauphin, Y., Liang, P., and Vaughan, J.~W.
  (eds.), \emph{Advances in Neural Information Processing Systems}, 2021.

\bibitem[Ljung et~al.(2012)Ljung, Pflug, and Walk]{ljung2012stochastic}
Ljung, L., Pflug, G., and Walk, H.
\newblock \emph{Stochastic approximation and optimization of random systems},
  volume~17.
\newblock Birkh{\"a}user, 2012.

\bibitem[Lucchi et~al.(2022)Lucchi, Proske, Orvieto, Bach, and
  Kersting]{lucchi2022fractional}
Lucchi, A., Proske, F., Orvieto, A., Bach, F., and Kersting, H.
\newblock On the theoretical properties of noise correlation in stochastic
  optimization.
\newblock In \emph{Advances in Neural Information Processing Systems}, 2022.

\bibitem[Malladi et~al.(2022)Malladi, Lyu, Panigrahi, and
  Arora]{Malladi2022AdamSDE}
Malladi, S., Lyu, K., Panigrahi, A., and Arora, S.
\newblock On the {SDEs} and scaling rules for adaptive gradient algorithms.
\newblock In \emph{Advances in Neural Information Processing Systems}, 2022.

\bibitem[Mandt et~al.(2015)Mandt, Hoffman, Blei, et~al.]{mandt2015continuous}
Mandt, S., Hoffman, M.~D., Blei, D.~M., et~al.
\newblock Continuous-time limit of stochastic gradient descent revisited.
\newblock \emph{NIPS-2015}, 2015.

\bibitem[Mandt et~al.(2017{\natexlab{a}})Mandt, Hoffman, and
  Blei]{mandt2017SGDasABI}
Mandt, S., Hoffman, M.~D., and Blei, D.~M.
\newblock Stochastic gradient descent as approximate {Bayesian} inference.
\newblock \emph{J. Mach. Learn. Res.}, 18\penalty0 (1):\penalty0 4873–4907,
  2017{\natexlab{a}}.
\newblock ISSN 1532-4435.

\bibitem[Mandt et~al.(2017{\natexlab{b}})Mandt, Hoffman, and
  Blei]{mandt2017stochastic}
Mandt, S., Hoffman, M.~D., and Blei, D.~M.
\newblock Stochastic gradient descent as approximate bayesian inference.
\newblock \emph{JMLR 2017}, 2017{\natexlab{b}}.

\bibitem[Mao(2007)]{mao2007stochastic}
Mao, X.
\newblock \emph{Stochastic differential equations and applications}.
\newblock Elsevier, 2007.

\bibitem[Mil’shtein(1986)]{mil1986weak}
Mil’shtein, G.
\newblock Weak approximation of solutions of systems of stochastic differential
  equations.
\newblock \emph{Theory of Probability \& Its Applications}, 30\penalty0
  (4):\penalty0 750--766, 1986.

\bibitem[Orvieto \& Lucchi(2019)Orvieto and Lucchi]{orvieto2019continuous}
Orvieto, A. and Lucchi, A.
\newblock Continuous-time models for stochastic optimization algorithms.
\newblock \emph{Advances in Neural Information Processing Systems}, 32, 2019.

\bibitem[Panigrahi et~al.(2019)Panigrahi, Somani, Goyal, and
  Netrapalli]{panigrahi2019non}
Panigrahi, A., Somani, R., Goyal, N., and Netrapalli, P.
\newblock Non-{G}aussianity of stochastic gradient noise.
\newblock \emph{SEDL workshop at NeurIPS 2019}, 2019.

\bibitem[Poggio et~al.(2017)Poggio, Kawaguchi, Liao, Miranda, Rosasco, Boix,
  Hidary, and Mhaskar]{poggio2017theory}
Poggio, T., Kawaguchi, K., Liao, Q., Miranda, B., Rosasco, L., Boix, X.,
  Hidary, J., and Mhaskar, H.
\newblock Theory of deep learning iii: explaining the non-overfitting puzzle.
\newblock \emph{arXiv preprint arXiv:1801.00173}, 2017.

\bibitem[Rangwani et~al.(2022)Rangwani, Aithal, Mishra, and
  Babu]{rangwani2022escaping}
Rangwani, H., Aithal, S.~K., Mishra, M., and Babu, R.~V.
\newblock Escaping saddle points for effective generalization on
  class-imbalanced data.
\newblock \emph{NeurIPS 2022}, 2022.

\bibitem[Risken(1996)]{risken1996fokker}
Risken, H.
\newblock Fokker-planck equation.
\newblock In \emph{The Fokker-Planck Equation}, pp.\  63--95. Springer, 1996.

\bibitem[Safran \& Shamir(2018)Safran and Shamir]{safran2018spurious}
Safran, I. and Shamir, O.
\newblock Spurious local minima are common in two-layer relu neural networks.
\newblock In \emph{International Conference on Machine Learning}, pp.\
  4433--4441. PMLR, 2018.

\bibitem[Sagun et~al.(2018)Sagun, Evci, Guney, Dauphin, and
  Bottou]{sagun2018empirical}
Sagun, L., Evci, U., Guney, V.~U., Dauphin, Y., and Bottou, L.
\newblock Empirical analysis of the hessian of over-parametrized neural
  networks.
\newblock \emph{ICLR 2018 Workshop Track}, 2018.
\newblock URL \url{https://openreview.net/forum?id=rJrTwxbCb}.

\bibitem[Simsekli et~al.(2019)Simsekli, Sagun, and
  Gurbuzbalaban]{simsekli2019tailindex}
Simsekli, U., Sagun, L., and Gurbuzbalaban, M.
\newblock A tail-index analysis of stochastic gradient noise in deep neural
  networks.
\newblock In \emph{International Conference on Machine Learning}, 2019.

\bibitem[Smith et~al.(2020)Smith, Elsen, and De]{smith2020sde}
Smith, S., Elsen, E., and De, S.
\newblock On the generalization benefit of noise in stochastic gradient
  descent.
\newblock In \emph{International Conference on Machine Learning}, 2020.

\bibitem[Su et~al.(2014)Su, Boyd, and Candes]{Su2014nesterov}
Su, W., Boyd, S., and Candes, E.
\newblock A differential equation for modeling {Nesterov’s} accelerated
  gradient method: Theory and insights.
\newblock In \emph{Advances in Neural Information Processing Systems}, 2014.

\bibitem[Ujv{\'a}ry et~al.(2022)Ujv{\'a}ry, Telek, Kerekes, M{\'e}sz{\'a}ros,
  and Husz{\'a}r]{ujvary2022rethinking}
Ujv{\'a}ry, S., Telek, Z., Kerekes, A., M{\'e}sz{\'a}ros, A., and Husz{\'a}r,
  F.
\newblock Rethinking sharpness-aware minimization as variational inference.
\newblock In \emph{OPT 2022: Optimization for Machine Learning (NeurIPS 2022
  Workshop)}, 2022.

\bibitem[Wen et~al.(2023)Wen, Ma, and Li]{wen2022does}
Wen, K., Ma, T., and Li, Z.
\newblock How does sharpness-aware minimization minimize sharpness?
\newblock \emph{ICLR 2023}, 2023.

\bibitem[Xie et~al.(2021)Xie, Sato, and Sugiyama]{xie2020diffusion}
Xie, Z., Sato, I., and Sugiyama, M.
\newblock A diffusion theory for deep learning dynamics: Stochastic gradient
  descent exponentially favors flat minima.
\newblock In \emph{International Conference on Learning Representations}, 2021.

\bibitem[Zhao et~al.(2022)Zhao, Lucchi, Proske, Orvieto, and
  Kersting]{zhao2022batch}
Zhao, J., Lucchi, A., Proske, F.~N., Orvieto, A., and Kersting, H.
\newblock Batch size selection by stochastic optimal control.
\newblock In \emph{Has it Trained Yet? NeurIPS 2022 Workshop}, 2022.

\bibitem[Zhou et~al.(2020)Zhou, Feng, Ma, Xiong, Hoi, et~al.]{zhou2020towards}
Zhou, P., Feng, J., Ma, C., Xiong, C., Hoi, S. C.~H., et~al.
\newblock Towards theoretically understanding why sgd generalizes better than
  adam in deep learning.
\newblock \emph{Advances in Neural Information Processing Systems},
  33:\penalty0 21285--21296, 2020.

\bibitem[Zhu et~al.(2019)Zhu, Wu, Yu, Wu, and Ma]{zhu2018anisotropic}
Zhu, Z., Wu, J., Yu, B., Wu, L., and Ma, J.
\newblock The anisotropic noise in stochastic gradient descent: Its behavior of
  escaping from sharp minima and regularization effects.
\newblock \emph{ICML 2019}, 2019.

\bibitem[Zhuang et~al.(2022)Zhuang, Gong, Yuan, Cui, Adam, Dvornek, Tatikonda,
  Duncan, and Liu]{zhuang2022surrogate}
Zhuang, J., Gong, B., Yuan, L., Cui, Y., Adam, H., Dvornek, N., Tatikonda, S.,
  Duncan, J., and Liu, T.
\newblock Surrogate gap minimization improves sharpness-aware training.
\newblock \emph{ICML 2022}, 2022.

\end{thebibliography}
\bibliographystyle{icml2023}

%%%%%%%%%%%%%%%%%%%%%%%%%%%%%%%%%%%%%%%%%%%%%%%%%%%%%%%%%%%%%%%%%%%%%%%%%%%%%%%
%%%%%%%%%%%%%%%%%%%%%%%%%%%%%%%%%%%%%%%%%%%%%%%%%%%%%%%%%%%%%%%%%%%%%%%%%%%%%%%
% APPENDIX
%%%%%%%%%%%%%%%%%%%%%%%%%%%%%%%%%%%%%%%%%%%%%%%%%%%%%%%%%%%%%%%%%%%%%%%%%%%%%%%
%%%%%%%%%%%%%%%%%%%%%%%%%%%%%%%%%%%%%%%%%%%%%%%%%%%%%%%%%%%%%%%%%%%%%%%%%%%%%%%
\newpage
\appendix
\onecolumn

%%%%%%%%%%%%%%%%%%%%%%%%%%%%%%%%%%%%%%%%%%%%%%%%%%%%%%%%%%%%%%%%%%%%%%%%%%%%
\section{Theoretical Framework - SDEs}\label{sec:theor}
%%%%%%%%%%%%%%%%%%%%%%%%%%%%%%%%%%%%%%%%%%%%%%%%%%%%%%%%%%%%%%%%%%%%%%%%%%%%

In the subsequent proofs, we will make repeated use of Taylor expansions in powers of $ \eta $. To simplify the presentation, we introduce the shorthand that whenever we write $ \mathcal{O}\left(\eta^\alpha\right) $, we mean that there exists a function $ K(x) \in G $ such that the error terms are bounded by $ K(x) \eta^\alpha $. For example, we write
$$
b(x+\eta)=b_0(x)+\eta b_1(x)+\mathcal{O}\left(\eta^2\right)
$$
to mean: there exists $ K \in G $ such that
$$
\left|b(x+\eta)-b_0(x)-\eta b_1(x)\right| \leq K(x) \eta^2 .
$$
Additionally, let us introduce some notation:

\begin{itemize}
\item A multi-index is $\alpha=\left(\alpha_1, \alpha_2, \ldots, \alpha_n\right)$ such that $\alpha_j \in\{0,1,2, \ldots\}$
\item $|\alpha|:=\alpha_1+\alpha_2+\cdots+\alpha_n$
\item $\alpha !:=\alpha_{1} ! \alpha_{2} ! \cdots \alpha_{n} !$
\item For $x=\left(x_1, x_2, \ldots, x_n\right) \in \mathbb{R}^n$, we define $x^\alpha:=x_1^{\alpha_1} x_2^{\alpha_2} \cdots x_n^{\alpha_n}$
\item For a multi-index $\beta$, $\partial_{\beta}^{|\beta|}f(x) := \frac{\partial^{|\beta|}}{\partial^{\beta_1}_{x_1}\partial^{\beta_2}_{x_2} \cdots \partial^{\beta_n}_{x_n} }f(x)$
\item We also denote the partial derivative with respect to $ x_{i} $ by $ \partial_{e_i} $. \\
\end{itemize}

 \begin{mybox}{gray}
\begin{lemma}[Lemma 1 \cite{li2017stochastic}] \label{lemma:li1}
Let $ 0<\eta<1 $. Consider a stochastic process $ X_t, t \geq 0 $ satisfying the SDE

$$
d X_t=b\left(X_t\right)dt+\eta^{\frac{1}{2}} \sigma\left(X_t\right) d W_t
$$

with $ X_0=x \in \mathbb{R}^d $ and $ b, \sigma $ together with their derivatives belong to $ G $. Define the one-step difference $ \Delta=X_\eta-x $, then we have

\begin{enumerate}
\item $ \mathbb{E} \Delta_{i}=b_{i} \eta+\frac{1}{2}\left[\sum_{j=1}^d b_{j} \partial_{e_j} b_{i}\right] \eta^2+\mathcal{O}\left(\eta^3\right) \quad \forall i = 1, \ldots,d$;
\item $ \mathbb{E} \Delta_{i} \Delta_{j}=\left[b_{i} b_{j}+\sigma \sigma_{(i j)}^T\right] \eta^2+\mathcal{O}\left(\eta^3\right) \quad \forall i,j = 1, \ldots,d$;
\item $ \mathbb{E} \prod_{j=1}^s \Delta_{\left(i_j\right)}=\mathcal{O}\left(\eta^3\right) $ for all $ s \geq 3, i_j=1, \ldots, d $.
\end{enumerate}
All functions above are evaluated at $ x $.
\end{lemma}
\end{mybox}

\begin{mybox}{gray}
\begin{theorem}[Theorem 2 and Lemma 5, \cite{mil1986weak}]\label{thm:mils}
Let the assumptions in Theorem \ref{thm:USAM_SDE} hold and let us define $\bar{\Delta}=x_1-x$ to be the increment in the discrete-time algorithm. If in addition there exists $K_1, K_2, K_3, K_4 \in G$ so that

\begin{enumerate}
\item $\left|\mathbb{E} \Delta_{i}-\mathbb{E} \bar{\Delta}_{i}\right| \leq K_1(x) \eta^{2}, \quad \forall i = 1, \ldots,d$;

\item $\left|\mathbb{E} \Delta_{i} \Delta_{j} - \mathbb{E} \bar{\Delta}_{i} \bar{\Delta}_{j}\right| \leq K_2(x) \eta^{2}, \quad \forall i,j = 1, \ldots, d$;
\item $\left|\mathbb{E} \prod_{j=1}^s \Delta_{i_j}-\mathbb{E} \prod_{j=1}^s \bar{\Delta}_{i_j}\right| \leq K_3(x) \eta^{2}, \quad \forall s \geq 3, \quad \forall i_j \in \{1, \ldots, d \}$;
\item $\mathbb{E} \prod_{j=1}^{ 3}\left|\bar{\Delta}_{i_j}\right| \leq K_4(x) \eta^{2}, \quad \forall i_j \in \{1, \ldots, d \}$.
\end{enumerate}
Then, there exists a constant $C$ so that for all $k=0,1, \ldots, N$ we have
$$
\left|\mathbb{E} g\left(X_{k \eta}\right)-\mathbb{E} g\left(x_k\right)\right| \leq C \eta.
$$
\end{theorem}
\end{mybox}

Before starting, we remind the reader that the order 1 weak approximation SDE of SGD (see \citet{li2017stochastic,li2019stochastic}) is given by 
\begin{equation} \label{eq:SGD_Equation_Insights}
dX_t = -\nabla f(X_t) d t + \sqrt{\eta}\left( \Sigma^{\text{SGD}}(X_t)\right)^{\frac{1}{2}}dW_t
\end{equation}
where $\Sigma^{\text{SGD}}(x)$ is the SGD covariance matrix defined as
\begin{equation}\label{eq:SGD_Covariance_Insights}
    \mathbb{E}\left[\left(\nabla f \left(x\right)-\nabla f_{\gamma}\left(x\right)\right)\left(\nabla f \left(x\right)-\nabla f_{\gamma}\left(x \right)\right)^T\right].
\end{equation}

%%%%%%%%%%%%%%%%%%%%%%%%%%%%%%%%%%%%%%%%%%%%%%%%%%%%%%%%%%%%%%%%%%%%%%%%%%%%
\subsection{Formal Derivation - USAM} \label{sec:formal_USAM}
%%%%%%%%%%%%%%%%%%%%%%%%%%%%%%%%%%%%%%%%%%%%%%%%%%%%%%%%%%%%%%%%%%%%%%%%%%%%

The next result is inspired by Theorem 1 of \cite{li2017stochastic} and is derived under some regularity assumption on the function $f$.

\begin{mybox}{gray}
\begin{assumption}
Assume that the following conditions on $f, f_i$ and their gradients are satisfied:
\begin{itemize}
\item $\nabla f, \nabla f_i $ satisfy a Lipschitz condition: there exists $ L>0 $ such that
$$
|\nabla f(x)-\nabla f(y)|+\sum_{i=1}^n\left|\nabla f_i(x)-\nabla f_i(y)\right| \leq L|x-y|;
$$
\item $ f, f_i $ and its partial derivatives up to order 7 belong to $ G $;
\item $ \nabla f, \nabla f_i $ satisfy a growth condition: there exists $ M>0 $ such that
$$
|\nabla f(x)|+\sum_{i=1}^n\left|\nabla f_i(x)\right| \leq M(1+|x|);
$$
\end{itemize}
\label{ass:regularity_f}
\end{assumption}
\end{mybox}

We will consider the stochastic process $ X_t \in \mathbb{R}^d $ defined by
\begin{equation}\label{eq:USAM_SDE}
dX_t = -\nabla \Tilde{f}^{\text{USAM}}(X_t) d t + \sqrt{\eta}\left( \Sigma^{\text{SGD}}(X_t)+ \rho \left( \tilde{\Sigma}(X_t) + \tilde{\Sigma}(X_t) ^{\top} \right) \right)^{\frac{1}{2}}dW_t
\end{equation}
where $$\Sigma^{\text{SGD}}(x):=\mathbb{E}\left[\left(\nabla f \left(x\right)-\nabla f_{\gamma}\left(x\right)\right)\left(\nabla f \left(x\right)-\nabla f_{\gamma}\left(x \right)\right)^T\right].
$$
is the usual covariance of SGD, while

\begin{equation} \label{eq:USAM_sigma_star}
\tilde{\Sigma}(x) :=\mathbb{E} \left[ \left( \nabla f\left(x\right) - \nabla f_{\gamma}\left(x\right) \right)\left(\mathbb{E} \left[ \nabla^2 f_{\gamma}(x) \nabla f_{\gamma}(x)\right] - \nabla^2 f_{\gamma}(x) \nabla f_{\gamma}(x) \right)^{\top} \right]
\end{equation}
and $$\Tilde{f}^{\text{USAM}}(x):= f(x) + \frac{\rho}{2} \mathbb{E} \left[ \lVert \nabla f_{\gamma}(x) \rVert^2_2\right].$$

In the following, we will use the notation

\begin{equation}\label{eq:USAM_Covariance}
\Sigma^{\text{USAM}}(x):= \left( \Sigma^{\text{SGD}}(X_t)+ \rho \left( \tilde{\Sigma}(X_t) + \tilde{\Sigma}(X_t) ^{\top} \right) \right)
\end{equation}

\begin{mybox}{gray}
\begin{theorem}[Stochastic modified equations] \label{thm:USAM_SDE}
Let $0<\eta<1, T>0$ and set $N=\lfloor T / \eta\rfloor$. Let $ x_k \in \mathbb{R}^d, 0 \leq k \leq N$ denote a sequence of USAM iterations defined by Eq.~\eqref{eq:USAM_Discr_Update}. Additionally, let us take 
\begin{equation}\label{eq:USAM_rho_theta_half}
\rho = \mathcal{O}\left(\eta^{\frac{1}{2}}\right).
\end{equation}

Consider the stochastic process $X_t$ defined in Eq.~\eqref{eq:USAM_SDE} and fix some test function $g \in G$ and suppose that $g$ and its partial derivatives up to order 6 belong to $G$.

Then, under Assumption~\ref{ass:regularity_f}, there exists a constant $ C>0 $ independent of $ \eta $ such that for all $ k=0,1, \ldots, N $, we have

$$
\left|\mathbb{E} g\left(X_{k \eta}\right)-\mathbb{E} g\left(x_k\right)\right| \leq C \eta^1 .
$$

That is, the SDE \eqref{eq:USAM_SDE} is an order $ 1 $ weak approximation of the SAM iterations \eqref{eq:USAM_Discr_Update}.
\end{theorem}
\end{mybox}

\begin{mybox}{gray}
\begin{lemma} \label{lemma:USAM_SDE}
Under the assumptions of Theorem \ref{thm:USAM_SDE}, let $ 0<\eta<1 $ and consider $ x_k, k \geq 0 $ satisfying the USAM iterations \eqref{eq:USAM_Discr_Update}
$$
x_{k+1}=x_k-\eta \nabla f_{\gamma_k}\left(x_k + \rho \nabla f_{\gamma_k}(x_k) \right)
$$
with $ x_0=x \in \mathbb{R}^d $. Additionally, we define $\partial_{e_i} \Tilde{f}^{\text{USAM}}(x) := \partial_{e_i} f(x) + \rho \mathbb{E} \left[ \sum_j \partial_{e_i+e_j}^{2} f_{\gamma}(x) \partial_{e_j}f_{\gamma}(x)\right]$. From the definition the one-step difference $ \bar{\Delta}=x_1-x $, then we have

\begin{enumerate}
\item $ \mathbb{E} \bar{\Delta}_{i}=-\partial_{e_i} \Tilde{f}^{\text{USAM}}(x) \eta +\mathcal{O}(\eta \rho^2) \quad \forall i = 1, \ldots,d$.
\item $ \mathbb{E} \bar{\Delta}_{i} \bar{\Delta}_{j}=\partial_{e_i} \Tilde{f}^{\text{USAM}}(x) \partial_{e_j} \Tilde{f}^{\text{USAM}}(x) \eta^2 + \Sigma^{\text{USAM}}_{(i j)} \eta^2 + \mathcal{O}\left(\eta^3 \right)\quad \forall i,j = 1, \ldots,d$. 
\item $\mathbb{E} \prod_{j=1}^s \Bar{\Delta}_{i_j} =\mathcal{O}\left(\eta^3\right) \quad \forall s \geq 3, \quad i_j \in \{ 1, \ldots, d\}.$
\end{enumerate}
and all the functions above are evaluated at $ x $.
\end{lemma}
\end{mybox}

Together with many other proofs in this Section, the following one relies on a Taylor expansion. The truncated terms are multiplied by a mixture of terms in $\eta$ and $\rho$. Therefore, a careful balancing of the relative size of these two quantities is needed as reflected in Equation \eqref{eq:USAM_rho_theta_half}.

\begin{proof}[Proof of Lemma \ref{lemma:USAM_SDE}]
Since the first step is to evaluate $\mathbb{E} \Delta_{i} = \mathbb{E} \left[ - \partial_{e_i} f_{\gamma}(x+ \rho \nabla f_{\gamma}(x)) \eta \right]$, we start by analyzing $\partial_{e_i} f_{\gamma}(x+ \rho \nabla f_{\gamma}(x))$, that is the partial derivative in the direction $e_i:=(0, \cdots, 0, \underset{i-th}{1}, 0, \cdots,0)$. Then, we have that

\begin{equation}
\partial_{e_i} f_{\gamma}(x+ \rho \nabla f_{\gamma}(x)) = \partial_{e_i} f_{\gamma}(x) + \sum_{\vert \alpha \vert=1}\partial_{e_i+\alpha}^{2} f_{\gamma}(x) \rho \partial_{\alpha}f_{\gamma}(x) + \mathcal{R}^{\partial_{e_i} f_{\gamma}(x)}_{x,1}(\rho \nabla f_{\gamma}(x)),
\end{equation}
where the residual is defined in Eq.~(4) of \cite{folland2005higher}. Therefore, for some constant $c\in (0,1)$, it holds that

\begin{equation}
\mathcal{R}^{\partial_{e_i} f_{\gamma}(x)}_{x,1}(\rho \nabla f_{\gamma}(x)) = \sum_{\vert \alpha \vert=2} \frac{\partial_{e_i+\alpha}^{3} f_{\gamma}(x + c \rho \nabla f_{\gamma}(x)) \rho^{2} \left( \nabla f_{\gamma}(x) \right)^{\alpha}}{\alpha !}.
\end{equation}
Therefore, we can rewrite it as

\begin{align} \label{eq:USAM_Estim_Rewritten}
\partial_{e_i} f_{\gamma}(x+ \rho \nabla f_{\gamma}(x)) = \partial_{e_i} f_{\gamma}(x) + \rho\sum_j \partial_{e_i+e_j}^{2} f_{\gamma}(x) \partial_{e_j}f_{\gamma}(x)+ \rho^2 \left[ \sum_{\vert \alpha \vert=2} \frac{\partial_{e_i+\alpha}^{3} f_{\gamma}(x + c \rho \nabla f_{\gamma}(x))\left( \nabla f_{\gamma}(x) \right)^{\alpha}}{\alpha !} \right]
\end{align}
Now, we observe that
\begin{equation}
K_{i}(x) := \left[ \sum_{\vert \alpha \vert=2} \frac{\partial_{e_i+\alpha}^{3} f_{\gamma}(x + c \rho \nabla f_{\gamma}(x))\left( \nabla f_{\gamma}(x) \right)^{\alpha}}{\alpha !} \right]
\end{equation}
is a finite sum of products of functions that by assumption are in $G$. Therefore, $K_{i}(x) \in G$ and$\Bar{K}_{i}(x) = \mathbb{E} \left[ K_{i}(x) \right] \in G$. Based on these definitions, we rewrite Eq.~\eqref{eq:USAM_Estim_Rewritten} as
\begin{equation} \label{eq:USAM_Estim_Rewritten_1}
\partial_{e_i} f_{\gamma}(x+ \rho \nabla f_{\gamma}(x)) = \partial_{e_i} f_{\gamma}(x) + \rho\sum_j \partial_{e_i+e_j}^{2} f_{\gamma}(x) \partial_{e_j}f_{\gamma}(x) + \rho^2 K_{i}(x).
\end{equation}
which implies that

\begin{equation} \label{eq:USAM_Estim_Rewritten_2}
\mathbb{E} \left[ \partial_{e_i} f_{\gamma}(x+ \rho \nabla f_{\gamma}(x)) \right] = \partial_{e_i} f(x) + \rho \mathbb{E} \left[ \sum_j \partial_{e_i+e_j}^{2} f_{\gamma}(x) \partial_{e_j}f_{\gamma}(x)\right] + \rho^2 \Bar{K}_{i}(x).
\end{equation}
Let us now remember that

\begin{equation} \label{eq:USAM_Real_Rewritten}
\partial_{e_i} \Tilde{f}^{\text{USAM}}(x)=\partial_{e_i} \left( f(x) + \frac{\rho}{2} \mathbb{E} \left[ \lVert \nabla f_{\gamma}(x) \rVert_2^2 \right] \right) = \partial_{e_i} f(x) + \rho\mathbb{E} \left[ \sum_j \partial_{e_i+e_j}^{2} f_{\gamma}(x) \partial_{e_j}f_{\gamma}(x) \right] 
\end{equation}

Therefore, by using Eq. \eqref{eq:USAM_Estim_Rewritten_2}, Eq. \eqref{eq:USAM_Real_Rewritten}, and the assumption \eqref{eq:USAM_rho_theta_half} we have that $\forall i = 1, \ldots,d$

\begin{equation} \label{eq:USAM_Estim_Rewritten_3}
\mathbb{E} \bar{\Delta}_{i} = -\partial_{e_i} \Tilde{f}^{\text{USAM}}(x) \eta + \eta \rho^2 \Bar{K}_{i}(x) = -\partial_{e_i} \Tilde{f}^{\text{USAM}}(x) \eta + \mathcal{O}\left(\eta^2 \right).
\end{equation}
Additionally, we have that

\begin{align}
 \mathbb{E} \bar{\Delta}_{i} \bar{\Delta}_{j} = & \text{Cov}(\bar{\Delta}_{i}, \bar{\Delta}_{j}) + \mathbb{E} \bar{\Delta}_{i} \mathbb{E}\bar{\Delta}_{j}  \nonumber \\
 & \overset{\eqref{eq:USAM_Estim_Rewritten_3}}{=} \text{Cov}(\bar{\Delta}_{i}, \bar{\Delta}_{j}) + \partial_{e_i} \Tilde{f}^{\text{USAM}} \partial_{e_j} \Tilde{f}^{\text{USAM}} \eta^2+ \eta^2 \rho^2 ( \partial_{e_i} \Tilde{f}^{\text{USAM}} \Bar{K}_{j}(x) + \partial_{e_j} \Tilde{f}^{\text{USAM}} \Bar{K}_{i}(x)) + \eta^2 \rho^4 \Bar{K}_{i}(x) \Bar{K}_{j}(x) \nonumber \\
 & = \text{Cov}(\bar{\Delta}_{i}, \bar{\Delta}_{j}) + \partial_{e_i} \Tilde{f}^{\text{USAM}} \partial_{e_j} \Tilde{f}^{\text{USAM}} \eta^2 + \mathcal{O}\left(\eta^2 \rho^2 \right) + \mathcal{O}\left(\eta^2 \rho^4 \right) \nonumber \\
 & =\partial_{e_i} \Tilde{f}^{\text{USAM}} \partial_{e_j} \Tilde{f}^{\text{USAM}} \eta^2+\text{Cov}(\bar{\Delta}_{i}, \bar{\Delta}_{j}) + \mathcal{O}\left(\eta^2 \rho^2 \right) + \mathcal{O}\left(\eta^2 \rho^4 \right)\quad \forall i,j = 1, \ldots,d \label{eq:USAM_Cov_final_step}
\end{align}
Let us now recall the expression \eqref{eq:USAM_sigma_star} of $\tilde{\Sigma}$ and the expression \eqref{eq:USAM_Covariance} of $\Sigma^{\text{USAM}}$. Then, we automatically have that

\begin{equation}
 \text{Cov}(\bar{\Delta}_{i}, \bar{\Delta}_{j}) =\eta^2 \left( \Sigma^{\text{SGD}}_{i,j}(x) + \rho \left[ \tilde{\Sigma}_{i,j}(x)+ \tilde{\Sigma}_{i,j}(x)^{\top} \right] + \mathcal{O}(\rho^2) \right) = \eta^2 \Sigma^{\text{USAM}}_{i,j}(x) +\mathcal{O}(\eta^2 \rho^2)
\end{equation}
Therefore, remembering Eq. \eqref{eq:USAM_Cov_final_step} and Eq. \eqref{eq:USAM_rho_theta_half} we have 

\begin{equation}
\mathbb{E} \bar{\Delta}_{i} \bar{\Delta}_{j} =\partial_{e_i} \Tilde{f}^{\text{USAM}} \partial_{e_j} \Tilde{f}^{\text{USAM}} \eta^2+ \Sigma^{\text{USAM}}_{i,j} \eta^2 + \mathcal{O}\left(\eta^3 \right), \quad \forall i,j = 1, \ldots,d
\end{equation}
Finally, with analogous considerations, it is obvious that under our assumptions

$$
\mathbb{E} \prod_{j=1}^s \bar{\Delta}_{i_j}= \mathcal{O}\left(\eta^s\right) \quad \forall s \geq 3, \quad i_j \in \{1, \ldots, d\}
$$
which in particular implies that
$$
\mathbb{E} \prod_{j=1}^3 \bar{\Delta}_{i_j}= \mathcal{O}\left(\eta^3\right), \quad i_j \in \{1, \ldots, d\}.
$$

\end{proof}

\paragraph{Additional Insights from Lemma \ref{lemma:USAM_SDE}.}
Let us notice that $\nabla f_{\gamma}(x)$ is dominated by a factor $M(1 + \vert x \vert)$, if all $\partial_{e_i+\alpha}^{3} f_{\gamma}(x)$ are limited by a common constant $L$, for some positive constant $C$ we have that

\begin{align}
\left| K_{i}(x) \right| = & \rho^2\left| \sum_{\vert \alpha \vert=2} \frac{\partial_{e_i+\alpha}^{3} f_{\gamma}(x + c \rho \nabla f_{\gamma}(x))\left( \nabla f_{\gamma}(x) \right)^{\alpha}}{\alpha !} \right| \\
 & \leq \rho^2 C L \lVert \nabla f_{\gamma}(x) \nabla f_{\gamma}(x)^{\top} \rVert_{F}^2 \leq \rho^2 C L d^2M^2(1 + \lvert x\rvert)^2
\end{align}
Therefore, $K_{i}(x)$ does not only lay in $G$, but has at most quadratic growth. \\

\begin{proof}[Proof of Theorem \ref{thm:USAM_SDE}] 
\label{proof:USAM_SDE}
To prove this result, all we need to do is check the conditions in Theorem \ref{thm:mils}. As we apply Lemma \ref{lemma:li1}, we make the following choices:

\begin{itemize}
\item $b(x)=-\nabla \Tilde{f}^{\text{USAM}}\left(x\right)$;
\item $\sigma(x) =\Sigma^{\text{USAM}}(x)^{\frac{1}{2}}$.
\end{itemize}
First of all, we notice that $\forall i = 1, \ldots, d$, it holds that

\begin{itemize}
\item $\mathbb{E} \bar{\Delta}_{i} \overset{\text{1. Lemma \ref{lemma:USAM_SDE}}}{=}-\partial_{e_i} \Tilde{f}^{\text{USAM}}(x)\eta+ \mathcal{O}(\eta^2)$;
\item $ \mathbb{E} \Delta_{i} \overset{\text{1. Lemma \ref{lemma:li1}}}{=} -\partial_{e_i} \Tilde{f}^{\text{USAM}}(x)\eta +\mathcal{O}\left(\eta^2\right)$.
\end{itemize}
Therefore, we have that for some $K_1(x) \in G$

\begin{equation}\label{eq:USAM_cond1}
\left|\mathbb{E} \Delta_{i}-\mathbb{E} \bar{\Delta}_{i}\right| \leq K_1(x) \eta^{2}, \quad \forall i = 1, \ldots,d.
\end{equation}
Additionally,we notice that $\forall i,j = 1, \ldots, d$, it holds that

\begin{itemize}
\item $ \mathbb{E} \bar{\Delta}_{i} \bar{\Delta}_{j} \overset{\text{2. Lemma \ref{lemma:USAM_SDE}}}{=}\partial_{e_i} \Tilde{f}^{\text{USAM}} \partial_{e_j} \Tilde{f}^{\text{USAM}} \eta^2+ \Sigma^{\text{USAM}}_{i,j} \eta^2 + \mathcal{O}\left(\eta^3 \right)$;
\item $ \mathbb{E} \Delta_{i} \Delta_{j} \overset{\text{2. Lemma \ref{lemma:li1}}}{=}\partial_{e_i} \Tilde{f}^{\text{USAM}} \partial_{e_j} \Tilde{f}^{\text{USAM}} \eta^2+ \Sigma^{\text{USAM}}_{i,j} \eta^2 + \mathcal{O}\left(\eta^3 \right)$.
\end{itemize}
Therefore, we have that for some $K_2(x) \in G$

\begin{equation}\label{eq:USAM_cond2}
\left|\mathbb{E} \Delta_{i} \Delta_{j} - \mathbb{E} \bar{\Delta}_{i} \bar{\Delta}_{j}\right| \leq K_2(x) \eta^{2}, \quad \forall i,j = 1, \ldots, d
\end{equation}
Additionally, we notice that $\forall s \geq 3, \forall i_j \in \{1, \ldots, d \}$, it holds that

\begin{itemize}
\item $ \mathbb{E} \prod_{j=1}^s \bar{\Delta}_{i_j}\overset{\text{3. Lemma \ref{lemma:USAM_SDE}}}{=}\mathcal{O}\left(\eta^3\right)$;
\item $ \mathbb{E} \prod_{j=1}^s \Delta_{i_j}\overset{\text{3. Lemma \ref{lemma:li1}}}{=}\mathcal{O}\left(\eta^3\right)$.
\end{itemize}
Therefore, we have that for some $K_3(x) \in G$

\begin{equation}\label{eq:USAM_cond3}
\left|\mathbb{E} \prod_{j=1}^s \Delta_{i_j}-\mathbb{E} \prod_{j=1}^s \bar{\Delta}_{i_j}\right| \leq K_3(x) \eta^{2}.
\end{equation}
Additionally, for some $K_4(x) \in G$, $\forall i_j \in \{1, \ldots, d \}$

\begin{equation} \label{eq:USAM_cond4}
\mathbb{E} \prod_{j=1}^{ 3}\left|\bar{\Delta}_{\left(i_j\right)}\right| \overset{\text{3. Lemma \ref{lemma:USAM_SDE}}}{\leq}K_4(x) \eta^{2}.
\end{equation}
Finally, Eq.~\eqref{eq:USAM_cond1}, Eq.~\eqref{eq:USAM_cond2}, Eq.~\eqref{eq:USAM_cond3}, and Eq.~\eqref{eq:USAM_cond4} allow us to conclude the proof.

\end{proof}

\begin{mybox}{gray}
\begin{corollary}\label{thm:USAM_SDE_Simplified}
Let us take the same assumptions of Theorem~\ref{thm:USAM_SDE}. Additionally, let us assume that the dynamics is near the minimizer. In this case, the noise structure is such that the stochastic gradient can be written as $\nabla f_{\gamma}(x) = \nabla f(x) + Z$ such that $Z$ is the noise that does not depend on $x$. Therefore, the SDE \eqref{eq:USAM_SDE} becomes

\begin{equation}\label{eq:USAM_SDE_Simplified}
dX_t = -\nabla \Tilde{f}^{\text{USAM}}(X_t) d t + \left( I_d + \rho \nabla^2 f(X_t) \right) \left( \eta \Sigma^{\text{SGD}}(X_t)\right)^{\frac{1}{2}} dW_t
\end{equation}
where $$\Sigma^{\text{SGD}}(x):=\mathbb{E}\left[\left(\nabla f \left(x\right)-\nabla f_{\gamma}\left(x\right)\right)\left(\nabla f \left(x\right)-\nabla f_{\gamma}\left(x \right)\right)^T\right]
$$
is the usual covariance of SGD, and $$\Tilde{f}^{\text{USAM}}(x) = f(x) + \frac{\rho}{2} \lVert \nabla f(x) \rVert^2_2.$$
\end{corollary}
\end{mybox}

\begin{proof}[Proof of Corollary \ref{thm:USAM_SDE_Simplified}]
Based on our assumption on the noise structure, we can rewrite Eq. \eqref{eq:USAM_sigma_star} of the matrix $\tilde{\Sigma}$ as

\begin{align}
\tilde{\Sigma}(x) & =\mathbb{E} \left[ \left( \nabla f\left(x\right) - \nabla f_{\gamma}\left(x\right) \right)\left(\mathbb{E} \left[ \nabla^2 f_{\gamma}(x) \nabla f_{\gamma}(x)\right] - \nabla^2 f_{\gamma}(x) \nabla f_{\gamma}(x) \right)^{\top} \right] \\
& = \nabla^2 f(x)\mathbb{E} \left[ \left( \nabla f\left(x\right) - \nabla f_{\gamma}\left(x\right) \right)\left(\nabla f(x) - \nabla f_{\gamma}(x) \right)^{\top} \right] 
\end{align}
Therefore, the Eq. \eqref{eq:USAM_Covariance} of the covariance $\Sigma^{\text{USAM}}$ becomes

\begin{equation}
 \Sigma^{\text{USAM}}(x)= \left( I_d+ 2 \rho \nabla^2 f(x) \right)\Sigma^{\text{SGD}}(X_t)
\end{equation}
which implies that

\begin{equation}
 (\Sigma^{\text{USAM}}(x))^{\frac{1}{2}} \approx \left( I_d+ \rho \nabla^2 f(x) \right)(\Sigma^{\text{SGD}}(X_t))^{\frac{1}{2}}.
\end{equation}
Finally, we have that

\begin{align}
\Tilde{f}^{\text{USAM}}(x) & := f(x) + \frac{\rho}{2} \mathbb{E} \left[ \lVert \nabla f_{\gamma}(x) \rVert^2_2\right] =f(x) + \frac{\rho}{2} \mathbb{E} \left[ \lVert \nabla f(x) \rVert^2_2 + Z^2 + 2 Z \nabla f_{\gamma}(x)\right] \\
& = f(x) + \frac{\rho}{2} \lVert \nabla f(x) \rVert^2_2 + \frac{\rho}{2} \mathbb{E} \left[ Z^2 \right] 
\end{align}
Since the component $\mathbb{E} \left[ Z^2 \right] $ is independent on $x$, we ignore it and conclude that 

$$
\Tilde{f}^{\text{USAM}}(x) = f(x) + \frac{\rho}{2} \lVert \nabla f(x) \rVert^2_2.
$$

\end{proof}

\subsubsection{USAM is SGD if \texorpdfstring{$\rho = \mathcal{O}(\eta)$}{Lg} }
%%%%%%%%%%%%%%%%%%%%%%%%%%%%%%%%%%%%%%%%%%%%%%%%%%%%%%%%%%%%%%%%%%%%%%%%%%%%

The following result is inspired by Theorem 1 of \cite{li2017stochastic}.

\begin{mybox}{gray}
\begin{theorem}[Stochastic modified equations] \label{thm:USAM_SGD}
Let $0<\eta<1, T>0$ and set $N=\lfloor T / \eta\rfloor$. Let $ x_k \in \mathbb{R}^d, 0 \leq k \leq N$ denote a sequence of USAM iterations defined by Eq.~\eqref{eq:USAM_Discr_Update}. Additionally, let us take 

\begin{equation}\label{eq:USAM_SGD_rho}
\rho = \mathcal{O}\left(\eta^{1}\right).
\end{equation}
Define $ X_t \in \mathbb{R}^d $ as the stochastic process satisfying the SDE

\begin{equation}\label{eq:USAM_SGD_SDE}
d X_t=-\nabla f\left(X_t\right) d t+\left(\eta \Sigma^{\text{SGD}}\left(X_t\right)\right)^{1 / 2} d W_t
\end{equation}

Such that $X_0=x_0$ and $$\Sigma^{\text{SGD}}(x):=\mathbb{E}\left[\left(\nabla f \left(x\right)-\nabla f_{\gamma}\left(x\right)\right)\left(\nabla f \left(x\right)-\nabla f_{\gamma}\left(x \right)\right)^T\right]
$$

Fix some test function $g \in G$ and suppose that $g$ and its partial derivatives up to order 6 belong to $G$.

Then, under Assumption~\ref{ass:regularity_f}, there exists a constant $ C>0 $ independent of $ \eta $ such that for all $ k=0,1, \ldots, N $, we have

$$
\left|\mathbb{E} g\left(X_{k \eta}\right)-\mathbb{E} g\left(x_k\right)\right| \leq C \eta^1 .
$$

That is, the SDE \eqref{eq:USAM_SGD_SDE} is an order $ 1 $ weak approximation of the USAM iterations \eqref{eq:USAM_Discr_Update}.
\end{theorem}
\end{mybox}

\begin{mybox}{gray}
\begin{lemma}\label{lemma:USAM_SGD}
Under the assumptions of Theorem \ref{thm:USAM_SGD}, let $ 0<\eta<1 $. Consider $ x_k, k \geq 0 $ satisfying the USAM iterations
$$
x_{k+1}=x_k-\eta \nabla f_{\gamma_k}\left(x_k + \rho \nabla f_{\gamma_k}(x_k) \right)
$$
with $ x_0=x \in \mathbb{R}^d $. From the definition the one-step difference $ \bar{\Delta}=x_1-x $, then we have

\begin{enumerate}
\item $ \mathbb{E} \bar{\Delta}_{i}=-\partial_{e_i}f(x)\eta+ \mathcal{O}(\eta^2) \quad \forall i = 1, \ldots,d$.
\item $ \mathbb{E} \bar{\Delta}_{i} \bar{\Delta}_{j}=\partial_{e_i} f \partial_{e_j} f \eta^2+\Sigma^{\text{SGD}}_{(i j)} \eta^2 + \mathcal{O}\left(\eta^3 \right)\quad \forall i,j = 1, \ldots,d$. 
\item $\mathbb{E} \prod_{j=1}^s \bar{\Delta}_{i_j} =\mathcal{O}\left(\eta^3\right) \quad \forall s \geq 3, \quad i_j \in \{ 1, \ldots, d\}.$
\end{enumerate}
All functions above are evaluated at $ x $.
\end{lemma}
\end{mybox}

\begin{proof}[Proof of Lemma \ref{lemma:USAM_SGD}]
First of all, we write that

\begin{equation}\label{eq:USAM_SGD_Estim}
 \partial_{e_i} f_{\gamma}(x+ \rho \nabla f_{\gamma}(x))= \partial_{e_i} f_{\gamma}(x) + \mathcal{R}^{\partial_{e_i} f_{\gamma}(x)}_{x,0}(\rho \nabla f_{\gamma}(x)),
\end{equation}
where the residual is defined in Eq.~(4) of \cite{folland2005higher}. Therefore, for some constant $c\in (0,1)$, it holds that

\begin{equation}\label{eq:USAM_SGD_Taylor}
\mathcal{R}^{\partial_{e_i} f_{\gamma}(x)}_{x,0}(\rho \nabla f_{\gamma}(x)) = \sum_{\vert \alpha \vert=1} \frac{\partial_{e_i+\alpha}^{2} f_{\gamma}(x + c \rho \nabla f_{\gamma}(x)) \rho^{1} \left( \nabla f_{\gamma}(x) \right)^{\alpha}}{\alpha !}.
\end{equation}
Let us now observe that $\mathcal{R}^{\partial_{e_i} f_{\gamma}(x)}_{x,0}(\rho \nabla f_{\gamma}(x))$ is a finite sum of products of functions in $G$ and that, therefore, it lies in $G$. Additionally, given its expression Eq.~\eqref{eq:USAM_SGD_Taylor}, we can factor out a common $\rho$ and have that $K(x) = \rho K_1(x)$ for some function $K_1(x) \in G$. Therefore, we rewrite Eq.~\eqref{eq:USAM_SGD_Estim} as

\begin{equation}\label{eq:USAM_SGD_Estim_Rewritten}
\partial_{e_i} f_{\gamma}(x+ \rho \nabla f_{\gamma}(x)) = \partial_{e_i} f_{\gamma}(x) + \rho K_1(x).
\end{equation}
First of all, we notice that if we define $ \Bar{K}_1(x) = \mathbb{E} \left[ K_{1}(x) \right]$, also $\Bar{K}_1(x) \in G$. Therefore, it holds that
\begin{equation}\label{eq:USAM_SGD_Estim_Rewritten_2}
\mathbb{E} \left[ \partial_{e_i} f_{\gamma}(x+ \rho \nabla f_{\gamma}(x)) \right] \overset{\eqref{eq:USAM_SGD_Estim_Rewritten}}{=} \partial_{e_i} f(x) + \rho \Bar{K}_1(x)
\end{equation}
Therefore, using assumption \eqref{eq:USAM_SGD_rho}, $\forall i = 1, \ldots,d$, we have that

\begin{equation}
\mathbb{E} \bar{\Delta}_{i} = -\partial_{e_i}f(x)\eta+ \eta \rho \Bar{K}_i(x)= -\partial_{e_i} f(x)\eta + \mathcal{O}\left(\eta^2\right)
\end{equation}
Additionally, by keeping in mind the definition of the covariance matrix $\Sigma$, We immediately have 

\begin{align}
 \mathbb{E} \bar{\Delta}_{i} \bar{\Delta}_{j} \overset{\eqref{eq:USAM_SGD_Estim_Rewritten}}{=} & Cov(\bar{\Delta}_{i}, \bar{\Delta}_{j}) + \mathbb{E} \bar{\Delta}_{i} \mathbb{E}\bar{\Delta}_{j} \nonumber \\
 & = \Sigma^{\text{SGD}}_{(i j)} \eta^2 + \partial_{e_i} f \partial_{e_j} f \eta^2+ \eta^2 \rho ( \partial_{e_i} f \Bar{K}_j(x) + \partial_{e_j} f \Bar{K}_i(x)) + \eta^2 \rho^2 \Bar{K}_i(x) \Bar{K}_j(x) \nonumber \\
 & = \Sigma^{\text{SGD}}_{(i j)} \eta^2 + \partial_{e_i} f \partial_{e_j} f \eta^2+ \mathcal{O}\left(\eta^2 \rho\right) + \mathcal{O}\left(\eta^2 \rho^2 \right) \nonumber \\
 & = \partial_{e_i} f \partial_{e_j} f \eta^2+\Sigma^{\text{SGD}}_{(i j)} \eta^2 + \mathcal{O}\left(\eta^3\right)\quad \forall i,j = 1, \ldots,d
\end{align}
Finally, with analogous considerations, it is obvious that under our assumptions

$$
\mathbb{E} \prod_{j=1}^s \bar{\Delta}_{i_j}= \mathcal{O}\left(\eta^3\right)\quad \forall s \geq 3, \quad i_j \in \{1, \ldots, d\}.
$$
\end{proof}

\begin{proof}[Proof of Theorem \ref{thm:USAM_SGD}] \label{proof:USAM_SGD}
To prove this result, all we need to do is check the conditions in Theorem \ref{thm:mils}. As we apply Lemma \ref{lemma:li1}, we make the following choices:

\begin{itemize}
\item $b(x)=-\nabla f\left(x\right)$,
\item $\sigma(x) =\Sigma^{\text{SGD}}(X_t)^{\frac{1}{2}}$;
\end{itemize}
First of all, we notice that $\forall i = 1, \ldots, d$, it holds that

\begin{itemize}
\item $\mathbb{E} \bar{\Delta}_{i} \overset{\text{1. Lemma \ref{lemma:USAM_SGD}}}{=}-\partial_{e_i} f (x )\eta+\mathcal{O}\left(\eta^2\right)$;
\item $ \mathbb{E} \Delta_{i} \overset{\text{1. Lemma \ref{lemma:li1}}}{=}-\partial_{e_i} f (x )\eta+\mathcal{O}\left(\eta^2\right)$.
\end{itemize}
Therefore, we have that for some $K_1(x) \in G$

\begin{equation}\label{eq:USAM_SGD_cond1}
\left|\mathbb{E} \Delta_{i}-\mathbb{E} \bar{\Delta}_{i}\right| \leq K_1(x) \eta^{2}, \quad \forall i = 1, \ldots, d.
\end{equation}
Additionally,we notice that $\forall i,j = 1, \ldots, d$, it holds that

\begin{itemize}
\item $ \mathbb{E} \bar{\Delta}_{i} \bar{\Delta}_{j} \overset{\text{2. Lemma \ref{lemma:USAM_SGD}}}{=}\partial_{e_i} f \partial_{e_j} f \eta^2+\Sigma^{\text{SGD}}_{(i j)} \eta^2 + \mathcal{O}\left(\eta^3\right)$;
\item $ \mathbb{E} \Delta_{i} \Delta_{j} \overset{\text{2. Lemma \ref{lemma:li1}}}{=}\partial_{e_i} f \partial_{e_j} f \eta^2+\Sigma^{\text{SGD}}_{(i j)} \eta^2 + \mathcal{O}\left(\eta^3\right)$.
\end{itemize}
Therefore, we have that for some $K_2(x) \in G$

\begin{equation}\label{eq:USAM_SGD_cond2}
\left|\mathbb{E} \Delta_{i} \Delta_{j} - \mathbb{E} \bar{\Delta}_{i} \bar{\Delta}_{j}\right| \leq K_2(x) \eta^{2}, \quad \forall i,j = 1, \ldots, d
\end{equation}
Additionally, we notice that $\forall s \geq 3, \forall i_j \in \{1, \ldots, d \}$, it holds that

\begin{itemize}
\item $ \mathbb{E} \prod_{j=1}^s \bar{\Delta}_{i_j}\overset{\text{3. Lemma \ref{lemma:USAM_SGD}}}{=}\mathcal{O}\left(\eta^3\right)$;
\item $ \mathbb{E} \prod_{j=1}^s \Delta_{i_j}\overset{\text{3. Lemma \ref{lemma:li1}}}{=}\mathcal{O}\left(\eta^3\right)$.
\end{itemize}
Therefore, we have that for some $K_3(x) \in G$

\begin{equation}\label{eq:USAM_SGD_cond3}
\left|\mathbb{E} \prod_{j=1}^s \Delta_{i_j}-\mathbb{E} \prod_{j=1}^s \bar{\Delta}_{i_j}\right| \leq K_3(x) \eta^{2}.
\end{equation}
Additionally, for some $K_4(x) \in G$, $\forall i_j \in \{1, \ldots, d \}$

\begin{equation} \label{eq:USAM_SGD_cond4}
\mathbb{E} \prod_{j=1}^{ 3}\left|\bar{\Delta}_{\left(i_j\right)}\right| \overset{\text{3. Lemma \ref{lemma:USAM_SGD}}}{\leq}K_4(x) \eta^{2}.
\end{equation}
Finally, Eq.~\eqref{eq:USAM_SGD_cond1}, Eq.~\eqref{eq:USAM_SGD_cond2}, Eq.~\eqref{eq:USAM_SGD_cond3}, and Eq.~\eqref{eq:USAM_SGD_cond4} allow us to conclude the proof.

\end{proof}

%%%%%%%%
\subsection{Formal Derivation - DNSAM}
%%%%%%%%
We now derive an SDE model for the DNSAM iteration given in \eqref{eq:DNSAM_Discr_Update} which we prove to be a 1-order weak approximation of such a discrete iteration. The following result is inspired by Theorem 1 of \cite{li2017stochastic}. We will consider the stochastic process $ X_t \in \mathbb{R}^d $ defined as the solution of 
the SDE

\begin{equation}\label{eq:DNSAM_SDE}
dX_t = -\nabla \Tilde{f}^{\text{DNSAM}}(X_t) d t +  \left( I_d + \rho \frac{\nabla^2 f(X_t)}{\lVert \nabla f(X_t) \rVert_2} \right) \left(\eta \Sigma^{\text{SGD}}(X_t)\right)^{\frac{1}{2}} dW_t
\end{equation}
where the regularized loss is
$$\Tilde{f}^{\text{DNSAM}}(x) = f(x) + \rho \lVert \nabla f(x) \rVert_2,$$
the covariance matrix is
\begin{equation}\label{eq:DNSAM_Covariance}
    \Sigma^{DNSAM}(x) := \Sigma^{SGD}(x)\left( I_d + 2 \rho \frac{\nabla^2 f(x)}{\lVert \nabla f(x) \rVert} \right)
\end{equation}
and $$\Sigma^{\text{SGD}}(x):=\mathbb{E}\left[\left(\nabla f \left(x\right)-\nabla f_{\gamma}\left(x\right)\right)\left(\nabla f \left(x\right)-\nabla f_{\gamma}\left(x \right)\right)^T\right]
$$ is the usual covariance of SGD.

\begin{mybox}{gray}
\begin{theorem}[Stochastic modified equations] \label{thm:DNSAM_SDE}
Let $0<\eta<1, T>0$ and set $N=\lfloor T / \eta\rfloor$. Let $ x_k \in \mathbb{R}^d, 0 \leq k \leq N$ denote a sequence of DNSAM iterations defined by Eq.~\eqref{eq:DNSAM_Discr_Update}. Additionally, let us assume that the noise structure is such that the stochastic gradient can be written as $\nabla f_{\gamma}(x) = \nabla f(x) + Z$ and

\begin{equation}\label{eq:DNSAM_rho_theta_half}
\rho = \mathcal{O}\left(\eta^{\frac{1}{2}}\right).
\end{equation}
Consider the stochastic process $X_t$ defined in Eq.~\eqref{eq:DNSAM_SDE} and fix some test function $g \in G$ and suppose that $g$ and its partial derivatives up to order 6 belong to $G$.

Then, under Assumption \ref{ass:regularity_f}, there exists a constant $ C>0 $ independent of $ \eta $ such that for all $ k=0,1, \ldots, N $, we have

$$
\left|\mathbb{E} g\left(X_{k \eta}\right)-\mathbb{E} g\left(x_k\right)\right| \leq C \eta^1 .
$$

That is, the SDE \eqref{eq:DNSAM_SDE} is an order $ 1 $ weak approximation of the DNSAM iterations \eqref{eq:DNSAM_Discr_Update}.
\end{theorem}
\end{mybox}

\begin{mybox}{gray}
\begin{lemma} \label{lemma:DNSAM_SDE}
Under the assumptions of Theorem \ref{thm:DNSAM_SDE}, let $ 0<\eta<1 $ and consider $ x_k, k \geq 0 $ satisfying the DNSAM iterations \eqref{eq:DNSAM_Discr_Update}
$$
 x_{k+1}=x_k-\eta \nabla f_{\gamma_k}\left(x_k + \rho \frac{\nabla f_{\gamma_k}(x_k)}{\lVert \nabla f(x_k) \rVert}\right)
$$
with $ x_0=x \in \mathbb{R}^d $. Additionally, we define $\partial_{e_i} \Tilde{f}^{\text{DNSAM}}(x) := \partial_{e_i} f(x) + \rho\frac{\sum_j \partial_{e_i+e_j}^{2} f(x) \partial_{e_j}f(x)}{\lVert \nabla f(x) \rVert}$.
From the definition the one-step difference $ \bar{\Delta}=x_1-x $, and we indicate with $\bar{\Delta}_{i}$ the $i$-th component of such difference. Then, we have

\begin{enumerate}
\item $ \mathbb{E} \bar{\Delta}_{i}=-\partial_{e_i} \Tilde{f}^{\text{DNSAM}}(x) \eta +\mathcal{O}(\eta \rho^2) \quad \forall i = 1, \ldots,d$;
\item $ \mathbb{E} \bar{\Delta}_{i} \bar{\Delta}_{j}=\partial_{e_i} \Tilde{f}^{\text{DNSAM}}(x) \partial_{e_j} \Tilde{f}^{\text{DNSAM}}(x) \eta^2 + \Sigma^{\text{DNSAM}}_{(i j)} \eta^2 + \mathcal{O}\left(\eta^3 \right)\quad \forall i,j = 1, \ldots,d$;
\item $\mathbb{E} \prod_{j=1}^s \Bar{\Delta_{i_j}} =\mathcal{O}\left(\eta^3\right) \quad \forall s \geq 3, \quad i_j \in \{ 1, \ldots, d\}.$
\end{enumerate}
and all the functions above are evaluated at $ x $.
\end{lemma}
\end{mybox}

\begin{proof}[Proof of Lemma \ref{lemma:DNSAM_SDE}]
Since the first step is to evaluate $\mathbb{E} \Delta_{i} =- \mathbb{E} \left[ \partial_{e_i} f_{\gamma}\left(x+ \frac{\rho}{\lVert \nabla f(x) \rVert}\nabla f_{\gamma}(x)\right) \eta \right]$, we start by analyzing $ \partial_{e_i} f_{\gamma}\left(x+ \frac{\rho}{\lVert \nabla f(x) \rVert}\nabla f_{\gamma}(x)\right)$, that is the partial derivative in the direction $e_i:=(0, \cdots, 0, \underset{i-th}{1}, 0, \cdots,0)$. 
Under the noise assumption $\nabla f_{\gamma}(x) = \nabla f(x) + Z$, we have that $\nabla^2 f_{\gamma}(x) = \nabla^2f(x)$. Then, we have that

\begin{equation}
 \partial_{e_i} f_{\gamma}\left(x+ \frac{\rho}{\lVert \nabla f(x) \rVert}\nabla f_{\gamma}(x)\right) = \partial_{e_i} f_{\gamma}(x) + \sum_{\vert \alpha \vert=1}\partial_{e_i+\alpha}^{2} f(x) \rho \frac{\partial_{\alpha}f_{\gamma}(x)}{\lVert \nabla f(x) \rVert} + \mathcal{R}^{\partial_{e_i} f_{\gamma}(x)}_{x,1}\left(\rho \frac{ \nabla f_{\gamma}(x)}{\lVert \nabla f(x) \rVert}\right),
\end{equation}
where the residual is defined in Eq.~(4) of \cite{folland2005higher}. Therefore, for some constant $c\in (0,1)$, it holds that

\begin{equation}
\mathcal{R}^{\partial_{e_i} f_{\gamma}(x)}_{x,1}\left(\rho \frac{ \nabla f_{\gamma}(x)}{\lVert \nabla f(x) \rVert}\right) = \sum_{\vert \alpha \vert=2} \frac{\partial_{e_i+\alpha}^{3} f_{\gamma}\left(x + c \rho\frac{ \nabla f_{\gamma}(x)}{\lVert \nabla f(x) \rVert}\right) \rho^{2} \left(\frac{ \nabla f_{\gamma}(x)}{\lVert \nabla f(x) \rVert} \right)^{\alpha}}{\alpha !}.
\end{equation}

Combining the last two equations, we obtain
\begin{align} \label{eq:DNSAM_Estim_Rewritten}
\partial_{e_i} f_{\gamma}\left(x+ \frac{\rho}{\lVert \nabla f(x) \rVert}\nabla f_{\gamma}(x)\right)& = \partial_{e_i} f_{\gamma}(x) + \frac{\rho}{\lVert \nabla f(x) \rVert} \sum_{\vert \alpha \vert=1}\partial_{e_i+\alpha}^{2} f(x) \partial_{\alpha}f_{\gamma}(x)\\
& + \rho^2 \sum_{\vert \alpha \vert=2} \frac{\partial_{e_i+\alpha}^{3} f_{\gamma}\left(x + c \rho\frac{ \nabla f_{\gamma}(x)}{\lVert \nabla f(x) \rVert}\right)\left(\frac{ \nabla f_{\gamma}(x)}{\lVert \nabla f(x) \rVert} \right)^{\alpha}}{\alpha !}.
\end{align}
Now, we observe that
\begin{equation}
K_{i}(x) := \left[ \sum_{\vert \alpha \vert=2} \frac{\partial_{e_i+\alpha}^{3} f_{\gamma}\left(x + c \rho\frac{ \nabla f(x)}{\lVert \nabla f(x) \rVert}\right)\left(\frac{ \nabla f_{\gamma}(x)}{\lVert \nabla f(x) \rVert} \right)^{\alpha}}{\alpha !} \right]
\end{equation}
is a finite sum of products of functions that by assumption are in $G$. Therefore, $K_{i}(x) \in G$ and $\Bar{K}_{i}(x) = \mathbb{E} \left[ K_{i}(x) \right] \in G$. Based on these definitions, we rewrite Eq.~\eqref{eq:DNSAM_Estim_Rewritten} as
\begin{equation} \label{eq:DNSAM_Estim_Rewritten_1}
 \partial_{e_i} f_{\gamma}\left(x+ \frac{\rho}{\lVert \nabla f(x) \rVert}\nabla f_{\gamma}(x)\right)= \partial_{e_i} f_{\gamma}(x) + \frac{\rho}{\lVert \nabla f(x) \rVert} \sum_{\vert \alpha \vert=1}\partial_{e_i+\alpha}^{2} f(x) \partial_{\alpha}f_{\gamma}(x) + \rho^2 K_{i}(x).
\end{equation}
which implies that
\begin{equation} \label{eq:DNSAM_Estim_Rewritten_2}
\mathbb{E} \left[ \partial_{e_i} f_{\gamma}\left(x+ \frac{\rho}{\lVert \nabla f(x) \rVert}\nabla f_{\gamma}(x)\right) \right] = \partial_{e_i} f(x) + \rho  \frac{ \sum_j \partial_{e_i+e_j}^{2} f(x) \partial_{e_j}f(x)}{\lVert \nabla f(x) \rVert}  + \rho^2 \Bar{K}_{i}(x),
\end{equation}
where we used the unbiasedness property of the stochastic gradients: $\E \nabla f_{\gamma}(x) = \nabla f(x)$. \\

Let us now remember that by definition
\begin{equation} \label{eq:DNSAM_Real_Rewritten}
\partial_{e_i} \Tilde{f}^{\text{DNSAM}}(x)=\partial_{e_i} f(x) + \rho  \frac{ \sum_j \partial_{e_i+e_j}^{2} f(x) \partial_{e_j}f(x)}{\lVert \nabla f(x) \rVert}.
\end{equation}
Therefore, by using Eq. \eqref{eq:DNSAM_Estim_Rewritten_2}, Eq. \eqref{eq:DNSAM_Real_Rewritten}, and the assumption \eqref{eq:DNSAM_rho_theta_half} we have that $\forall i = 1, \ldots,d$,
\begin{equation} \label{eq:DNSAM_Estim_Rewritten_3}
\mathbb{E} \bar{\Delta}_{i} = -\partial_{e_i} \Tilde{f}^{\text{DNSAM}}(x) \eta + \eta \rho^2 \Bar{K}_{i}(x) = -\partial_{e_i} \Tilde{f}^{\text{DNSAM}}(x) \eta + \mathcal{O}\left(\eta^2 \right).
\end{equation}
We now observe that the covariance matrix of the difference between the drift $ \nabla f(x) + \rho \frac{\nabla^2 f(x) \nabla f(x)}{\lVert \nabla f(x) \rVert}$ of the SDE \eqref{eq:DNSAM_SDE} and the gradient $\nabla f_{\gamma}\left(x+ \frac{\rho}{\lVert \nabla f(x) \rVert}\nabla f_{\gamma}(x)\right) = \nabla f_{\gamma}(x) + \rho \frac{\nabla^2 f(x) \nabla f_{\gamma}(x)}{\lVert \nabla f(x) \rVert} + \rho^2 K(x)$ in the discrete algorithm \eqref{eq:DNSAM_Discr_Update} is

\begin{align}
    \Bar{\Sigma} := \mathbb{E} & \left[  \left( \nabla f(x) + \rho \frac{\nabla^2 f(x) \nabla f(x)}{\lVert \nabla f(x) \rVert} - \nabla f_{\gamma}(x) - \rho \frac{\nabla^2 f(x) \nabla f_{\gamma}(x)}{\lVert \nabla f(x) \rVert} - \rho^2 K(x) \right) \right. \\
    & \left. \left( \nabla f(x) + \rho \frac{\nabla^2 f(x) \nabla f(x)}{\lVert \nabla f(x) \rVert} - \nabla f_{\gamma}(x) - \rho \frac{\nabla^2 f(x) \nabla f_{\gamma}(x)}{\lVert \nabla f(x) \rVert} - \rho^2 K(x) \right)^{\top}  \right] \\
    & = \Sigma^{SGD}\left( I_d + 2 \rho \frac{\nabla^2 f(x)}{\lVert \nabla f(x) \rVert} \right) + \mathcal{O}(\rho^2) = \Sigma^{DNSAM}(x) + \mathcal{O}(\rho^2).
\end{align}
Therefore, we have that

\begin{align}
 \mathbb{E} \bar{\Delta}_{i} \bar{\Delta}_{j} = & \text{Cov}(\bar{\Delta}_{i}, \bar{\Delta}_{j}) + \mathbb{E} \bar{\Delta}_{i} \mathbb{E}\bar{\Delta}_{j} \nonumber \\
 & \overset{\eqref{eq:DNSAM_Estim_Rewritten_3}}{=}  \text{Cov}(\bar{\Delta}_{i}, \bar{\Delta}_{j}) + \partial_{e_i} \Tilde{f}^{\text{DNSAM}} \partial_{e_j} \Tilde{f}^{\text{DNSAM}} \eta^2+ \eta^2 \rho^2 ( \partial_{e_i} \Tilde{f}^{\text{DNSAM}} \Bar{K}_{j}(x) + \partial_{e_j} \Tilde{f}^{\text{DNSAM}} \Bar{K}_{i}(x)) + \eta^2 \rho^4 \Bar{K}_{i}(x) \Bar{K}_{j}(x) \nonumber \\
 & = \text{Cov}(\bar{\Delta}_{i}, \bar{\Delta}_{j}) + \partial_{e_i} \Tilde{f}^{\text{DNSAM}} \partial_{e_j} \Tilde{f}^{\text{DNSAM}} \eta^2 + \mathcal{O}\left(\eta^2 \rho^2 \right) + \mathcal{O}\left(\eta^2 \rho^4 \right) \nonumber \\
 & =\partial_{e_i} \Tilde{f}^{\text{DNSAM}} \partial_{e_j} \Tilde{f}^{\text{DNSAM}} \eta^2+\text{Cov}(\bar{\Delta}_{i}, \bar{\Delta}_{j}) + \mathcal{O}\left(\eta^2 \rho^2 \right) + \mathcal{O}\left(\eta^2 \rho^4 \right) \quad \forall i,j = 1, \ldots,d.\label{eq:DNSAM_Cov_final_step}
\end{align}

By the definitions of $\Sigma^{\text{DNSAM}}$ and of $\Bar{\Sigma}(x)$, we have
\begin{equation}
 \text{Cov}(\bar{\Delta}_{i}, \bar{\Delta}_{j}) = \eta^2\Bar{\Sigma}_{i,j}(x) = \eta^2 \Sigma^{\text{DNSAM}}_{i,j}(x) +\mathcal{O}(\eta^2 \rho^2).
\end{equation}
Therefore, remembering Eq. \eqref{eq:DNSAM_Cov_final_step} and Eq. \eqref{eq:DNSAM_rho_theta_half} we have 

\begin{equation}
\mathbb{E} \bar{\Delta}_{i} \bar{\Delta}_{j} =\partial_{e_i} \Tilde{f}^{\text{DNSAM}} \partial_{e_j} \Tilde{f}^{\text{DNSAM}} \eta^2+ \Sigma^{\text{DNSAM}}_{i,j} \eta^2 + \mathcal{O}\left(\eta^3 \right), \quad \forall i,j = 1, \ldots,d.
\end{equation}
Finally, with analogous considerations, it is obvious that under our assumptions

$$
\mathbb{E} \prod_{j=1}^s \bar{\Delta}_{i_j}= \mathcal{O}\left(\eta^s\right) \quad \forall s \geq 3, \quad i_j \in \{1, \ldots, d\}
$$
which in particular implies that
$$
\mathbb{E} \prod_{j=1}^3 \bar{\Delta}_{i_j}= \mathcal{O}\left(\eta^3\right), \quad i_j \in \{1, \ldots, d\}.
$$

\end{proof}

\begin{proof}[Proof of Theorem \ref{thm:DNSAM_SDE}] 
\label{proof:DNSAM_SDE}
To prove this result, all we need to do is check the conditions in Theorem \ref{thm:mils}. As we apply Lemma \ref{lemma:li1}, we make the following choices:

\begin{itemize}
\item $b(x)=-\nabla \Tilde{f}^{\text{DNSAM}}\left(x\right)$;
\item $\sigma(x) =\Sigma^{\text{DNSAM}}(x)^{\frac{1}{2}}$.
\end{itemize}
First of all, we notice that $\forall i = 1, \ldots,d$, it holds that

\begin{itemize}
\item $\mathbb{E} \bar{\Delta}_{i} \overset{\text{1. Lemma \ref{lemma:DNSAM_SDE}}}{=}-\partial_{e_i} \Tilde{f}^{\text{DNSAM}}(x)\eta+ \mathcal{O}(\eta^2)$;
\item $ \mathbb{E} \Delta_{i} \overset{\text{1. Lemma \ref{lemma:li1}}}{=} -\partial_{e_i} \Tilde{f}^{\text{DNSAM}}(x)\eta +\mathcal{O}\left(\eta^2\right)$.
\end{itemize}
Therefore, we have that for some $K_1(x) \in G$

\begin{equation}\label{eq:DNSAM_cond1}
\left|\mathbb{E} \Delta_{i}-\mathbb{E} \bar{\Delta}_{i}\right| \leq K_1(x) \eta^{2}, \quad \forall i = 1, \ldots,d.
\end{equation}
Additionally,we notice that $\forall i,j = 1, \ldots,d$, it holds that

\begin{itemize}
\item $ \mathbb{E} \bar{\Delta}_{i} \bar{\Delta}_{j} \overset{\text{2. Lemma \ref{lemma:DNSAM_SDE}}}{=}\partial_{e_i} \Tilde{f}^{\text{DNSAM}} \partial_{e_j} \Tilde{f}^{\text{DNSAM}} \eta^2+ \Sigma^{\text{DNSAM}}_{i,j} \eta^2 + \mathcal{O}\left(\eta^3 \right)$;
\item $ \mathbb{E} \Delta_{i} \Delta_{j} \overset{\text{2. Lemma \ref{lemma:li1}}}{=}\partial_{e_i} \Tilde{f}^{\text{DNSAM}} \partial_{e_j} \Tilde{f}^{\text{DNSAM}} \eta^2+ \Sigma^{\text{DNSAM}}_{i,j} \eta^2 + \mathcal{O}\left(\eta^3 \right)$.
\end{itemize}
Therefore, we have that for some $K_2(x) \in G$

\begin{equation}\label{eq:DNSAM_cond2}
\left|\mathbb{E} \Delta_{i} \Delta_{j} - \mathbb{E} \bar{\Delta}_{i} \bar{\Delta}_{j}\right| \leq K_2(x) \eta^{2}, \quad \forall i,j = 1, \ldots,d.
\end{equation}
Additionally, we notice that $\forall s \geq 3, \forall i_j \in \{1, \ldots,d \}$, it holds that

\begin{itemize}
\item $ \mathbb{E} \prod_{j=1}^s \bar{\Delta}_{i_j}\overset{\text{3. Lemma \ref{lemma:DNSAM_SDE}}}{=}\mathcal{O}\left(\eta^3\right)$;
\item $ \mathbb{E} \prod_{j=1}^s \Delta_{i_j}\overset{\text{3. Lemma \ref{lemma:li1}}}{=}\mathcal{O}\left(\eta^3\right)$.
\end{itemize}
Therefore, we have that for some $K_3(x) \in G$

\begin{equation}\label{eq:DNSAM_cond3}
\left|\mathbb{E} \prod_{j=1}^s \Delta_{i_j}-\mathbb{E} \prod_{j=1}^s \bar{\Delta}_{i_j}\right| \leq K_3(x) \eta^{2}.
\end{equation}
Additionally, for some $K_4(x) \in G$, $\forall i_j \in \{1, \ldots,d \}$

\begin{equation} \label{eq:DNSAM_cond4}
\mathbb{E} \prod_{j=1}^{ 3}\left|\bar{\Delta}_{\left(i_j\right)}\right| \overset{\text{3. Lemma \ref{lemma:DNSAM_SDE}}}{\leq}K_4(x) \eta^{2}.
\end{equation}
Finally, Eq.~\eqref{eq:DNSAM_cond1}, Eq.~\eqref{eq:DNSAM_cond2}, Eq.~\eqref{eq:DNSAM_cond3}, and Eq.~\eqref{eq:DNSAM_cond4} allow us to conclude the proof.

\end{proof}

\subsubsection{DNSAM is SGD if \texorpdfstring{$\rho = \mathcal{O}(\eta)$}{Lg} }
%%%%%%%%%%%%%%%%%%%%%%%%%%%%%%%%%%%%%%%%%%%%%%%%%%%%%%%%%%%%%%%%%%%%%%%%%%%%

The following result is inspired by Theorem 1 of \cite{li2017stochastic}. We will consider the stochastic process $ X_t \in \mathbb{R}^d $ defined as the solution of 
the SDE
\begin{equation}\label{eq:DNSAM_SGD_SDE}
d X_t=-\nabla f\left(X_t\right) d t+\left(\eta \Sigma^{\text{SGD}}\left(X_t\right)\right)^{1 / 2} d W_t
\end{equation}

Such that $X_0=x_0$ and $$\Sigma^{\text{SGD}}(x):=\mathbb{E}\left[\left(\nabla f \left(x\right)-\nabla f_{\gamma}\left(x\right)\right)\left(\nabla f \left(x\right)-\nabla f_{\gamma}\left(x \right)\right)^T\right]
$$

\begin{mybox}{gray}
\begin{theorem}[Stochastic modified equations] \label{thm:DNSAM_SGD}
Let $0<\eta<1, T>0$ and set $N=\lfloor T / \eta\rfloor$. Let $ x_k \in \mathbb{R}^d, 0 \leq k \leq N$ denote a sequence of DNSAM iterations defined by Eq.~\eqref{eq:DNSAM_Discr_Update}. Additionally, let us take 

\begin{equation}\label{eq:DNSAM_SGD_rho}
\rho = \mathcal{O}\left(\eta^{1}\right).
\end{equation}
Consider the stochastic process $X_t$ defined in Eq.~\eqref{eq:DNSAM_SGD_SDE} and fix some test function $g \in G$ and suppose that $g$ and its partial derivatives up to order 6 belong to $G$.
Then, under Assumption \ref{ass:regularity_f}, there exists a constant $ C>0 $ independent of $ \eta $ such that for all $ k=0,1, \ldots, N $, we have

$$
\left|\mathbb{E} g\left(X_{k \eta}\right)-\mathbb{E} g\left(x_k\right)\right| \leq C \eta^1 .
$$

That is, the SDE \eqref{eq:DNSAM_SGD_SDE} is an order $ 1 $ weak approximation of the SAM iterations \eqref{eq:SAM_Discr_Update}.
\end{theorem}
\end{mybox}

\begin{proof}
    The proof is completely similar to that of Theorem \ref{thm:USAM_SGD} presented before and of Theorem \ref{thm:SAM_SGD} presented later. 
\end{proof}

%%%%%%%%%%%%%%%%%%%%%%%%%%%%%%%%%%%%%%%%%%%%%%%%%%%%%%%%%%%%%%%%%%%%%%%%%%%%
\subsection{Formal Derivation - SAM} \label{sec:formal_SAM}
%%%%%%%%%%%%%%%%%%%%%%%%%%%%%%%%%%%%%%%%%%%%%%%%%%%%%%%%%%%%%%%%%%%%%%%%%%%%

The following result is inspired by Theorem 1 of \cite{li2017stochastic}. We will consider the stochastic process $ X_t \in \mathbb{R}^d $ defined as the solution of 
the SDE

\begin{equation}\label{eq:SAM_SDE}
dX_t = -\nabla \Tilde{f}^{\text{SAM}}(X_t) d t + \sqrt{\eta}\left( \Sigma^{\text{SGD}}(X_t)+ \rho \left( \Hat{\Sigma}(X_t) + \Hat{\Sigma}(X_t) ^{\top} \right) \right)^{\frac{1}{2}}dW_t
\end{equation}
where $$\Sigma^{\text{SGD}}(x):=\mathbb{E}\left[\left(\nabla f \left(x\right)-\nabla f_{\gamma}\left(x\right)\right)\left(\nabla f \left(x\right)-\nabla f_{\gamma}\left(x \right)\right)^T\right]
$$
is the usual covariance of SGD, while

\begin{equation} \label{eq:SAM_sigma_star}
\Hat{\Sigma}(x) :=\mathbb{E} \left[ \left( \nabla f\left(x\right) - \nabla f_{\gamma}\left(x\right) \right)\left( \E \left[\frac{\nabla^2 f_{\gamma}(x)\nabla f_{\gamma}(x) }{\lVert \nabla f_{\gamma}(x) \rVert_2} \right]- \frac{\nabla^2 f_{\gamma}(x)\nabla f_{\gamma}(x) }{\lVert \nabla f_{\gamma}(x) \rVert_2} \right)^{\top} \right]
\end{equation}
and $$\Tilde{f}^{\text{SAM}}(x):= f(x) + \rho \mathbb{E} \left[ \lVert \nabla f_{\gamma}(x) \rVert_2\right].$$
In the following, we will use the notation

\begin{equation}\label{eq:SAM_Covariance}
\Sigma^{\text{SAM}}(x):= \left( \Sigma^{\text{SGD}}(X_t)+ \rho \left( \Hat{\Sigma}(X_t) + \Hat{\Sigma}(X_t) ^{\top} \right) \right).
\end{equation}

\begin{mybox}{gray}
\begin{theorem}[Stochastic modified equations] \label{thm:SAM_SDE}
Let $0<\eta<1, T>0$ and set $N=\lfloor T / \eta\rfloor$. Let $ x_k \in \mathbb{R}^d, 0 \leq k \leq N$ denote a sequence of SAM iterations defined by Eq.~\eqref{eq:SAM_Discr_Update}. Additionally, let us take 

\begin{equation}\label{eq:SAM_rho_theta_half}
\rho = \mathcal{O}\left(\eta^{\frac{1}{2}}\right).
\end{equation}
Consider the stochastic process $X_t$ defined in Eq.~\eqref{eq:SAM_SDE} and fix some test function $g \in G$ and suppose that $g$ and its partial derivatives up to order 6 belong to $G$.

Then, under Assumption \ref{ass:regularity_f}, there exists a constant $ C>0 $ independent of $ \eta $ such that for all $ k=0,1, \ldots, N $, we have

$$
\left|\mathbb{E} g\left(X_{k \eta}\right)-\mathbb{E} g\left(x_k\right)\right| \leq C \eta^1 .
$$

That is, the SDE \eqref{eq:SAM_SDE} is an order $ 1 $ weak approximation of the SAM iterations \eqref{eq:SAM_Discr_Update}.
\end{theorem}
\end{mybox}

\begin{remark}\label{remark:guarantee}
Denote by $b$ and $\sigma $ the (Borel measurable) drift and diffusion
coefficient in (\ref{eq:SAM_SDE}), respectively. Suppose that there exists a
non-negative function $F\in L^{d+1}(\left[ 0,\infty \right) \times \mathbb{R}%
^{d})$ such that%
\begin{equation}
\left\Vert b(t,x)\right\Vert \leq K+F(t,x)  \label{A1}
\end{equation}%
$dt\times dx-a.e.$ for some constant $K\geq 0$. Further, assume that $\sigma
\sigma ^{\top }$ is strongly elliptic, that
is there exists a $\delta \in (0,1)$ such that for all $t\geq 0$, $x\in 
\mathbb{R}^{d}$%
\begin{equation}
\delta I_d\leq \sigma \sigma ^{\top }(t,x)\leq \delta ^{-1}I_d,  \label{A2}
\end{equation}%
where $I_d\in \mathbb{R}^{d\times d}$ is the unit matrix. Then there exists a
(global) weak solution to (\ref{eq:SAM_SDE}). Moreover, if there is a (global) weak
solution to (\ref{eq:SAM_SDE}), $b\in L_{loc}^{2d+2}(\left[ 0,\infty \right) \times 
\mathbb{R}^{d})$, $\sigma $ is locally Lipschitz continuous uniformly in
time and if (\ref{A2}) holds, then there exists a unique strong solution to (%
\ref{eq:SAM_SDE}). See \cite{GM}.
\end{remark}

\begin{mybox}{gray}
\begin{lemma} \label{lemma:SAM_SDE}
Under the assumptions of Theorem \ref{thm:SAM_SDE}, let $ 0<\eta<1 $ and consider $ x_k, k \geq 0 $ satisfying the USAM iterations \eqref{eq:SAM_Discr_Update}
$$
 x_{k+1}=x_k-\eta \nabla f_{\gamma_k}\left(x_k + \rho \frac{\nabla f_{\gamma_k}(x_k)}{\lVert \nabla f_{\gamma}(x_k) \rVert}\right)
$$
with $ x_0=x \in \mathbb{R}^d $. Additionally, we define $\partial_{e_i} \Tilde{f}^{\text{SAM}}(x) := \partial_{e_i} f(x) + \rho\mathbb{E} \left[ \frac{\sum_j \partial_{e_i+e_j}^{2} f_{\gamma}(x) \partial_{e_j}f_{\gamma}(x)}{\lVert \nabla f_{\gamma}(x) \rVert} \right]$. From the definition the one-step difference $ \bar{\Delta}=x_1-x $, then we have

\begin{enumerate}
\item $ \mathbb{E} \bar{\Delta}_{i}=-\partial_{e_i} \Tilde{f}^{\text{SAM}}(x) \eta +\mathcal{O}(\eta \rho^2) \quad \forall i = 1, \ldots,d$;
\item $ \mathbb{E} \bar{\Delta}_{i} \bar{\Delta}_{j}=\partial_{e_i} \Tilde{f}^{\text{SAM}}(x) \partial_{e_j} \Tilde{f}^{\text{SAM}}(x) \eta^2 + \Sigma^{\text{SAM}}_{(i j)} \eta^2 + \mathcal{O}\left(\eta^3 \right)\quad \forall i,j = 1, \ldots,d$;
\item $\mathbb{E} \prod_{j=1}^s \Bar{\Delta}_{i_j} =\mathcal{O}\left(\eta^3\right) \quad \forall s \geq 3, \quad i_j \in \{ 1, \ldots, d\}.$
\end{enumerate}
and all the functions above are evaluated at $ x $.
\end{lemma}
\end{mybox}

\begin{proof}[Proof of Lemma \ref{lemma:SAM_SDE}]
Since the first step is to evaluate $\mathbb{E} \Delta_{i} =- \mathbb{E} \left[ \partial_{e_i} f_{\gamma}\left(x+ \frac{\rho}{\lVert \nabla f_{\gamma}(x) \rVert}\nabla f_{\gamma}(x)\right) \eta \right]$, we start by analyzing $ \partial_{e_i} f_{\gamma}\left(x+ \frac{\rho}{\lVert \nabla f_{\gamma}(x) \rVert}\nabla f_{\gamma}(x)\right)$, that is the partial derivative in the direction $e_i:=(0, \cdots, 0, \underset{i-th}{1}, 0, \cdots,0)$. Then, we have that

\begin{equation}
 \partial_{e_i} f_{\gamma}\left(x+ \frac{\rho}{\lVert \nabla f_{\gamma}(x) \rVert}\nabla f_{\gamma}(x)\right) = \partial_{e_i} f_{\gamma}(x) + \sum_{\vert \alpha \vert=1}\partial_{e_i+\alpha}^{2} f_{\gamma}(x) \rho \frac{\partial_{\alpha}f_{\gamma}(x)}{\lVert \nabla f_{\gamma}(x) \rVert} + \mathcal{R}^{\partial_{e_i} f_{\gamma}(x)}_{x,1}\left(\rho \frac{ \nabla f_{\gamma}(x)}{\lVert \nabla f_{\gamma}(x) \rVert}\right)
\end{equation}
Where the residual is defined in Eq.~(4) of \cite{folland2005higher}. Therefore, for some constant $c\in (0,1)$, it holds that

\begin{equation}
\mathcal{R}^{\partial_{e_i} f_{\gamma}(x)}_{x,1}\left(\rho \frac{ \nabla f_{\gamma}(x)}{\lVert \nabla f_{\gamma}(x) \rVert}\right) = \sum_{\vert \alpha \vert=2} \frac{\partial_{e_i+\alpha}^{3} f_{\gamma}\left(x + c \rho\frac{ \nabla f_{\gamma}(x)}{\lVert \nabla f_{\gamma}(x) \rVert}\right) \rho^{2} \left(\frac{ \nabla f_{\gamma}(x)}{\lVert \nabla f_{\gamma}(x) \rVert} \right)^{\alpha}}{\alpha !}.
\end{equation}
Therefore, we can rewrite it as

\begin{align} \label{eq:SAM_Estim_Rewritten}
\partial_{e_i} f_{\gamma}\left(x+ \frac{\rho}{\lVert \nabla f_{\gamma}(x) \rVert}\nabla f_{\gamma}(x)\right)& = \partial_{e_i} f_{\gamma}(x) + \frac{\rho}{\lVert \nabla f_{\gamma}(x) \rVert} \sum_{\vert \alpha \vert=1}\partial_{e_i+\alpha}^{2} f_{\gamma}(x) \partial_{\alpha}f_{\gamma}(x)\\
& + \rho^2 \sum_{\vert \alpha \vert=2} \frac{\partial_{e_i+\alpha}^{3} f_{\gamma}\left(x + c \rho\frac{ \nabla f_{\gamma}(x)}{\lVert \nabla f_{\gamma}(x) \rVert}\right)\left(\frac{ \nabla f_{\gamma}(x)}{\lVert \nabla f_{\gamma}(x) \rVert} \right)^{\alpha}}{\alpha !}.
\end{align}
Now, we observe that
\begin{equation}
K_{i}(x) := \left[ \sum_{\vert \alpha \vert=2} \frac{\partial_{e_i+\alpha}^{3} f_{\gamma}\left(x + c \rho\frac{ \nabla f_{\gamma}(x)}{\lVert \nabla f_{\gamma}(x) \rVert}\right)\left(\frac{ \nabla f_{\gamma}(x)}{\lVert \nabla f_{\gamma}(x) \rVert} \right)^{\alpha}}{\alpha !} \right]
\end{equation}
is a finite sum of products of functions that by assumption are in $G$. Therefore, $K_{i}(x) \in G$ and $\Bar{K}_{i}(x) = \mathbb{E} \left[ K_{i}(x) \right] \in G$. Based on these definitions, we rewrite Eq.~\eqref{eq:SAM_Estim_Rewritten} as
\begin{equation} \label{eq:SAM_Estim_Rewritten_1}
 \partial_{e_i} f_{\gamma}\left(x+ \frac{\rho}{\lVert \nabla f_{\gamma}(x) \rVert}\nabla f_{\gamma}(x)\right)= \partial_{e_i} f_{\gamma}(x) + \frac{\rho}{\lVert \nabla f_{\gamma}(x) \rVert} \sum_{\vert \alpha \vert=1}\partial_{e_i+\alpha}^{2} f_{\gamma}(x) \partial_{\alpha}f_{\gamma}(x) + \rho^2 K_{i}(x).
\end{equation}
which implies that

\begin{equation} \label{eq:SAM_Estim_Rewritten_2}
\mathbb{E} \left[ \partial_{e_i} f_{\gamma}\left(x+ \frac{\rho}{\lVert \nabla f_{\gamma}(x) \rVert}\nabla f_{\gamma}(x)\right) \right] = \partial_{e_i} f(x) + \rho \mathbb{E} \left[ \frac{ \sum_j \partial_{e_i+e_j}^{2} f_{\gamma}(x) \partial_{e_j}f_{\gamma}(x)}{\lVert \nabla f_{\gamma}(x) \rVert} \right] + \rho^2 \Bar{K}_{i}(x).
\end{equation}
Let us now remember that

\begin{equation} \label{eq:SAM_Real_Rewritten}
\partial_{e_i} \Tilde{f}^{\text{SAM}}(x)=\partial_{e_i} \left( f(x) + \rho \mathbb{E} \left[ \lVert \nabla f_{\gamma}(x) \rVert_2 \right] \right) = \partial_{e_i} f(x) + \rho \mathbb{E} \left[ \frac{ \sum_j \partial_{e_i+e_j}^{2} f_{\gamma}(x) \partial_{e_j}f_{\gamma}(x)}{\lVert \nabla f_{\gamma}(x) \rVert} \right]
\end{equation}

Therefore, by using Eq. \eqref{eq:SAM_Estim_Rewritten_2}, Eq. \eqref{eq:SAM_Real_Rewritten}, and the assumption \eqref{eq:SAM_rho_theta_half} we have that $\forall i = 1, \ldots,d$

\begin{equation} \label{eq:SAM_Estim_Rewritten_3}
\mathbb{E} \bar{\Delta}_{i} = -\partial_{e_i} \Tilde{f}^{\text{SAM}}(x) \eta + \eta \rho^2 \Bar{K}_{i}(x) = -\partial_{e_i} \Tilde{f}^{\text{SAM}}(x) \eta + \mathcal{O}\left(\eta^2 \right).
\end{equation}
Additionally, we have that

\begin{align}
 \mathbb{E} \bar{\Delta}_{i} \bar{\Delta}_{j} = & \text{Cov}(\bar{\Delta}_{i}, \bar{\Delta}_{j}) + \mathbb{E} \bar{\Delta}_{i} \mathbb{E}\bar{\Delta}_{j} \nonumber \\
 & \overset{\eqref{eq:SAM_Estim_Rewritten_3}}{=}  \text{Cov}(\bar{\Delta}_{i}, \bar{\Delta}_{j}) + \partial_{e_i} \Tilde{f}^{\text{SAM}} \partial_{e_j} \Tilde{f}^{\text{SAM}} \eta^2+ \eta^2 \rho^2 ( \partial_{e_i} \Tilde{f}^{\text{SAM}} \Bar{K}_{j}(x) + \partial_{e_j} \Tilde{f}^{\text{SAM}} \Bar{K}_{i}(x)) + \eta^2 \rho^4 \Bar{K}_{i}(x) \Bar{K}_{j}(x) \nonumber \\
 & = \text{Cov}(\bar{\Delta}_{i}, \bar{\Delta}_{j}) + \partial_{e_i} \Tilde{f}^{\text{SAM}} \partial_{e_j} \Tilde{f}^{\text{SAM}} \eta^2 + \mathcal{O}\left(\eta^2 \rho^2 \right) + \mathcal{O}\left(\eta^2 \rho^4 \right) \nonumber \\
 & =\partial_{e_i} \Tilde{f}^{\text{SAM}} \partial_{e_j} \Tilde{f}^{\text{SAM}} \eta^2+\text{Cov}(\bar{\Delta}_{i}, \bar{\Delta}_{j}) + \mathcal{O}\left(\eta^2 \rho^2 \right) + \mathcal{O}\left(\eta^2 \rho^4 \right) \quad \forall i,j = 1, \ldots,d.\label{eq:SAM_Cov_final_step}
\end{align}
Let us now recall the expression \eqref{eq:SAM_sigma_star} of $\Hat{\Sigma}$ and the expression \eqref{eq:SAM_Covariance} of $\Sigma^{\text{SAM}}$. Then, we automatically have that

\begin{equation}
 \text{Cov}(\bar{\Delta}_{i}, \bar{\Delta}_{j}) =\eta^2 \left( \Sigma^{\text{SGD}}_{i,j}(x) + \rho \left[ \Hat{\Sigma}_{i,j}(x)+ \Hat{\Sigma}_{i,j}(x)^{\top} \right] + \mathcal{O}(\rho^2) \right) = \eta^2 \Sigma^{\text{SAM}}_{i,j}(x) +\mathcal{O}(\eta^2 \rho^2).
\end{equation}
Therefore, remembering Eq. \eqref{eq:SAM_Cov_final_step} and Eq. \eqref{eq:SAM_rho_theta_half} we have 

\begin{equation}
\mathbb{E} \bar{\Delta}_{i} \bar{\Delta}_{j} =\partial_{e_i} \Tilde{f}^{\text{SAM}} \partial_{e_j} \Tilde{f}^{\text{SAM}} \eta^2+ \Sigma^{\text{SAM}}_{i,j} \eta^2 + \mathcal{O}\left(\eta^3 \right), \quad \forall i,j = 1, \ldots,d.
\end{equation}
Finally, with analogous considerations, it is obvious that under our assumptions

$$
\mathbb{E} \prod_{j=1}^s \bar{\Delta}_{i_j}= \mathcal{O}\left(\eta^s\right) \quad \forall s \geq 3, \quad i_j \in \{1, \ldots, d\}
$$
which in particular implies that
$$
\mathbb{E} \prod_{j=1}^3 \bar{\Delta}_{i_j}= \mathcal{O}\left(\eta^3\right), \quad i_j \in \{1, \ldots, d\}.
$$

\end{proof}

\begin{proof}[Proof of Theorem \ref{thm:SAM_SDE}] 
\label{proof:SAM_SDE}
To prove this result, all we need to do is check the conditions in Theorem \ref{thm:mils}. As we apply Lemma \ref{lemma:li1}, we make the following choices:

\begin{itemize}
\item $b(x)=-\nabla \Tilde{f}^{\text{SAM}}\left(x\right)$;
\item $\sigma(x) =\Sigma^{\text{SAM}}(x)^{\frac{1}{2}}$.
\end{itemize}
First of all, we notice that $\forall i = 1, \ldots,d$, it holds that

\begin{itemize}
\item $\mathbb{E} \bar{\Delta}_{i} \overset{\text{1. Lemma \ref{lemma:SAM_SDE}}}{=}-\partial_{e_i} \Tilde{f}^{\text{SAM}}(x)\eta+ \mathcal{O}(\eta^2)$;
\item $ \mathbb{E} \Delta_{i} \overset{\text{1. Lemma \ref{lemma:li1}}}{=} -\partial_{e_i} \Tilde{f}^{\text{SAM}}(x)\eta +\mathcal{O}\left(\eta^2\right)$.
\end{itemize}
Therefore, we have that for some $K_1(x) \in G$

\begin{equation}\label{eq:SAM_cond1}
\left|\mathbb{E} \Delta_{i}-\mathbb{E} \bar{\Delta}_{i}\right| \leq K_1(x) \eta^{2}, \quad \forall i = 1, \ldots,d.
\end{equation}
Additionally,we notice that $\forall i,j = 1, \ldots,d$, it holds that

\begin{itemize}
\item $ \mathbb{E} \bar{\Delta}_{i} \bar{\Delta}_{j} \overset{\text{2. Lemma \ref{lemma:SAM_SDE}}}{=}\partial_{e_i} \Tilde{f}^{\text{SAM}} \partial_{e_j} \Tilde{f}^{\text{SAM}} \eta^2+ \Sigma^{\text{SAM}}_{i,j} \eta^2 + \mathcal{O}\left(\eta^3 \right)$;
\item $ \mathbb{E} \Delta_{i} \Delta_{j} \overset{\text{2. Lemma \ref{lemma:li1}}}{=}\partial_{e_i} \Tilde{f}^{\text{SAM}} \partial_{e_j} \Tilde{f}^{\text{SAM}} \eta^2+ \Sigma^{\text{SAM}}_{i,j} \eta^2 + \mathcal{O}\left(\eta^3 \right)$.
\end{itemize}
Therefore, we have that for some $K_2(x) \in G$

\begin{equation}\label{eq:SAM_cond2}
\left|\mathbb{E} \Delta_{i} \Delta_{j} - \mathbb{E} \bar{\Delta}_{i} \bar{\Delta}_{j}\right| \leq K_2(x) \eta^{2}, \quad \forall i,j = 1, \ldots,d.
\end{equation}
Additionally, we notice that $\forall s \geq 3, \forall i_j \in \{1, \ldots,d \}$, it holds that

\begin{itemize}
\item $ \mathbb{E} \prod_{j=1}^s \bar{\Delta}_{i_j}\overset{\text{3. Lemma \ref{lemma:SAM_SDE}}}{=}\mathcal{O}\left(\eta^3\right)$;
\item $ \mathbb{E} \prod_{j=1}^s \Delta_{i_j}\overset{\text{3. Lemma \ref{lemma:li1}}}{=}\mathcal{O}\left(\eta^3\right)$.
\end{itemize}
Therefore, we have that for some $K_3(x) \in G$

\begin{equation}\label{eq:SAM_cond3}
\left|\mathbb{E} \prod_{j=1}^s \Delta_{i_j}-\mathbb{E} \prod_{j=1}^s \bar{\Delta}_{i_j}\right| \leq K_3(x) \eta^{2}.
\end{equation}
Additionally, for some $K_4(x) \in G$, $\forall i_j \in \{1, \ldots,d \}$

\begin{equation} \label{eq:SAM_cond4}
\mathbb{E} \prod_{j=1}^{ 3}\left|\bar{\Delta}_{\left(i_j\right)}\right| \overset{\text{3. Lemma \ref{lemma:USAM_SDE}}}{\leq}K_4(x) \eta^{2}.
\end{equation}
Finally, Eq.~\eqref{eq:SAM_cond1}, Eq.~\eqref{eq:SAM_cond2}, Eq.~\eqref{eq:SAM_cond3}, and Eq.~\eqref{eq:SAM_cond4} allow us to conclude the proof.

\end{proof}

\begin{mybox}{gray}
\begin{corollary}\label{thm:SAM_SDE_Simplified}
Let us take the same assumptions of Theorem \eqref{thm:SAM_SDE}. Additionally, let us assume that the dynamics is near the minimizer. In this case, the noise structure is such that the stochastic gradient can be written as $\nabla f_{\gamma}(x) = \nabla f(x) + Z$ such that $Z$ is the noise that does not depend on $x$. In this case, the SDE \eqref{eq:SAM_SDE} becomes

\begin{equation}\label{eq:SAM_SDE_Simpler}
dX_t  = -\nabla \Tilde{f}^{\text{SAM}}(X_t) d t +  \sqrt{\eta\left( \Sigma^{\text{SGD}}(X_t)+ \rho H_t \left( \Bar{\Sigma}(X_t) + \Bar{\Sigma}(X_t) ^{\top} \right) \right)}dW_t \nonumber
\end{equation}
%where $\Sigma^{\text{SGD}}(x)$ is given by Eq. \eqref{eq:SGD_Covariance_Insights} 
where $H_t := \nabla^2 f(X_t)$ and $\Bar{\Sigma}(x)$ is defined as

\begin{equation} \label{eq:SAM_sigma_bar}
\mathbb{E} \left[ \left( \nabla f\left(x\right) - \nabla f_{\gamma}\left(x\right) \right) \left( \E \left[\frac{\nabla f_{\gamma}(x) }{\lVert \nabla f_{\gamma}(x) \rVert_2} \right]- \frac{\nabla f_{\gamma}(x) }{\lVert \nabla f_{\gamma}(x) \rVert_2} \right)^{\top} \right], \nonumber
\end{equation}
and $\Tilde{f}^{\text{SAM}}(x):= f(x) + \rho \mathbb{E} \left[ \lVert \nabla f_{\gamma}(x) \rVert_2\right].$
\end{corollary}
\end{mybox}

\begin{proof}[Proof of Corollary \ref{thm:SAM_SDE_Simplified}]
It follows immediately by substituting the expression for the perturbed gradient.

\end{proof}

\subsubsection{SAM is SGD if \texorpdfstring{$\rho = \mathcal{O}(\eta)$}{Lg} }
%%%%%%%%%%%%%%%%%%%%%%%%%%%%%%%%%%%%%%%%%%%%%%%%%%%%%%%%%%%%%%%%%%%%%%%%%%%%

The following result is inspired by Theorem 1 of \cite{li2017stochastic}. We will consider the stochastic process $ X_t \in \mathbb{R}^d $ defined as the solution of 
the SDE
\begin{equation}\label{eq:SAM_SGD_SDE}
d X_t=-\nabla f\left(X_t\right) d t+\left(\eta \Sigma^{\text{SGD}}\left(X_t\right)\right)^{1 / 2} d W_t
\end{equation}

Such that $X_0=x_0$ and $$\Sigma^{\text{SGD}}(x):=\mathbb{E}\left[\left(\nabla f \left(x\right)-\nabla f_{\gamma}\left(x\right)\right)\left(\nabla f \left(x\right)-\nabla f_{\gamma}\left(x \right)\right)^T\right]
$$

\begin{mybox}{gray}
\begin{theorem}[Stochastic modified equations] \label{thm:SAM_SGD}
Let $0<\eta<1, T>0$ and set $N=\lfloor T / \eta\rfloor$. Let $ x_k \in \mathbb{R}^d, 0 \leq k \leq N$ denote a sequence of SAM iterations defined by Eq.~\eqref{eq:SAM_Discr_Update}. Additionally, let us take 

\begin{equation}\label{eq:SAM_SGD_rho}
\rho = \mathcal{O}\left(\eta^{1}\right).
\end{equation}
Consider the stochastic process $X_t$ defined in Eq.~\eqref{eq:SAM_SGD_SDE} and fix some test function $g \in G$ and suppose that $g$ and its partial derivatives up to order 6 belong to $G$.
Then, under Assumption \ref{ass:regularity_f}, there exists a constant $ C>0 $ independent of $ \eta $ such that for all $ k=0,1, \ldots, N $, we have

$$
\left|\mathbb{E} g\left(X_{k \eta}\right)-\mathbb{E} g\left(x_k\right)\right| \leq C \eta^1 .
$$

That is, the SDE \eqref{eq:SAM_SGD_SDE} is an order $ 1 $ weak approximation of the SAM iterations \eqref{eq:SAM_Discr_Update}.
\end{theorem}
\end{mybox}

\begin{lemma}\label{lemma:SAM_SGD}
Under the assumptions of Theorem \ref{thm:SAM_SGD}, let $ 0<\eta<1 $. Consider $ x_k, k \geq 0 $ satisfying the SAM iterations
$$
x_{k+1}=x_k-\eta \nabla f_{\gamma_k}\left(x_k + \rho \frac{\nabla f_{\gamma_k}(x_k)}{\lVert \nabla f_{\gamma_k}(x_k) \rVert}\right)
$$
with $ x_0=x \in \mathbb{R}^d $. From the definition the one-step difference $ \bar{\Delta}=x_1-x $, then we have

\begin{enumerate}
\item $ \mathbb{E} \bar{\Delta}_{i}=-\partial_{e_i}f(x)\eta+ \mathcal{O}(\eta^2) \quad \forall i = 1, \ldots,d$.
\item $ \mathbb{E} \bar{\Delta}_{i} \bar{\Delta}_{j}=\partial_{e_i} f \partial_{e_j} f \eta^2+\Sigma^{\text{SGD}}_{(i j)} \eta^2 + \mathcal{O}\left(\eta^3 \right)\quad \forall i,j = 1, \ldots,d$. 
\item $\mathbb{E} \prod_{j=1}^s \bar{\Delta}_{i_j} =\mathcal{O}\left(\eta^3\right) \quad \forall s \geq 3, \quad i_j \in \{ 1, \ldots, d\}.$
\end{enumerate}
All functions above are evaluated at $ x $.
\end{lemma}

\begin{proof}[Proof of Lemma \ref{lemma:SAM_SGD}]
First of all, we write that

\begin{equation}\label{eq:SAM_SGD_Estim}
 \partial_{e_i} f_{\gamma}\left(x + \rho \frac{\nabla f_{\gamma}(x)}{\lVert \nabla f_{\gamma}(x) \rVert}\right)= \partial_{e_i} f_{\gamma}(x) + \mathcal{R}^{\partial_{e_i} f_{\gamma}(x)}_{x,0}\left(\rho \frac{\nabla f_{\gamma}(x)}{\lVert \nabla f_{\gamma}(x) \rVert}\right),
\end{equation}
where the residual is defined in Eq.~(4) of \cite{folland2005higher}. Therefore, for some constant $c\in (0,1)$, it holds that

\begin{equation}\label{eq:SAM_SGD_Taylor}
\mathcal{R}^{\partial_{e_i} f_{\gamma}(x)}_{x,0}\left(\rho \frac{\nabla f_{\gamma}(x)}{\lVert \nabla f_{\gamma}(x) \rVert}\right) = \sum_{\vert \alpha \vert=1} \frac{\partial_{e_i+\alpha}^{2} f_{\gamma}\left(x + c \rho \frac{\nabla f_{\gamma}(x)}{\lVert \nabla f_{\gamma}(x) \rVert}\right) \rho^{1} \left( \frac{\nabla f_{\gamma}(x)}{\lVert \nabla f_{\gamma}(x) \rVert} \right)^{\alpha}}{\alpha !}.
\end{equation}
Let us now observe that $\mathcal{R}^{\partial_{e_i} f_{\gamma}(x)}_{x,0}\left(\rho \frac{\nabla f_{\gamma}(x)}{\lVert \nabla f_{\gamma}(x) \rVert}\right)$ is a finite sum of products of functions in $G$ and that, therefore, it lies in $G$. Additionally, given its expression Eq.~\eqref{eq:SAM_SGD_Taylor}, we can factor out a common $\rho$ and have that $K(x) = \rho K_1(x)$ for some function $K_1(x) \in G$. Therefore, we rewrite Eq.~\eqref{eq:SAM_SGD_Estim} as

\begin{equation}\label{eq:SAM_SGD_Estim_Rewritten}
\partial_{e_i} f_{\gamma}\left(x + \rho \frac{\nabla f_{\gamma}(x)}{\lVert \nabla f_{\gamma}(x) \rVert}\right) = \partial_{e_i} f_{\gamma}(x) + \rho K_1(x).
\end{equation}
First of all, we notice that if we define $ \Bar{K}_1(x) = \mathbb{E} \left[ K_{1}(x) \right]$, also $\Bar{K}_1(x) \in G$. Therefore, it holds that
\begin{equation}\label{eq:SAM_SGD_Estim_Rewritten_2}
\mathbb{E} \left[ \partial_{e_i} f_{\gamma}\left(x + \rho \frac{\nabla f_{\gamma}(x)}{\lVert \nabla f_{\gamma}(x) \rVert}\right) \right] \overset{\eqref{eq:SAM_SGD_Estim_Rewritten}}{=} \partial_{e_i} f(x) + \rho \Bar{K}_1(x)
\end{equation}
Therefore, using assumption \eqref{eq:SAM_SGD_rho}, $\forall i = 1, \ldots,d$, we have that

\begin{equation}
\mathbb{E} \bar{\Delta}_{i} = -\partial_{e_i}f(x)\eta+ \eta \rho \Bar{K}_i(x)= -\partial_{e_i} f(x)\eta + \mathcal{O}\left(\eta^2\right)
\end{equation}
Additionally, by keeping in mind the definition of the covariance matrix $\Sigma$, We immediately have 

\begin{align}
 \mathbb{E} \bar{\Delta}_{i} \bar{\Delta}_{j} \overset{\eqref{eq:SAM_SGD_Estim_Rewritten}}{=} & Cov(\bar{\Delta}_{i}, \bar{\Delta}_{j}) + \mathbb{E} \bar{\Delta}_{i} \mathbb{E}\bar{\Delta}_{j} \nonumber \\
 & = \Sigma^{\text{SGD}}_{(i j)} \eta^2 + \partial_{e_i} f \partial_{e_j} f \eta^2+ \eta^2 \rho ( \partial_{e_i} f \Bar{K}_j(x) + \partial_{e_j} f \Bar{K}_i(x)) + \eta^2 \rho^2 \Bar{K}_i(x) \Bar{K}_j(x) \nonumber \\
 & = \Sigma^{\text{SGD}}_{(i j)} \eta^2 + \partial_{e_i} f \partial_{e_j} f \eta^2+ \mathcal{O}\left(\eta^2 \rho\right) + \mathcal{O}\left(\eta^2 \rho^2 \right) \nonumber \\
 & = \partial_{e_i} f \partial_{e_j} f \eta^2+\Sigma^{\text{SGD}}_{(i j)} \eta^2 + \mathcal{O}\left(\eta^3\right) \quad \forall i,j = 1, \ldots,d
\end{align}
Finally, with analogous considerations, it is obvious that under our assumptions

$$
\mathbb{E} \prod_{j=1}^s \bar{\Delta}_{i_j}= \mathcal{O}\left(\eta^3\right) \quad \forall s \geq 3, \quad i_j \in \{1, \ldots, d\}.
$$
\end{proof}

\begin{proof}[Proof of Theorem \ref{thm:SAM_SGD}] \label{proof:SAM_SGD}
To prove this result, all we need to do is check the conditions in Theorem \ref{thm:mils}. As we apply Lemma \ref{lemma:li1}, we make the following choices:

\begin{itemize}
\item $b(x)=-\nabla f\left(x\right)$,
\item $\sigma(x) =\Sigma^{\text{SGD}}(X_t)^{\frac{1}{2}}$;
\end{itemize}
First of all, we notice that $\forall i = 1, \ldots, d$, it holds that

\begin{itemize}
\item $\mathbb{E} \bar{\Delta}_{i} \overset{\text{1. Lemma \ref{lemma:SAM_SGD}}}{=}-\partial_{e_i} f (x )\eta+\mathcal{O}\left(\eta^2\right)$;
\item $ \mathbb{E} \Delta_{i} \overset{\text{1. Lemma \ref{lemma:li1}}}{=}-\partial_{e_i} f (x )\eta+\mathcal{O}\left(\eta^2\right)$.
\end{itemize}
Therefore, we have that for some $K_1(x) \in G$

\begin{equation}\label{eq:SAM_SGD_cond1}
\left|\mathbb{E} \Delta_{i}-\mathbb{E} \bar{\Delta}_{i}\right| \leq K_1(x) \eta^{2}, \quad \forall i = 1, \ldots, d.
\end{equation}
Additionally,we notice that $\forall i,j = 1, \ldots, d$, it holds that

\begin{itemize}
\item $ \mathbb{E} \bar{\Delta}_{i} \bar{\Delta}_{j} \overset{\text{2. Lemma \ref{lemma:SAM_SGD}}}{=}\partial_{e_i} f \partial_{e_j} f \eta^2+\Sigma^{\text{SGD}}_{(i j)} \eta^2 + \mathcal{O}\left(\eta^3\right)$;
\item $ \mathbb{E} \Delta_{i} \Delta_{j} \overset{\text{2. Lemma \ref{lemma:li1}}}{=}\partial_{e_i} f \partial_{e_j} f \eta^2+\Sigma^{\text{SGD}}_{(i j)} \eta^2 + \mathcal{O}\left(\eta^3\right)$.
\end{itemize}
Therefore, we have that for some $K_2(x) \in G$

\begin{equation}\label{eq:SAM_SGD_cond2}
\left|\mathbb{E} \Delta_{i} \Delta_{j} - \mathbb{E} \bar{\Delta}_{i} \bar{\Delta}_{j}\right| \leq K_2(x) \eta^{2}, \quad \forall i,j = 1, \ldots, d
\end{equation}
Additionally, we notice that $\forall s \geq 3, \forall i_j \in \{1, \ldots,d \}$, it holds that

\begin{itemize}
\item $ \mathbb{E} \prod_{j=1}^s \bar{\Delta}_{i_j}\overset{\text{3. Lemma \ref{lemma:SAM_SGD}}}{=}\mathcal{O}\left(\eta^3\right)$;
\item $ \mathbb{E} \prod_{j=1}^s \Delta_{i_j}\overset{\text{3. Lemma \ref{lemma:li1}}}{=}\mathcal{O}\left(\eta^3\right)$.
\end{itemize}
Therefore, we have that for some $K_3(x) \in G$

\begin{equation}\label{eq:SAM_SGD_cond3}
\left|\mathbb{E} \prod_{j=1}^s \Delta_{i_j}-\mathbb{E} \prod_{j=1}^s \bar{\Delta}_{i_j}\right| \leq K_3(x) \eta^{2}.
\end{equation}
Additionally, for some $K_4(x) \in G$, $\forall i_j \in \{1, \ldots,d \}$

\begin{equation} \label{eq:SAM_SGD_cond4}
\mathbb{E} \prod_{j=1}^{ 3}\left|\bar{\Delta}_{\left(i_j\right)}\right| \overset{\text{3. Lemma \ref{lemma:SAM_SGD}}}{\leq}K_4(x) \eta^{2}.
\end{equation}
Finally, Eq.~\eqref{eq:SAM_SGD_cond1}, Eq.~\eqref{eq:SAM_SGD_cond2}, Eq.~\eqref{eq:SAM_SGD_cond3}, and Eq.~\eqref{eq:SAM_SGD_cond4} allow us to conclude the proof.

\end{proof}

%%%%%%%%%%%%%%%%%%%%%%%%%%%%%%%%%%%%%%%%%%%%%%%%%%%%%%%%%%%%%%%%%%%%%%%%%%%%
\section{Random SAM}
%%%%%%%%%%%%%%%%%%%%%%%%%%%%%%%%%%%%%%%%%%%%%%%%%%%%%%%%%%%%%%%%%%%%%%%%%%%%

Following \cite{ujvary2022rethinking} (Algorithm 2), we define Random SAM (RSAM) as the following discrete algorithm

\begin{equation}\label{eq:RSAM_Discr_Update}
x_{k+1}=x_k-\eta \mathbb{E}_{\epsilon \sim \mathcal{N}(0, \Sigma)} \nabla f \gamma_k\left(x_k+\epsilon\right).
\end{equation}
As a first attempt, we focus on the case where $\Sigma = \sigma^2 I_d$.

\subsection{Formal Derivation - RSAM} \label{sec:formal_RSAM}
%%%%%%%%%%%%%%%%%%%%%%%%%%%%%%%%%%%%%%%%%%%%%%%%%%%%%%%%%%%%%%%%%%%%%%%%%%%%
We will consider the stochastic process $ X_t \in \mathbb{R}^d $ defined by
\begin{equation}\label{eq:RSAM_SDE}
dX_t = -\nabla \Tilde{f}^{\text{RSAM}}(X_t) d t + \sqrt{\eta}\left( \Sigma^{\text{SGD}}(X_t)+ \frac{\sigma^2}{2} \left( \tilde{\Sigma}(X_t) + \tilde{\Sigma}(X_t) ^{\top} \right) \right)^{\frac{1}{2}}dW_t
\end{equation}
where $$\Sigma^{\text{SGD}}(x):=\mathbb{E}\left[\left(\nabla f \left(x\right)-\nabla f_{\gamma}\left(x\right)\right)\left(\nabla f \left(x\right)-\nabla f_{\gamma}\left(x \right)\right)^T\right]
$$
is the usual covariance of SGD, while

\begin{equation} \label{eq:RSAM_sigma_star}
\tilde{\Sigma}(x) :=\mathbb{E} \left[ \left( \nabla f\left(x\right) - \nabla f_{\gamma}\left(x\right) \right)\left(\mathbb{E} \left[ \nabla^3 f_{\gamma}(x) [I_d] \right] - \nabla^3 f_{\gamma}(x) [I_d] \right)^{\top} \right]
\end{equation}
and $$\Tilde{f}^{\text{RSAM}}(x):= f(x) + \frac{\sigma^2}{2} Tr(\nabla^2 f(x)).$$
In the following, we will use the notation

\begin{equation}\label{eq:RSAM_Covariance}
\Sigma^{\text{RSAM}}(x):= \left( \Sigma^{\text{SGD}}(X_t)+ \frac{\sigma^2}{2}  \left( \tilde{\Sigma}(X_t) + \tilde{\Sigma}(X_t) ^{\top} \right) \right).
\end{equation}

\begin{mybox}{gray}
\begin{theorem}[Stochastic modified equations] \label{thm:RSAM_SDE}
Let $0<\eta<1, T>0$ and set $N=\lfloor T / \eta\rfloor$. Let $ x_k \in \mathbb{R}^d, 0 \leq k \leq N$ denote a sequence of RSAM iterations defined by Eq.~\eqref{eq:RSAM_Discr_Update}. Additionally, let us take 
\begin{equation}\label{eq:RSAM_rho_theta_half}
\sigma = \mathcal{O}\left(\eta^{\frac{1}{3}}\right).
\end{equation}

Consider the stochastic process $X_t$ defined in Eq.~\eqref{eq:RSAM_SDE} and fix some test function $g \in G$ and suppose that $g$ and its partial derivatives up to order 6 belong to $G$.

Then, under Assumption~\ref{ass:regularity_f}, there exists a constant $ C>0 $ independent of $ \eta $ such that for all $ k=0,1, \ldots, N $, we have

$$
\left|\mathbb{E} g\left(X_{k \eta}\right)-\mathbb{E} g\left(x_k\right)\right| \leq C \eta^1 .
$$

That is, the SDE \eqref{eq:RSAM_SDE} is an order $ 1 $ weak approximation of the RSAM iterations \eqref{eq:RSAM_Discr_Update}.
\end{theorem}
\end{mybox}

\begin{mybox}{gray}
\begin{lemma} \label{lemma:RSAM_SDE}
Under the assumptions of Theorem \ref{thm:RSAM_SDE}, let $ 0<\eta<1 $ and consider $ x_k, k \geq 0 $ satisfying the RSAM iterations \eqref{eq:RSAM_Discr_Update}
$$
x_{k+1}=x_k-\eta \mathbb{E}_{\epsilon \sim \mathcal{N}(0, \Sigma)} \nabla f \gamma_k\left(x_k+\epsilon\right).
$$
where $\Sigma = \sigma^2 I_d$ and $ x_0=x \in \mathbb{R}^d $. From the definition the one-step difference $ \bar{\Delta}=x_1-x $, then we have

\begin{enumerate}
\item $ \mathbb{E} \bar{\Delta}_{i}=-\partial_{e_i} \Tilde{f}^{\text{RSAM}}(x) \eta +\mathcal{O}(\eta \sigma^3) \quad \forall i = 1, \ldots,d$.
\item $ \mathbb{E} \bar{\Delta}_{i} \bar{\Delta}_{j}=\partial_{e_i} \Tilde{f}^{\text{RSAM}}(x) \partial_{e_j} \Tilde{f}^{\text{RSAM}}(x) \eta^2 + \Sigma^{\text{RSAM}}_{(i j)} \eta^2 + \mathcal{O}\left(\eta^3 \right)\quad \forall i,j = 1, \ldots,d$. 
\item $\mathbb{E} \prod_{j=1}^s \Bar{\Delta}_{i_j} =\mathcal{O}\left(\eta^3\right) \quad \forall s \geq 3, \quad i_j \in \{ 1, \ldots, d\}.$
\end{enumerate}
and all the functions above are evaluated at $ x $.
\end{lemma}
\end{mybox}

\begin{proof}[Proof of Lemma \ref{lemma:RSAM_SDE}]

We perform a Taylor expansion of $\partial_i f(\cdot)$ around $x_k$
$$
\begin{aligned}
x_{k+1}^i=x_k^i - \mathbb{E}_{\epsilon \sim \mathcal{N}(0, \Sigma)}\left[ \eta \partial_i f\left(x_k\right)-\eta \sum_j \partial_{i j}^2 f\left(x_k\right) \epsilon_k^j - \frac{\eta}{2} \sum_{j, l} \partial_{i j l}^3 f\left(x_k\right) \epsilon_k^j \epsilon_k^l+\mathcal{O}\left(\eta\left\|\epsilon_k\right\|^3\right) \right],
\end{aligned}
$$
and we notice that the term $\frac{\eta}{2} \sum_{j, l} \partial_{i j l}^3 f\left(x_k\right) \epsilon_k^j \epsilon_k^l$ is equal to $\frac{\eta}{2} \partial_i \sum_{j l} \partial_{j l}^2 f\left(x_k\right) \epsilon_k^j \epsilon_k^l$ due to Clairaut's theorem (assuming that $f$ has continuous fourth-order partial derivatives). By exploiting that $\epsilon_k$ has mean zero and covariance $\sigma^2 I_d$, we have that
$$
\mathbb{E}\left[x_{k+1} - x_k\right]=-\eta \nabla \tilde{f}^{\text{RSAM}}\left(x_k\right)+\mathcal{O}\left(\eta \mathbb{E}\left[\left\|\epsilon_k\right\|^3\right]\right) = -\eta \nabla \tilde{f}^{\text{RSAM}}\left(x_k\right)+\mathcal{O}\left(\eta\sigma^3\right),
$$
where the modified loss $\Tilde{f}^{\text{RSAM}}$ is given by
$$
\Tilde{f}^{\text{RSAM}}(x):=f(z)+\frac{\sigma^2}{2} \operatorname{Tr}\left(\nabla^2 f(x)\right).
$$
Therefore, using \eqref{eq:RSAM_rho_theta_half}, we have that $\forall i = 1, \ldots,d$

\begin{equation} \label{eq:RSAM_Estim_Rewritten_3}
\mathbb{E} \bar{\Delta}_{i} = -\partial_{e_i} \Tilde{f}^{\text{RSAM}}(x) \eta + \mathcal{O}\left(\eta^2 \right).
\end{equation}
Additionally, we have that

\begin{align}
 \mathbb{E} \bar{\Delta}_{i} \bar{\Delta}_{j} = \partial_{e_i} \Tilde{f}^{\text{RSAM}} \partial_{e_j} \Tilde{f}^{\text{RSAM}} \eta^2+\text{Cov}(\bar{\Delta}_{i}, \bar{\Delta}_{j}) + \mathcal{O}\left(\eta^3\right)\quad \forall i,j = 1, \ldots,d. \label{eq:RSAM_Cov_final_step}
\end{align}
Let us now recall the expression \eqref{eq:RSAM_sigma_star} of $\tilde{\Sigma}$ and the expression \eqref{eq:RSAM_Covariance} of $\Sigma^{\text{RSAM}}$. Then, we automatically have that

\begin{equation}
 \text{Cov}(\bar{\Delta}_{i}, \bar{\Delta}_{j}) =\eta^2 \left( \Sigma^{\text{SGD}}_{i,j}(x) + \frac{\sigma^2}{2} \left[ \tilde{\Sigma}_{i,j}(x)+ \tilde{\Sigma}_{i,j}(x)^{\top} \right] + \mathcal{O}(\sigma^3) \right) = \eta^2 \Sigma^{\text{RSAM}}_{i,j}(x) +\mathcal{O}(\eta^2 \sigma^3)
\end{equation}
Therefore, remembering Eq. \eqref{eq:RSAM_Cov_final_step} and Eq. \eqref{eq:RSAM_rho_theta_half} we have 

\begin{equation}
\mathbb{E} \bar{\Delta}_{i} \bar{\Delta}_{j} =\partial_{e_i} \Tilde{f}^{\text{RSAM}} \partial_{e_j} \Tilde{f}^{\text{RSAM}} \eta^2+ \Sigma^{\text{SAM}}_{i,j} \eta^2 + \mathcal{O}\left(\eta^3 \right), \quad \forall i,j = 1, \ldots,d
\end{equation}
Finally, with analogous considerations, it is obvious that under our assumptions

$$
\mathbb{E} \prod_{j=1}^s \bar{\Delta}_{i_j}= \mathcal{O}\left(\eta^s\right) \quad \forall s \geq 3, \quad i_j \in \{1, \ldots, d\}
$$
which in particular implies that
$$
\mathbb{E} \prod_{j=1}^3 \bar{\Delta}_{i_j}= \mathcal{O}\left(\eta^3\right), \quad i_j \in \{1, \ldots, d\}.
$$

\end{proof}

\begin{proof}[Proof of Theorem \ref{thm:RSAM_SDE}] 
\label{proof:RSAM_SDE}
To prove this result, all we need to do is check the conditions in Theorem \ref{thm:mils}. As we apply Lemma \ref{lemma:li1}, we make the following choices:

\begin{itemize}
\item $b(x)=-\nabla \Tilde{f}^{\text{RSAM}}\left(x\right)$;
\item $\sigma(x) =\Sigma^{\text{RSAM}}(x)^{\frac{1}{2}}$.
\end{itemize}
First of all, we notice that $\forall i = 1, \ldots, d$, it holds that

\begin{itemize}
\item $\mathbb{E} \bar{\Delta}_{i} \overset{\text{1. Lemma \ref{lemma:RSAM_SDE}}}{=}-\partial_{e_i} \Tilde{f}^{\text{RSAM}}(x)\eta+ \mathcal{O}(\eta^2)$;
\item $ \mathbb{E} \Delta_{i} \overset{\text{1. Lemma \ref{lemma:li1}}}{=} -\partial_{e_i} \Tilde{f}^{\text{RSAM}}(x)\eta +\mathcal{O}\left(\eta^2\right)$.
\end{itemize}
Therefore, we have that for some $K_1(x) \in G$

\begin{equation}\label{eq:RSAM_cond1}
\left|\mathbb{E} \Delta_{i}-\mathbb{E} \bar{\Delta}_{i}\right| \leq K_1(x) \eta^{2}, \quad \forall i = 1, \ldots,d.
\end{equation}
Additionally,we notice that $\forall i,j = 1, \ldots, d$, it holds that

\begin{itemize}
\item $ \mathbb{E} \bar{\Delta}_{i} \bar{\Delta}_{j} \overset{\text{2. Lemma \ref{lemma:RSAM_SDE}}}{=}\partial_{e_i} \Tilde{f}^{\text{RSAM}} \partial_{e_j} \Tilde{f}^{\text{RSAM}} \eta^2+ \Sigma^{\text{RSAM}}_{i,j} \eta^2 + \mathcal{O}\left(\eta^3 \right)$;
\item $ \mathbb{E} \Delta_{i} \Delta_{j} \overset{\text{2. Lemma \ref{lemma:li1}}}{=}\partial_{e_i} \Tilde{f}^{\text{RSAM}} \partial_{e_j} \Tilde{f}^{\text{RSAM}} \eta^2+ \Sigma^{\text{RSAM}}_{i,j} \eta^2 + \mathcal{O}\left(\eta^3 \right)$.
\end{itemize}
Therefore, we have that for some $K_2(x) \in G$

\begin{equation}\label{eq:RSAM_cond2}
\left|\mathbb{E} \Delta_{i} \Delta_{j} - \mathbb{E} \bar{\Delta}_{i} \bar{\Delta}_{j}\right| \leq K_2(x) \eta^{2}, \quad \forall i,j = 1, \ldots, d
\end{equation}
Additionally, we notice that $\forall s \geq 3, \forall i_j \in \{1, \ldots, d \}$, it holds that

\begin{itemize}
\item $ \mathbb{E} \prod_{j=1}^s \bar{\Delta}_{i_j}\overset{\text{3. Lemma \ref{lemma:RSAM_SDE}}}{=}\mathcal{O}\left(\eta^3\right)$;
\item $ \mathbb{E} \prod_{j=1}^s \Delta_{i_j}\overset{\text{3. Lemma \ref{lemma:li1}}}{=}\mathcal{O}\left(\eta^3\right)$.
\end{itemize}
Therefore, we have that for some $K_3(x) \in G$

\begin{equation}\label{eq:RSAM_cond3}
\left|\mathbb{E} \prod_{j=1}^s \Delta_{i_j}-\mathbb{E} \prod_{j=1}^s \bar{\Delta}_{i_j}\right| \leq K_3(x) \eta^{2}.
\end{equation}
Additionally, for some $K_4(x) \in G$, $\forall i_j \in \{1, \ldots, d \}$

\begin{equation} \label{eq:RSAM_cond4}
\mathbb{E} \prod_{j=1}^{ 3}\left|\bar{\Delta}_{\left(i_j\right)}\right| \overset{\text{3. Lemma \ref{lemma:RSAM_SDE}}}{\leq}K_4(x) \eta^{2}.
\end{equation}
Finally, Eq.~\eqref{eq:RSAM_cond1}, Eq.~\eqref{eq:RSAM_cond2}, Eq.~\eqref{eq:RSAM_cond3}, and Eq.~\eqref{eq:RSAM_cond4} allow us to conclude the proof.

\end{proof}

%%%%%%%%%%%%%%%%%%%%%%%%%%%%%%%%%%%%%%%%%%%%%%%%%%%%%%%%%%%%%%%%%%%%%%%%%%%%
\section{Convergence Analysis: Quadratic Loss}
%%%%%%%%%%%%%%%%%%%%%%%%%%%%%%%%%%%%%%%%%%%%%%%%%%%%%%%%%%%%%%%%%%%%%%%%%%%%

\subsection{ODE USAM}
%%%%%%%%%%%%%%%%%%%%%%%%%%%%%%%%%%%%%%%%%%%%%%%%%%%%%%%%%%%%%%%%%%%%%%%%%%%%

Let us study the quadratic loss function $f(x) = x^{\top}H x$ where $H$ is a diagonal matrix of eigenvalues $(\lambda_1, \dots, \lambda_d)$ such that $\lambda_1 \geq \lambda_1 \geq \dots \geq \lambda_d$. Under the dynamics of the ODE of USAM, we have that

\begin{equation}
dX_t= -H \left( X_t + \rho H X_t\right) dt = -H \left( I_d + \rho H \right)X_tdt,
\end{equation}
which, for the single component gives us the following dynamics

\begin{equation}
dX^{j}_t = - \lambda_j(1 + \rho \lambda_j) X^{j}_tdt
\end{equation}
whose solution is

\begin{equation}
X^{j}_t = X^{j}_0 e^{-\lambda_j (1 + \rho \lambda_j)t}.
\end{equation}

\begin{mybox}{gray}
\begin{lemma}\label{lemma:USAM_Quad_PSD}
For all $\rho>0$, if all the eigenvalues of $H$ are positive, then
\begin{equation}
X^{j}_t \overset{t \rightarrow \infty}{\rightarrow} 0, \quad \forall j \in \{1, \dots, d \}
\end{equation}
\end{lemma}
\end{mybox}

\begin{proof}[Proof of Lemma \ref{lemma:USAM_Quad_PSD}]
For each $j \in \{1, \cdots, d \}$, we have that $$X^{j}_{t} = X^{j}_{0} e^{-\lambda_j(1+\rho\lambda_j)}.$$ Therefore, since the exponent is always negative, $X^{j}_{t} \rightarrow 0$ as $t \rightarrow \infty$.

\end{proof}

\begin{mybox}{gray}
\begin{lemma}\label{lemma:USAM_Quad_Ind}
Let $H$ have at least one strictly negative eigenvalue and let $\lambda_{*}$ be the largest negative eigenvalue of $H$. Then, for all $\rho > -\frac{1}{\lambda_{*}}$,
\begin{equation}
X^{j}_t \overset{t \rightarrow \infty}{\rightarrow} 0, \quad \forall j \in \{1, \dots, d \}.
\end{equation}
\end{lemma}
\end{mybox}

\begin{proof}[Proof of Lemma \ref{lemma:USAM_Quad_Ind}]
For each $j \in \{1, \cdots, d \}$, we have that $$X^{j}_{t} = X^{j}_{0} e^{-\lambda_j(1+\rho\lambda_j)}.$$ Therefore, if $\lambda_j>0$, the exponent is always negative for each value of $\rho>0$. Therefore, $X^{j}_{t} \rightarrow 0$ as $t \rightarrow \infty$. Differently, if $\lambda_j <0$, the exponent $-\lambda_j(1+\rho\lambda_j)$ is negative only if $\rho > -\frac{1}{\lambda_{*}}$ where $\lambda_{*}$ is the largest negative eigenvalue of $H$. Therefore, if$\rho > -\frac{1}{\lambda_{*}}$, $X^{j}_{0} \rightarrow 0$ if $t \rightarrow \infty$.

\end{proof}

\subsection{SDE USAM - Stationary Distribution}
%%%%%%%%%%%%%%%%%%%%%%%%%%%%%%%%%%%%%%%%%%%%%%%%%%%%%%%%%%%%%%%%%%%%%%%%%%%%

Let us consider the noisy quadratic model $f(x) = \frac{1}{2} x^\top H x$, where $H$ is a symmetric matrix. Then, based on Theorem \eqref{thm:USAM_SDE} in the case where $\Sigma(x) = \varsigma I_d$, the corresponding SDE is give by

\begin{equation} \label{eq:SDE_USAM_Quad}
 dX_t = -H\left(I_d + \rho H\right)X_t dt + \left[(I_d+\rho H)\sqrt{\eta}\varsigma \right] dW_t. 
\end{equation}

\begin{mybox}{gray}
\begin{theorem}[Stationary distribution - PSD Case.] \label{thm:USAM_stat_distr_quad_PSD}
For any $\rho>0$, the stationary distribution of Eq. \eqref{eq:SDE_USAM_Quad} is

\begin{align}
P \left(x, \infty \mid \rho \right) = \sqrt{\frac{\lambda_i}{\pi \eta \varsigma^2 } \frac{1}{1 + \rho \lambda_i}} \exp \left[-\frac{\lambda_i}{\eta \varsigma^2 } \frac{1}{1 + \rho \lambda_i} x^2\right]
\end{align}

where $(\lambda_1,\dots,\lambda_d)$ are the eigenvalues of $H$ and $\lambda_i>0, \forall i \in \{ 1, \cdots, d \}$.
\end{theorem}
\end{mybox}

More interestingly, if $\rho$ is too large, this same conclusion holds even for a saddle point.

\begin{mybox}{gray}
\begin{theorem}[Stationary distribution - Indefinite Case.] \label{thm:USAM_stat_distr_quad_Indef}
Let $(\lambda_1,\dots,\lambda_d)$ are the eigenvalues of $H$ such that there exists at least one which is strictly negative. If $\rho > -\frac{1}{\lambda_{*}}$ where $\lambda_{*}$ is the largest negative eigenvalue of H, then the stationary distribution of Eq. \eqref{eq:SDE_USAM_Quad} is

\begin{align}
P \left(x, \infty \mid \rho \right) = \sqrt{\frac{\lambda_i}{\pi \eta \varsigma^2 } \frac{1}{1 + \rho \lambda_i}} \exp \left[-\frac{\lambda_i}{\eta \varsigma^2 } \frac{1}{1 + \rho \lambda_i} x^2\right]
\end{align}
\end{theorem}
\end{mybox}

\begin{proof}[Proof of Theorem \ref{thm:USAM_stat_distr_quad_PSD}]
Note that Eq. \eqref{eq:SDE_USAM_Quad} is a linear SDE, and that drift and diffusion matrices are co-diagonalizable: Let $H = U \Lambda U^\top$ be one eigenvalue decomposition of $H$, with $\Lambda = \text{diag}(\lambda_1,\dots,\lambda_d)$. If we plug this in, we get $$dX_t = -U\left(\Lambda + \rho \Lambda^2\right)U^\top X_t dt + U\left[(I_d+\rho \Lambda)\sqrt{\eta}\varsigma \right]U^\top dW_t.$$
Let us multiply the LHS with $U^\top$, then
$$d(U^\top X_t) = -\left(\Lambda + \rho \Lambda^2\right)(U^\top X_t) dt + \left[(I_d+\rho \Lambda)\sqrt{\eta}\varsigma \right]U^\top dW_t.$$
Finally, note that $U^\top dW_t = dW_t$ in law, so we can write
$$d(U^\top X_t) = -\left(\Lambda + \rho \Lambda^2\right)(U^\top X_t) dt + \left[(I_d+\rho \Lambda)\sqrt{\eta}\varsigma \right] dW_t.$$
This means that the coordinates of the vector $Y =U^\top X$ evolve independently
$$dY_t = -\left(\Lambda + \rho \Lambda^2\right)Y_t dt + \left[(I_d+\rho \Lambda)\sqrt{\eta}\varsigma \right] dW_t,$$
since $\Lambda$ is diagonal. Therefore for the $i$-th component $Y_i$ we can write

\begin{equation}\label{eq:USAM_OU}
dY_{i,t} = -\left(\lambda_i + \rho \lambda_i^2\right)Y_{i,t} dt + \left[(1+\rho \lambda_i)\sqrt{\eta}\varsigma \right] dW_{i,t}.
\end{equation}

Note that this is a simple one-dimensional Ornstein–Uhlenbeck process $dY_t = -\theta Y_t dt + \sigma dW_t$~($\theta>0,\sigma \neq 0$) with parameters
\begin{equation}
\theta = \lambda_i \left( 1 + \rho \lambda_i\right)>0\quad \text{and} \quad \sigma = (1+\rho \lambda_i)\sqrt{\eta}\varsigma>0
\end{equation}
Therefore, from Section 4.4.4 of \citep{gardiner1985handbook}, we get that

\begin{equation}
\E[Y_t] = e^{-\theta t}Y_0,\quad \text{Var}(Y_t) = \frac{\sigma^2}{2\theta}(1 - e^{-2\theta t}).
\end{equation}
In our case we have that 
\begin{equation}
\E[Y_t] = e^{-\theta t}Y_0 \rightarrow 0 \quad \text{and} \quad\text{Var}(Y_t) = \frac{\sigma^2}{2\theta}(1 - e^{-2\theta t}) \rightarrow \frac{\sigma^2}{2\theta} = \frac{\eta \varsigma^2 }{2 \lambda_i}(1 + \rho \lambda_i).
\end{equation}
Additionally, using the Fokker–Planck equation, see Section 5.3 of \citep{risken1996fokker}, we have the following formula for the stationary distribution of each eigendirection. Indeed, let us recall that for $D := \frac{\sigma^2}{2}$, the probability density function is

\begin{equation}
P\left(x, t \mid x^{\prime}, t^{\prime}, \rho \right)=\sqrt{\frac{\theta}{2 \pi D\left(1-e^{-2 \theta\left(t-t^{\prime}\right)}\right)}} \exp \left[-\frac{\theta}{2 D} \frac{\left(x-x^{\prime} e^{-\theta\left(t-t^{\prime}\right)}\right)^2}{1-e^{-2 \theta\left(t-t^{\prime}\right)}}\right].
\end{equation}
Therefore, the stationary distribution is

\begin{align}
P \left(x, \infty \mid \rho \right) & =\sqrt{\frac{\theta}{2 \pi D}} \exp \left[-\frac{\theta }{2 D} x^2\right] \nonumber \\
& = \sqrt{\frac{\theta}{\pi \sigma^2}} \exp \left[-\frac{\theta }{\sigma^2} x^2\right] \nonumber \\
& = \sqrt{\frac{\lambda_i}{\pi \eta \varsigma^2 } \frac{1}{1 + \rho \lambda_i}} \exp \left[-\frac{\lambda_i}{\eta \varsigma^2 } \frac{1}{1 + \rho \lambda_i} x^2\right].
\end{align}
To conclude,
\begin{equation}
Y_{i,\infty}\sim\mathcal{N}\left(0,\frac{\eta \varsigma^2 }{\lambda_i}(1 + \rho \lambda_i)\right).
\end{equation}
Since all of the eigenvalues are positive, this distribution has more variance than SGD on each direction.
\end{proof} 
Since the proof of Theorem \ref{thm:USAM_stat_distr_quad_Indef} is perfectly similar to that of Theorem \ref{thm:USAM_stat_distr_quad_PSD}, we skip it. Additionally, a very analogous result holds true even if all the eigenvalues are strictly negative and thus the quadratic has a single maximum as a critical point. From these results, we understand that under certain circumstances, USAM might be attracted not only by the minimum but possibly also by a saddle or a maximum. This is fully consistent with the results derived for the ODE of USAM in Lemma \ref{lemma:USAM_Quad_PSD} and Lemma \ref{lemma:USAM_Quad_Ind}.

\begin{mybox}{gray}
\begin{observation}[Suboptimality under the Stationary Distribution -- comparison to SGD]\label{observation:subopt}

In the special case where the stochastic process has reached stationarity, one can approximate the loss landscape with a quadratic loss~\citep{jastrzkebski2017three}. By further assuming that $\Sigma^{SGD}(x) = H$ (see e.g.\cite{sagun2018empirical,zhu2018anisotropic}), Theorem \eqref{thm:USAM_SDE} implies that for USAM
\begin{equation} \label{eq:SDE_USAM_Quad_New}
 dX_t = -H\left(I_d + \rho H\right)X_t dt + \left[(I_d+\rho H)\sqrt{\eta} \sqrt{H} \right] dW_t. 
\end{equation}
Up to a change of variable, we assume $H$ to be diagonal and therefore
\begin{equation}\label{eq:USAM_Suboptimality}
    \E_{USAM} \left[f\right] = \frac{1}{2} \sum_{i=1}^{d} \lambda_i \E [X_i^2] = \frac{\eta}{4} \sum_{i=1}^{d} \lambda_i (1 + \rho \lambda_i)^2 = \frac{\eta}{4} \left( Tr(H) + 2 \rho Tr(H^2) + \rho^2 Tr(H^3)\right) \gg \E_{SGD} \left[f\right],
\end{equation}
where subscripts indicate that $f$ is being optimized with SGD and USAM, respectively.

Regarding DNSAM, Theorem \eqref{thm:DNSAM_SDE} implies that

\begin{equation}
dX_t = -H \left( I_d + \frac{\rho H}{\lVert H X_t \rVert} \right)X_t dt + \sqrt{\eta} \sqrt{H} \left( I_d + \frac{\rho H}{\lVert H X_t \rVert} \right)dW_t.
\end{equation}
Therefore, we argue that DNSAM has to have a suboptimality with respect to SGD which is even larger than that of USAM. Intuitively, when $\lVert H X_t \rVert<1$, the variance of DNSAM is larger than that of USAM. Therefore, its suboptimality has to be larger as well. Finally, the behavior of SAM is close to that of DNSAM, but less pronounced because the denominator can never get too close to $0$ due to the noise injection.

\end{observation}
\end{mybox}

\subsection{ODE SAM}
%%%%%%%%%%%%%%%%%%%%%%%%%%%%%%%%%%%%%%%%%%%%%%%%%%%%%%%%%%%%%%%%%%%%%%%%%%%%

W.l.o.g, we take $H$ to be diagonal and if it has negative eigenvalues, we denote the largest negative eigenvalue with $\lambda_{*}$. Let us recall that the ODE of SAM for the quadratic loss is given by

\begin{equation}\label{eq:SAM_ODE_Quad}
dX_t = -H \left( I_d + \frac{\rho H}{\lVert H X_t \rVert} \right)X_tdt
\end{equation}

\begin{mybox}{gray}
\begin{lemma} \label{lemma:SAM_ODE_Converg}
For all $\rho>0$, if $H$ is PSD, the origin is (locally) asymptotically stable. Additionally, if $H$ is not PSD, if $\lVert H X_t \rVert \leq-\rho \lambda_{*}$, then the origin is still (locally) asymptotically stable.
\end{lemma}
\end{mybox}

\begin{proof}[Proof of Lemma \ref{lemma:SAM_ODE_Converg}]
Let $V(x):= \frac{x^{\top} K x}{2}$ be the Lyapunov function, where $K$ is a diagonal matrix with positive eigenvalues $\left(k_1, \cdots, k_d \right)$. Therefore, we have

\begin{equation}
V(X_t) = \frac{1}{2} \sum_{i=1}^{d} k_i \left(X_t^{i}\right)^2>0
\end{equation}
and
\begin{equation}
\Dot{V}(X_t) = \sum_{i=1}^{d} k_i X_t^{i}\Dot{X}_t^{i} = \sum_{i=1}^{d} k_i (-\lambda_i) \left( 1 + \frac{\rho \lambda_i}{\lVert H X_t \rVert} \right)X_t^{i}X_t^{i}dt = - \sum_{i=1}^{d} k_i \lambda_i \left( 1 + \frac{\rho \lambda_i}{\lVert H X_t \rVert} \right)\left(X_t^{i}\right)^2 dt.
\end{equation}
Let us analyze the terms $$k_i \lambda_i \left( 1 + \frac{\rho \lambda_i}{\lVert H X_t \rVert} \right)\left(X_t^{i}\right)^2.$$ When $\lambda_i>0$, these quantities are all positive and the proof is concluded. However, if there exists $\lambda_i<0$, these quantities are positive only if $\left( 1 + \frac{\rho \lambda_i}{\lVert H X_t \rVert} \right) \leq 0$, that is if $\lVert H X_t \rVert \leq -\rho \lambda_i$. Therefore, a sufficient condition for $\Dot{V}(X_t) \leq 0$ is that
$$\lVert H X_t \rVert\leq-\rho \lambda_i, \quad \forall i \text{ s.t. } \lambda_i<0.$$
Based on Theorem 1.1 of \cite{mao2007stochastic}, we conclude that if $\lVert H X_t \rVert \leq-\rho$ where $\lambda_{*}$ is the largest negative eigenvalue of H, $V(X_t)>0$ and $\Dot{V}(X_t) \leq 0$, and that therefore the dynamics of $X_t$ is bounded inside this compact set and cannot diverge.
\end{proof}
From this result, we understand that the dynamics of the ODE of USAM might converge to a saddle or even a maximum if it gets too close to it.

%%%%%%%%%%%%%%%%%%%%%%%%%%%%%%%%%%%%%%%%%%%%%%%%%%%%%%%%%%%%%%%%%%%%%%%%%%%%
\subsection{SDE DNSAM}
%%%%%%%%%%%%%%%%%%%%%%%%%%%%%%%%%%%%%%%%%%%%%%%%%%%%%%%%%%%%%%%%%%%%%%%%%%%%

W.l.o.g, we take $H$ to be diagonal and if it has negative eigenvalues, we denote the largest negative eigenvalue with $\lambda_{*}$. Based on Eq. \eqref{eq:SDE_SAM_Quad} and in the case where $\Sigma^{\text{SGD}}=\varsigma^2 I_d$, the SDE of DNSAM for the quadratic loss is given by

\begin{equation}\label{eq:SDE_SAM_Quad}
dX_t = -H \left( I_d + \frac{\rho H}{\lVert H X_t \rVert} \right)X_t dt + \sqrt{\eta} \varsigma \left( I_d + \frac{\rho H}{\lVert H X_t \rVert} \right)dW_t
\end{equation}

\begin{mybox}{gray}
\begin{observation}
\label{lemma:SAM_SDE_Converg}
For all $\rho>0$, there exists an $\epsilon>0$ such that if $\lVert H X_t \rVert \in \left( \epsilon, -\rho \lambda_{*} \right) $, the dynamics of $X_t$ is attracted towards the origin. If the eigenvalues are all positive, the condition is $\lVert H X_t \rVert \in \left( \epsilon, \infty \right)$. On the contrary, if $\lVert H X_t \rVert<\epsilon$, then the dynamics is pushed away from the origin. 
\end{observation}
\end{mybox}

\begin{proof}[Formal calculations to support Observation \ref{lemma:SAM_SDE_Converg}]
Let $V(t,x):= e^{-t} \frac{x^{\top} K x}{2}$ be the Lyapunov function, where $K$ is a diagonal matrix with strictly positive eigenvalues $\left(h_1, \cdots, h_d \right)$. Therefore, we have

\begin{equation}
V(X_t) = e^{-t} \frac{1}{2} \sum_{i=1}^{d} k_i \left(X_t^{i}\right)^2>0
\end{equation}
and
\begin{align}
LV(t, X_t) & = -e^{-t} \frac{1}{2} \sum_{i=1}^{d} k_i \left(X_t^{i}\right)^2 + e^{-t} \sum_{i=1}^{d} k_i (-\lambda_i) \left( 1 + \frac{\rho \lambda_i}{\lVert H X_t \rVert} \right)\left(X_t^{i}\right)^2 + e^{-t} \frac{\eta \varsigma^2}{2}\sum_{i=1}^{d} k_i\left( 1 + \frac{\rho \lambda_i}{\lVert H X_t \rVert} \right)^2 \nonumber \\
& =-e^{-t} \left( \frac{1}{2} \sum_{i=1}^{d} k_i \left(X_t^{i}\right)^2 + \sum_{i=1}^{d} k_i \lambda_i \left( 1 + \frac{\rho \lambda_i}{\lVert H X_t \rVert} \right)\left(X_t^{i}\right)^2 - \frac{\eta \varsigma^2}{2}\sum_{i=1}^{d} k_i\left( 1 + \frac{\rho \lambda_i}{\lVert H X_t \rVert} \right)^2\right)
\end{align}
Let us analyze the terms $$k_i \lambda_i \left( 1 + \frac{\rho \lambda_i}{\lVert H X_t \rVert} \right)\left(X_t^{i}\right)^2.$$ When $\lambda_i>0$, these quantities are all positive. When $\lambda_i<0$, these quantities are positive only if $\left( 1 + \frac{\rho \lambda_i}{\lVert H X_t \rVert} \right) \leq 0$, that is if $\lVert H X_t \rVert \leq -\rho \lambda_i$. Let us now assume that $$\lVert H X_t \rVert\leq-\rho \lambda_i, \quad \forall i \text{ s.t. } \lambda_i<0.$$
that is, $\lVert H X_t \rVert \leq-\rho \lambda_{*}$ such that $\lambda_{*}$. Then, we observe that 

\begin{itemize}
\item If $\lVert H X_t \rVert \rightarrow 0$, $LV(t, X_t)\geq 0$
\item If $\varsigma$ is small enough, for $\lVert H X_t \rVert \approx -\rho \lambda_{*}$, $LV(t, X_t)\leq 0$
\end{itemize}
Given that all functions and functionals involved are continuous, there exists $\epsilon>0$ such that 

\begin{itemize}
\item If $\lVert H X_t \rVert < \epsilon$, $LV(t, X_t)\geq 0$
\item If $\varsigma$ is small enough, for $\lVert H X_t \rVert \in (\epsilon, -\rho \lambda_{*})$, $LV(t, X_t)\leq 0$
\end{itemize}

\end{proof}
Based on Theorem 2.2 of \cite{mao2007stochastic}, we understand that if the dynamics is sufficiently close to the origin, it gets pulled towards it, but if gets too close, it gets repulsed from it. If there is no negative eigenvalue, the same happens but the dynamics can never get close to the minimum.

%%%%%%%%%%%%%%%%%%%%%%%%%%%%%%%%%%%%%%%%%%%%%%%%%%%%%%%%%%%%%%%%%%%%%%%%%%%%
\section{Experiments}\label{appendix:Exper}
%%%%%%%%%%%%%%%%%%%%%%%%%%%%%%%%%%%%%%%%%%%%%%%%%%%%%%%%%%%%%%%%%%%%%%%%%%%%

In this section, we provide additional details regarding the validation that the SDEs we proposed indeed weakly approximate the respective algorithms. We do so on a quadratic landscape, on a classification task with a deep linear model, a binary classification task with a deep nonlinear model, and a gregression task with a teacher-student model. Since our SDEs prescribe the calculation of the Hessian of the whole neural network at each iteration step, this precludes us from testing our theory on large-scale models.

\subsection{SDE Validation} \label{appendix:validation}

\paragraph{Quadratic}
In this paragraph, we provide the details of the Quadratic experiment. We optimize the loss function $f(x) = \frac{1}{2} x^{\top} H x$ of dimension $d=20$. The Hessian $H$ is a random SPD matrix generated using the standard Gaussian matrix $A \in \mathbb{R}^{d \times 2 d}$ as $H=A A^{\top} /(2 d)$. The noise used to perturb the gradients is $Z \sim \mathcal{N}(0, \Sigma)$ where $\Sigma = \sigma I_d$ and $\sigma = 0.01$. We use $\eta = 0.01$, $\rho \in \{0.001, 0.01, 0.1, 0.5\}$. The results are averaged over $3$ experiments.

\paragraph{Deep Linear Classification}
In this paragraph, we provide the details of the Deep Linear Classification experiment. This is a classification task on the Iris Database \cite{Dua:2019}. The model is a Linear MLP with 1 hidden layer with a width equal to the number of features and we optimize the cross-entropy loss function. The noise used to perturb the gradients is $Z \sim \mathcal{N}(0, \Sigma)$ where $\Sigma = \sigma I_d$ and $\sigma = 0.01$. We use $\eta = 0.01$, $\rho \in \{0.001, 0.01, 0.1, 0.2\}$. The results are averaged over $3$ experiments.

\paragraph{Deep Nonlinear Classification}
In this paragraph, we provide the details of the Deep Nonlinear Classification experiment. This is a binary classification task on the Breast Cancer Database \cite{Dua:2019}. The model is a Nonlinear MLP with 1 hidden layer with a width equal to the number of features, sigmoid activation function, and we optimize the $\ell^2$-regularized logistic loss  with parameter $\lambda = 0.1$. The noise used to perturb the gradients is $Z \sim \mathcal{N}(0, \Sigma)$ where $\Sigma = \sigma I_d$ and $\sigma = 0.01$. We use $\eta = 0.01$, $\rho \in \{0.001, 0.01, 0.1, 0.5\}$. The results are averaged over $3$ experiments.

\paragraph{Deep Teacher-Student Model}
In this paragraph, we provide the details of the Teacher-Student experiment. This is a regression task where the database is generated by the Teacher model based on random inputs in $\mathbb{R}^5$ and output in $\mathbb{R}$. The Teacher model is a deep linear MLP with $20$ hidden layers with $10$ nodes, while the Student is a deep nonlinear MLP with $20$ hidden layers and $10$ nodes. We optimize the MSE loss. The noise used to perturb the gradients is $Z \sim \mathcal{N}(0, \Sigma)$ where $\Sigma = \sigma I_d$ and $\sigma = 0.001$. We use $\eta = 0.001$, $\rho \in \{0.0001, 0.001, 0.03, 0.05\}$. The results are averaged over $3$ experiments.

\subsubsection{Importance of the Additional Noise.} \label{appendix:ImportanceCovariance}

In this section, we empirically test the importance of using the correct diffusion terms the USAM SDE Eq. \eqref{eq:USAM_SDE_Simplified_Insights} and the DNSAM SDE Eq. \eqref{eq:DNSAM_SDE_Insights}. Let us introduce two new SDEs where the diffusion term is the one of the SGD SDE in Eq. \eqref{eq:SGD_Equation_Insights} rather than that from the correct SDEs:

\begin{equation}\label{eq:USAM_SDE_NO_COV}
  dX_t = -\nabla \Tilde{f}^{\text{USAM}}(X_t) d t + \sqrt{\eta\left( \Sigma^{\text{SGD}}(X_t)\right)}dW_t, \quad \text{where} \quad  \Tilde{f}^{\text{USAM}}(x):= f(x) + \frac{\rho}{2} \lVert \nabla f(x) \rVert^2_2.
\end{equation}

\begin{equation}\label{eq:DNSAM_SDE_NO_COV}
  dX_t = -\nabla \Tilde{f}^{\text{DNSAM}}(X_t) d t + \sqrt{\eta\left( \Sigma^{\text{SGD}}(X_t)\right)}dW_t, \quad \text{where} \quad \Tilde{f}^{\text{DNSAM}}(x):= f(x) + \rho \lVert \nabla f(x) \rVert_2.
\end{equation}

In Figure \ref{fig:Verif_Rho_NO_COV_USAM} we observe how approximating USAM with the SGD SDE (Eq. \eqref{eq:SGD_Equation_Insights}) brings a large error in all four cases. Introducing the correct drift but excluding the correct covariance, i.e. using Eq. \eqref{eq:USAM_SDE_NO_COV}, reduces the error, but the best performer is the complete USAM SDE Eq. \eqref{eq:USAM_SDE_Insights}. From Figure \ref{fig:Verif_Rho_NO_COV_SAM}, the same observations hold for DNSAM.

From these experiments, we understand that naively adding noise to the ODEs of USAM and SAM does not provide SDEs with sufficient approximation power. This is consistent with our proofs.

\begin{figure}%
    \centering
    \subfloat{{\includegraphics[width=0.24\linewidth]{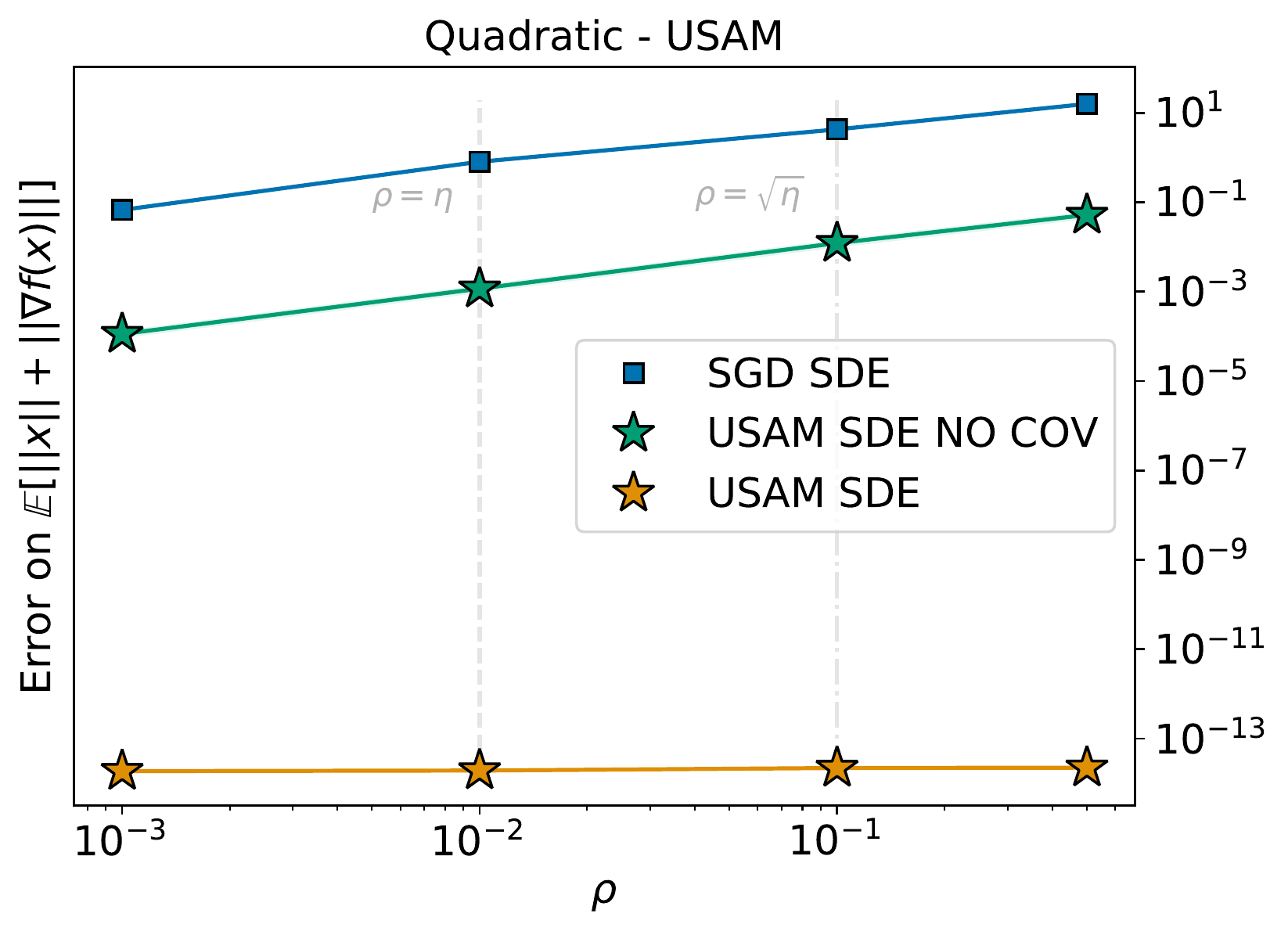} }}%
    \subfloat{{\includegraphics[width=0.24\linewidth]{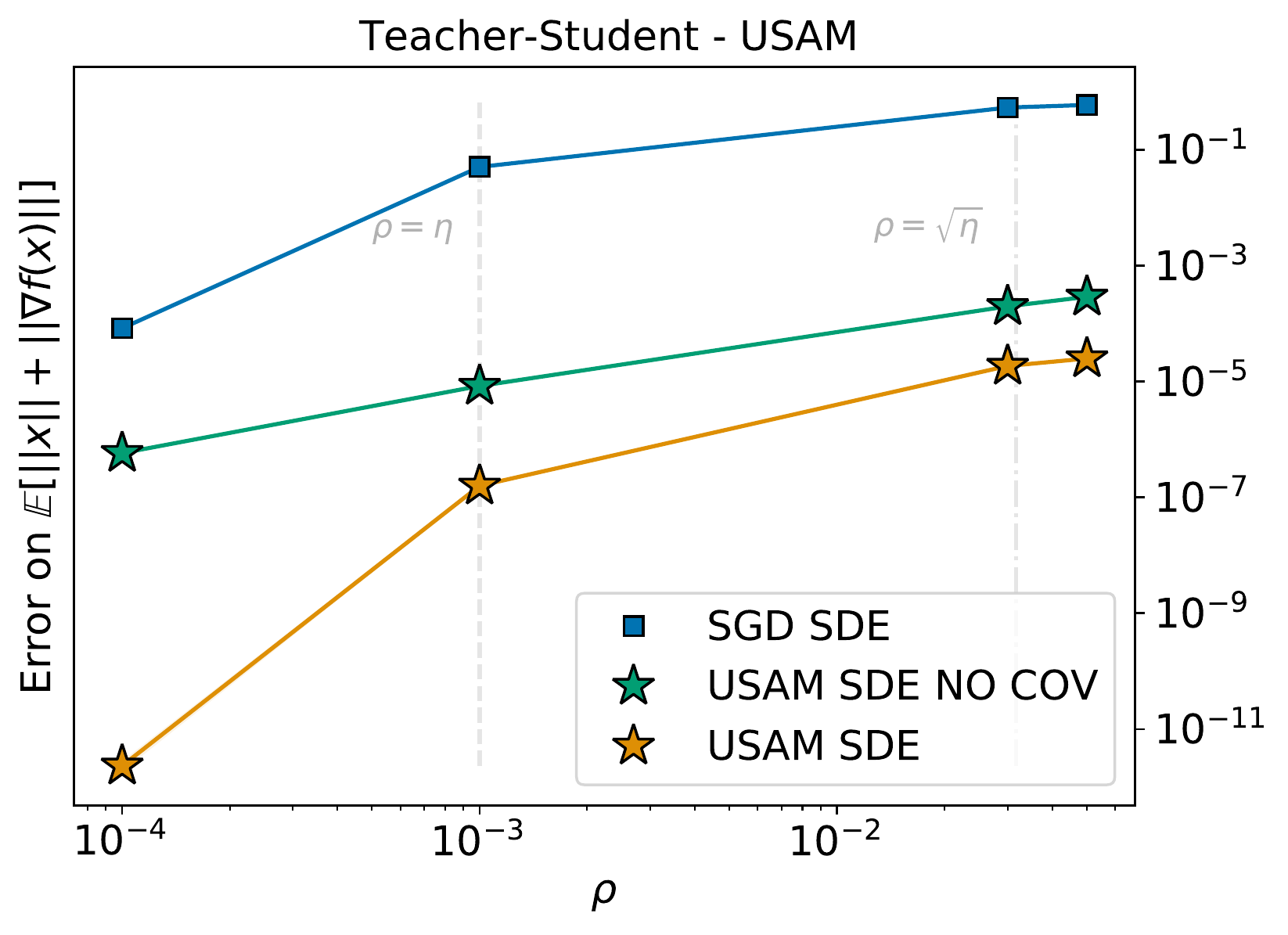} }}
    \subfloat{{\includegraphics[width=0.24\linewidth]{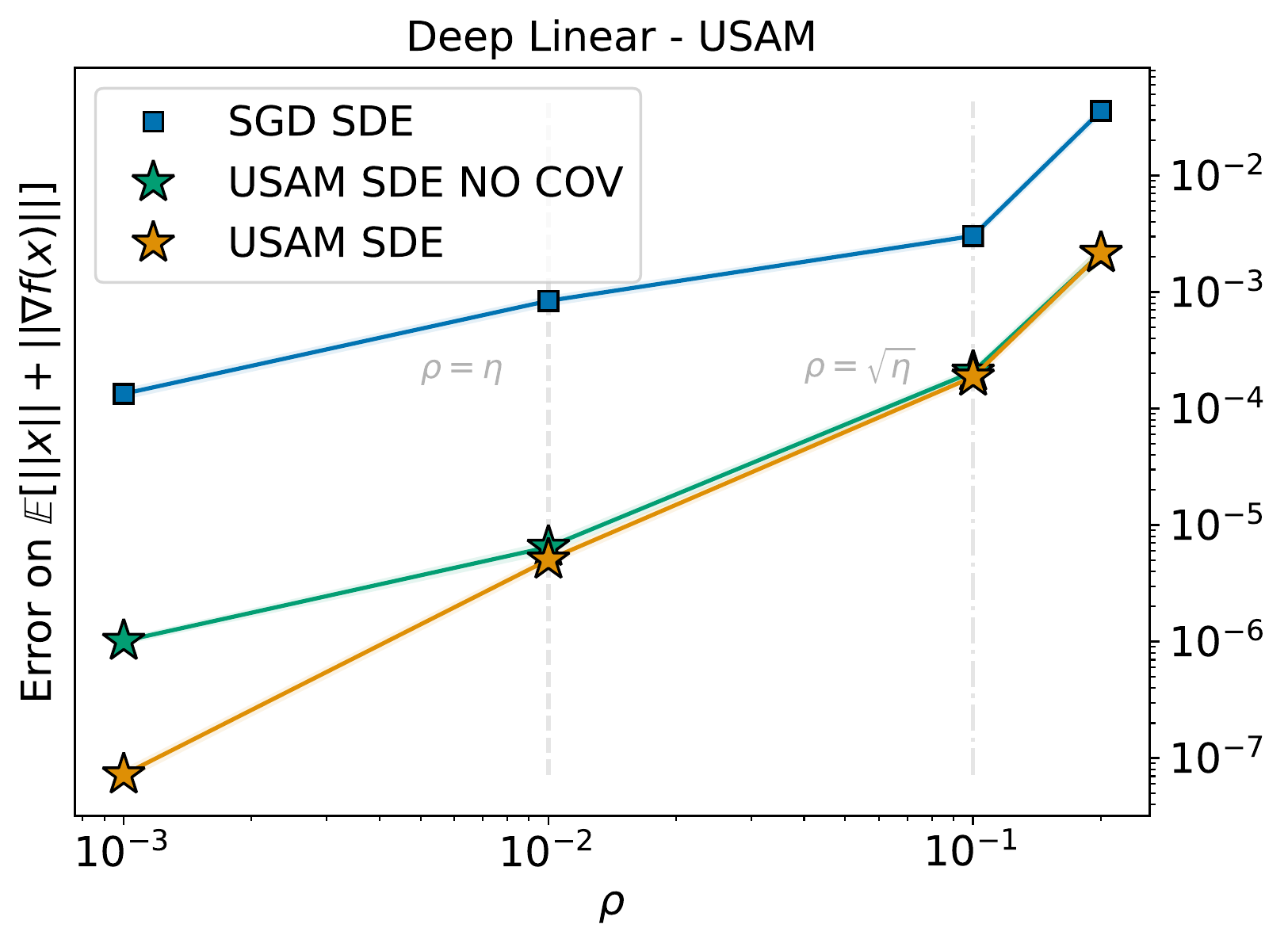} }}%
    \subfloat{{\includegraphics[width=0.24\linewidth]{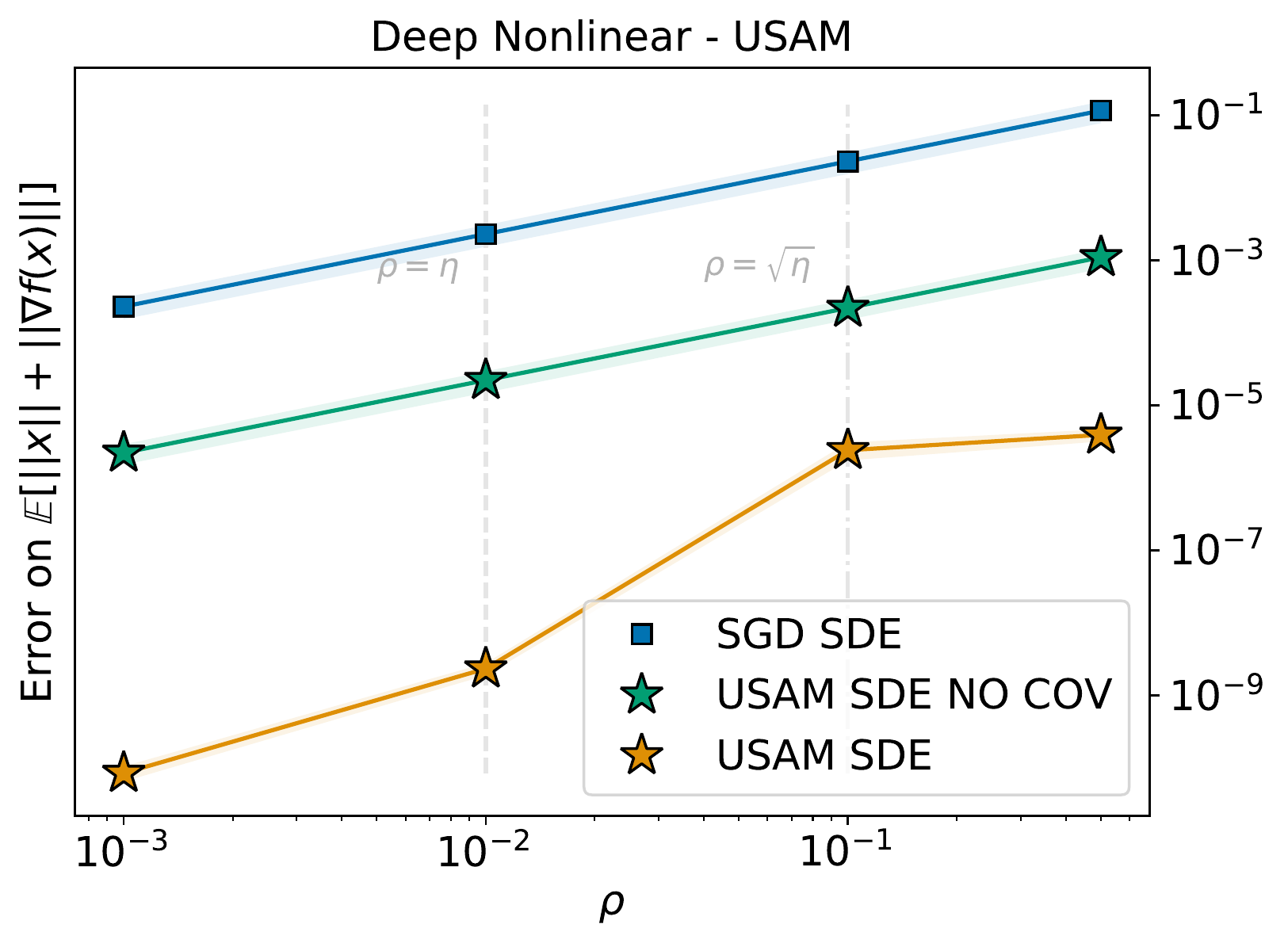} }}%
    \caption{USAM - Comparison in terms of $g_1(x)$ with respect to $\rho$ - Quadratic (left); Teacher-Student (center-left); Deep linear class (center-right); Deep Nonlinear class (right).}%
    \label{fig:Verif_Rho_NO_COV_USAM}%
\end{figure}

\begin{figure}%
    \centering
    \subfloat{{\includegraphics[width=0.24\linewidth]{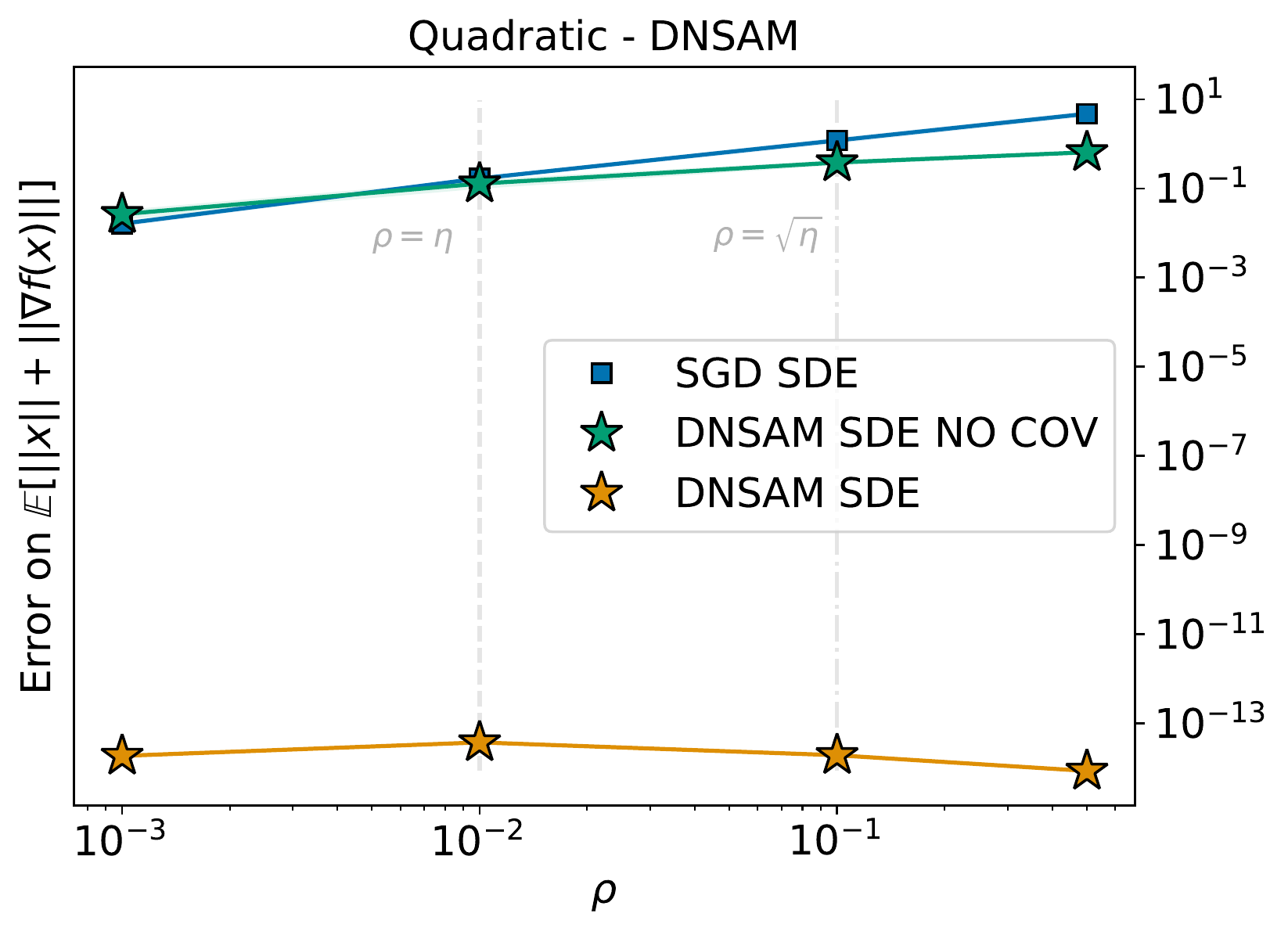} }}%
    \subfloat{{\includegraphics[width=0.24\linewidth]{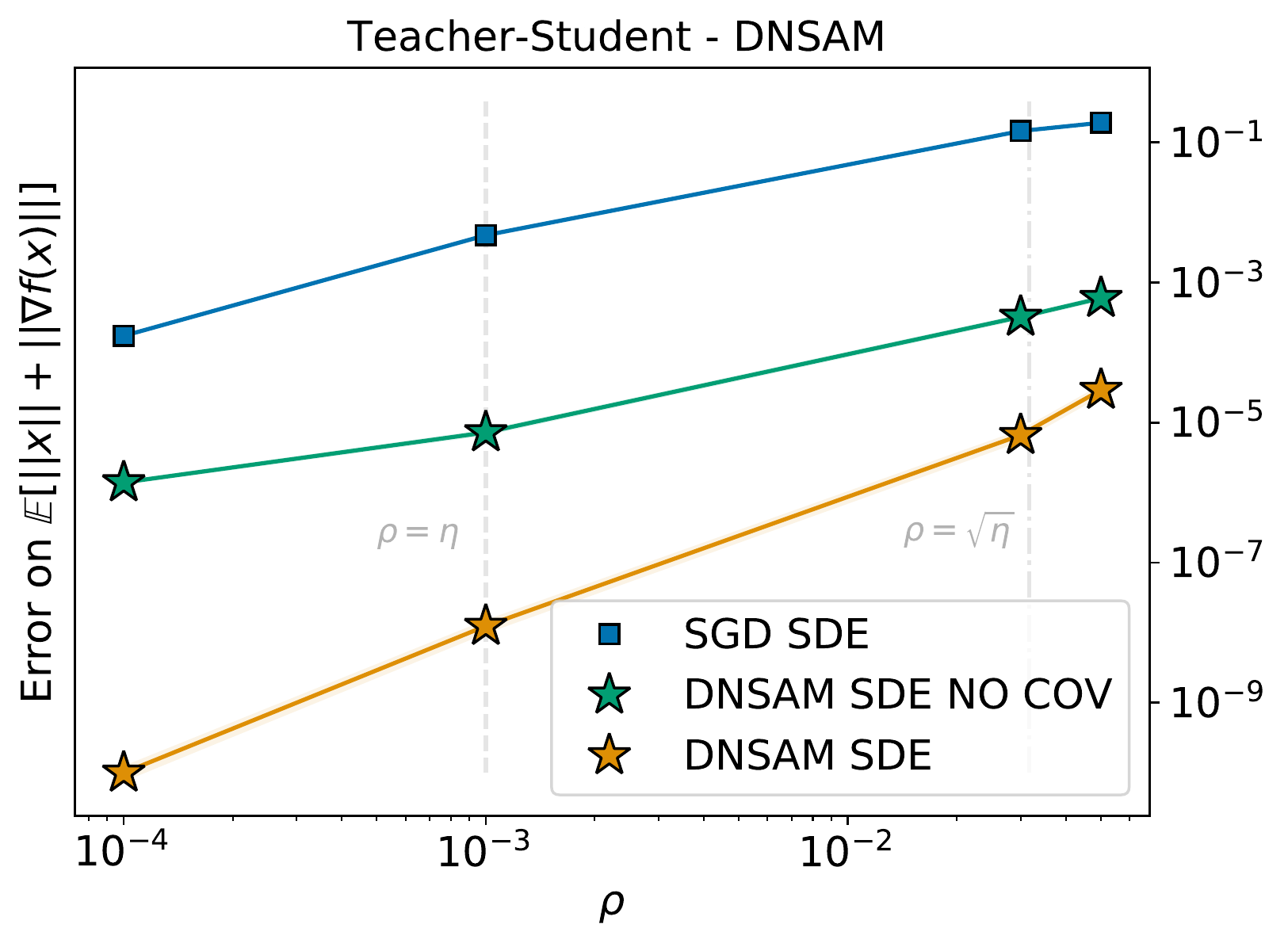} }}
    \subfloat{{\includegraphics[width=0.24\linewidth]{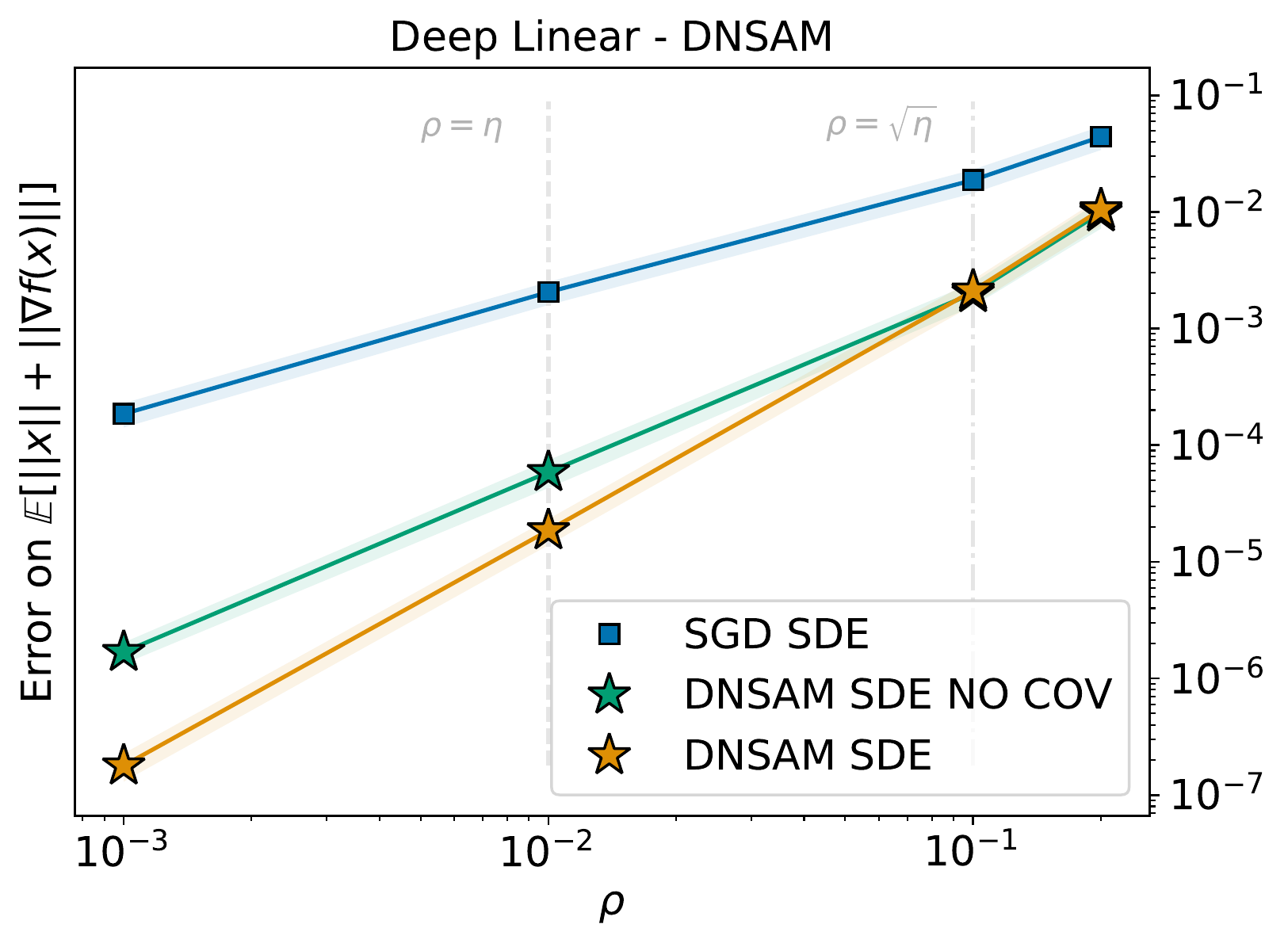} }}%
    \subfloat{{\includegraphics[width=0.24\linewidth]{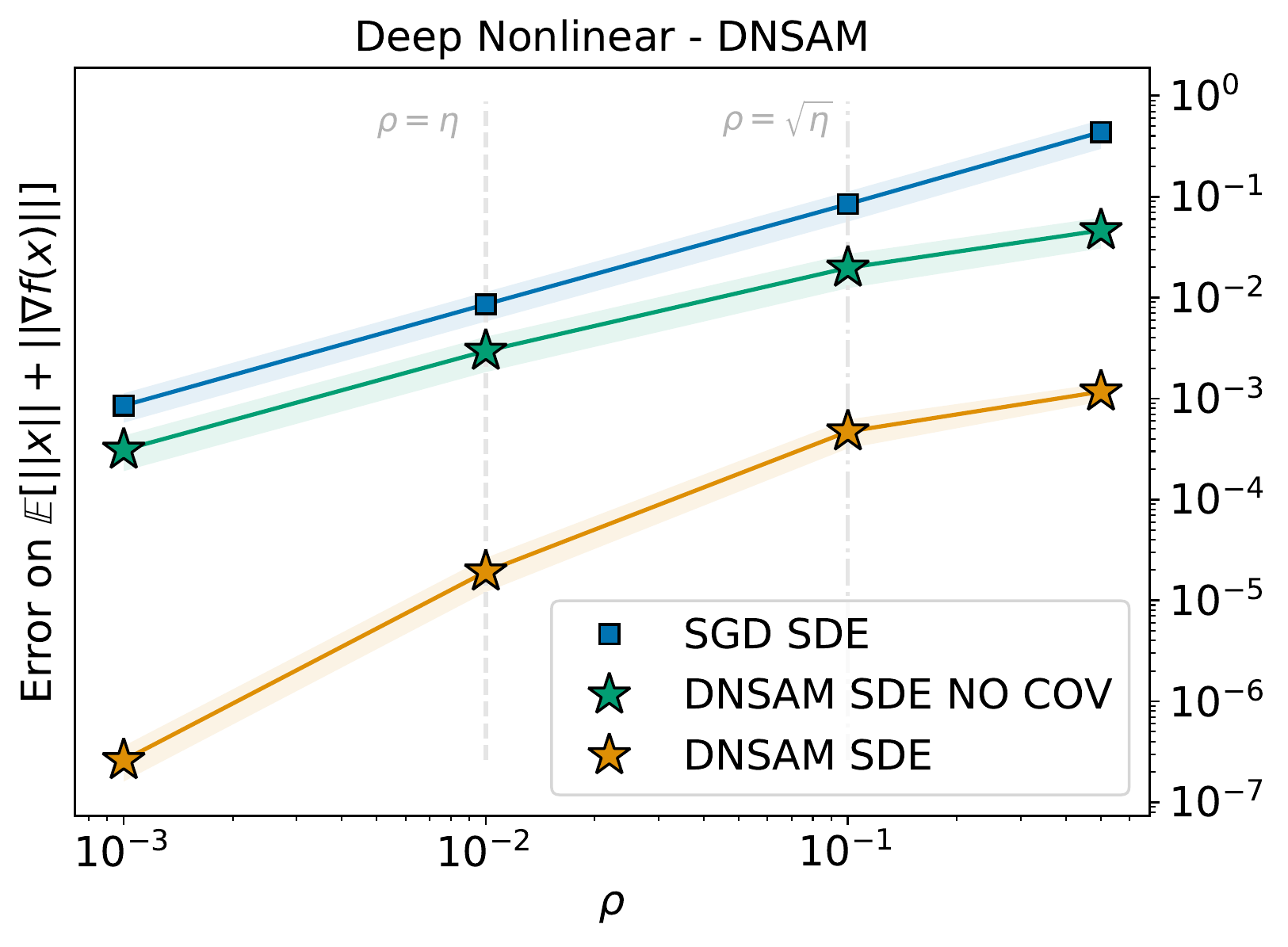} }}%
    \caption{DNSAM - Comparison in terms of $g_1(x)$ with respect to $\rho$ - Quadratic (left); Teacher-Student (center-left); Deep linear class (center-right); Deep Nonlinear class (right).}%
    \label{fig:Verif_Rho_NO_COV_SAM}%
\end{figure}

\subsection{Quadratic Landscape}\label{appendix:quadratic}

\paragraph{Interplay Between Hessian, $\rho$, and the Noise.}

In this paragraph, we provide additional details regarding the interplay between the Hessian, $\rho$, and the noise. In the first experiment represented in Figure \ref{fig:SAM_SGD_Hessian_Noise}, we fix $\rho = \sqrt{\eta}$, where $\eta=0.001$ is the learning rate. Then, we fix the Hessian $H \in \mathbb{R}^{100 \times 100}$ to be diagonal with random positive eigenvalues. Then, we select the scaling factors $\sigma \in \{1, 2, 4 \}$. For each value of $\sigma$, we optimize the quadratic loss with SGD and DNSAM where the hessian is scaled up by a factor $\sigma$. The starting point is $x_0 = (0.02, \cdots, 0.02)$ and the number of iterations is $20000$. The results are averaged over $5$ runs.

In the second experiment represented in Figure \ref{fig:SAM_SGD_Rho_Noise}, we fix the Hessian $H$ with random positive eigenvalues. then, we select $\rho = \sqrt{\eta}$, where $\eta=0.001$ is the learning rate. Then, we select the scaling factors $\sigma \in \{1, 2, 4 \}$. For each value of $\sigma$, we optimize the quadratic loss with SGD and DNSAM where the hessian is fixed and $\rho$ is scaled up by a factor $\sigma$. The starting point is $x_0 = (0.02, \cdots, 0.02)$ and the number of iterations is $20000$. The results are averaged over $5$ runs.

The very same setup holds for the experiments carried out for USAM and is represented in Figure \ref{fig:USAM_SGD_Hessian_Noise} and Figure \ref{fig:USAM_SGD_Rho_Noise}.

The very same setup holds for the experiments carried out for SAM and is represented in Figure \ref{fig:True_SAM_SGD_Hessian_Noise} and Figure \ref{fig:True_SAM_SGD_Rho_Noise}.

\begin{figure}%
    \centering
    \subfloat{{\includegraphics[width=0.32\linewidth]{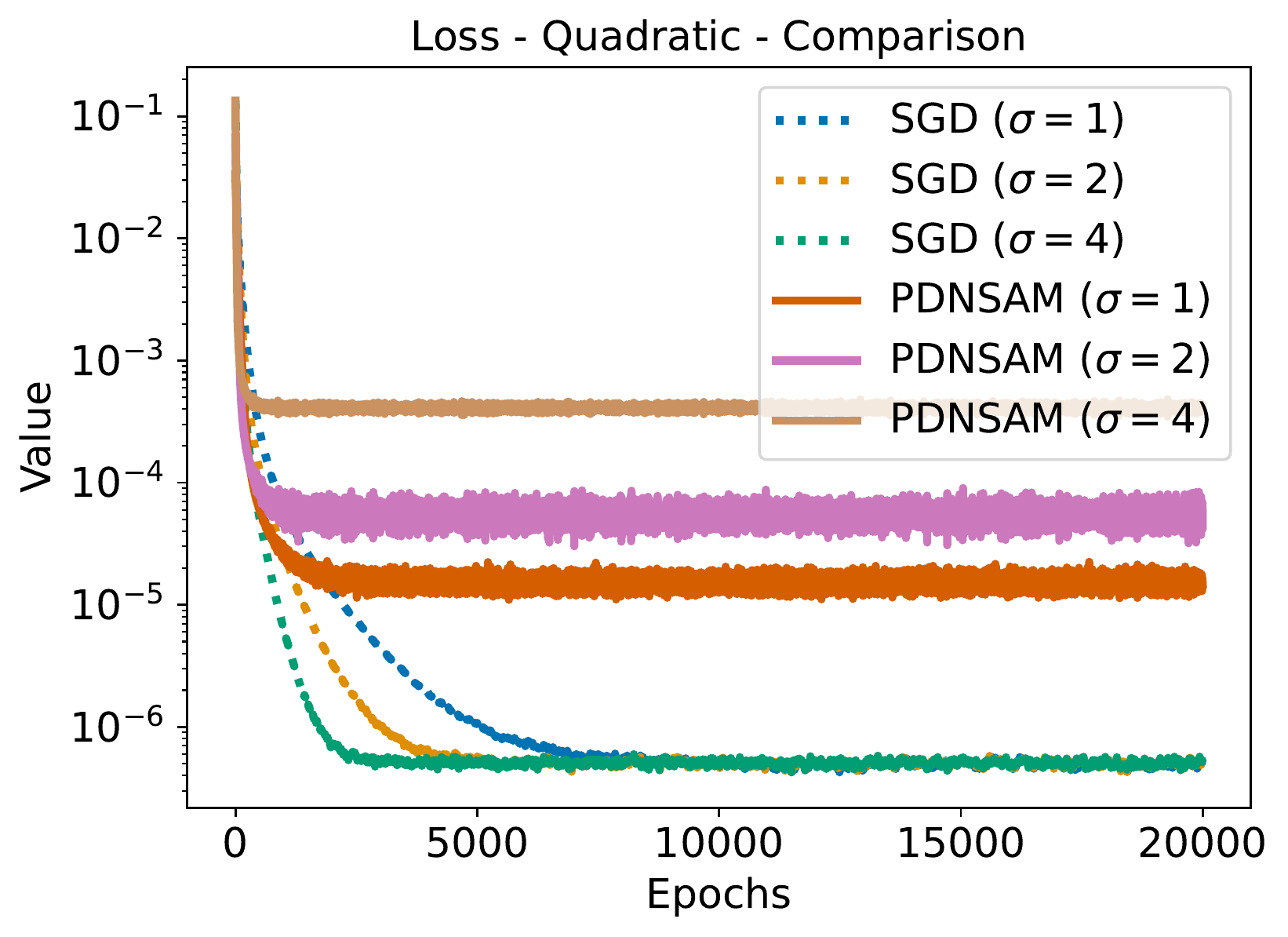} }}%
    \subfloat{{\includegraphics[width=0.32\linewidth]{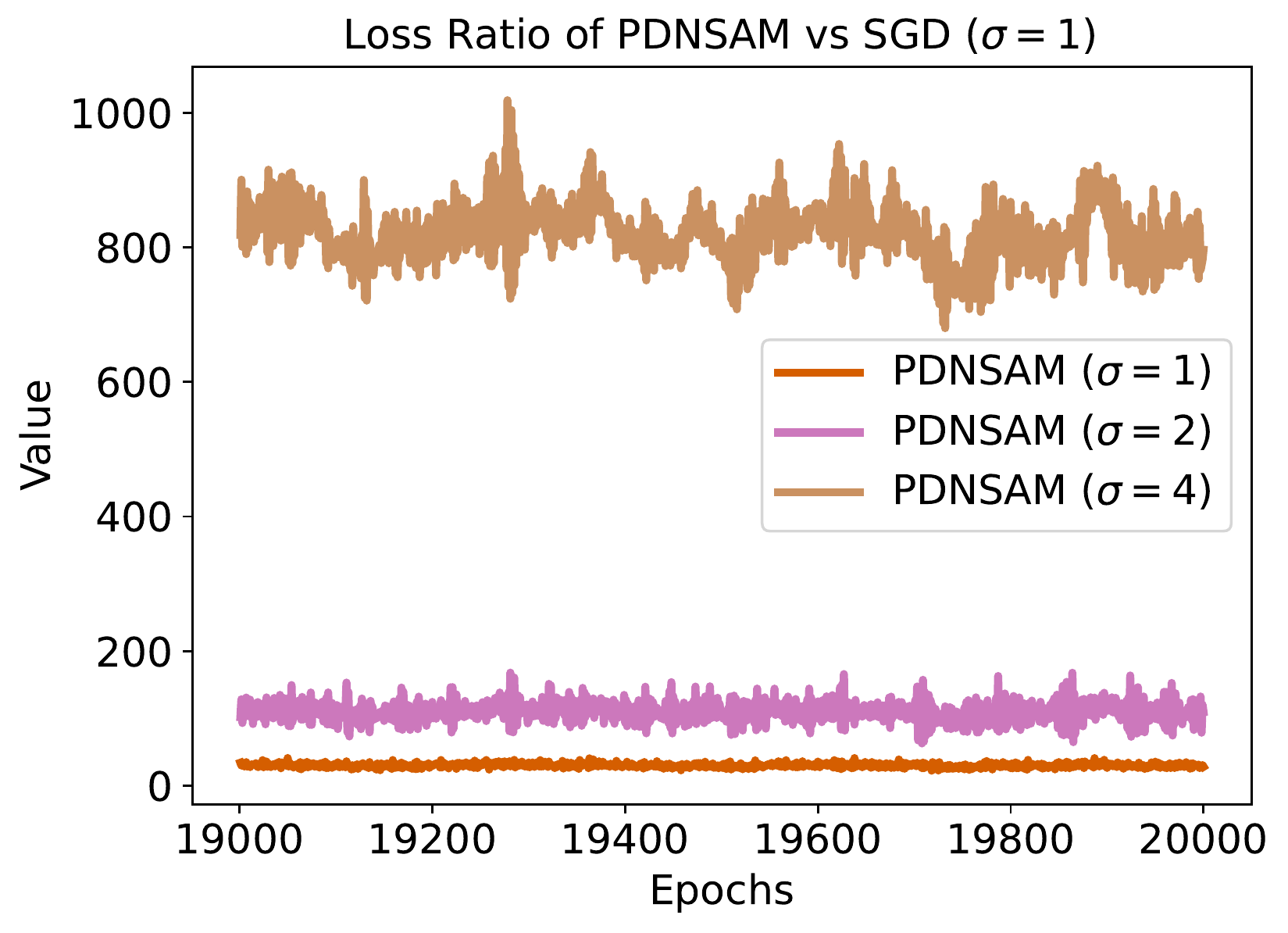} }}%
    \caption{Role of the Hessian - Left: Comparison between SGD and DNSAM for fixed rho and larger Hessians. Right: Ratio between the Loss of DNSAM for different scaling of the Hessian by the loss of the unscaled case of SGD.}%

    \label{fig:SAM_SGD_Hessian_Noise}%
\end{figure}

\begin{figure}%
    \centering
    \subfloat{{\includegraphics[width=0.32\linewidth]{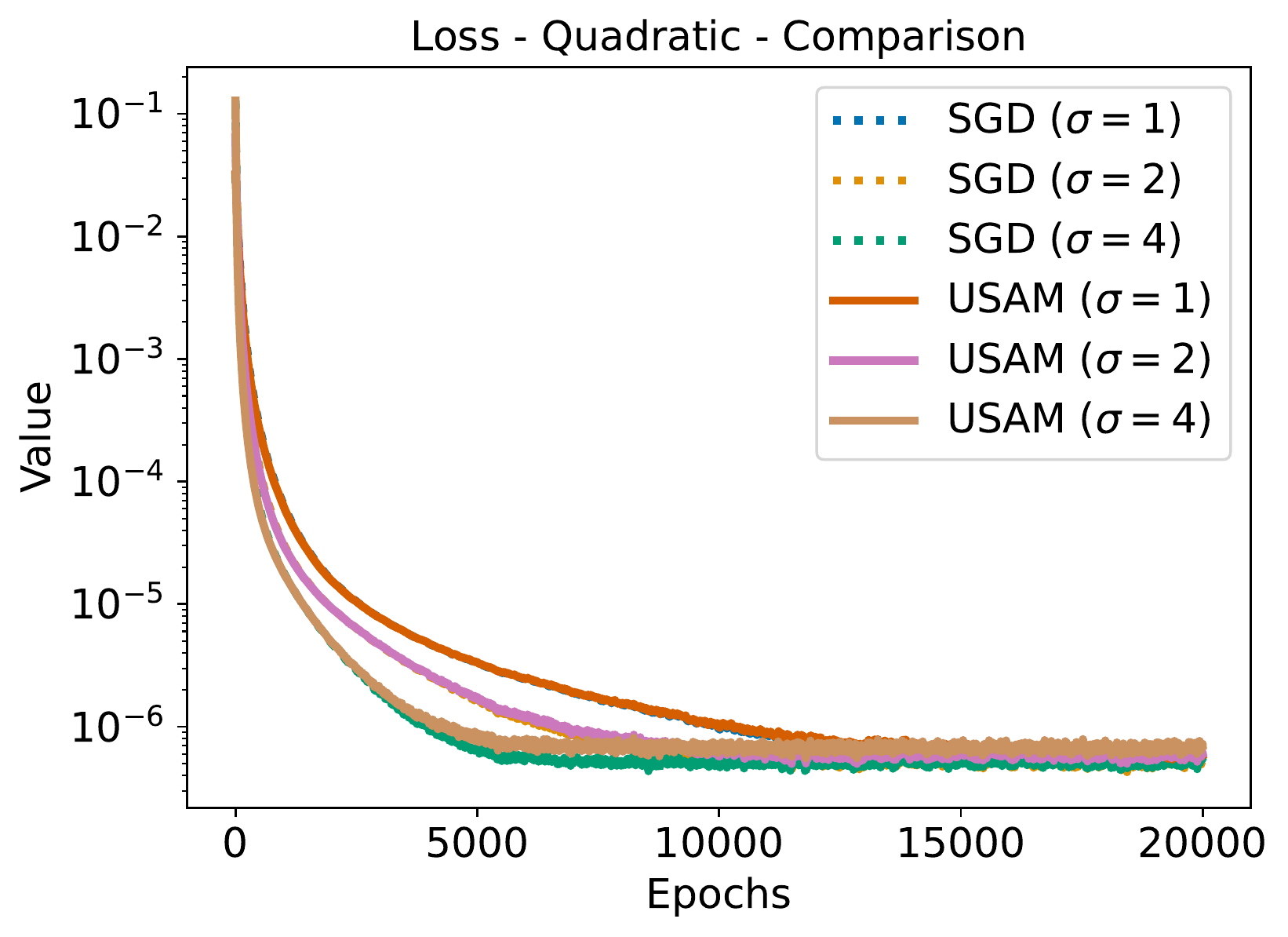} }}%
    \subfloat{{\includegraphics[width=0.32\linewidth]{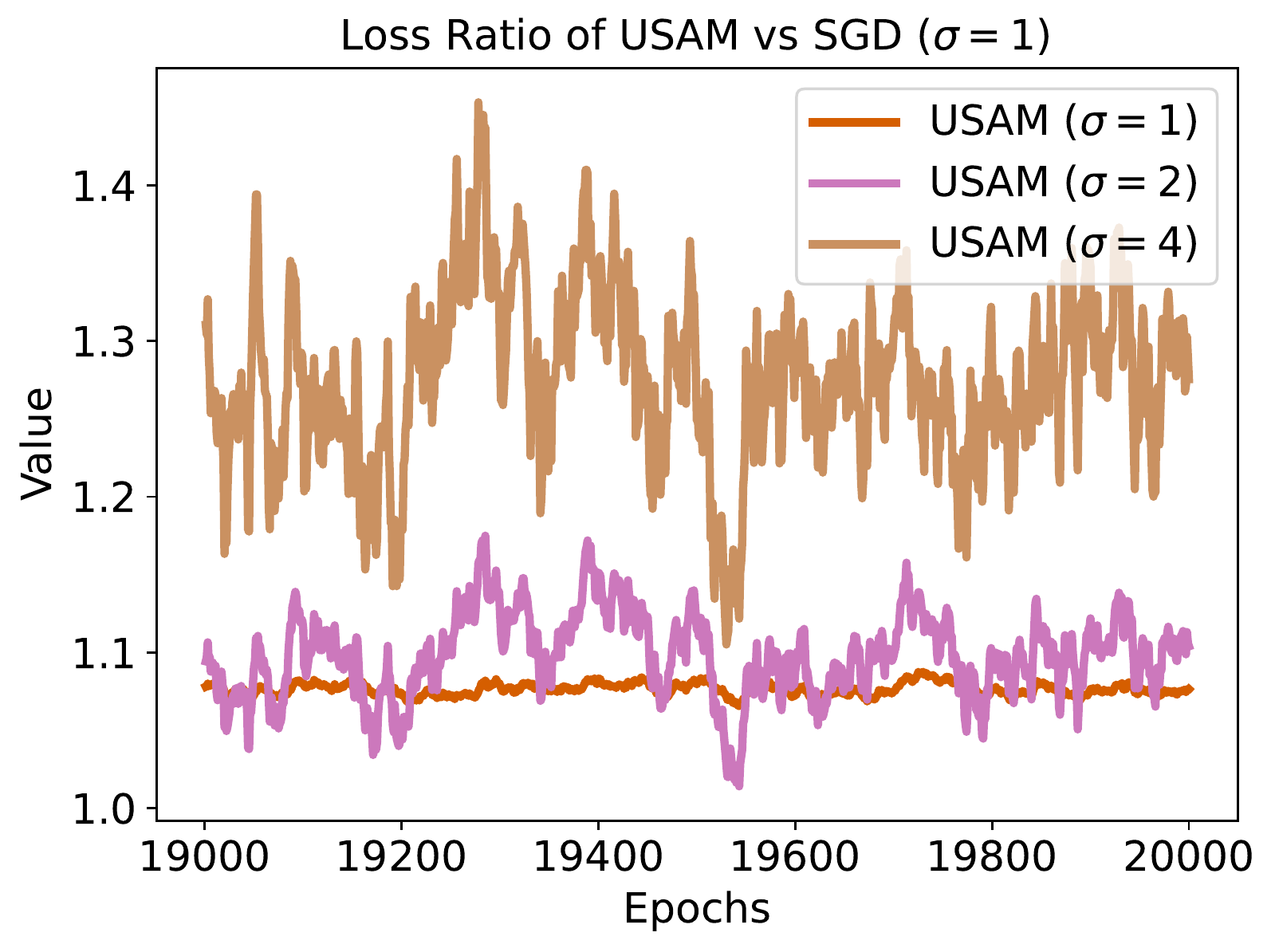} }}%
    \caption{Role of the Hessian - Left: Comparison between SGD and USAM for fixed rho and larger Hessians. Right: Ratio between the Loss of USAM for different scaling of the Hessian by the loss of the unscaled case of SGD.}%

    \label{fig:USAM_SGD_Hessian_Noise}%
\end{figure}

\begin{figure}%
    \centering
    \subfloat{{\includegraphics[width=0.32\linewidth]{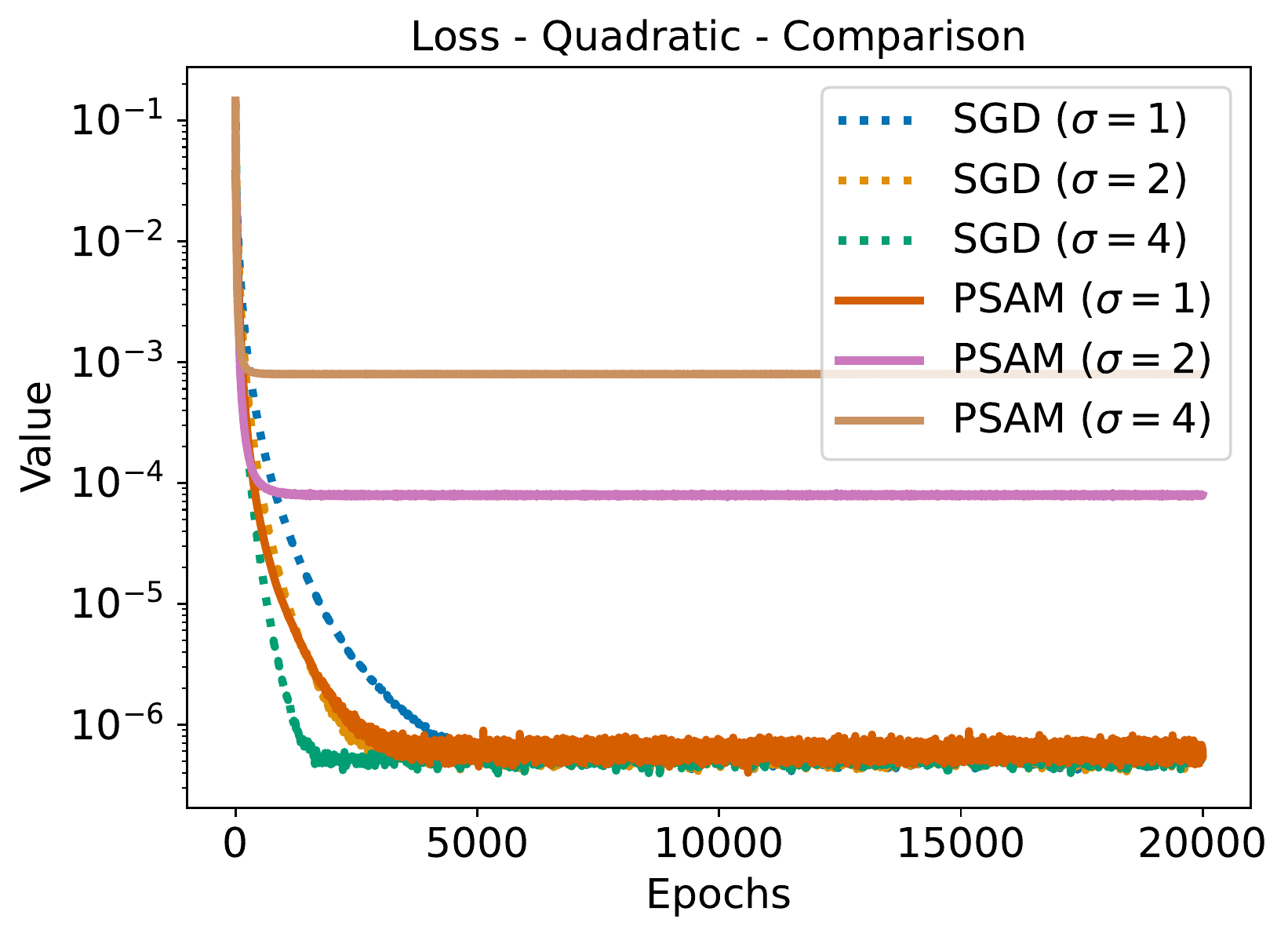} }}%
    \subfloat{{\includegraphics[width=0.32\linewidth]{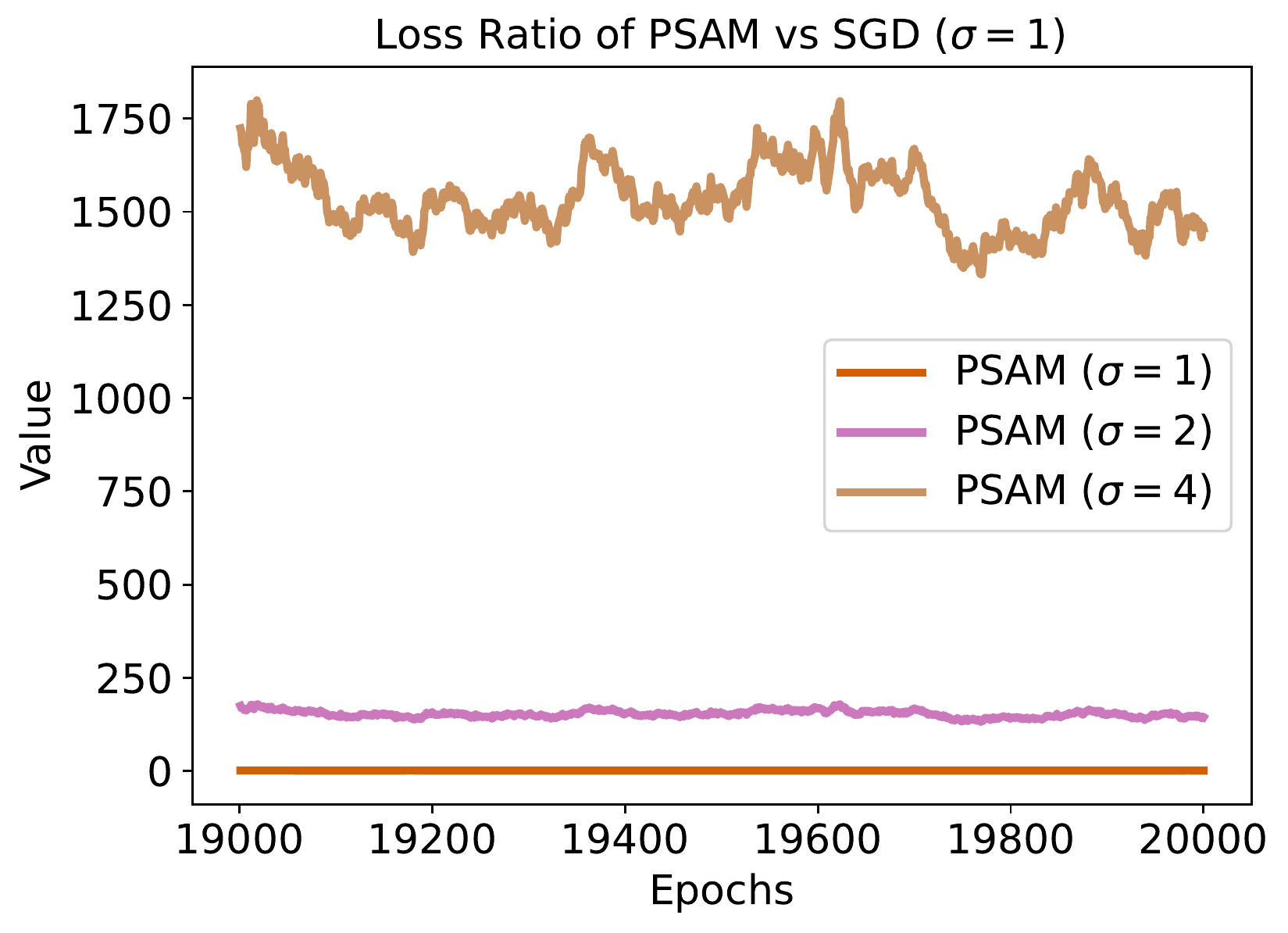} }}%
    \caption{Role of the Hessian - Left: Comparison between SGD and SAM for fixed rho and larger Hessians. Right: Ratio between the Loss of PSAM for different scaling of the Hessian by the loss of the unscaled case of SGD.}%

    \label{fig:True_SAM_SGD_Hessian_Noise}%
\end{figure}

\begin{figure}%
    \centering
    \subfloat{{\includegraphics[width=0.32\linewidth]{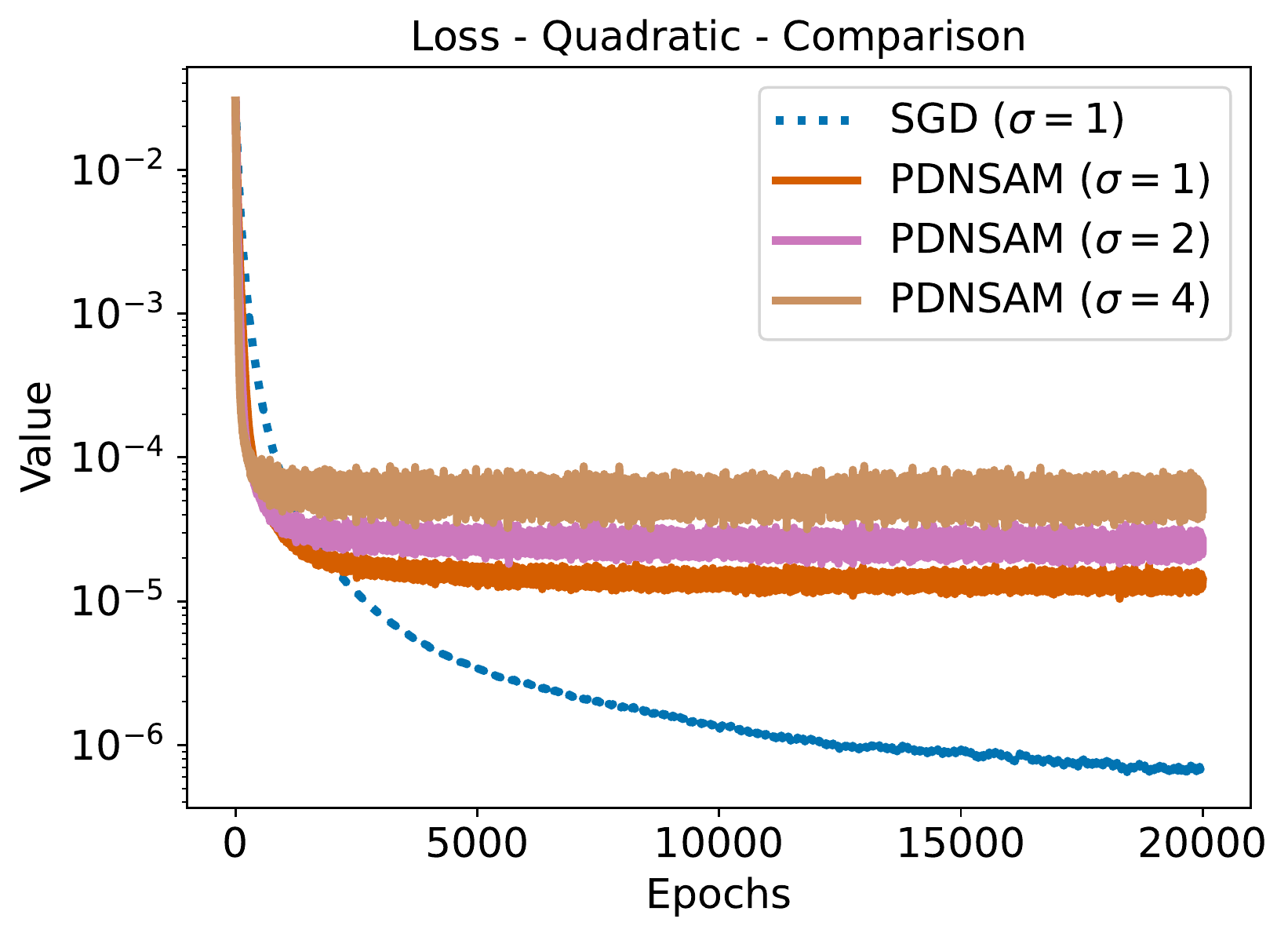} }}%
    \subfloat{{\includegraphics[width=0.32\linewidth]{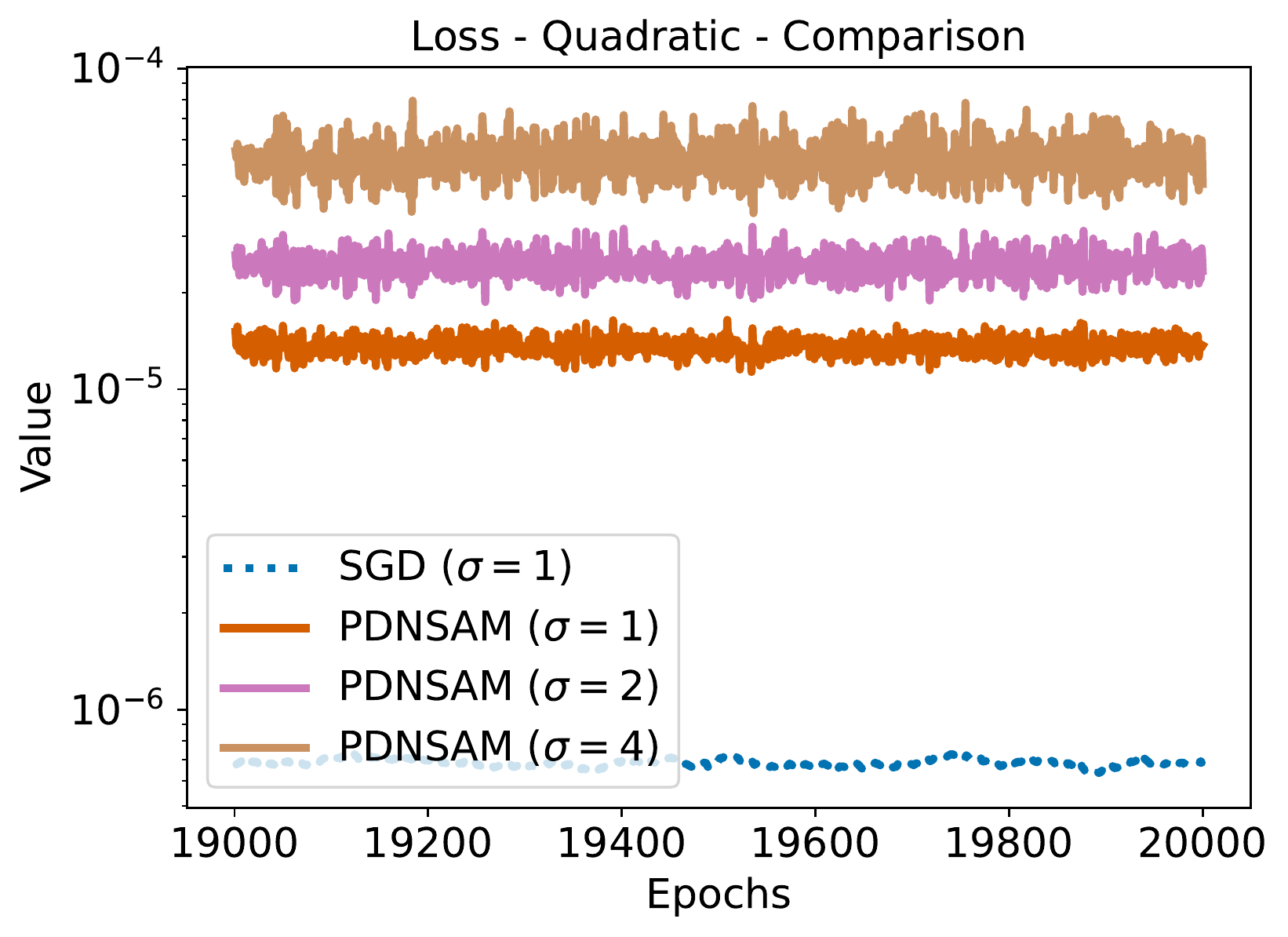} }}%
    \subfloat{{\includegraphics[width=0.32\linewidth]{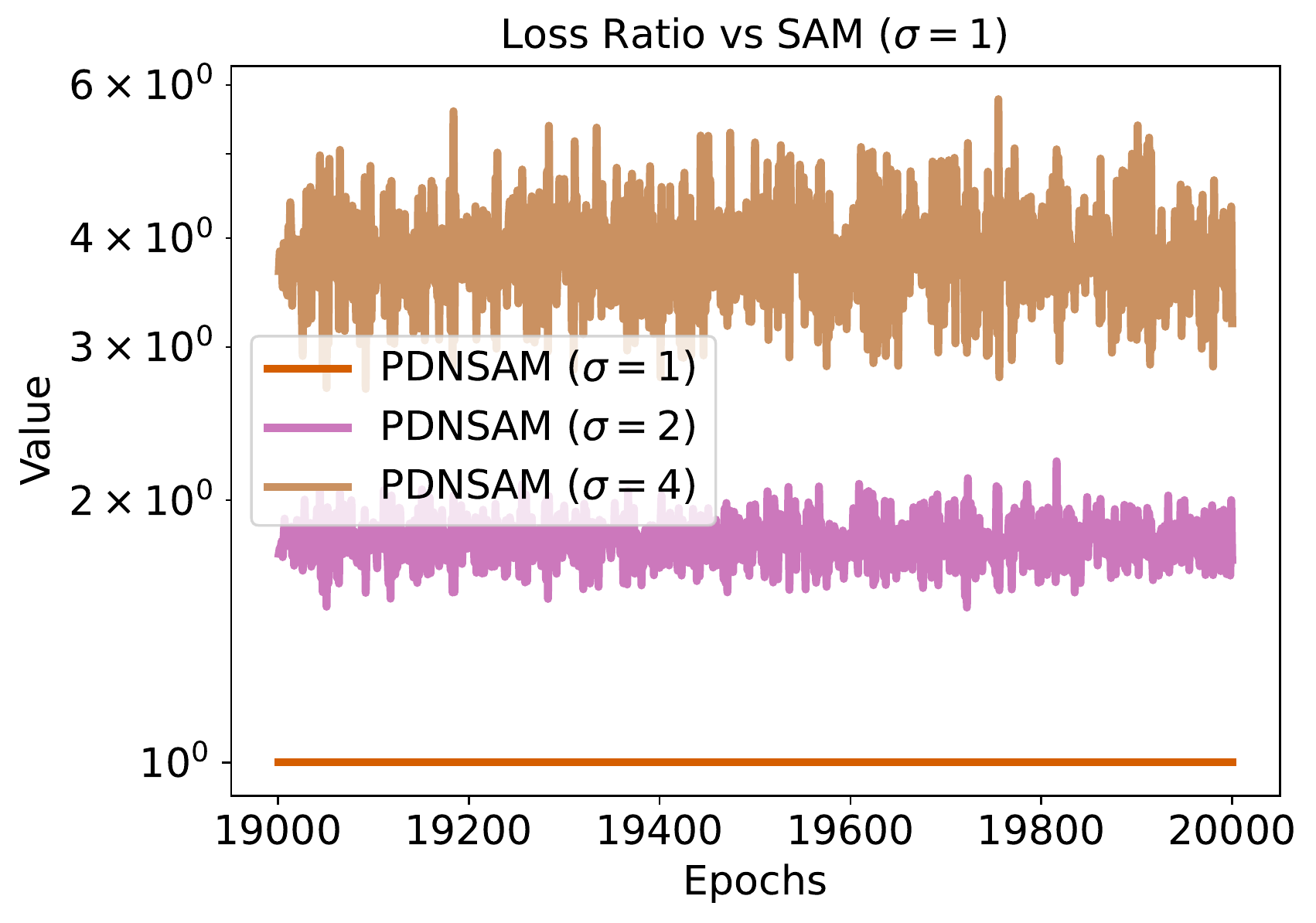} }}%
    \caption{Role of $\rho$ - Left: Comparison between SGD and DNSAM for fixed hessian and larger $\rho$ values. Center: Zoom at convergence. Right: Ratio between the Loss of DNSAM for different scaling of the Hessian by the loss of the unscaled case of DNSAM.}%

    \label{fig:SAM_SGD_Rho_Noise}%
\end{figure}

\begin{figure}%
    \centering
    \subfloat{{\includegraphics[width=0.30\linewidth]{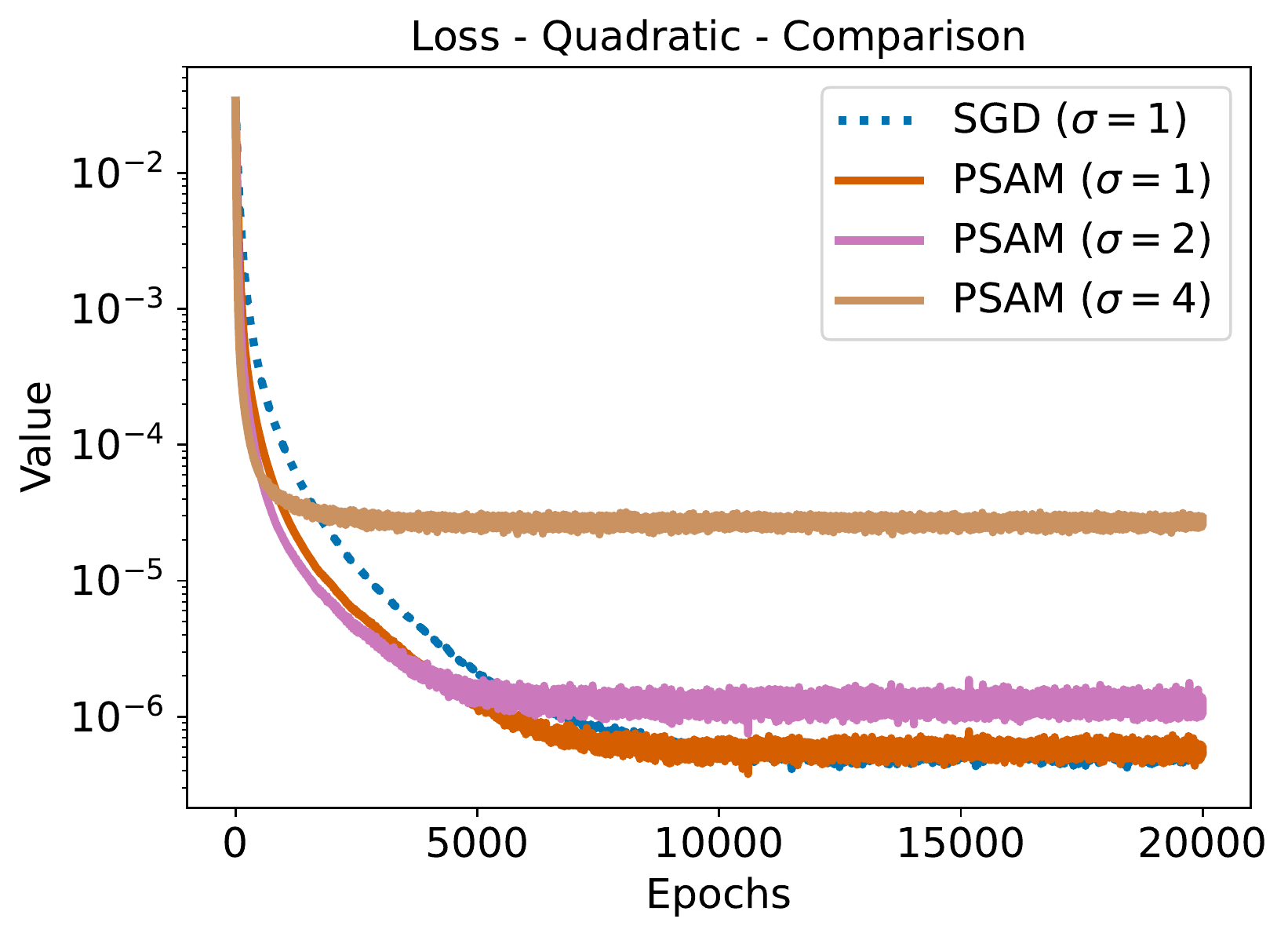} }}%
    \subfloat{{\includegraphics[width=0.34\linewidth]{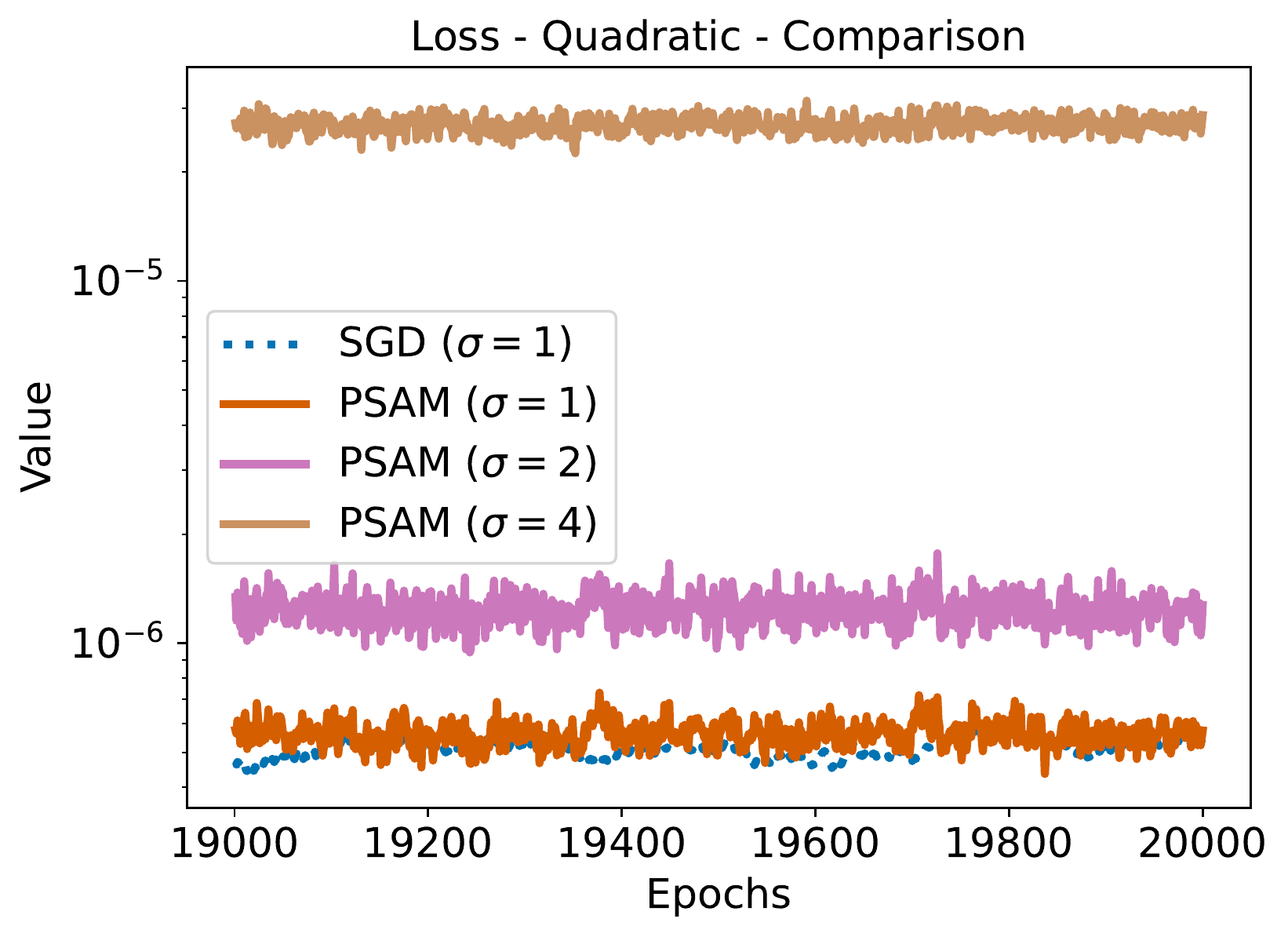} }}%
    \subfloat{{\includegraphics[width=0.33\linewidth]{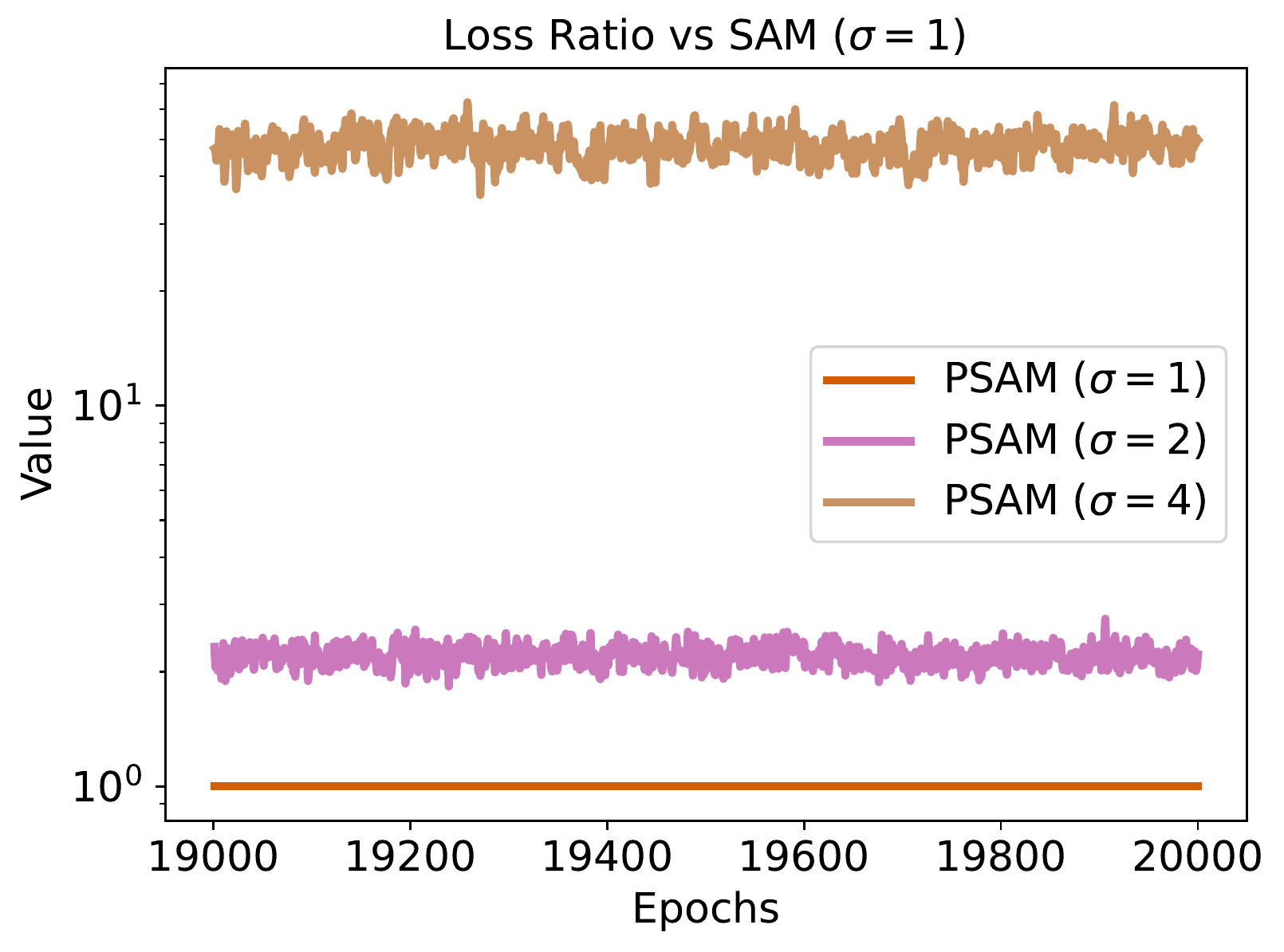} }}%
    \caption{Role of $\rho$ - Left: Comparison between SGD and PSAM for fixed hessian and larger $\rho$ values. Center: Zoom at convergence. Right: Ratio between the Loss of PSAM for different scaling of the Hessian by the loss of the unscaled case of PSAM.}%

    \label{fig:True_SAM_SGD_Rho_Noise}%
\end{figure}

\begin{figure}%
    \centering
    \subfloat{{\includegraphics[width=0.30\linewidth]{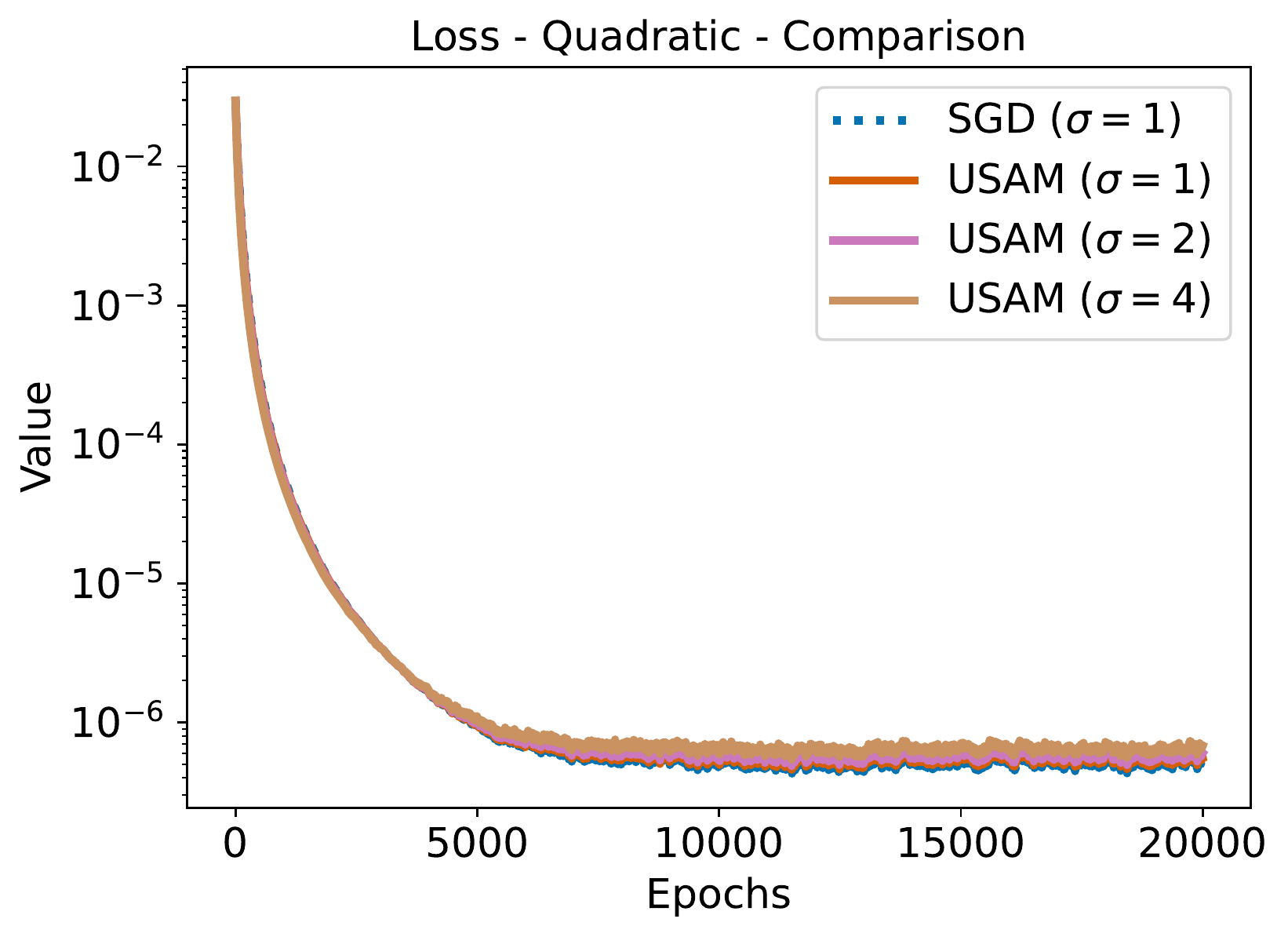} }}%
    \subfloat{{\includegraphics[width=0.34\linewidth]{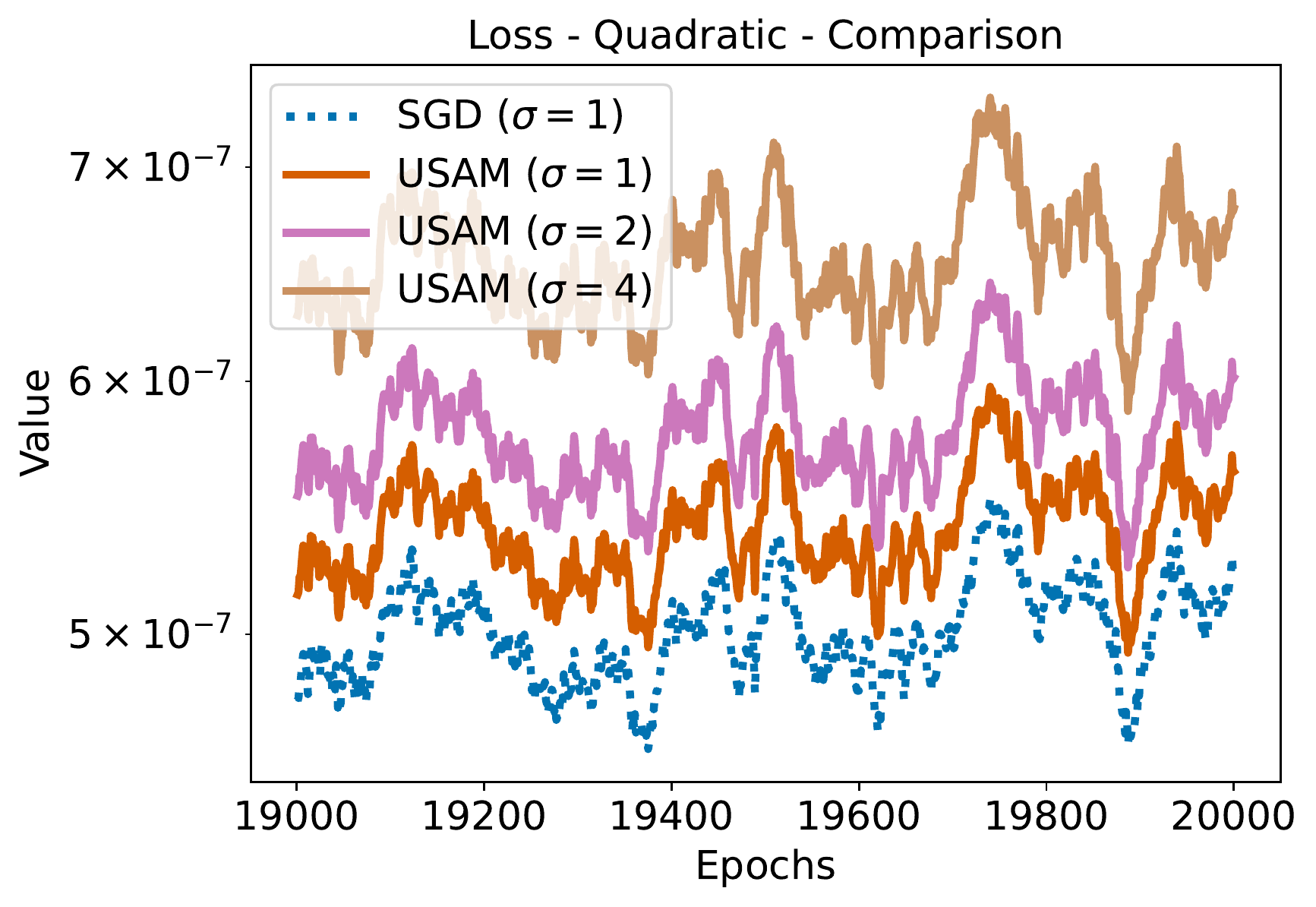} }}%
    \subfloat{{\includegraphics[width=0.33\linewidth]{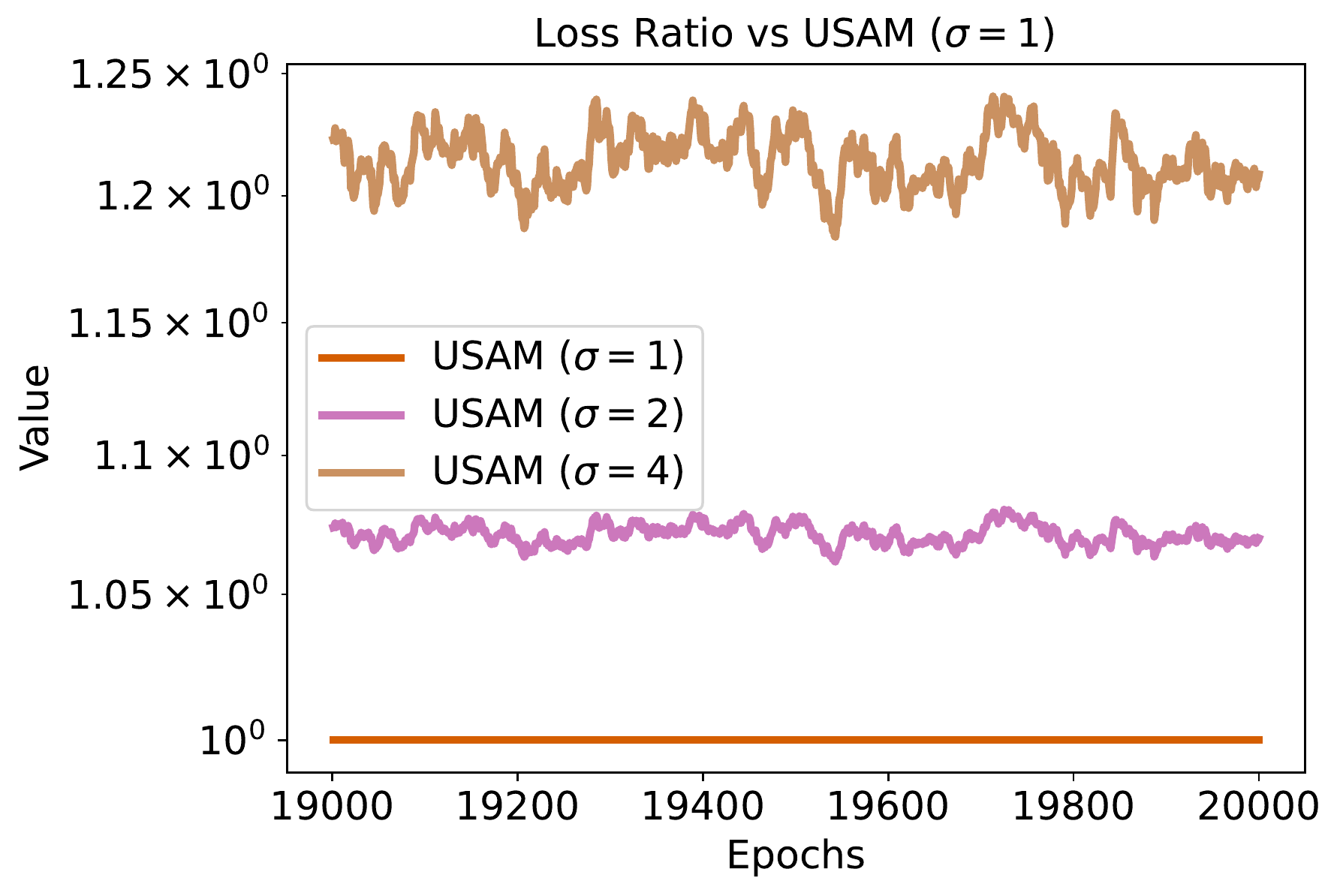} }}%
    \caption{Role of $\rho$ - Left: Comparison between SGD and USAM for fixed hessian and larger $\rho$ values. Center: Zoom at convergence. Right: Ratio between the Loss of USAM for different scaling of the Hessian by the loss of the unscaled case of USAM.}%

    \label{fig:USAM_SGD_Rho_Noise}%
\end{figure}

\paragraph{Stationary Distribution Convex Case} In this paragraph, we provide the details of the experiment about the dynamics of the SDE of DNSAM in the quadratic convex case of dimension $2$. The hessian is diagonal with both eigenvalues equal to $1$. We select $\rho = \sqrt{\eta}$, where $\eta=0.001$ is the learning rate. In the first image on the left of Figure \ref{fig:SAM_Convex_StatDistr}, we show the distribution of $10^{5}$ trajectories all starting at $(0.02, 0.02)$ after $5 \cdot 10^{4}$ iterations. In the second image, we plot the number of trajectories that at a certain time are inside a ball of radius $0.007$, e.g. close to the origin. As we can see in greater detail in the third image, all of them are initialized outside such a ball, then they get attracted inside, and around the $600$-th iteration they get repulsed out of it. We highlight that the proportion of points inside/outside the ball is relatively stable. In the fourth image, we count the number of trajectories that are jumping in or out of such a ball. All of the trajectories enter the ball between the $400$-th and $500$-th iteration, and then they start jumping in and out after the iteration $600$. We conclude that this experimental evidence are supporting the claim that the origin attracts the dynamics, but repulses it at the moment that the trajectories get too close to it. 

\begin{figure}%
    \centering
    \subfloat{{\includegraphics[width=0.24\linewidth]{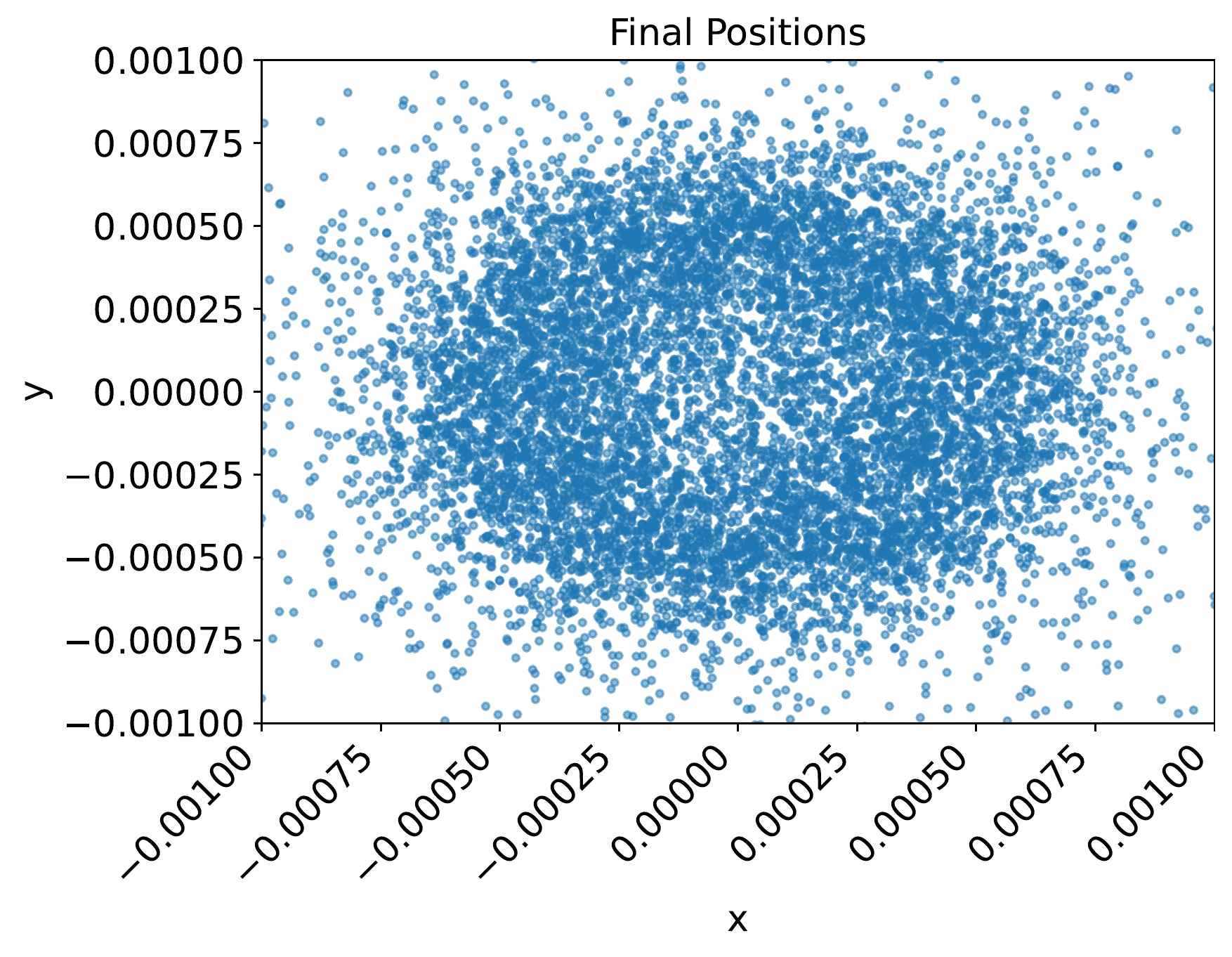} }}%
    \subfloat{{\includegraphics[width=0.24\linewidth]{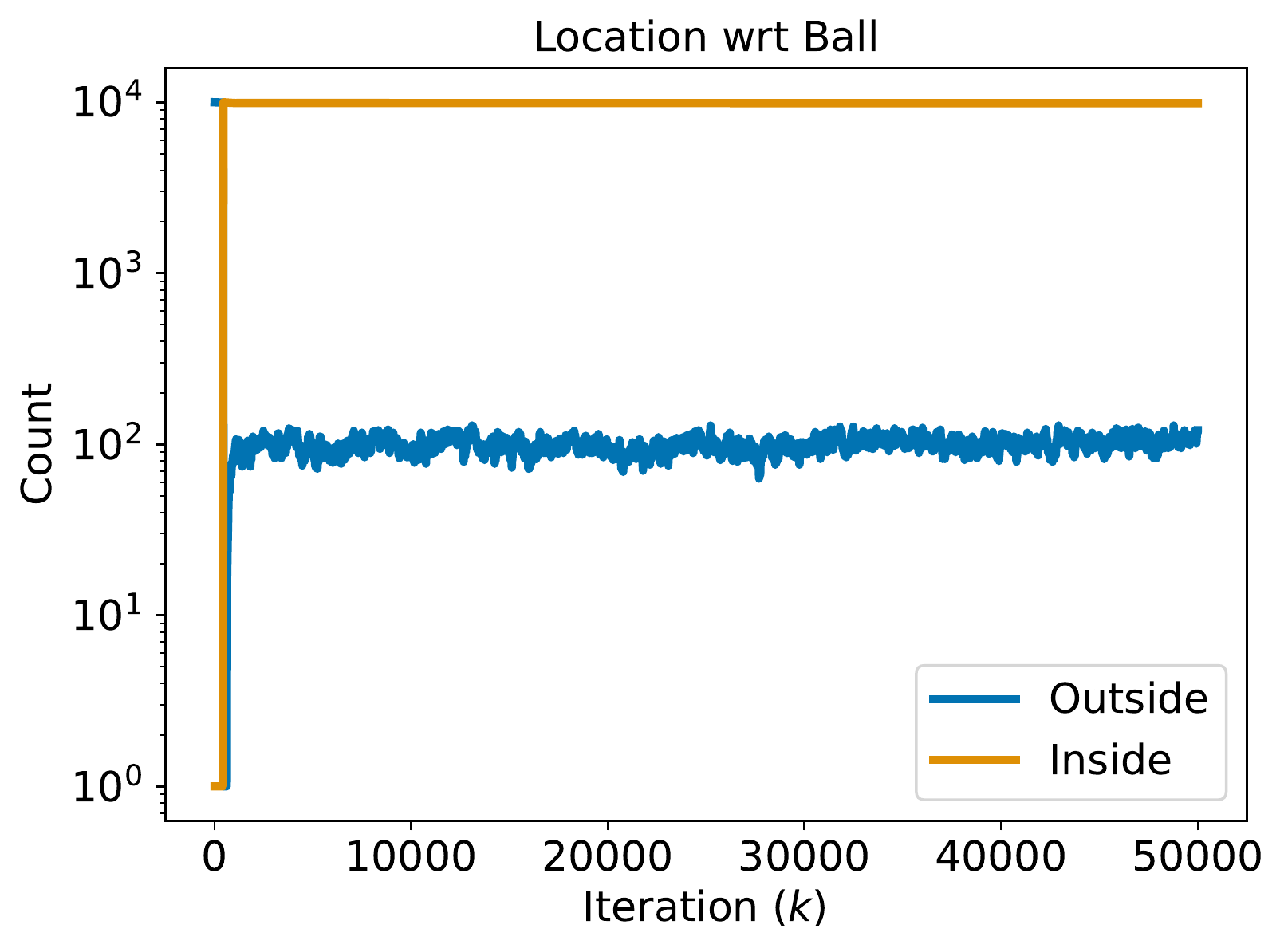} }}%
    \subfloat{{\includegraphics[width=0.24\linewidth]{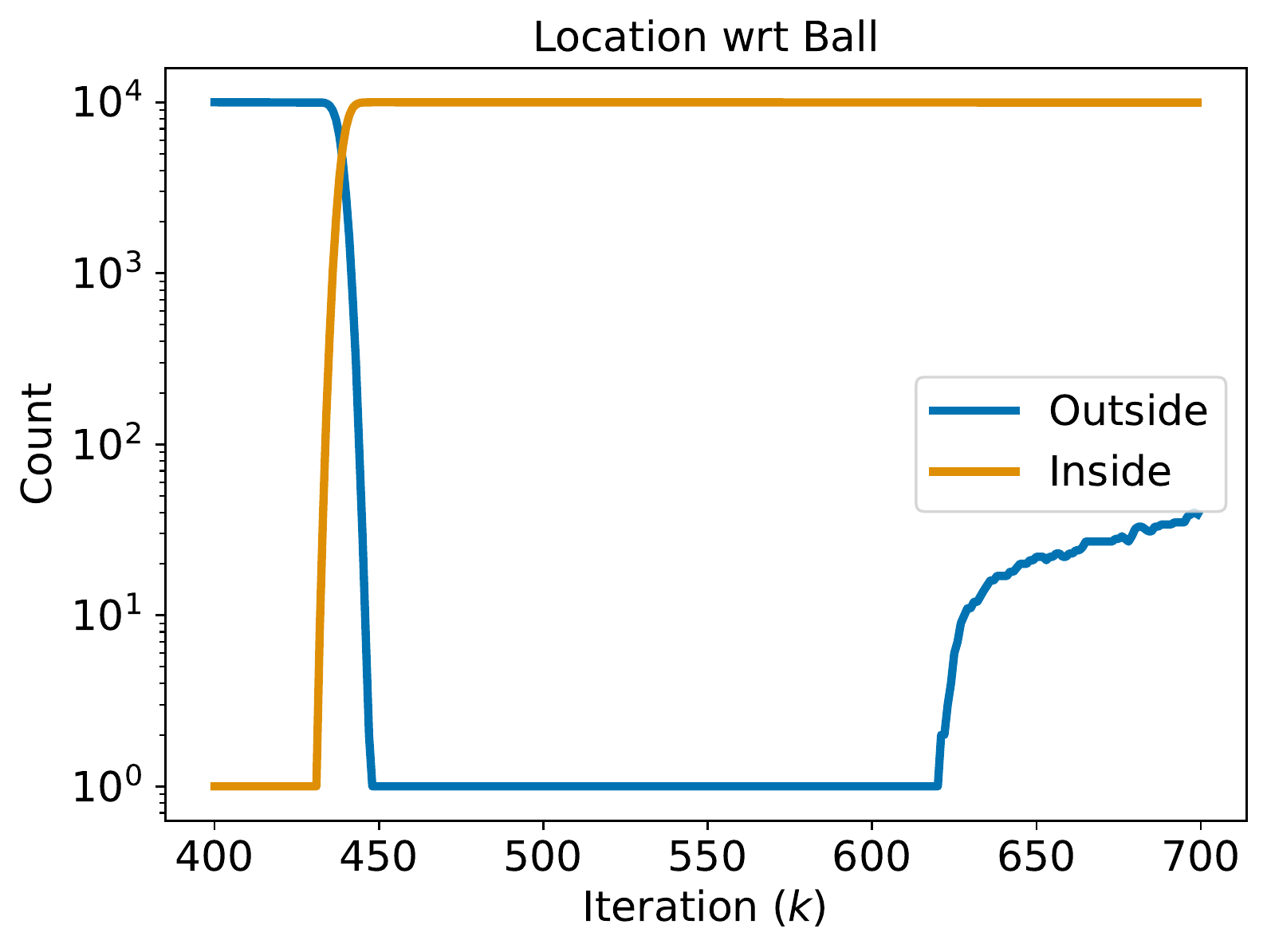} }}%
    \subfloat{{\includegraphics[width=0.24\linewidth]{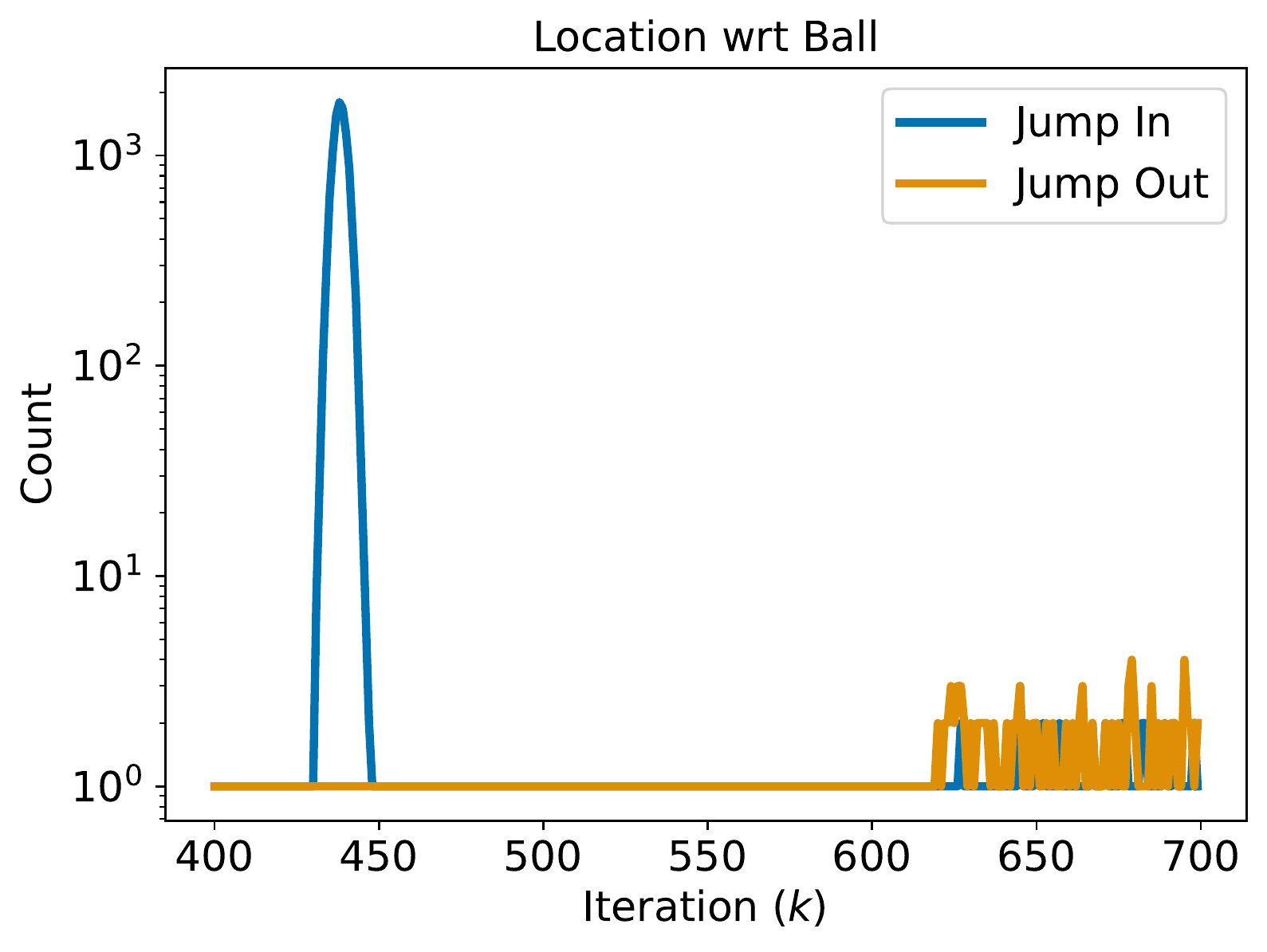} }}%
        \caption{Convex Quadratic - Left: Distribution points around the origin is scarcer near to the origin; Center-Left: Number of trajectories outside a small ball around the origin increases over time; Center-Right: All the trajectories eventually enter the ball and then start exiting it; Right: There is a constant oscillation of points in and out of the ball.}%
    \label{fig:SAM_Convex_StatDistr}%
\end{figure}

\paragraph{Stationary Distribution Saddle Case} In this paragraph, we provide the details of the experiment about the dynamics of the SDE of DNSAM in the quadratic saddle case of dimension $2$. The hessian is diagonal with eigenvalues equal to $1$ and $-1$. We select $\rho = \sqrt{\eta}$, where $\eta=0.001$ is the learning rate. In the first image on the left of Figure \ref{fig:SAM_Saddle_StatDistr}, we show the distribution of $10^{5}$ trajectories all starting at $(0.02, 0.02)$ after $5 \cdot 10^{4}$ iterations. In the second image, we plot the number of trajectories that at a certain time are inside a ball of radius $0.007$, e.g. close to the origin. As we can see in greater detail in the third image, all of them are initialized outside such a ball, then they get attracted inside, and around the $1200$-th iteration they get repulsed out of it. We highlight that the proportion of points outside the ball is stably increasing, meaning that the trajectories are slowly escaping from the saddle. In the fourth image, we count the number of trajectories that are jumping in or out of such a ball. All of the trajectories enter the ball between the $950$-th and $1000$-th iteration, and then they start jumping in and out after the iteration $1200$. We conclude that this experimental evidence are supporting the claim that the origin attracts the dynamics, but repulses it at the moment that the trajectories get too close to it, even when this is a saddle.

\paragraph{Escaping the Saddle - Low Dimensional}
In this paragraph, we provide details for the Escaping the Saddle experiment in dimension $d=2$. As in the previous experiment, the saddle is a quadratic of dimension $2$ and its hessian is diagonal with eigenvalues equal to $1$ and $-1$. We select $\rho = \sqrt{\eta}$, where $\eta=0.001$ is the learning rate. We initialize the GD, USAM, SAM, SGD, PUSAM, DNSAM, and PSAM in the point $x_{0}=(0,0.01)$, e.g. in the direction of the fastest escape from the saddle. In the left of Figure \ref{fig:SAM_Saddle_Escape}, we observe that GD and USAM manage to escape the saddle while SAM remains stuck. We highlight that running for more iterations would not change this as SAM is oscillating across the origin. In the second figure, we observe that the stochastic optimizers escape the saddle quicker than their deterministic counterpart and even PSAM and DNSAM manage to escape. Results are averaged over $3$ runs.

\paragraph{Escaping the Saddle - High Dimensional}
In this paragraph, we provide details for the Escaping the Saddle experiment in dimension $d=400$. We fix the Hessian $H \in \mathbb{R}^{400 \times 400}$ to be diagonal with random positive eigenvalues. To simulate a saddle, we flip the sign of the smallest $10$ eigenvalues. We select $\rho = \sqrt{\eta}$, where $\eta=0.001$ is the learning rate. We study the optimization dynamics of SAM, PSAM, and DNPSAM as we initialize the process closer and closer to the saddle in the origin. The starting point $x_0=(1, \cdots, 1)$ is scaled with factors $\sigma \in \{ 10^{0}, 10^{-4}, 10^{-8} \}$ and we notice that the one scaled with $\sigma = 1$ escapes slowly from the saddle. The one scaled with $\sigma=10^{-8}$ experiences a sharp spike in volatility and jumps away from the origin and ends up escaping the saddle faster than the previous case. Finally, the one scaled with $\sigma=10^{-4}$ stays trapped in the saddle. Results are represented in Figure \ref{fig:SAM_Saddle_Escape}. Results are averaged over $3$ runs.

\begin{figure}%
    \centering
    \subfloat{{\includegraphics[width=0.24\linewidth]{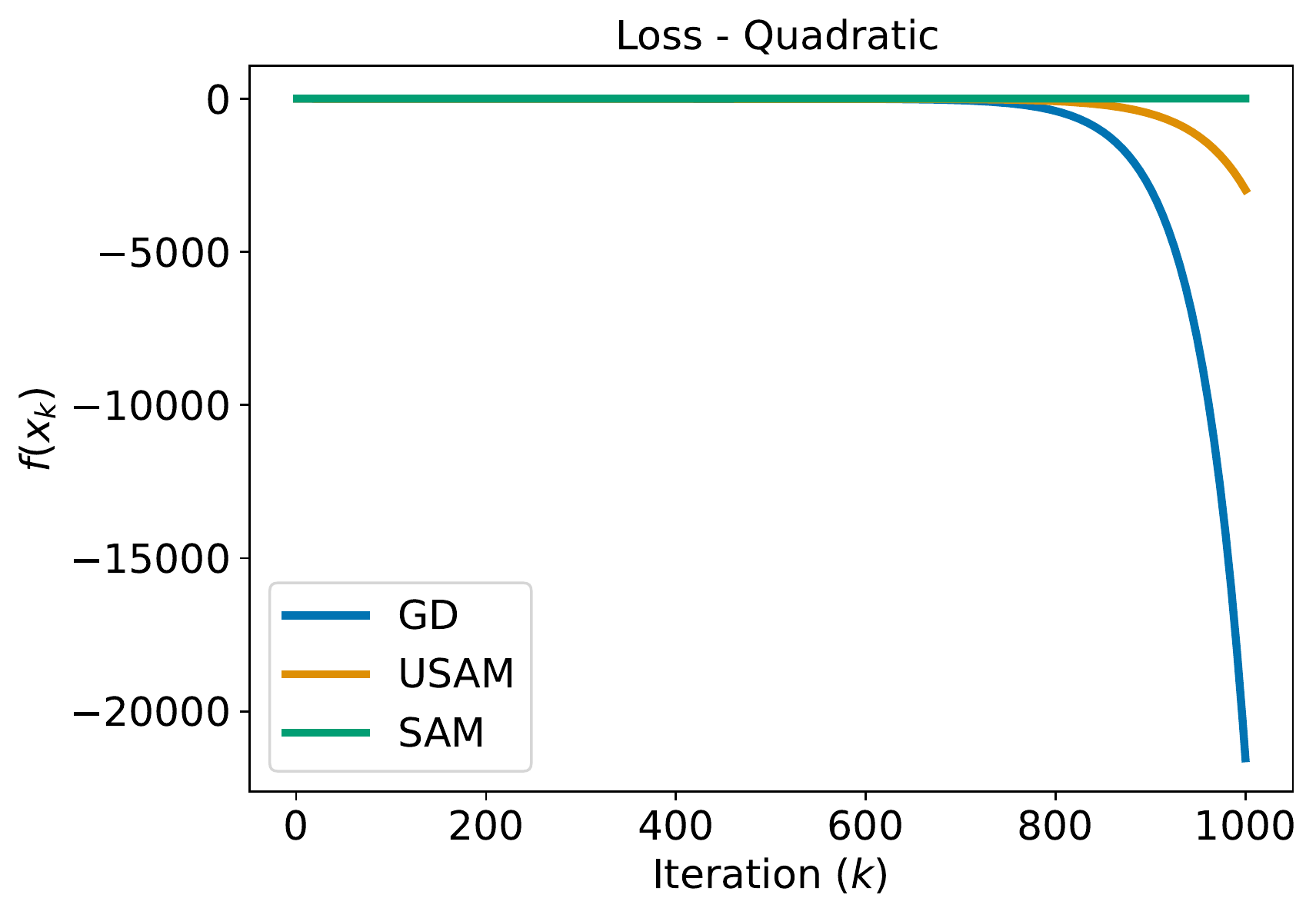} }}%
    \subfloat{{\includegraphics[width=0.24\linewidth]{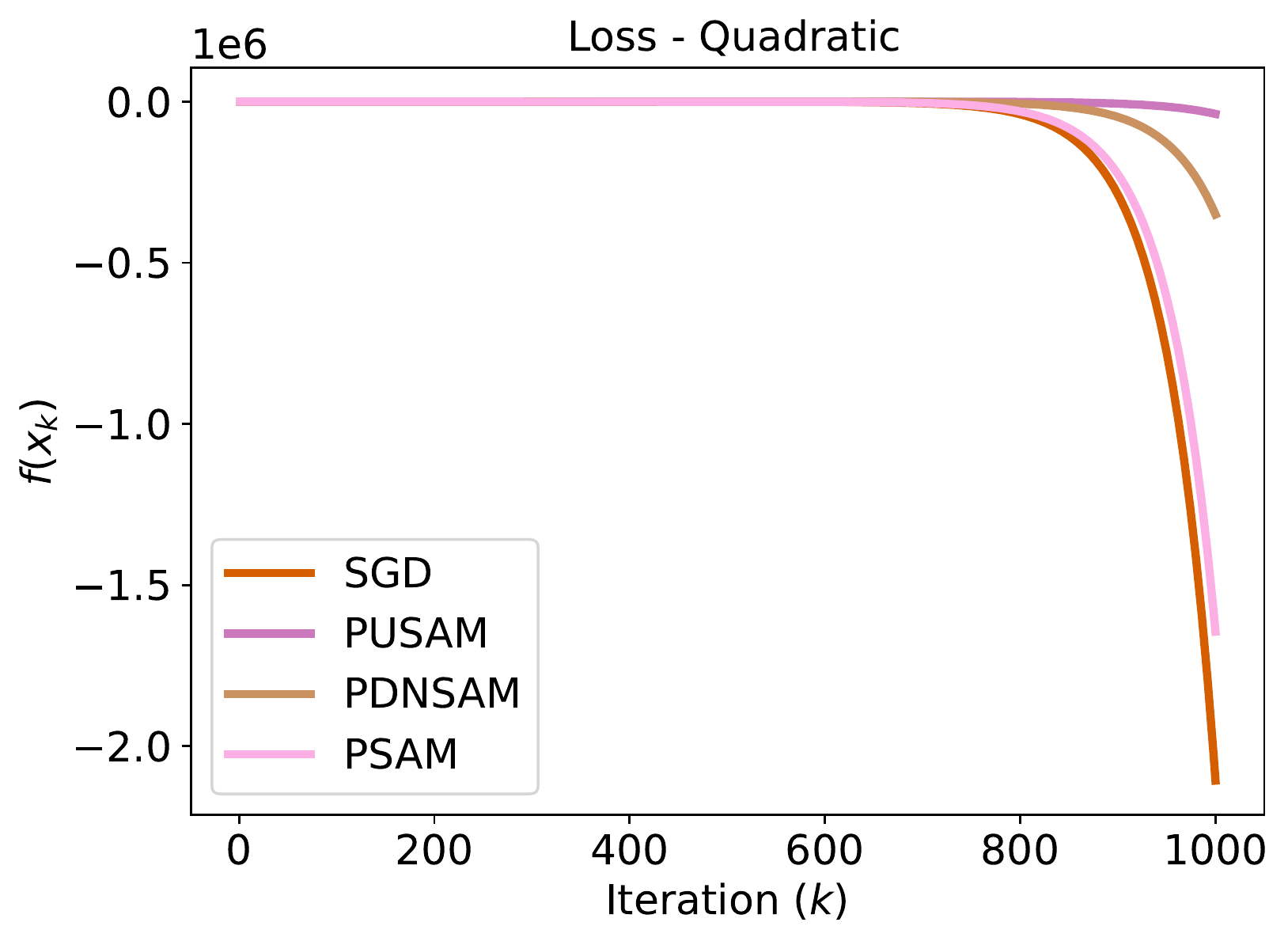} }}%
    \subfloat{{\includegraphics[width=0.24\linewidth]{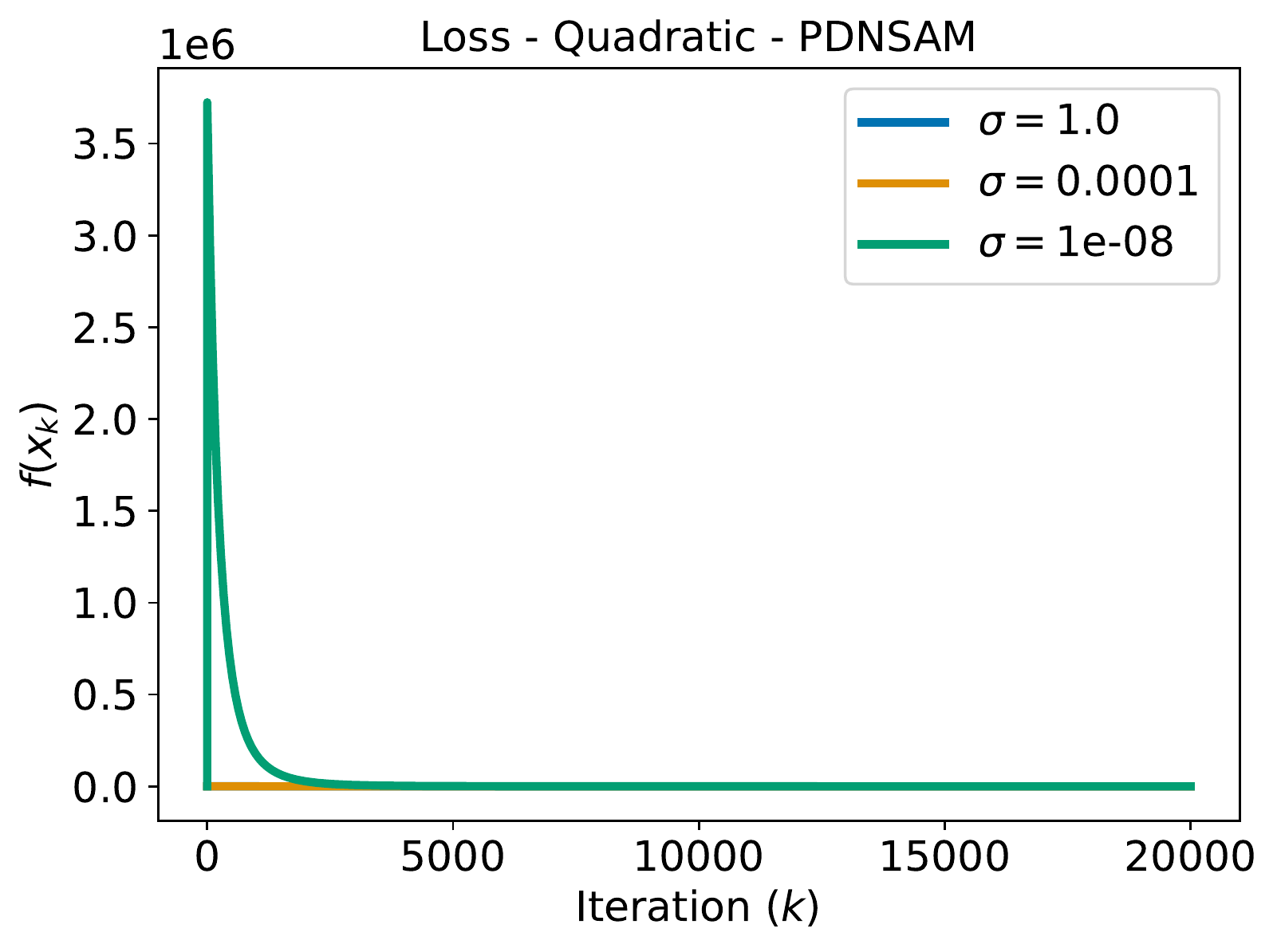} }}%
    \subfloat{{\includegraphics[width=0.24\linewidth]{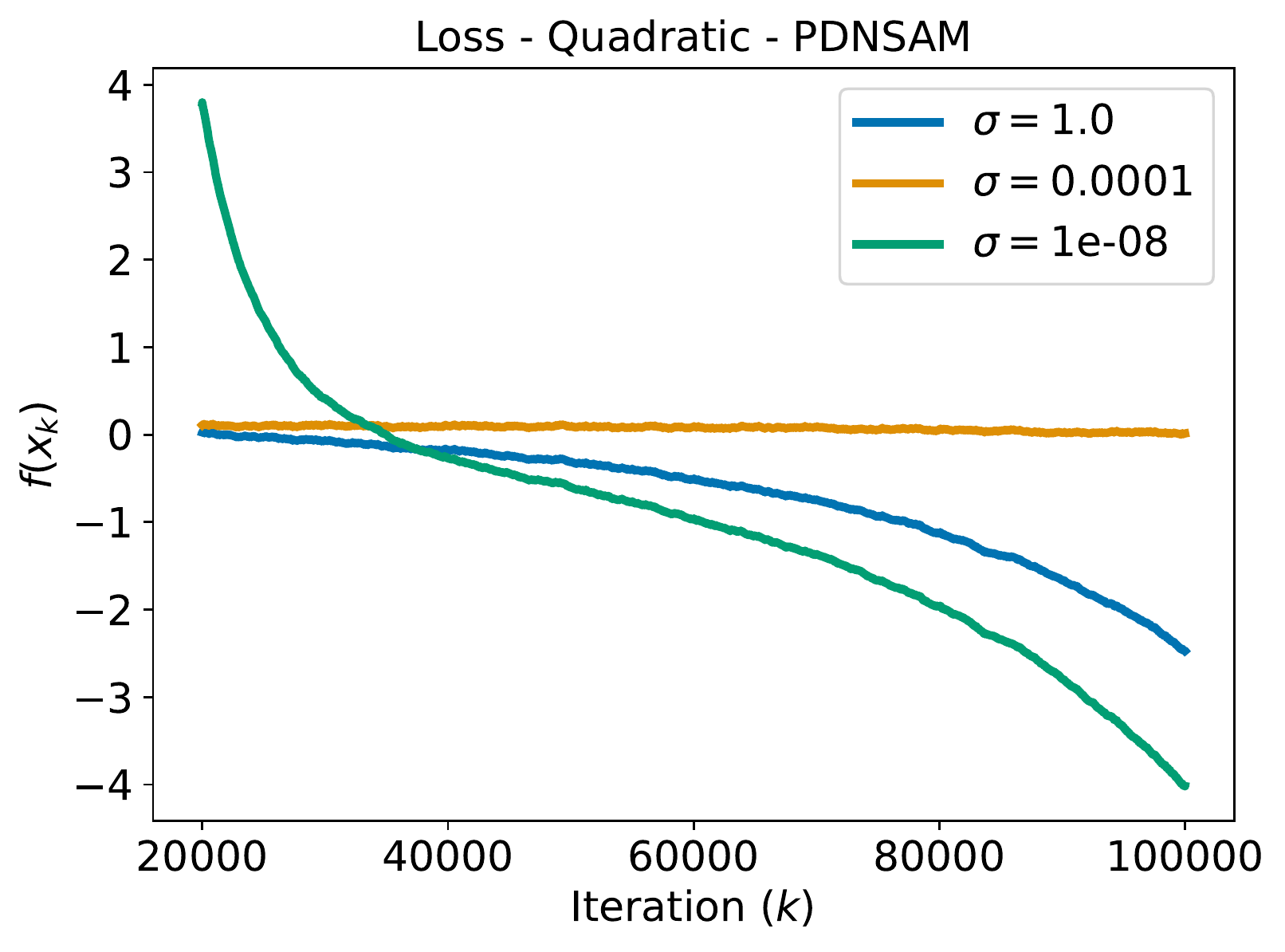} }} \\
    \subfloat{{\includegraphics[width=0.24\linewidth]{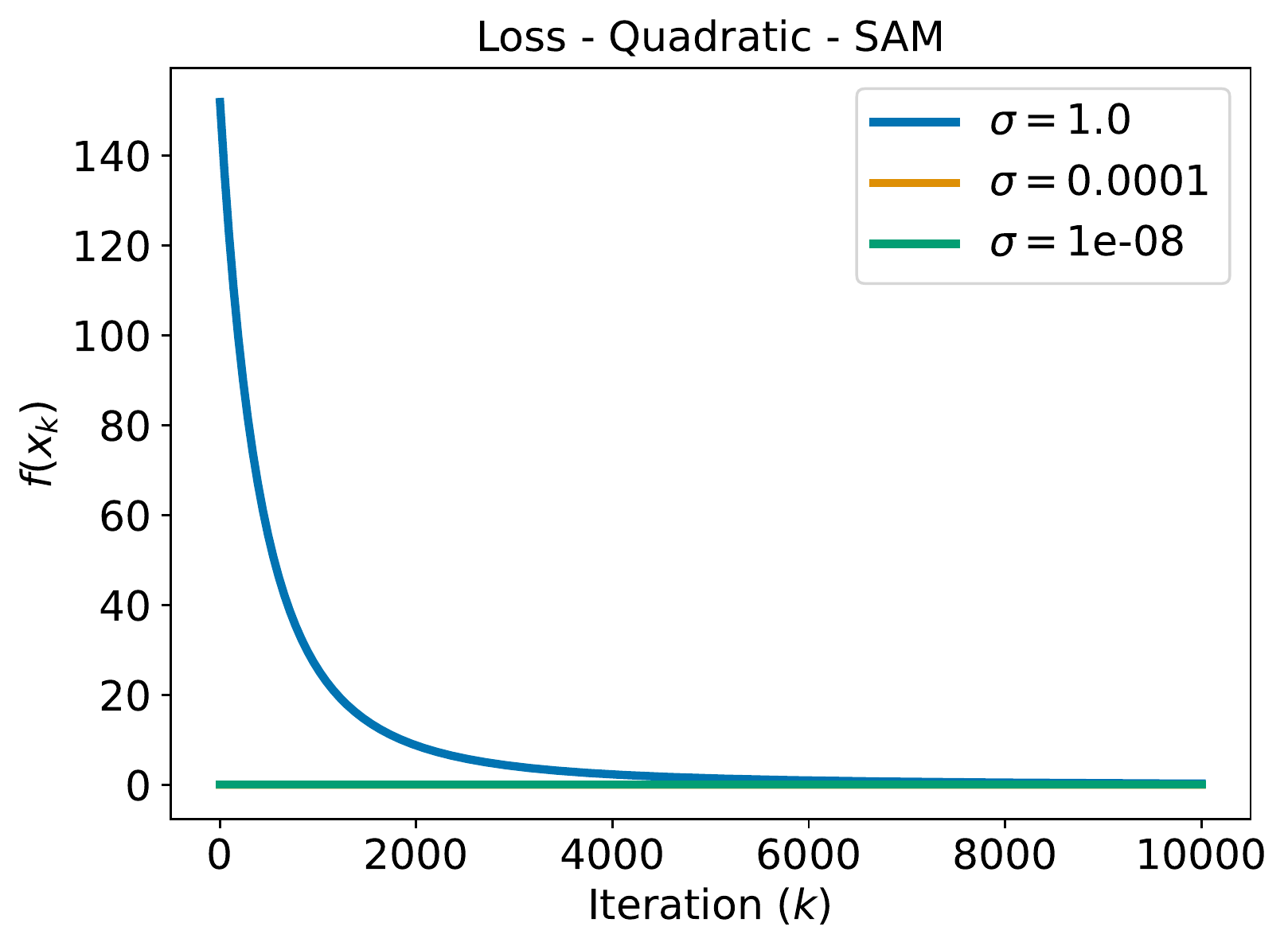} }}%
    \subfloat{{\includegraphics[width=0.24\linewidth]{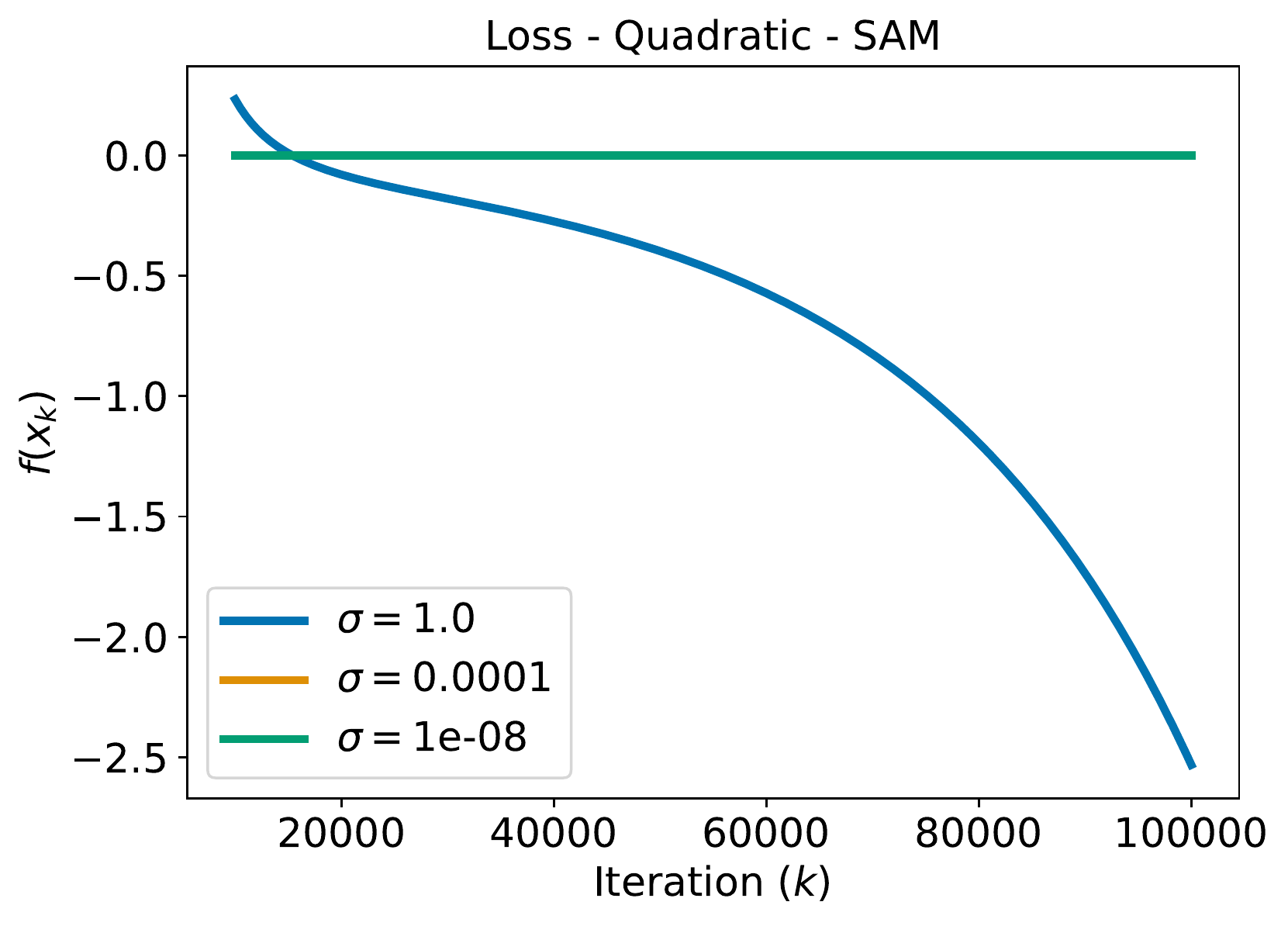} }}
    \subfloat{{\includegraphics[width=0.24\linewidth]{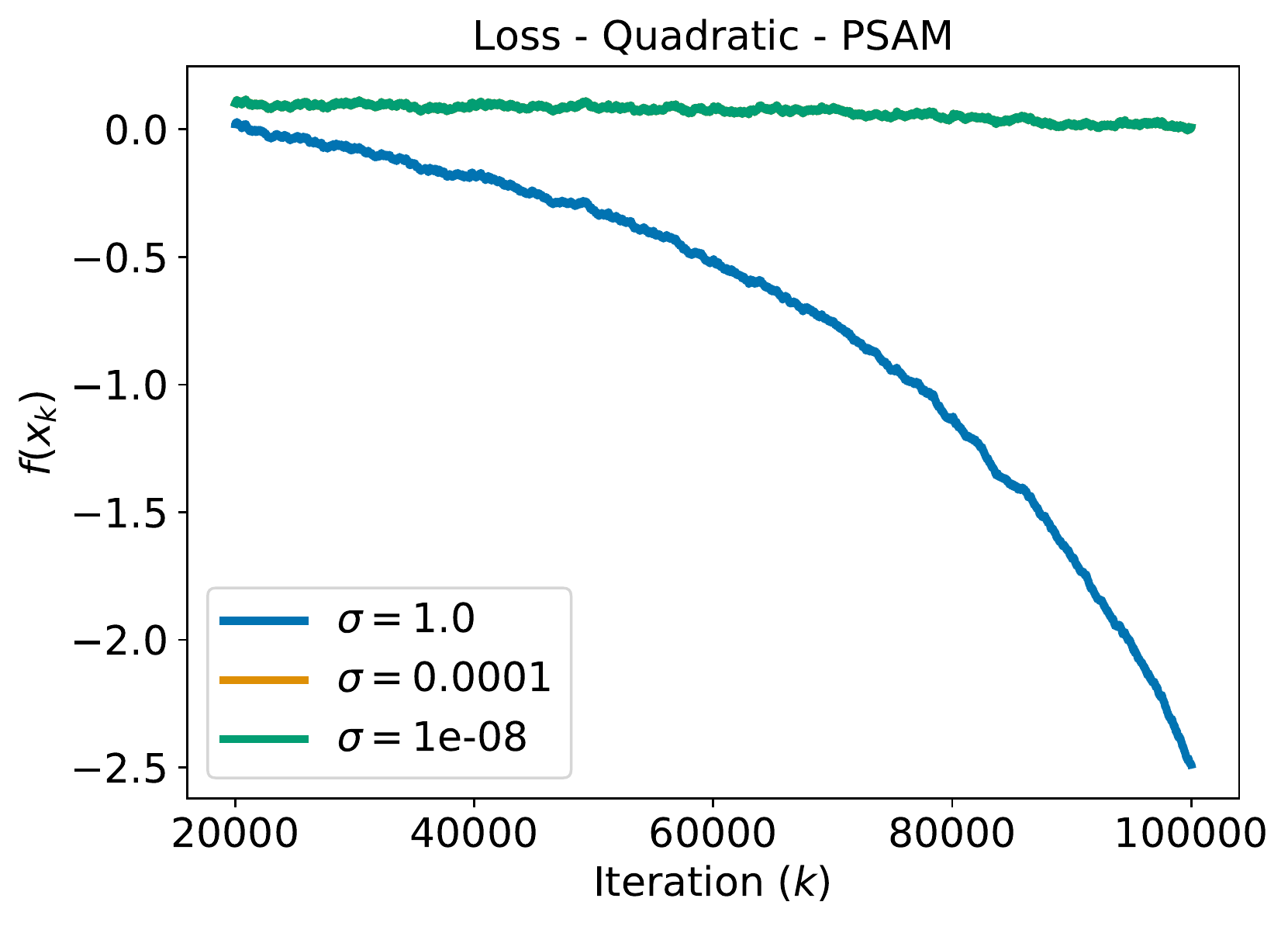} }}%
    \subfloat{{\includegraphics[width=0.24\linewidth]{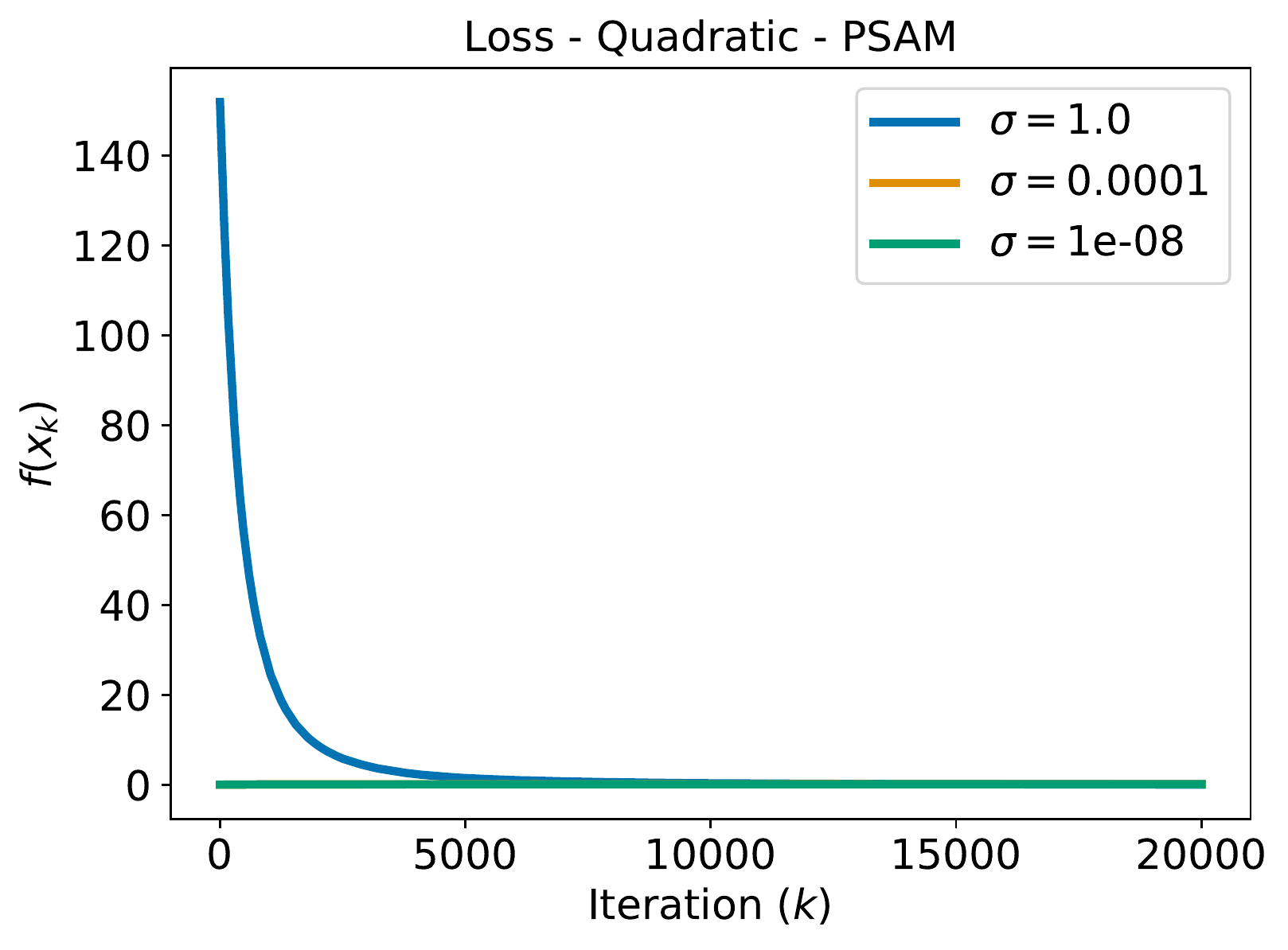} }}
    \caption{Escaping the Saddle - Top-Left: Comparison between GD, SAM, and USAM at escaping from a quadratic saddle. Top-Center-Left: Comparison between PGD, PSAM, DNSAM, and PUSAM at escaping from a quadratic saddle. Top-Center-Right: If close enough to the origin, DNSAM escapes from the origin immediately due to a volatility spike. Top-Right: The DNSAM initialized far from the origin starts escaping from it and the PSAM which jumped away from it efficiently escapes the saddle. \\
    Bottom: Both SAM and PSAM get stuck in the origin if initialized too close to it.}%
    \label{fig:SAM_Saddle_Escape}%
\end{figure}

\subsection{Linear Autoencoder}\label{appendix:autoencoder}
In this paragraph, we provide additional details regarding the Linear Autoencoder experiment. In this experiment, we approximate the Identity matrix of dimension 20 as the product of two square matrices $W1$ and $W2$. As described in \cite{kunin2019loss}, there is a saddle of the loss function around the origin. Inspired by the results obtained for the quadratic landscape, we test if SAM and its variants struggle to escape this saddle as well. To do this, we initialize the two matrices with entries normally distributed. We select $\rho = \sqrt{\eta}$, where $\eta=0.001$ is the learning rate. Then, we study the dynamics of the optimization process in case we scale the matrices by a factor $\sigma \in \{10^{-2}, 10^{-3}, 5 \cdot 10^{-3}, 10^{-4}, 10^{-5}\}$. As we can see from the first image of Figure \ref{fig:SAM_Autoencoder}, initializing SAM far from the origin, that is $\sigma = 0.01$, allows SAM to optimize the loss. Decreasing $\sigma$ implies that SAM becomes slower and slower at escaping the loss up to not being able to escape it anymore. The second image shows that the very same happens if we use DNSAM. However, if the process is initialized extremely close to the origin, that is $\sigma = 10^{-5}$, then the process enjoys a volatility spike that pushes it away from the origin. This allows the process to escape the saddle quickly and effectively. In the third image, we observe that similarly to SAM, PSAM becomes slower at escaping if $\sigma$ is lower. In the fourth image, we compare the behavior of GD, USAM, SAM, PGD, PUSAM, DNSAM, and PSAM where $\sigma = 10^{-5}$. We observe that DNSAM is the fastest algorithm to escape the saddle, followed by the others. As expected, SAM does not escape. Results are averaged over $3$ runs.

\subsection{Embedded Saddle}\label{appendix:embedded}
In this paragraph, we provide additional details regarding the Embedded Saddle experiment. In this experiment, we optimize a regularized quadratic $d$-dimensional landscape $L(x) = \frac{1}{2} x^{\top} H x + \lambda \sum_{i=1}^{d} x_{i}^4 $. As described in \cite{lucchi2022fractional}, if $H$ is not PSD, there is a saddle of the loss function around the origin and local minima away from it. We fix the Hessian $H \in \mathbb{R}^{400 \times 400}$ to be diagonal with random positive eigenvalues. To simulate a saddle, we flip the sign of the smallest $10$ eigenvalues. The regularization parameter is fixed at $\lambda=0.001$. We use $\eta = 0.005$, $\rho = \sqrt{\eta}$, run for $200000$ and average over $3$ runs. In Figure \ref{fig:SAM_EmbSaddle} we see the very same observations we had for the Autoencoder.

\vspace{-3.3cm}

\begin{figure}%
    \centering
    \subfloat{{\includegraphics[width=0.24\linewidth]{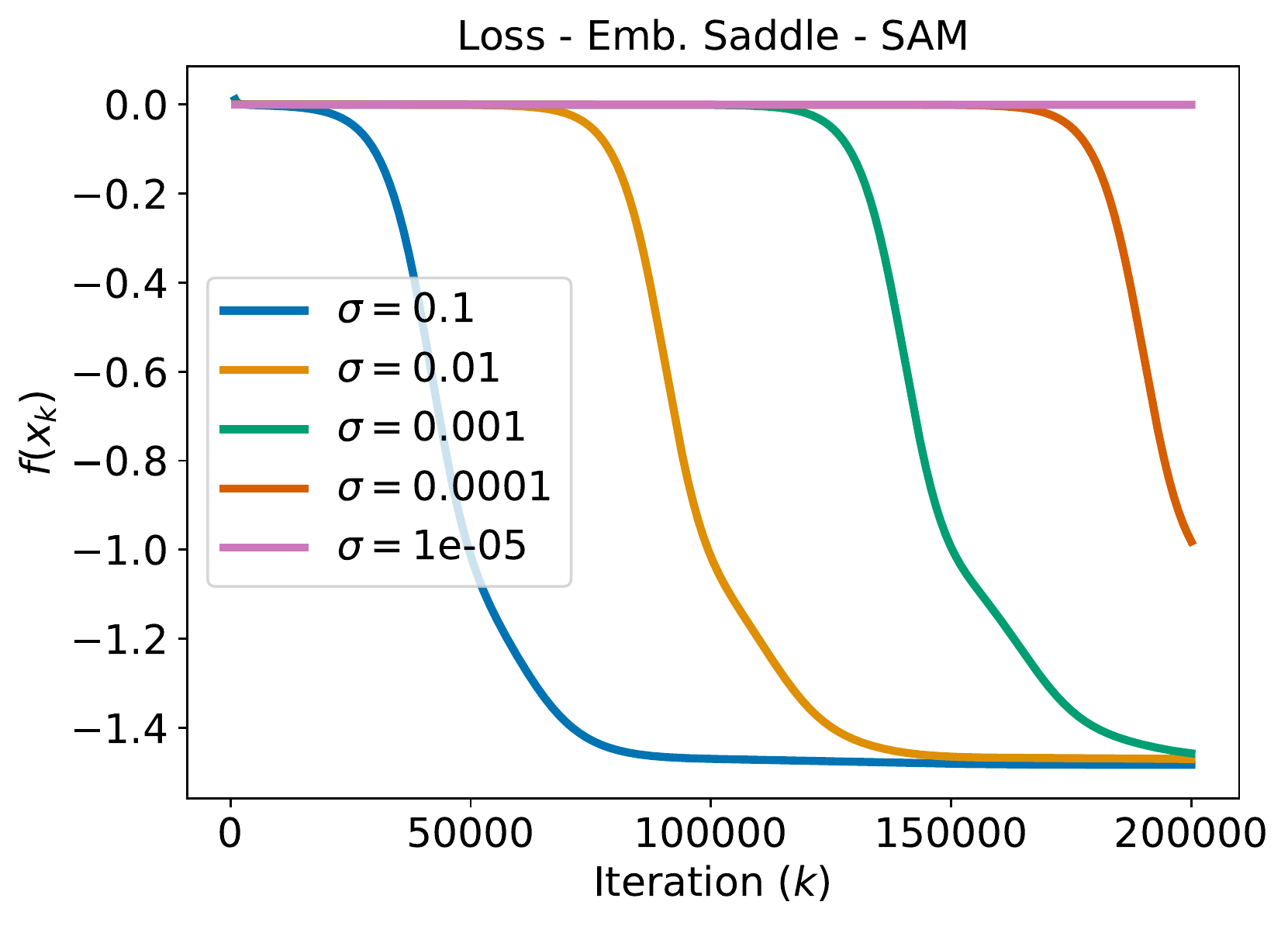} }}%
    \subfloat{{\includegraphics[width=0.24\linewidth]{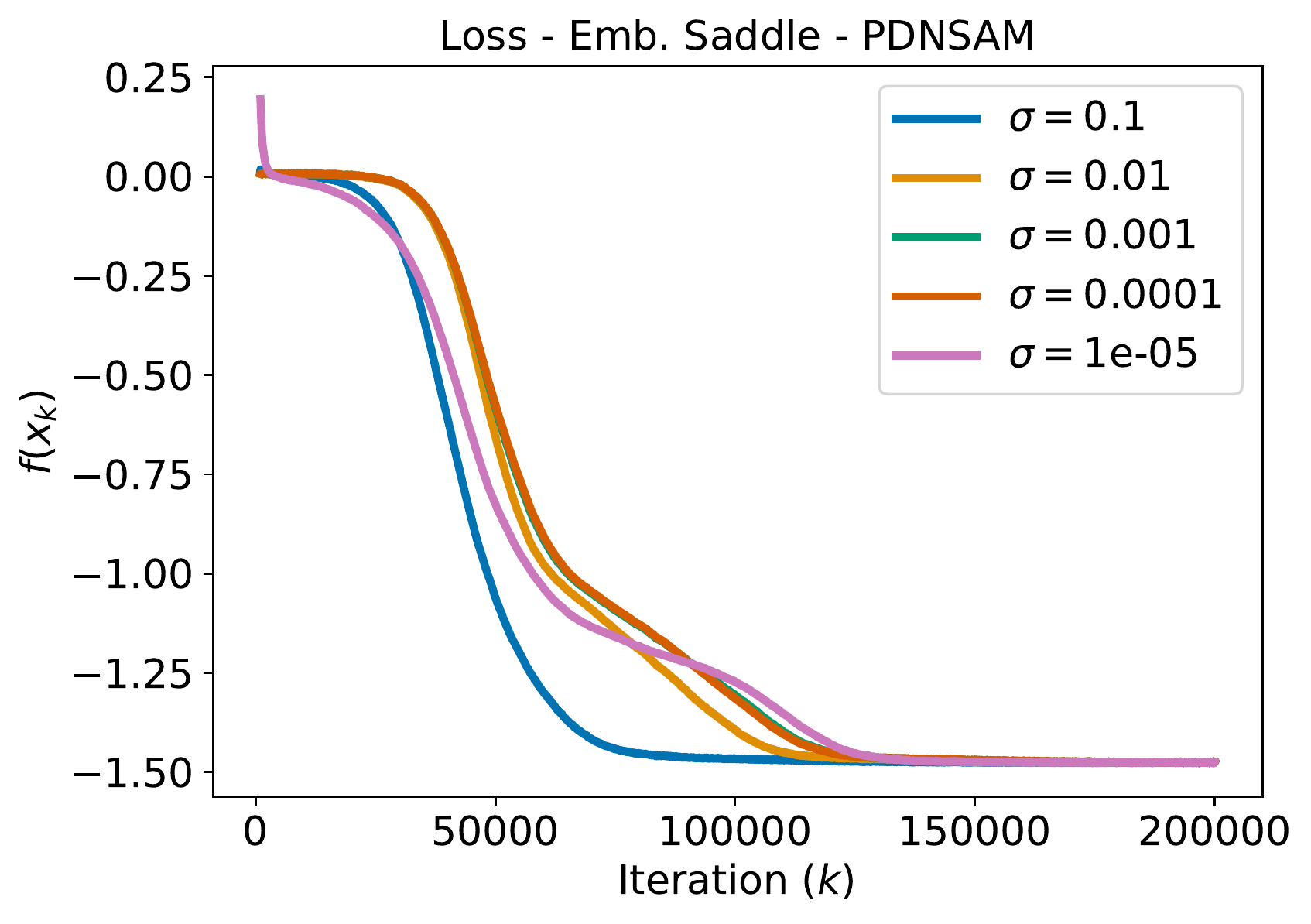} }} %
    \subfloat{{\includegraphics[width=0.24\linewidth]{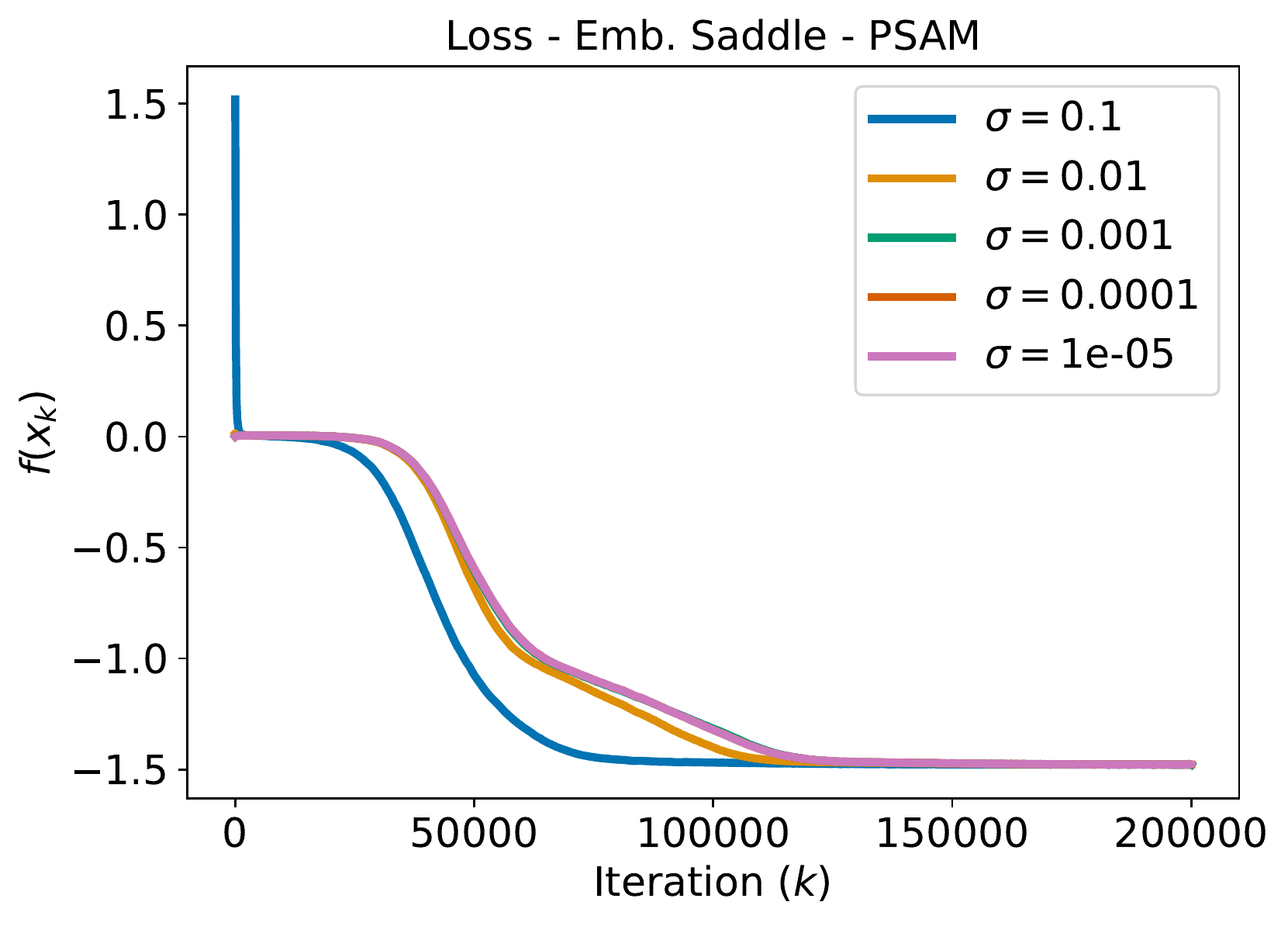} }}%
    \subfloat{{\includegraphics[width=0.24\linewidth]{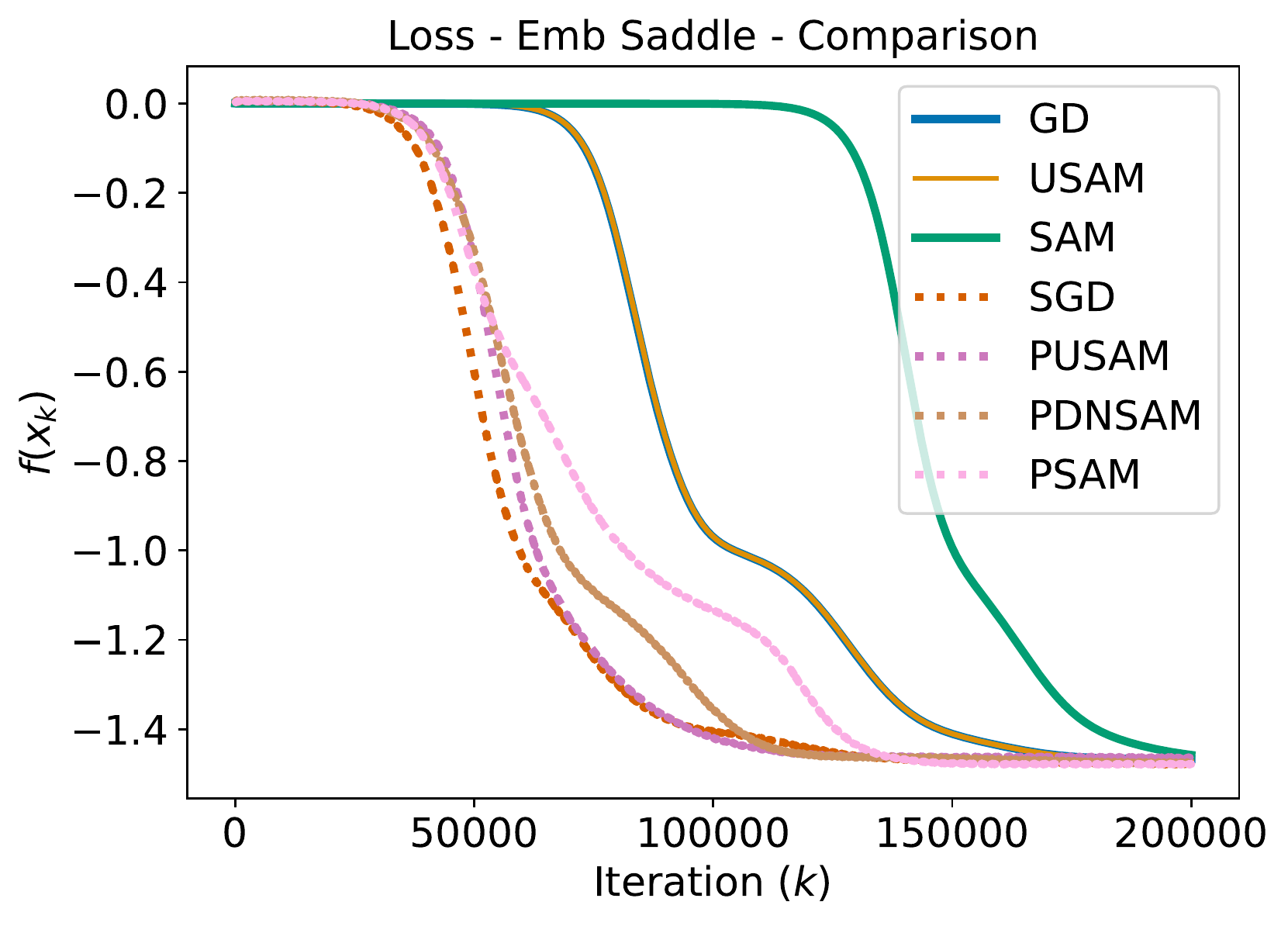} }}%

    \caption{Embedded Saddle - Left: SAM does not escape the saddle if it is too close to it. Center-Left: DNSAM always escapes it, but more slowly if initialized closer to the origin. If extremely close, it recovers speed thanks to a volatility spike. Center-Right: Similarly to SAM, PSAM gets progressively slower the closer it gets initialized to the origin. Right: SAM is stuck while the other optimizers manage to escape. }%
    \label{fig:SAM_EmbSaddle}%
\end{figure}

\begin{table}[!ht]
\centering
\begin{tabular}{|l|l|l|l|}
\hline Name  & Algorithm & Theorem for SDE \\
\hline SGD   & $x_{k+1} = x_k - \eta \nabla f_{\gamma_k}(x_k)$  & Theorem 1 \citep{li2017stochastic} \\

\hline SAM   & $x_{k+1}=x_k-\eta \nabla f_{\gamma_k}\left(x_k + \rho \frac{\nabla f_{\gamma_k}(x_k)}{\lVert \nabla f_{\gamma_k}(x_k) \rVert}\right)$ & Theorem \ref{thm:SAM_SDE_Insights} \\

\hline USAM  & $x_{k+1}=x_k-\eta \nabla f_{\gamma_k}\left(x_k + \rho \nabla f_{\gamma_k}(x_k) \right)$ & Theorem \ref{thm:USAM_SDE_Insights} \\

\hline DNSAM & $x_{k+1}=x_k-\eta \nabla f_{\gamma_k}\left(x_k + \rho \frac{\nabla f_{\gamma_k}(x_k)}{\lVert \nabla f_{\gamma_k}(x_k) \rVert}\right) $ & Not Available \\

\hline PGD   & $x_{k+1} = x_k - \eta \nabla f(x_k) + \eta Z$ & Theorem 1 \citep{li2017stochastic}\\

\hline PSAM  & $x_{k+1}=x_k-\eta \nabla f\left(x_k + \rho \frac{\nabla f(x_k) + Z}{\lVert \nabla f(x_k) + Z \rVert}\right) + \eta Z$ & Theorem \ref{thm:SAM_SDE_Simplified_Insights} \\

\hline PUSAM & $x_{k+1}=x_k-\eta \nabla f\left(x_k + \rho \nabla f(x_k) + Z \right) + \eta Z$ & Theorem \ref{thm:USAM_SDE_Simplified_Insights} \\

\hline PDNSAM & $x_{k+1}=x_k-\eta \nabla f\left(x_k + \rho \frac{\nabla f(x_k) + Z}{\lVert \nabla f(x_k) \rVert}\right) + \eta Z $ & Theorem \ref{thm:DNSAM_SDE_Insights} \\
\hline
\end{tabular}
%\vspace{0.1em}
\caption{Comparison of algorithms for methods analyzed in the paper.
The learning rate is $\eta$, the radius is $\rho$, and $Z \sim \mathcal{N}(0, \Sigma)$ is the injected noise.}
\label{table:AlgoComparison}
\end{table}

\begin{table}[!ht]
\centering
\begin{tabular}{|l|l|}
\hline Name  &  Drift Term  \\
\hline SGD   & $-\nabla f(x)$  \\

\hline SAM   & $-\nabla \left( f(x) + \rho \mathbb{E} \left[ \lVert \nabla f_{\gamma}(x) \rVert_2\right] \right)$ \\

\hline USAM  & $ - \nabla \left( f(x) + \frac{\rho}{2} \mathbb{E} \left[ \lVert \nabla f_{\gamma}(x) \rVert^2_2\right] \right)$ \\

\hline PGD   & $-\nabla f(x)$ \\

\hline PSAM  & $-\nabla \left( f(x) + \rho \mathbb{E} \left[ \lVert \nabla f_{\gamma}(x) \rVert_2\right] \right)$ \\

\hline PUSAM & $- \nabla \left( f(x) + \frac{\rho}{2} \lVert \nabla f(x) \rVert^2_2 \right)$ \\

\hline (P)DNSAM & $- \nabla \left( f(x) + \rho \lVert \nabla f(x) \rVert_2 \right)$ \\

\hline
\end{tabular}
%\vspace{0.1em}
\caption{Comparison of the drift terms of the SDEs for methods analyzed in the paper. }
\label{table:SDEComparisonDrift}
\end{table}

\begin{table}[!ht]
\centering
\begin{tabular}{|l|l|l|l|l|l|}
\hline Name  &  Diffusion Term & $\tilde{\Sigma}(x)$ \\
\hline SGD   & $\sqrt{\eta}\left( \Sigma(x)\right)^{\frac{1}{2}}$ & \\

\hline SAM   & $\sqrt{\eta\left( \Sigma(x)+ \rho \left( \tilde{\Sigma}(x) + \tilde{\Sigma}(x) ^{\top} \right) \right)}$ & $\mathbb{E} \left[ \left( \nabla f\left(x\right) - \nabla f_{\gamma}\left(x\right) \right) \cdot \left( \E \left[\frac{H_{\gamma}(x)\nabla f_{\gamma}(x) }{\lVert \nabla f_{\gamma}(x) \rVert_2} \right]- \frac{H_{\gamma}(x)\nabla f_{\gamma}(x) }{\lVert \nabla f_{\gamma}(x) \rVert_2} \right)^{\top} \right] $ \\

\hline USAM  & $\sqrt{\eta\left( \Sigma(x)+ \rho \left( \tilde{\Sigma}(x) + \tilde{\Sigma}(x) ^{\top} \right) \right)}$ & $\mathbb{E} \left[ \left( \nabla f\left(x\right) - \nabla f_{\gamma}\left(x\right) \right) \left(\mathbb{E} \left[ H_{\gamma}(x) \nabla f_{\gamma}(x)\right] - H_{\gamma}(x) \nabla f_{\gamma}(x) \right)^{\top} \right] $\\

\hline PGD   & $\sqrt{\eta \Sigma(x)}$ & \\

\hline PSAM  & $\sqrt{\eta\left( \Sigma(x)+ \rho \left( \Bar{\Sigma}(x) + \Bar{\Sigma}(x) ^{\top} \right) \right)}$ &  $H(x) \mathbb{E} \left[ \left( \nabla f\left(x\right) - \nabla f_{\gamma}\left(x\right) \right) \cdot \left( \E \left[\frac{\nabla f_{\gamma}(x) }{\lVert \nabla f_{\gamma}(x) \rVert_2} \right]- \frac{\nabla f_{\gamma}(x) }{\lVert \nabla f_{\gamma}(x) \rVert_2} \right)^{\top} \right]$\\

\hline PUSAM & $\left(I_d+ \rho H(x) \right)\left(\eta \Sigma\left(x\right)\right)^{1 / 2}$ &\\

\hline (P)DNSAM &  $ \left( I_d + \rho \frac{H(x)}{\lVert \nabla f(x) \rVert_2} \right) \left(\eta \Sigma(x)\right)^{\frac{1}{2}}$ & \\

\hline
\end{tabular}
%\vspace{0.1em}
\caption{Comparison of the diffusion terms of SDEs for methods analyzed in the paper. The matrix $\Sigma(x)$ is equal to $\Sigma(x)^{\text{SGD}}$ and $H(x)=\nabla^2 f(x)$.}
\label{table:SDEComparisonDiff}
\end{table}

\end{document}